\documentclass{article}

\usepackage{microtype}
\usepackage{graphicx}
\usepackage{booktabs} 

\usepackage{hyperref}

\usepackage[accepted]{icml2024}

\usepackage{amsmath}
\usepackage{amssymb}
\usepackage{mathtools}
\usepackage{amsthm}

\usepackage[capitalize,noabbrev]{cleveref}

\theoremstyle{plain}
\newtheorem{theorem}{Theorem}[section]

\newtheorem{lemma}[theorem]{Lemma}

\theoremstyle{definition}
\newtheorem{definition}[theorem]{Definition}

\newtheorem*{conjecture*}{Conjecture}

\theoremstyle{remark}

\usepackage[textsize=tiny]{todonotes}

\usepackage{amsmath,amsfonts,bm}

\def\figref#1{figure~\ref{#1}}

\def\secref#1{section~\ref{#1}}

\def\eqref#1{equation~\ref{#1}}

\def\1{\bm{1}}

\def\rva{{\mathbf{a}}}

\def\rvu{{\mathbf{i}}}

\def\rvr{{\mathbf{r}}}
\def\rvs{{\mathbf{s}}}

\def\rvu{{\mathbf{u}}}
\def\rvv{{\mathbf{v}}}
\def\rvw{{\mathbf{w}}}
\def\rvx{{\mathbf{x}}}
\def\rvy{{\mathbf{y}}}
\def\rvz{{\mathbf{z}}}

\def\rmS{{\mathbf{S}}}

\def\rmW{{\mathbf{W}}}
\def\rmX{{\mathbf{X}}}

\def\rmZ{{\mathbf{Z}}}

\DeclareMathAlphabet{\mathsfit}{\encodingdefault}{\sfdefault}{m}{sl}
\SetMathAlphabet{\mathsfit}{bold}{\encodingdefault}{\sfdefault}{bx}{n}

\DeclareMathOperator{\sign}{sign}

\usepackage{url}
\usepackage{enumitem}
\usepackage{bbm}
\graphicspath{{./figures/}}
\usepackage{amsthm}
\usepackage{nicefrac}
\usepackage{subcaption}

\newcommand{\wt}[1]{\widetilde{#1}}

\newcommand{\exrisk}[2]{\mathcal{R}_{#1}(#2)}
\newcommand{\oodrisk}[1]{\mathcal{R}_\mathrm{OOD}(#1)}

\newcommand{\zspace}{\mathcal{Z}}

\newcommand{\xspace}{\mathcal{X}}
\newcommand{\yspace}{\mathcal{Y}}

\newcommand{\dtrainset}{\mathbb{D}_\mathrm{train}}
\newcommand{\dset}{\mathbb{D}}
\newcommand{\dtrain}{\mathcal{D}_\mathrm{train}}

\newcommand{\dtest}{\mathcal{D}_\mathrm{test}}

\newcommand{\snorm}[1]{\lVert #1 \rVert_1}
\newcommand{\norm}[1]{\lVert #1\rVert_2}
\newcommand{\pnorm}[1]{\lVert #1 \rVert_p}
\newcommand{\bignorm}[1]{\big\lVert #1\big\rVert_2}

\newcommand{\dtcore}{\Delta_{\mathrm{core}}^{(t)}}
\newcommand{\dtbg}{\Delta_{\mathrm{bg}}^{(t)}}

\newcommand{\abs}[1]{|#1|}
\newcommand{\bigabs}[1]{\big|#1\big|}

\newcommand{\inprod}[2]{\langle {#1}, {#2} \rangle}
\newcommand{\inp}[2]{\inprod{#1}{#2}}
\newcommand{\biginprod}[2]{\big\langle {#1}, {#2} \big\rangle}
\newcommand{\Biginprod}[2]{\Big\langle {#1}, {#2} \Big\rangle}

\newcommand{\relu}{\mathsf{ReLU}}

\newcommand{\poly}[1]{\mathsf{poly}(#1)}

\newcommand{\its}[2]{#1^{(#2)}}

\newcommand{\indicator}[1]{\mathbf{1}_{#1}}
\newcommand{\ind}[1]{\indicator{#1}}

\newcommand{\di}{d_0}
\newcommand{\dcore}{d_\mathrm{core}}
\newcommand{\dbg}{d_\mathrm{bg}}

\newcommand{\score}{\mathcal{S}_\mathrm{core}}
\newcommand{\sbg}{\mathcal{S}_\mathrm{bg}}

\newcommand{\prob}[2]{\mathbf{Pr}_{#1}{#2}}

\newcommand{\expect}[2]{\mathbb{E}_{#1}{#2}}

\newcommand{\defeq}{\vcentcolon=}

\newcommand{\gkyt}[1]{g_{k,y,{#1}}^{(t)}}
\newcommand{\gknyt}[1]{g_{k,-y,{#1}}^{(t)}}
\newcommand{\gkytj}{g_{k,y,j}^{(t)}}

\newcommand{\nset}[2]{\mathcal{N}_{#1}^{(#2)}}
\newcommand{\nyt}{\nset{y}{t}}

\newcommand{\nyinit}{\nset{y}{0}}
\newcommand{\ny}[1]{\mathcal{N}_y^{(#1)}}

\newcommand{\npos}[1]{\mathcal{N}_1^{(#1)}}

\newcommand{\muj}{\mu_j}
\newcommand{\wkt}{\its{\rvw_k}{t}}
\newcommand{\wkinit}{\its{\rvw_k}{0}}
\newcommand{\wk}[1]{\its{\rvw_k}{#1}}
\newcommand{\mj}{\bm{m}_j}

\newcommand{\zj}{\rvz_j}

\newcommand{\htx}{h^{(t)}(\rvx)}

\newcommand{\hx}[1]{h^{(#1)}(\rvx)}

\renewcommand{\eqref}[1]{(\ref{#1})}
\renewcommand{\figref}[1]{\hyperref[#1]{Figure~\ref{#1}}}
\renewcommand{\secref}[1]{\hyperref[#1]{Section~\ref{#1}}}
\newcommand{\tableref}[1]{\hyperref[#1]{Table~\ref{#1}}}
\newcommand{\defref}[1]{\hyperref[#1]{Definition~\ref{#1}}}
\newcommand{\theoref}[1]{\hyperref[#1]{Theorem~\ref{#1}}}
\newcommand{\lemmaref}[1]{\hyperref[#1]{Lemma~\ref{#1}}}
\newcommand{\propref}[1]{\hyperref[#1]{Proposition~\ref{#1}}}
\newcommand{\cororef}[1]{\hyperref[#1]{Corollary~\ref{#1}}}
\newcommand{\conjref}[1]{\hyperref[#1]{Conjecture~\ref{#1}}}
\newcommand{\appref}[1]{\hyperref[#1]{Appendix~\ref{#1}}}

\definecolor{darkorange}{rgb}{1.0, 0.45, 0.0}

\icmltitlerunning{Feature Contamination}

\begin{document}

\twocolumn[
\icmltitle{Feature Contamination: Neural Networks Learn Uncorrelated Features\\ and Fail to Generalize}



\icmlsetsymbol{equal}{*}

\begin{icmlauthorlist}
\icmlauthor{Tianren Zhang}{equal,yyy}
\icmlauthor{Chujie Zhao}{equal,yyy}
\icmlauthor{Guanyu Chen}{yyy}
\icmlauthor{Yizhou Jiang}{yyy}
\icmlauthor{Feng Chen}{yyy,lab}
\end{icmlauthorlist}

\icmlaffiliation{yyy}{Department of Automation, Tsinghua University, Beijing, China}
\icmlaffiliation{lab}{LSBDPA Beijing Key Laboratory, Beijing, China}

\icmlcorrespondingauthor{Feng Chen}{chenfeng@mail.tsinghua.edu.cn}

\icmlkeywords{Out-of-Distribution Generalization, Distribution Shift, Spurious Correlations, Neural Networks, Representation Learning}

\vskip 0.3in
]



\printAffiliationsAndNotice{\icmlEqualContribution} 

\begin{abstract}
Learning representations that generalize under distribution shifts is critical for building robust machine learning models.
However, despite significant efforts in recent years, algorithmic advances in this direction have been limited.
In this work, we seek to understand the fundamental difficulty of out-of-distribution generalization with deep neural networks. We first empirically show that perhaps surprisingly, even allowing a neural network to \emph{explicitly} fit the representations obtained from a teacher network that \emph{can} generalize out-of-distribution is insufficient for the generalization of the student network. Then, by a theoretical study of two-layer ReLU networks optimized by stochastic gradient descent (SGD) under a structured feature model, we identify a fundamental yet unexplored feature learning proclivity of neural networks, \emph{feature contamination}: neural networks can learn \emph{uncorrelated} features together with predictive features, resulting in generalization failure under distribution shifts.
Notably, this mechanism essentially differs from the prevailing narrative in the literature that attributes the generalization failure to spurious correlations. Overall, our results offer new insights into the non-linear feature learning dynamics of neural networks and highlight the necessity of considering inductive biases in out-of-distribution generalization.\footnote{Code is available at \url{https://github.com/trzhang0116/feature-contamination}.}
\end{abstract}

\section{Introduction}
\label{sec:intro}

The capability of generalizing under distribution shifts is crucial for machine learning systems to be deployed in the wild~\citep{amodei_concrete_2016,ferrari_recognition_2018,koh_wilds_2021}. In the last decade, it has proved that the conventional principle of empirical risk minimization (ERM), when combined with deep neural networks, can lead to remarkable {in-distribution (ID) generalization} performance given sufficient training data.
Nevertheless, this powerful paradigm can often fail in \emph{out-of-distribution (OOD) generalization}, where distribution shifts occur due to data variations that are not well-covered in training~\citep{torralba_unbiased_2011,ferrari_recognition_2018,geirhos_generalisation_2018,degrave_ai_2021}.

In response, recent years have witnessed a surge of developing algorithms that promote OOD generalization. However, the effectiveness of many proposed algorithms has been called into question by recent work~\citep{gulrajani_search_2021,koh_wilds_2021},
in which no tested algorithm exhibits a significant advantage over ERM under fair comparisons.
On the other hand, it turns out that the most effective means of improving OOD generalization to date is pre-training on a more diverse dataset~\citep{taori_measuring_2020,wiles_fine-grained_2022}, with notable examples including CLIP~\citep{radford_learning_2021} and GPT~\citep{brown_language_2020}. Yet, using additional pre-training data also blurs the notion of ``OOD'' itself since it essentially expands the training distribution.
Moreover, it has been observed that when the test distribution differs from the pre-training distribution, pre-trained models can also suffer from performance degradation~\citep{bommasani_opportunities_2022,liu_evaluating_2023,li_task_2023}.

The limited algorithmic success underlines the necessity of identifying and understanding the fundamental factors behind OOD generalization. In particular, a prevailing narrative in the literature attributes the OOD generalization failure to \emph{spurious correlations}~\citep{arjovsky_invariant_2019,nagarajan_understanding_2021,scholkopf_toward_2021}. This explanation is inspired by the observation that the representations learned by ERM can absorb features that have \emph{correlational} yet non-causal relationships with the output~\citep{ferrari_recognition_2018,geirhos_shortcut_2020}, and it has motivated a main line of algorithmic endeavor of designing better representation learning objectives in recent years~\citep{arjovsky_invariant_2019,krueger_out--distribution_2021,mitrovic_representation_2021,chen_iterative_2022,shi_gradient_2022}. However, despite being intuitive, it remains elusive how much this failure mode actually contributes to the OOD generalization failure in practice---as we will elaborate in the following sections, there exists a major OOD generalization gap in many tasks that \emph{cannot} be straightforwardly explained by spurious correlations, implying that there must exist some more dominant factors.

On the theoretical side, a series of work has been devoted to analyzing the failure modes of OOD generalization. However, existing analysis has two major limitations: (\romannumeral 1) conceptually, most studies only consider the failure mode due to spurious correlations; (\romannumeral 2)
technically,
most studies either only consider \emph{linear} models such as linear classification over prescribed features or neural tangent kernels
\citep{arjovsky_invariant_2019,sagawa_investigation_2020,nagarajan_understanding_2021,xu_how_2021,ahuja_empirical_2021,ahuja_invariance_2021,pezeshki_gradient_2021,chen_iterative_2022,wang_provable_2022,rosenfeld_domain-adjusted_2022,abbe_generalization_2023,chen_understanding_2023}, or only consider arbitrary \emph{unstructured} models without taking into the account the role of optimization~\citep{rosenfeld_risks_2021,kamath_does_2021,ye_towards_2021}---this makes them unable to capture the inductive biases of today's most widely used model class, i.e., \emph{neural networks}. As a result, it has been observed that many OOD generalization algorithms that enjoy provable guarantees in their theoretical models do not excel in practice~\citep{gulrajani_search_2021}.

Overall, the above results imply that current explanations and theoretical models on OOD generalization may \emph{not} faithfully reflect real-world distribution shifts. Motivated by the gap between theory and practice, we argue that taking into account the inductive biases of neural networks is not only important but also \emph{necessary} for understanding OOD generalization in the era of deep learning.

\subsection{Our Results and Implications}
\label{sec:summary}

In this work, we set out to understand the fundamental difficulty of OOD generalization with deep neural networks:

\textbf{Empirically,} inspired by the ongoing trend of designing specific representation learning objectives for OOD generalization~\citep{arjovsky_invariant_2019,gulrajani_search_2021}, we investigate what will happen in an ``ideal'' setting where good representations are \emph{explicitly given} during training. Concretely, we show on a range of distribution shift benchmarks that perhaps surprisingly, even if we allow a neural network to explicitly fit the representations obtained from a teacher network that \emph{can} generalize out-of-distribution, the performance of the student network can still significantly deteriorate under distribution shifts. Our results thus imply that only considering the effect of the representation learning objective is \emph{insufficient} for understanding OOD generalization in practice without considering the inductive biases in optimization. Moreover, we show that the above generalization failure \emph{cannot} be simply explained by spurious correlations or other existing explanations in the literature on OOD generalization.

\textbf{Theoretically,} we prove that in certain structured binary classification tasks where the data is generated from generalizable \emph{core features} and other \emph{background features} (formal definitions in~\secref{sec:setup}), a randomly initialized two-layer ReLU neural network trained by SGD can achieve ID generalization given sufficient iterations, yet fails to generalize OOD. In particular, we show that the above failure mode differs fundamentally from prior work as it holds even when:
\begin{itemize}[leftmargin=1em]
\item Background features are \emph{uncorrelated} with the label (this excludes the failure mode due to spurious correlations).
\item Ground-truth labels can be perfectly predicted by core features (this excludes the failure mode due to lacking informative features for prediction).
\item Core features and background features are distributed in orthogonal subspaces (this excludes the failure mode due to non-linearly entangled features in the input that may be hard to disentangle for the neural network in training).
\end{itemize}

\begin{figure*}[t]
\centering
\includegraphics[width=0.92\linewidth]{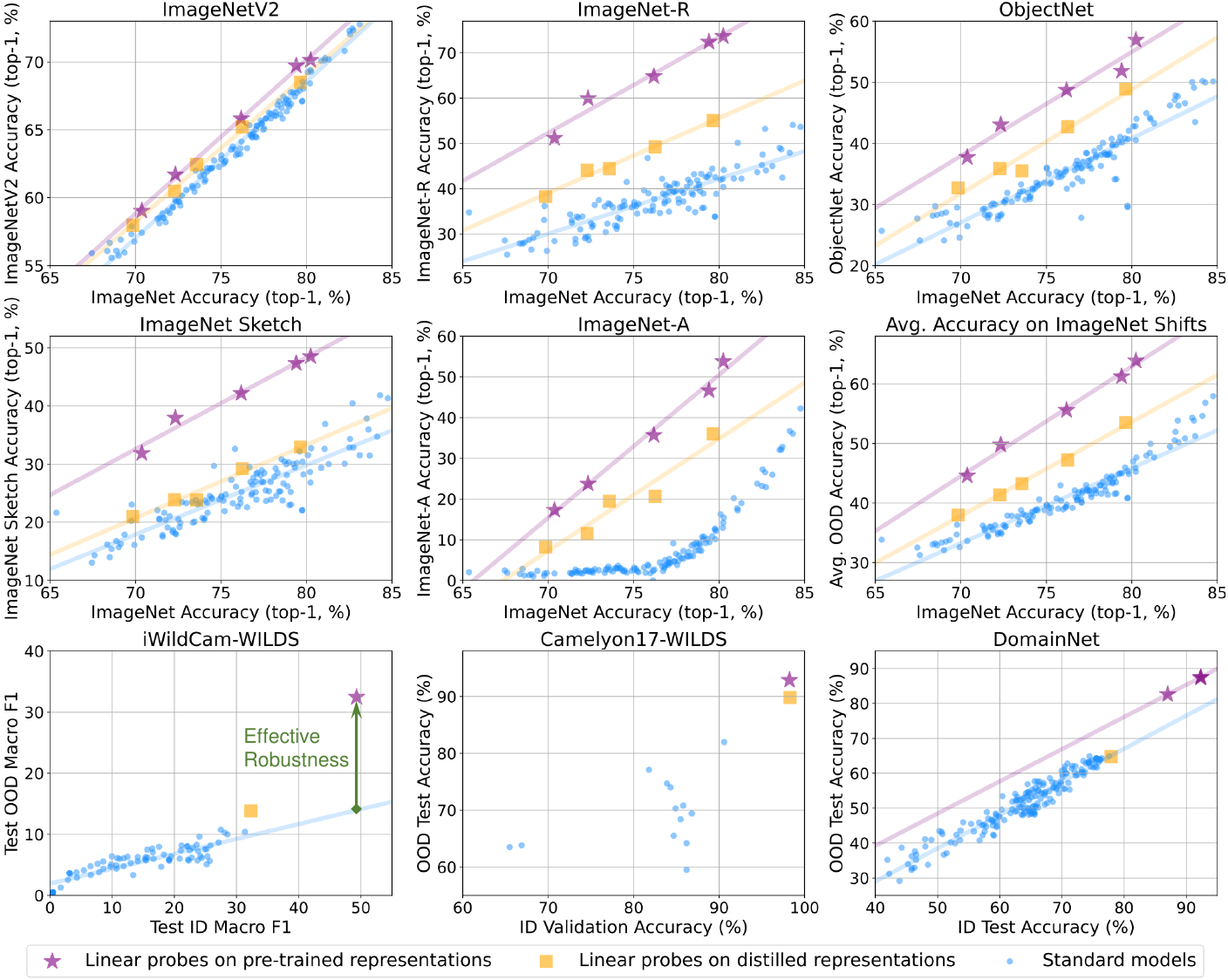}
\vspace{-0.5em}
\caption{OOD performance ($y$-axes) v.s. ID performance ($x$-axes) for three model families including (\romannumeral 1) linear probes on pre-trained representations ({\color{Purple}purple stars}), (\romannumeral 2) linear probes on distilled representations ({\color{Orange}orange squares}), and (\romannumeral 3) standard models trained on ID data ({\color{NavyBlue}blue circles}).
The $y$-axis of the sixth panel stands for the average accuracy on ImageNet-based OOD test sets, averaged from the first five panels. Please refer to~\secref{appsec:distill} for more details on each model family.}
\label{fig:results}
\vspace{-0.9em}
\end{figure*}

Instead, we identify that the above failure stems from a fundamental yet unexplored feature learning proclivity, which we name \emph{feature contamination}, of neural networks. In brief, feature contamination indicates that
during the learning of core features, SGD-trained neural networks also learn background features simultaneously, 
even when background features are \emph{uncorrelated} with the label and in the presence of weight decay. The reason for this phenomenon is that the neurons in the network tend to have \emph{asymmetric activation} for different classes, resulting in non-zero expected gradient projections onto \emph{both} the core feature subspace and the background feature subspace. This eventually leads to additional risks under distribution shifts due to the coupling of core and background features in the neurons' pre-activation.
Moreover, we formally show that ReLU networks and linear networks are {\emph{provably different}} in our setting with the latter exhibiting no such behavior, suggesting a separation between linear and non-linear models.
Finally, we present empirical evidence on deep neural networks that connects feature contamination to the empirical OOD generalization failure observed in our experiments.

At a high level, we expect that feature contamination as a \emph{novel inductive bias of SGD-trained neural networks} may also be used in more general contexts. For example, it may serve as a new perspective for understanding the \emph{feature learning} process of (deep) neural networks, complementing other known inductive biases of neural networks such as the simplicity bias~\citep{arpit_closer_2017,shah_pitfalls_2020}.

\section{Good Representations Are Hard to Learn Even when Explicitly Given in Training}
\label{sec:main_exp}

We begin our analysis by an empirical study inspired by recent algorithmic explorations in OOD generalization. Existing work has made various attempts to learn OOD-generalizable models by designing auxiliary \emph{representation learning} objectives beyond minimizing the prediction risk (see the baseline models in~\secref{appsec:domainnet} for some examples).
Typically, those objectives reflect the premise of the properties of ``good'' representations.
For example, a major line of work focuses on learning invariant representations across multiple training domains~\citep{arjovsky_invariant_2019,chen_iterative_2022,shi_gradient_2022}, with the aim of removing domain-specific spurious correlations.
Given the limited success of existing algorithms, here we would like to investigate the fundamental limitations of such representation learning methods.
However, a main confounder in our study is that it is often unclear whether optimizing certain objectives is indeed effective for shaping the representation to satisfy the ideal properties---for example, it has been shown that optimizing some invariant representation learning objectives may lead to representations that are not truly invariant~\citep{kamath_does_2021,rosenfeld_risks_2021}.


To ablate the potential sub-optimality in the representation learning objective, we focus on an ``ideal'' scenario where the model has \emph{explicit access to good representations} in training.
Concretely, we leverage large-scale pre-trained models such as CLIP~\citep{radford_learning_2021}, which has shown remarkable robustness against distribution shifts, to extract good representations for each input: given a pre-trained CLIP encoder as a \emph{teacher encoder}, we randomly initialize another \emph{student encoder} with the \emph{same} architecture. We then train the student encoder by minimizing the Euclidean distance between its output representations and the representations extracted by the teacher encoder, a process known as representation distillation~\citep{hinton_distilling_2014,tian_contrastive_2020}. Finally, we evaluate the ID and the OOD performance of both models.
In total, our experiments span six different pre-trained models and eight extensively benchmarked distribution shift datasets, including five ImageNet-based natural distribution shift datasets~\citep{taori_measuring_2020}, two in-the-wild distribution shift datasets from WILDS~\citep{koh_wilds_2021}, and a domain generalization dataset DomainNet~\citep{peng_moment_2019}.
Please see~\secref{appsec:distill} for more experimental details.

\textbf{Evaluation protocol.} {We evaluate the ID and the OOD performance of pre-trained and distilled encoders by training linear probes on top of their output representations on the ID training set and then evaluate those linear probes on both ID and OOD test sets. Note that under our protocol, the linear probes still face OOD generalization tasks on the OOD test set, albeit with representations instead of raw images as inputs.} To compare the OOD generalization ability of different models, we follow the evaluation protocol of \emph{effective robustness}~\citep{taori_measuring_2020}, which quantifies a model's distribution shift robustness as its OOD performance advantage over a baseline representing the OOD performance of standard models trained on ID data. Following~\citet{taori_measuring_2020}, we illustrate the effective robustness of our models using scatter plots, with $x$-axes representing ID performance and $y$-axes representing OOD performance.

\textbf{Results.} As shown in~\figref{fig:results}, linear probes on distilled representations exhibit consistent effective robustness gains over standard models.
This is not very surprising given that the distilled models have additional supervision provided by the representations obtained from the teacher models in training, while standard models do not. However, the upshot is that \emph{even with explicit access to good representations, the OOD generalization performance of distilled models still lags far behind their pre-trained counterparts}.
For example, distilled models only close about half of the average effective robustness gap between standard models and pre-trained models in ImageNet-based datasets, with even worse performance on iWildCam and DomainNet. Note that this is not due to the failure in distillation, as distilled models do achieve similar \emph{ID performance} to that of the pre-trained models.
Our results thus suggest that the limited algorithmic success in OOD generalization cannot be simply explained by not having a ``good enough'' representation learning objective.

\textbf{What is the cause of the above failure?} First, one may argue that spurious correlations can still play a role here, as the representations extracted by pre-trained models may also contain spurious correlations to the label even if they achieve generally good OOD performance. While we do acknowledge this possibility, we emphasize that spurious correlations \emph{cannot} explain the large \emph{OOD performance gap} between the distilled and the pre-trained models since we would expect them to be similarly impacted by the spurious correlations in their representations.
Another plausible explanation is data leakage, i.e., CLIP may have ``seen'' many OOD examples in its pre-training stage and thus can extract richer predictive features for OOD examples~\citep{zhang_learning_2023}.
However, this possibility is nullified by a recent study~\citep{mayilvahanan_does_2024}, which shows that CLIP's distribution shift robustness persists even when OOD examples are pruned from its pre-training dataset.

In a nutshell, we argue that \emph{existing explanations are insufficient to account for the above OOD generalization gap}. This suggests that taking into the account the inductive biases of SGD-trained neural networks are necessary for understanding the OOD generalization failure in practice. In the following sections, we will formally identify feature contamination as a novel OOD generalization failure mode and further connect it to the results in this section.
\section{A Theoretical Model of OOD Generalization}
\label{sec:setup}


\textbf{Notation.} We use $[d]$ to denote the set $\{1, \ldots, d\}$ for positive integers $d$.
For a set $\mathcal{S}$, we denote its cardinality by $|\mathcal{S}|$.
For a vector $\rvu$, we denote its $\ell^2$-norm by $\lVert\rvu\rVert_2$.
We denote the inner product of two vectors $\rvu$ and $\rvv$ by $\langle\rvu,\rvv\rangle$. We use the standard big-O notation: $O(\cdot)$, $\Omega(\cdot)$, $\Theta(\cdot)$, $o(\cdot)$, as well as their soft-O variants such as $\wt{\Theta}(\cdot)$ to hide logarithmic factors. For some parameter $d$, we use $\poly{d}$ to denote $\Theta(d^C)$ with some unspecified large constant $C$. 
We use $\ind{E}$ to denote the indicator function for an event $E$.

\subsection{OOD Generalization Problem Setup}
\label{subsec:dgm}

\textbf{Task and data.} We consider a binary classification task with an input space $\mathcal{X}\subseteq\mathbb{R}^d$, a label space $\mathcal{Y} = \{-1, 1\}$, a model class $\mathcal{H}:\mathcal{X}\to\mathbb{R}$, and a loss function $\ell:\yspace\times\yspace\to\mathbb{R}$.
For every distribution $\mathcal{D}$ over $\xspace\times\yspace$ and model $h\in\mathcal{H}$, the expected risk of $h$ on $\mathcal{D}$ is given by $\exrisk{\mathcal{D}}{h}\defeq\mathbb{E}_{(\rvx, \rvy)\sim\mathcal{D}}\ell(h(\rvx),\rvy)$.
In an OOD generalization problem, there exist a set of distributions $\mathbb{D}$ that consists of all possible distributions to which we would like our model to generalize. In training, we have access to a training distribution set $\mathbb{D}_\mathrm{train}\subsetneq\mathbb{D}$, where $\mathbb{D}_\mathrm{train}$ may contain one or multiple training distributions. Following prior work~\citep{arjovsky_invariant_2019,sagawa_distributionally_2020,nagarajan_understanding_2021,rosenfeld_risks_2021}, we aim to select a model $h\in\mathcal{H}$ to minimize the \emph{OOD risk}, defined as the worst-case expected risk on $\mathbb{D}$:
\begin{equation}
\oodrisk{h}\defeq \max_{\mathcal{D}\in\mathbb{D}} \mathcal{R}_\mathcal{D}(h).
\label{eq:ood_risk}
\end{equation}
It is clear that without further assumptions on $\dtrainset$ and $\dset$, OOD generalization is impossible since no model can generalize to an arbitrary distribution.
Fortunately, real-world distribution shifts are often \emph{structured} with some structural similarities shared by different distributions. We can thus hope that such structures can be captured by certain algorithms to train models that can generalize OOD.

To formalize this, in this work we assume that both ID and OOD data are generated by a dictionary $\bm{M} = (\bm{m}_1,\ldots,\bm{m}_{\di})\in\mathbb{R}^{d\times\di}$ consisting of $\di$ features with each feature $\bm{m}_i\in\mathbb{R}^d$. Throughout the paper, we work with the case where $\di$ is sufficiently large and $d \in [\Omega(\di^{2.01}),\poly{\di}]$. For simplicity, we assume that every feature satisfies $\norm{\bm{m}_i}=1$ and different features are orthogonal: $\forall i\ne j\in[\di], \inprod{\bm{m}_i}{\bm{m}_j} = 0$.\footnote{Another advantage of assuming orthogonal features is that this prevents the network from learning background features due to their correlations with core features. We note that our results can be extended to more general settings without orthogonality.}

Among all features in $\bm{M}$, we assume that there are $\dcore$ features consistently correlating with the label in all distributions in $\dset$. We denote the index set of those features by $\score\subsetneq[\di]$ and refer to them as {\textbf{core features}} since they are consistently predictive of the label in all distributions. We refer to the remaining features as {\textbf{background features}} and denote their index set by $\sbg = [\di]\setminus\score$ with $\dbg\defeq|\sbg|=\di-\dcore$. We assume that $\dcore=\Theta(\di)$ and $\dbg = \Theta(\frac{\di}{\log\di})$, so that the number of both core features and background features is non-negligible. With the above definitions, we introduce our ID and OOD data generation model in~\defref{def:dgp}.
\begin{definition}[ID and OOD data generation]
\label{def:dgp}
Under the feature model stated above, consider a training distribution (ID data distribution) $\dtrain\in\dtrainset$ and a test distribution (OOD data distribution) $\dtest\in\dset\setminus\dtrainset$.\footnote{Note that $\dtrain$ and $\dtest$ can also be (weighted) \emph{mixtures} of multiple distributions in $\dtrainset$ and $\mathbb{D}\setminus\dtrainset$, respectively.}
Each example $(\rvx, \rvy)\sim \mathcal{D}\in\{\dtrain,\dtest\}$ is generated as follows:\vspace{-0.5em}
\begin{enumerate}[leftmargin=1.5em]
\setlength\itemsep{0.1em}
\item Sample a label $\rvy$ from the uniform distribution over $\yspace$.\vspace{-0.2em}
\item Sample a weight vector $\rvz=(\rvz_1,\ldots,\rvz_{\di})\in\mathbb{R}^{\di}$ where different coordinates of $\rvz$ are independent random variables generated as follows:\vspace{-0.2em}
\begin{itemize}[leftmargin=1em]
	\item \textbf{{ID data ($\mathcal{D} = \dtrain$):}} for every $j\in[\di]$,
	sample $\rvz_j$ from some distribution $\mathcal{D}_j$ over $[0, 1]$ such that its moments satisfy $\mu_{jp}\defeq\expect{\mathcal{D}_j}{\zj^p} = \Theta(1)$ for $p\in[3]$ and its variance satisfies $\sigma^2_j = \Theta(1)$.
	\item \textbf{{OOD data ($\mathcal{D} = \dtest$):}} for every $j\in[\di]$, if $j\in\score$, sample $\rvz_j$ from $\mathcal{D}_j$ over $[0, 1]$; if $j\in\sbg$, sample $\rvz_j$ from some distribution $\mathcal{D}'_j$ over $[-1, 0]$ such that $\expect{\mathcal{D}_j'}{\zj} = -\Theta(1)$.\vspace{-0.2em}
\end{itemize}
\item Generate $\rvx = \sum_{j\in\score}\rvy\rvz_j\bm{m}_j + \sum_{j\in\sbg}\rvz_j\bm{m}_j$.
\end{enumerate}
\end{definition}

\textbf{Remarks on data generation.} Our data model formalizes a structured OOD generalization setup reflecting several facets of real-world OOD generalization problems:\vspace{-0.4em}
\begin{itemize}[leftmargin=1.em]
\setlength\itemsep{0.05em}
\item The explicit separation of core and background features captures structural assumptions that make OOD generalization tractable: under the distribution shifts on background features, there still exists a set of core features that enable robust classification. Hence, a model that discards background features and retains core features can generalize OOD. This rules out the ill-posed case where the ID data is not informative enough to learn a generalizable model~\citep{tripuraneni_theory_2020,xu_how_2021,kumar_fine-tuning_2022}, and is also the key intuition of many OOD generalization algorithms aiming to learn invariant representations~\citep{gulrajani_search_2021}.
\item The weights of background features are assumed to be independent of the label, rendering background features and labels \emph{uncorrelated}. This differs from prior OOD generalization analysis~\citep{arjovsky_invariant_2019,sagawa_investigation_2020,nagarajan_understanding_2021,rosenfeld_risks_2021} where background features are assumed to be \emph{spuriously correlated} with the label and hence useful for prediction. We intentionally make this assumption to ``ablate'' the effect of spurious correlations in feature learning.\footnote{On the other hand, our results can be extended to the settings where some background features have spurious correlations.}
\end{itemize}

\subsection{Model and Training}
\label{sec:model}

\textbf{Model.} We consider a model class $\mathcal{H}$ representing width-$m$ two-layer neural networks with ReLU activation. Formally, given hidden-layer weights $\rmW = (\rvw_1,\ldots,\rvw_m) \in\mathbb{R}^{d\times m}$ and output-layer weights $\rva = (a_1,\ldots,a_m)^\top\in\mathbb{R}^m$, the output of a model $h\in\mathcal{H}$ given an input $\rvx\in\xspace$ is defined as
\begin{equation}
h(\rvx) = \sum\nolimits_{k\in[m]}a_k\cdot\relu{(\inprod{\rvw_{k}}{\rvx})},
\label{eq:nn}
\end{equation}
where $\relu(u) = \max\{u,0\},u\in\mathbb{R}$. Similar to practical design choices, we consider an \emph{overparameterized} setting where $m \in [\Theta(\di),\Theta(d)]$. We initialize each weight vector $\rvw_k,k\in[m]$ by sampling $\its{\rvw_i}{0}\sim\mathcal{N}(\mathbf{0},\sigma^2_0\bm{I}_d)$ with $\sigma_0^2 = \frac{1}{d}$.
We randomly initialize output-layer weights $\rva$ by sampling $a_k\sim \mathsf{Uniform}\{-\frac{1}{m},\frac{1}{m}\}$ independently for each $k\in[m]$. To simplify our analysis, we keep output-layer weights $\rva$ fixed during training, which is a common assumption in analyzing two-layer neural networks~\citep{allen-zhu_feature_2021,karp_local_2021,allen-zhu_towards_2023}.

\begin{figure*}[t]
\centering
\includegraphics[width=0.95\linewidth]{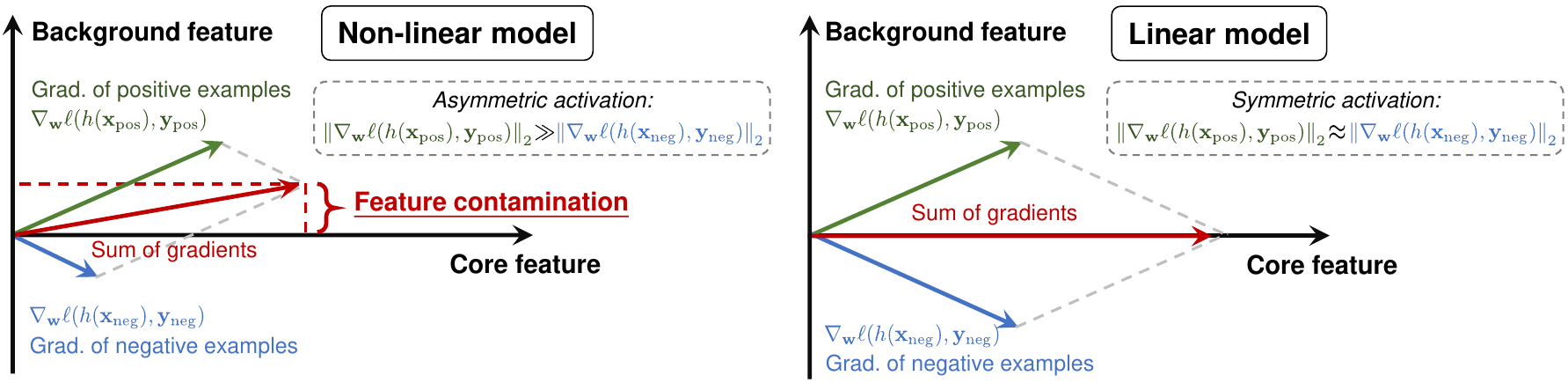}
\vspace{-0.7em}
\caption{A diagram of feature contamination in our binary classification setting. \textbf{Left:} for models with non-linear activation functions such as ReLU, activation asymmetry leads to non-zero gradient projections onto background features. \textbf{Right:} for linear models, background features are cancelled out in the gradients, exhibiting no feature contamination.
}
\label{fig:theory}
\vspace{-0.5em}
\end{figure*}

\textbf{Training.} We train the network using SGD to minimize a standard hinge loss $\ell(y,y') = \max\{1-yy', 0\}$ with step size $\eta > 0$ for $T$ iterations. We also include a weight decay with strength $\lambda = O(\frac{\di}{m^{1.01}})$ for regularization. At each iteration $t\in\{0,\ldots,T\}$, we i.i.d. sample a batch of examples $\{(\its{\rvx_i}{t},\its{\rvy_i}{t})\}_{i\in[N]}\sim\dtrain^N$ with batch size $N = \poly{d}$ and consider the following empirical loss:
\begin{equation}
\widehat{\mathcal{L}}(\its{h}{t}) = \frac{1}{N}\sum_{i\in[N]} \ell\left(\its{h}{t}(\its{\rvx_i}{t}), \its{\rvy_i}{t}\right) + \frac{\lambda}{2} \sum_{k\in[m]}\bignorm{\its{\rvw_k}{t}}^2,
\label{eq:batch_loss}
\end{equation}
where we use $\its{h}{t}$ to denote the model at iteration $t$, with weights $\its{\rmW}{t} = (\its{\rvw_1}{t},\ldots,\its{\rvw_m}{t})$.
The SGD update for each weight vector $\rvw_k,k\in[m]$ is then given by
\begin{equation}
\its{\rvw_k}{t+1} = \its{\rvw_k}{t} - \eta \nabla_{\its{\rvw_k}{t}} \widehat{\mathcal{L}}(\its{h}{t}).
\label{eq:sgd}
\end{equation}

\section{Main Theoretical Results}
\label{sec:main}

In this section, we present our main theoretical results, provide mathematical reasoning of why feature contamination happens, and discuss its impact on generalization.
We also include numerical results and show that our findings can be extended to more general data and neural network models.

\textbf{Technical challenges.} As we have discussed in~\secref{sec:intro}, most existing theoretical work on OOD generalization \emph{separates generalization and optimization} and directly studies the \emph{global minimizers} of their training objecives without considering optimization dynamics. By contrast, without a unique global minimizer, our setup requires an explicit analysis on the SGD optimization trajectory, which is known to be challenging due to its \emph{non-convex} and \emph{non-linear} nature. Prior work has studied fine-tuning pre-trained models for OOD generalization in the context of two-layer \emph{linear networks}~\citep{kumar_fine-tuning_2022,lee_surgical_2023}. Analyzing non-linear networks further requires a careful treatment on the activation property of the neurons, which results in SGD dynamics that essentially deviate from linear networks.

\begin{figure*}[t]
\centering
\includegraphics[width=0.93\linewidth]{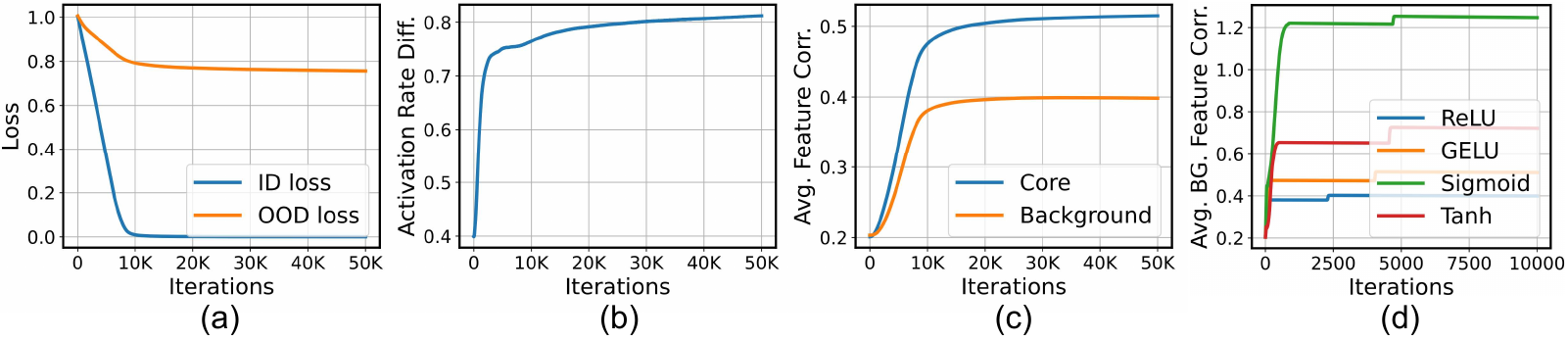}
\vspace{-0.7em}
\caption{Numerical results. \textbf{(a)} \emph{ID and OOD risks:} During training, ID loss quickly approaches zero, while OOD loss stays high. \textbf{(b)} \emph{Activation asymmetry:} the difference of average neuron activation rates for different classes largely increases during training. \textbf{(c)} \emph{Feature contamination:} the average correlations between neuron weights and both core features and \emph{uncorrelated background features} increase in training. \textbf{(d)} Feature contamination also occurs in more general settings with different activation functions. Please refer to~\secref{appsubsec:numerical} for more details and results.
}
\label{fig:numerical}
\vspace{-1.1em}
\end{figure*}

\textbf{Our approach.} At a high level, our analysis is based on the construction of two neuron subsets $\nyt$ (see \defref{def:neuron}) for $y\in\yspace=\{-1,1\}$ at iteration $t\in\{0,\ldots,T\}$ so that each subset has cardinality $\Theta(m)$ and its neurons are randomly initialized to have large enough expected correlations with the examples from the class $y$ (i.e., ``winning the lottery tickets''~\citep{frankle_lottery_2019,allen-zhu_feature_2021}). We then apply the Berry-Esseen theorem to bound the class-conditional activation probabilities of ReLU for the neurons in those subsets. By a careful treatment of the activation probabilities as the neurons evolve during training, we can bound the expected gradients for each neuron in $\nyt$ at every step $t$, hence iteratively tracking its weight updates throughout training. This treatment allows us to characterize the output of the network up to constant factors while avoiding the nuisance of analyzing the activation probability of \emph{every} neuron in the network, which turns out to be very challenging. For ease of presentation, in the sequel we separate our main results into four parts and introduce them progressively, with an illustration of our key ideas in~\figref{fig:theory}. Complete proofs of all theoretical results are deferred to~\hyperlink{app:theory}{Appendix \uppercase\expandafter{\romannumeral 1}}.

\textbf{\emph{{1.} Neuron activation is asymmetric.}} Our key insight is that during training, every neuron in $\nyt$ has the incentive to be positively correlated with the examples from at most one class $\rvy_\mathrm{pos} = y$ (whether $y=1$ or $y=-1$ depends on the random initialization of the neuron); we refer to those examples as \emph{positive examples} $(\rvx_\mathrm{pos}, \rvy_\mathrm{pos})\sim\dtrain|\rvy=\rvy_\mathrm{pos}$ for that neuron. Correspondingly, we refer to examples from the other class $\rvy_\mathrm{neg} = -y$ as \emph{negative examples} $(\rvx_\mathrm{neg}, \rvy_\mathrm{neg})\sim\dtrain|\rvy=\rvy_\mathrm{neg}$ for the neuron. Due to randomness at initialization, we can show that $|\nyinit| = \Theta(m)$ for both $y\in\{-1,1\}$ and, after sufficient SGD iterations, all neurons in $\nyt$ will accumulate (in expectation) positive correlations with examples from $y$ and negative correlations with examples from $-y$, resulting in class-wise asymmetry in their activation as shown by~\theoref{theo:activation}.

\begin{theorem}[Activation asymmetry]
\label{theo:activation}
For every $\eta \le \frac{1}{\poly{\di}}$ and every $y\in\yspace$, there exists $T_0 = \wt{\Theta}(\frac{m}{\eta\sqrt{d}})$ such that w.h.p., for every $t\ge T_0$, there exist $\Theta(m)$ neurons in which the weight $\rvw_k^{(t)}$ for each neuron satisfies:
\begin{equation}
\begin{aligned}
&\prob{\rvx|\rvy=y\sim\dtrain}{[\inp{\wk{t}}{\rvx} \ge 0]} = 1 - O\left(\di^{-\frac{1}{2}}\right),\\
&\prob{\rvx|\rvy=-y\sim\dtrain}{[\inp{\wk{t}}{\rvx} \ge 0]} = o(1).
\end{aligned}
\end{equation}
\end{theorem}

\begin{figure}[t]
\centering
\includegraphics[width=0.49\linewidth]{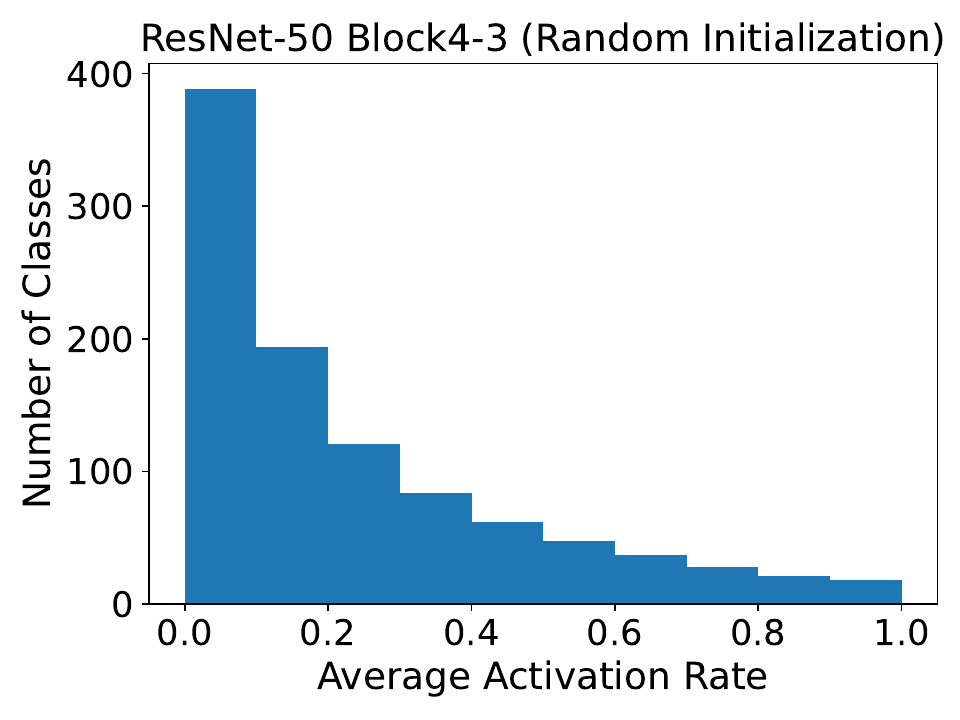}
\includegraphics[width=0.49\linewidth]{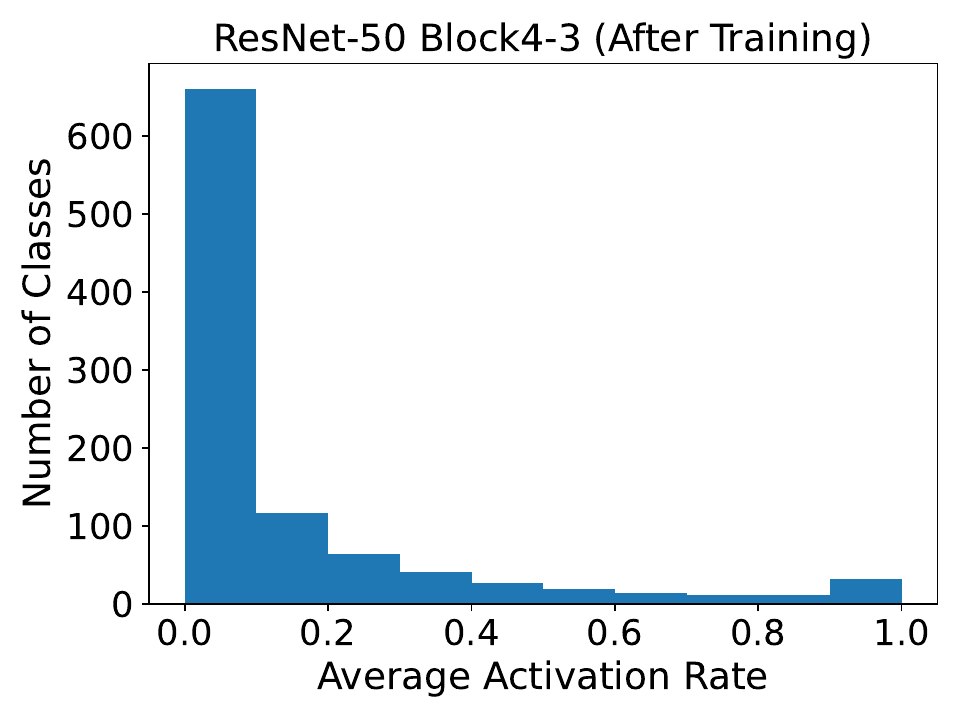}
\vspace{-0.5em}
\caption{\emph{Class-averaged} activation rate histograms of a randomly initialized CLIP-RN50 \textbf{(left)} and a distilled CLIP-RN50 \textbf{(right)}. After training, more classes have smaller average activation rates close to zero and only a small number of classes have large average activation rates.
}
\label{fig:histograms}
\vspace{-1.0em}
\end{figure}

\emph{\textbf{2. Activation asymmetry leads to feature contamination.}}
Note that for every $k\in[m]$, the weight vector of the $k$-th neuron (we will also refer to it as the \emph{learned feature} of the neuron) after $t$ iterations can be written as
\begin{equation}
\its{\rvw_k}{t} = \sum_{j\in\score }\inprod{\its{\rvw_k}{t}}{\bm{m}_j}\bm{m}_j + \sum_{j\in\sbg}\inprod{\its{\rvw_k}{t}}{\bm{m}_j}\bm{m}_j + \mathrm{res},
\label{eq:weight_decompose}
\end{equation}
where the residual term satisfies $\inprod{\mathrm{res}}{\bm{m}_j} = 0$ for every $j\in[\di]$ and thus can be neglected.
Intuitively, Eq.~\eqref{eq:weight_decompose} indicates that the learned feature can be decomposed into its projections onto different feature vectors.
Meanwhile, as we will prove in~\lemmaref{lemma:corr_gradient}, at iteration $t$, the gradient projection onto background features for every neuron $k\in \nyt$ satisfies: for every $j\in\sbg$,
\begin{equation}
\begin{aligned}
&\inprod{-\nabla_{\its{\rvw_k}{t}} \widehat{\mathcal{L}}(\its{h}{t})}{\bm{m}_j}\propto\\
&\qquad \expect{(\rvx,\rvy)}{(\indicator{\rvy=\rvy_\mathrm{pos}} - \indicator{\rvy=\rvy_\mathrm{neg}}) \indicator{\inprod{\its{\rvw_k}{t}}{\rvx} \ge 0} \rvz_j}.
\label{eq:approx_grad}
\end{aligned}
\end{equation}
By~\theoref{theo:activation}, we then have that for at least $\Theta(m)$ neurons in $\nyt$, $\expect{\rvx|\rvy=\rvy_\mathrm{pos}}{\ind{\inp{\wkt}{\rvx}\ge 0}}$ would be much larger than $\expect{\rvx|\rvy=\rvy_{\mathrm{neg}}}{\ind{\inp{\wkt}{\rvx}\ge 0}}$, resulting in a positive gradient projection onto \emph{every} background feature $\mj$ \emph{regardless of its correlation with the label}. We refer to this feature learning proclivity of neural networks as \textbf{feature contamination}. Formally,~\theoref{theo:feature} shows that this will result in the neurons' learned features accumulating both correlated core features and \emph{uncorrelated} background features.
\begin{theorem}[Learned features]
\label{theo:feature}
For every $\eta \le \frac{1}{\poly{\di}}$ and every $y\in\yspace$, there exists $T_1 = \Theta(\frac{m}{\eta\di})$ such that w.h.p., after $T_1$ iterations, there exist $\Theta(m)$ neurons in which the weight $\wk{T_1}$ for each neuron satisfies the following:
\begin{equation}
\begin{aligned}
&\sum\nolimits_{j\in\score}\mu_{j1}\inp{\wk{T_1}}{\mj} = y\cdot\Theta(1),\\
&\sum\nolimits_{j\in\sbg}\mu_{j1}\inp{\wk{T_1}}{\mj} = \wt{\Theta}(1).
\label{eq:feature}
\end{aligned}
\end{equation}
\end{theorem}

\begin{figure}[t]
\centering
\includegraphics[width=0.49\linewidth]{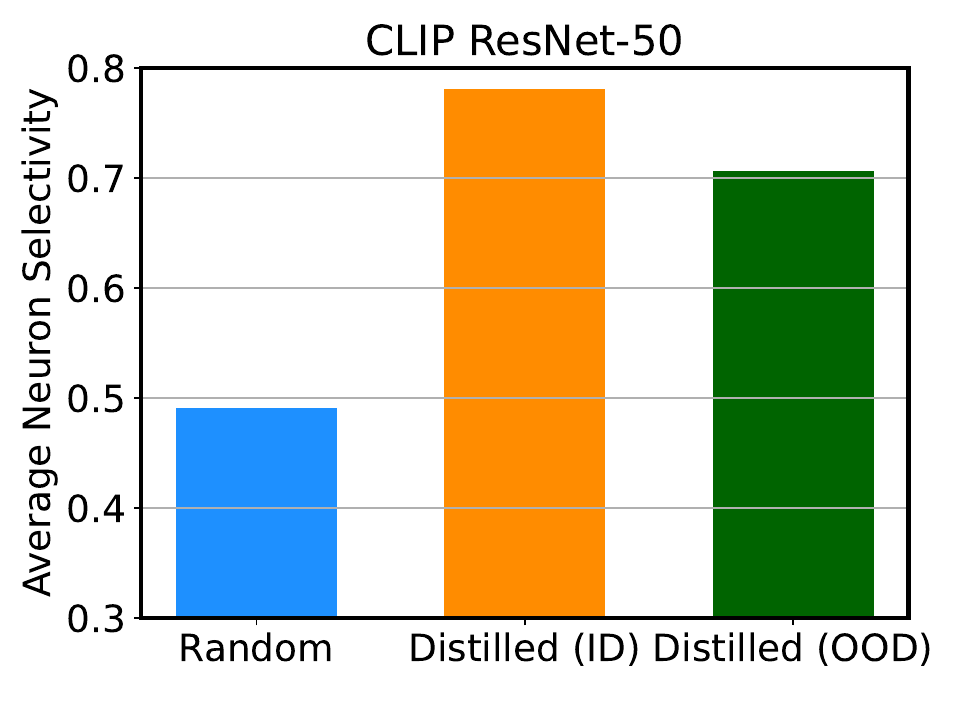}
\includegraphics[width=0.49\linewidth]{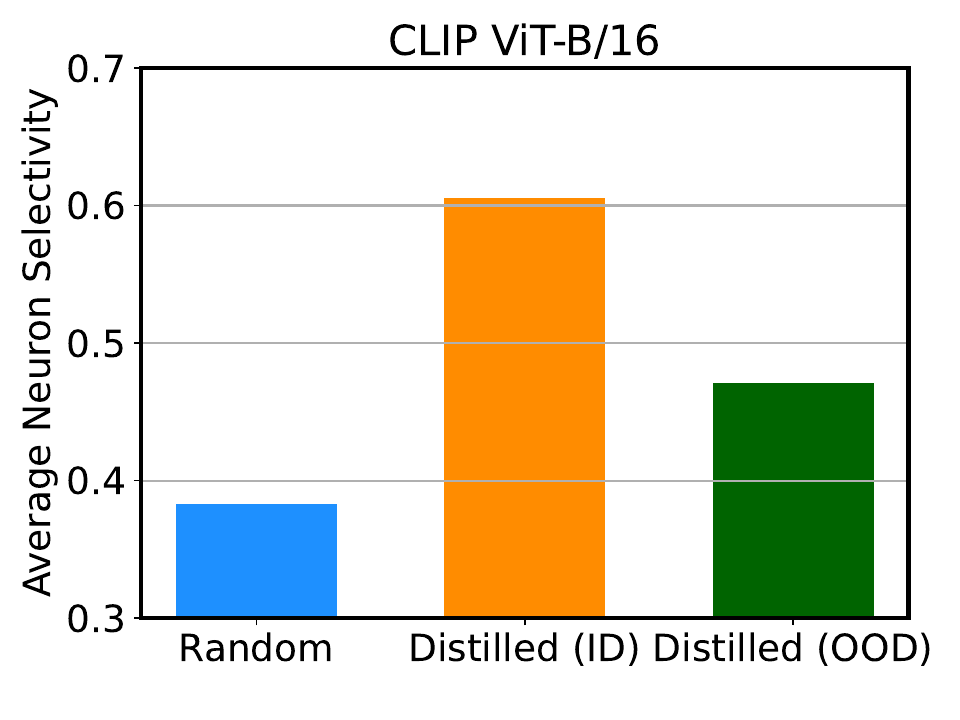}
\vspace{-0.5em}
\caption{Average \emph{neuron selectivity} of random and distilled CLIP-RN50 \textbf{(left)} and CLIP-ViT-B/16 \textbf{(right)} models.
Distilled models have larger selectivity compared with random models and exhibit a selectivity drop in OOD data. Please refer to~\secref{appsec:selectivity} for more details.
}
\label{fig:selectivity}
\vspace{-1.2em}
\end{figure}

\begin{figure*}[t]
\centering
\includegraphics[width=0.96\linewidth]{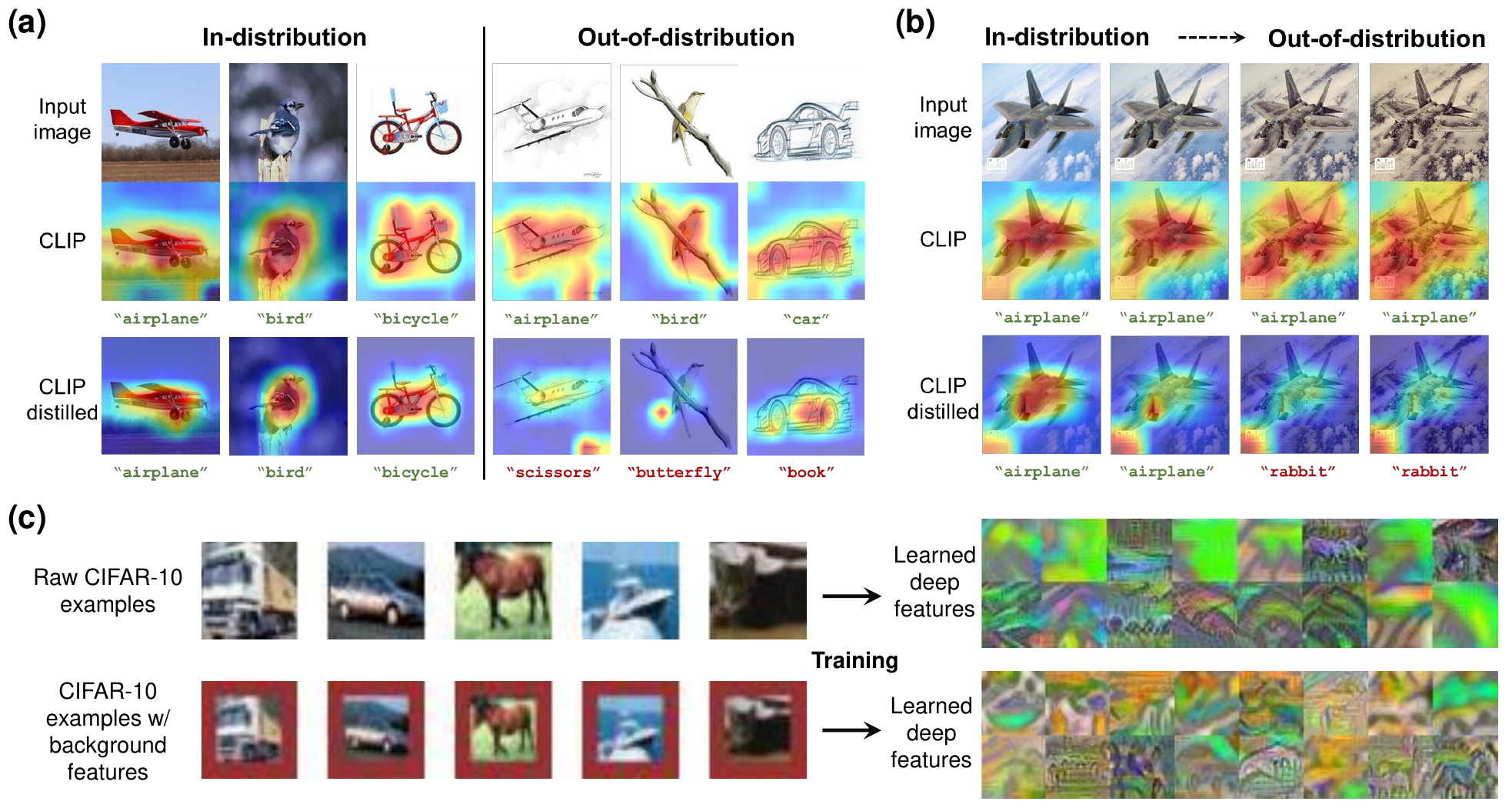}
\vspace{-0.5em}
\caption{Empirical evidence for feature contamination in \emph{deep} neural networks. \textbf{(a)} Grad-CAM results for CLIP-RN50 and a distilled CLIP-RN50 on DomainNet. \textbf{(b)} Grad-CAM results for CLIP-RN50 and a distilled CLIP-RN50 on controlled distribution shifts. For the distilled model, while the weights of core objects are dominant for ID images, they are \emph{reduced} under distribution shifts. \textbf{(c)} Example images of the raw CIFAR-10 dataset and our modified version with background color features that are \emph{uncorrelated} with the label, and the visualization of deep features learned by a ResNet on both datasets.}
\label{fig:visualization}
\end{figure*}

\emph{\textbf{3. Feature contamination induces large OOD risk.}}
Intuitively, this result is a direct consequence of the \emph{coupling} of core features and background features in the neurons' pre-activation as shown by~\theoref{theo:feature}. With this coupling, negative shifts of background features can \emph{reduce the activation of the neuron}, resulting in OOD risk. In extreme cases, when the pre-activation of a neuron is reduced to negative, the contribution of the core features extracted by the neuron will also diminish. Formally,~\theoref{theo:risk} quantifies such impact of feature contamination on ID and OOD risks.

\begin{theorem}[ID and OOD risks]
\label{theo:risk}
For every $\eta \le \frac{1}{\poly{\di}}$, there exists $T_2 = \wt{\Theta}(\frac{m}{\eta\di})$ such that w.h.p., after $T_2$ iterations, the trained model $\its{h}{T_2}$ satisfies the following:
\begin{equation}
\exrisk{\dtrain}{\its{h}{T_2}} \le o(1),\ \oodrisk{\its{h}{T_2}} = \wt{\Theta}(1).
\end{equation}

\end{theorem}
\vspace{-0.4em}

\emph{\textbf{4. Linear networks are provably free from feature contamination.}} Finally, to further understand the role of non-linearity, we prove that if we ``remove'' the non-linearity in the network by replacing each ReLU with the identity function,
then feature contamination will no longer occur.

\begin{theorem}[Linear networks]
\label{prop:linear}
If we replace the ReLU functions in the network with identity functions and keep other conditions the same as in~\theoref{theo:feature}, then with high probability, we have $|\inprod{\its{\rvw_k}{T_1}}{\bm{m}_j}| \le \wt{O}(\frac{1}{\sqrt{d}})$ for every $k\in[m]$ and every $j\in\sbg$.\vspace{-0.5em}
\end{theorem}
The main intuition of~\theoref{prop:linear} is that without non-linearity, the activation magnitude for the examples from different classes would be no longer asymmetric: for two-layer linear networks, the gradient projection onto background features is akin to Eq.~\eqref{eq:approx_grad} but without the activation derivative $\ind{\inp{\wkt}{\rvx}\ge 0}$. We then have $\inprod{-\nabla_{\its{\rvw_k}{t}} \widehat{\mathcal{L}}(\its{h}{t})}{\bm{m}_j}\approx 0$ for every $j\in\sbg$ since positive gradients and negative gradients on $\rvz_j$ will now cancel out. As a result, the background features will not be accumulated during SGD.

\textbf{Numerical results.} As empirical evidence that corroborates our theory, we provide numerical results in~\figref{fig:numerical} that demonstrate the existence of activation asymmetry and feature contamination under our setting.

\textbf{Extensions to more general settings.} To show that feature contamination also occurs \emph{beyond} our theoretical setting, we also conduct numerical experiments with several relaxations from our setting, including (\romannumeral 1) \emph{non-linear} relationships between core features and the label, (\romannumeral 2) more activation functions other than ReLU, and (\romannumeral 3) optimizers with adaptive step sizes other than vanilla SGD. As shown by~\figref{fig:numerical}(d), feature contamination consistently occurs in more general settings. In~\secref{appsubsec:numerical}, we also provide numerical results in regression (representation distillation) tasks to show that feature contamination also occurs beyond classification.

\section{Feature Contamination in Practice}
\label{sec:evidence}

In this section, we present empirical evidence that connects our theoretical results to practical OOD generalization failure and discuss possible solutions.

\textbf{Activation asymmetry in deep networks.} In our theoretical model, feature contamination stems from the asymmetric activation of the neurons. To examine whether deep networks trained on real-world data also exhibit this behavior, we compute the \emph{class-averaged} activation rate histograms of the ResNet and ViT models trained in our experiments in~\secref{sec:main_exp}. As shown in~\figref{fig:histograms} and~\secref{appsec:histograms}, both models exhibit activation asymmetry after training, with low average activation rates that are close to zero for most classes and high average activation rates for only a small number of classes. Moreover, we also adopt a more quantitative metric termed \emph{neuron selectivity} that measures the difference of neuron activation magnitudes of different classes. As shown in~\figref{fig:selectivity}, distilled models have considerably larger average neuron selectivity than random models, which corroborates our theory.
Please refer to~\secref{appsec:selectivity} for the implementation details of computing neuron selectivity.

\textbf{Prediction heatmap visualization.} In~\figref{fig:visualization}(a), we visualize the prediction heatmaps of CLIP-RN50 and a distilled CLIP-RN50 in the DomainNet dataset using Grad-CAM~\citep{chattopadhay2018grad}.
An intriguing phenomenon revealed by the heatmaps is that for the distilled model, while the weights of core objects are dominant for ID images, they are \emph{reduced} under distribution shifts, resulting in OOD generalization failure. Quantitatively, a similar observation is shown in~\figref{fig:selectivity}, where the selectivity of neurons drops in OOD data.
Here we argue that \emph{feature contamination explains this phenomenon}: the Grad-CAM score of each feature map is calculated by differentiating the classification score with respect to the feature map, thus being proportional to the corresponding neurons' activation. Hence, when the core features and the background features are coupled in the neurons' pre-activation (\theoref{theo:feature}), the shift of background features can reduce the activation,
which in turn reduces the Grad-CAM score of the feature map as visually observed. The reduction in neuron selectivity in OOD data can also be explained in a similar way.

\textbf{Prediction heatmaps under controlled distribution shifts.} In~\figref{fig:visualization}(b), we consider a synthetic OOD generalization task based on image style transfer where we can manually control the ``degree'' of distribution shifts by controlling the amount of style change. This setting closely matches our data model in~\defref{def:dgp} by keeping core features intact while \emph{only} changing background features in the ID images. As shown in the figure, the prediction heatmaps exhibit visually similar patterns as the heatmaps of natural OOD images. This implies that our data model indeed captures some key characteristics of real-world distribution shifts and backs the explanation of feature contamination.

\textbf{Visualizing contaminated deep features.} To qualitatively show the impact of feature contamination on the learned features of \emph{deep} neural networks, we visualize the features learned by a ResNet on a modified CIFAR-10 dataset in~\figref{fig:visualization}(c), where the original images are padded with background red pixels \emph{uncorrelated} with labels. Compared to the original CIFAR-10 dataset, the learned features on the modified CIFAR-10 dataset exhibits evident color differences, indicating that feature contamination also occurs in \emph{deep features}. See~\secref{appsec:visualization} for more details and results.

\textbf{Discussion on possible solutions.} In our theoretical model, although gradient descent accumulates both core and background features in the weight space due to feature contamination, there still exists a \emph{subspace} (corresponding to the span of core features) where background features are not accumulated. Hence, constraining the SGD updates to this subspace would possibly lead to ideal generalization. This is consistent with prior results~\citep{idnani_dont_2023} showing that projecting the network's intermediate representations onto certain subspaces may improve OOD generalization. However, how to effectively find the correct subspace for projection without the explicit access to the core and background feature subspaces remains an open problem.

\section{Conclusion}
\label{sec:discussion}


In this section, we discuss potential implications of our results and list important future directions.


\textbf{Takeaway 1: OOD generalization algorithms need to consider inductive biases.} Many existing studies on OOD generalization motivate and analyze their algorithms using linear or unstructured models that do not capture the inductive biases of neural networks. Our results imply that OOD generalization may not be feasible without considering such inductive biases, calling for explicitly incorporating them into principled algorithm design.


\textbf{Takeaway 2: Non-linearity in neural networks elicits new OOD generalization challenges.} As we formally show in~\secref{sec:main}, feature contamination essentially stems from the gradient descent optimization process of non-linear neural networks even with \emph{uncorrelated} background features, thus being orthogonal to the prevailing narrative of spurious correlations. This provides a new perspective on OOD generalization and may inspire new algorithmic design.

\textbf{Takeaway 3: Learned features may behave very differently from prescribed ones.} Many existing studies on OOD generalization explicitly or implicitly assume that we can directly work on a set of \emph{well-separated} core/spurious features. While this assumption helps build intuitions, our results highlight that it can also be misleading since the features \emph{learned} by neural networks may manifest in a \emph{non-linearly coupled} manner, thus often diverging from the intuitions for prescribed, well-separated features.

\subsection{Limitations, a Conjecture, and Future Work}
\label{subsec:conjecture}

While this work takes a step towards fully understanding OOD generalization in practice, our results still leave much room for improvement such as extensions to more general data distributions, multi-class classification, and more complicated network architectures.
Meanwhile, while our current setup focuses on training from scratch, we envision that the viewpoint of feature contamination may also be helpful in analyzing the effect of \emph{pre-training} on OOD generalization. In particular, we have the following conjecture:
\begin{conjecture*}
Pre-training on a sufficiently diverse dataset does not remove uncorrelated features, but \emph{linearizes} those features in the model's representations, hence mitigating feature contamination and improving OOD generalization.
\end{conjecture*}
We provide preliminary empirical evidence that supports the above conjecture in~\secref{appsec:conjecture}, as well as more discussion on related empirical observations in recent work~\citep{gandelsman_interpreting_2024,mayilvahanan_does_2024}. Yet, we believe that rigorously proving this conjecture requires a more fine-grained treatment for the pre-training data distribution and the SGD dynamics and thus leave it as future work.


\section*{Acknowledgements}
This work was supported in part by the National Key Research and Development Program of China under STI 2030-Major Projects 2021ZD0200300, and in part by the National Natural Science Foundation of China under Grant 62176133, and in part by the Tsinghua-Guoqiang research program under Grant 2019GQG0006.

\section*{Impact Statement}
This paper presents work that aims to advance our understanding of the feature learning process of neural networks and its impact on generalization under distribution shifts, which may benefit building machine learning models that are more generalizable, robust, and trustworthy.







\bibliography{ref}

\begin{thebibliography}{89}
\providecommand{\natexlab}[1]{#1}
\providecommand{\url}[1]{\texttt{#1}}
\expandafter\ifx\csname urlstyle\endcsname\relax
  \providecommand{\doi}[1]{doi: #1}\else
  \providecommand{\doi}{doi: \begingroup \urlstyle{rm}\Url}\fi

\bibitem[Abbe et~al.(2023)Abbe, Bengio, Lotfi, and
  Rizk]{abbe_generalization_2023}
Abbe, E., Bengio, S., Lotfi, A., and Rizk, K.
\newblock Generalization on the unseen, logic reasoning and degree curriculum.
\newblock In \emph{International {Conference} on {Machine} {Learning}}, 2023.

\bibitem[Ahuja et~al.(2021{\natexlab{a}})Ahuja, Caballero, Zhang, Bengio,
  Mitliagkas, and Rish]{ahuja_invariance_2021}
Ahuja, K., Caballero, E., Zhang, D., Bengio, Y., Mitliagkas, I., and Rish, I.
\newblock Invariance principle meets information bottleneck for
  out-of-distribution generalization.
\newblock In \emph{Advances in {Neural} {Information} {Processing} {Systems}},
  pp.\  3438--3450, 2021{\natexlab{a}}.

\bibitem[Ahuja et~al.(2021{\natexlab{b}})Ahuja, Wang, Dhurandhar, Shanmugam,
  and Varshney]{ahuja_empirical_2021}
Ahuja, K., Wang, J., Dhurandhar, A., Shanmugam, K., and Varshney, K.~R.
\newblock Empirical or invariant risk minimization? {A} sample complexity
  perspective.
\newblock In \emph{{ICLR}}, 2021{\natexlab{b}}.

\bibitem[Allen-Zhu \& Li(2021)Allen-Zhu and Li]{allen-zhu_feature_2021}
Allen-Zhu, Z. and Li, Y.
\newblock Feature purification: {How} adversarial training performs robust deep
  learning.
\newblock \emph{arXiv preprint arXiv:2005.10190}, 2021.

\bibitem[Allen-Zhu \& Li(2023{\natexlab{a}})Allen-Zhu and
  Li]{allen-zhu_physics_2023}
Allen-Zhu, Z. and Li, Y.
\newblock Physics of language models: {Part} 3.1, knowledge storage and
  extraction.
\newblock \emph{arXiv preprint arXiv:2309.14316}, 2023{\natexlab{a}}.

\bibitem[Allen-Zhu \& Li(2023{\natexlab{b}})Allen-Zhu and
  Li]{allen-zhu_towards_2023}
Allen-Zhu, Z. and Li, Y.
\newblock Towards understanding ensemble, knowledge distillation and
  self-distillation in deep learning.
\newblock In \emph{International {Conference} on {Learning} {Representations}},
  2023{\natexlab{b}}.

\bibitem[Amodei et~al.(2016)Amodei, Olah, Steinhardt, Christiano, Schulman, and
  Mané]{amodei_concrete_2016}
Amodei, D., Olah, C., Steinhardt, J., Christiano, P., Schulman, J., and Mané,
  D.
\newblock Concrete problems in {AI} safety.
\newblock \emph{arXiv preprint arXiv:1606.06565}, 2016.

\bibitem[Arjovsky et~al.(2019)Arjovsky, Bottou, Gulrajani, and
  Lopez-Paz]{arjovsky_invariant_2019}
Arjovsky, M., Bottou, L., Gulrajani, I., and Lopez-Paz, D.
\newblock Invariant risk minimization.
\newblock \emph{arXiv preprint arXiv:1907.02893}, 2019.

\bibitem[Arpit et~al.(2017)Arpit, Jastrzebski, Ballas, Krueger, Bengio, Kanwal,
  Maharaj, Fischer, Courville, Bengio, and Lacoste-Julien]{arpit_closer_2017}
Arpit, D., Jastrzebski, S., Ballas, N., Krueger, D., Bengio, E., Kanwal, M.~S.,
  Maharaj, T., Fischer, A., Courville, A., Bengio, Y., and Lacoste-Julien, S.
\newblock A closer look at memorization in deep networks.
\newblock In \emph{International {Conference} on {Machine} {Learning}}, pp.\
  233--242, 2017.

\bibitem[Barbu et~al.(2019)Barbu, Mayo, Alverio, Luo, Wang, Gutfreund,
  Tenenbaum, and Katz]{barbu_objectnet_2019}
Barbu, A., Mayo, D., Alverio, J., Luo, W., Wang, C., Gutfreund, D., Tenenbaum,
  J., and Katz, B.
\newblock {ObjectNet}: {A} large-scale bias-controlled dataset for pushing the
  limits of object recognition models.
\newblock In \emph{Advances in {Neural} {Information} {Processing} {Systems}},
  pp.\  9448--9458, 2019.

\bibitem[Beery et~al.(2018)Beery, Van~Horn, and
  Perona]{ferrari_recognition_2018}
Beery, S., Van~Horn, G., and Perona, P.
\newblock Recognition in terra incognita.
\newblock In \emph{{ECCV}}, volume 11220, pp.\  472--489, 2018.

\bibitem[Bommasani et~al.(2022)Bommasani, Hudson, Adeli, Altman, Arora, von
  Arx, Bernstein, Bohg, Bosselut, Brunskill, Brynjolfsson, Buch, Card,
  Castellon, Chatterji, Chen, Creel, Davis, Demszky, Donahue, Doumbouya,
  Durmus, Ermon, Etchemendy, Ethayarajh, Fei-Fei, Finn, Gale, Gillespie, Goel,
  Goodman, Grossman, Guha, Hashimoto, Henderson, Hewitt, Ho, Hong, Hsu, Huang,
  Icard, Jain, Jurafsky, Kalluri, Karamcheti, Keeling, Khani, Khattab, Koh,
  Krass, Krishna, Kuditipudi, Kumar, Ladhak, Lee, Lee, Leskovec, Levent, Li,
  Li, Ma, Malik, Manning, Mirchandani, Mitchell, Munyikwa, Nair, Narayan,
  Narayanan, Newman, Nie, Niebles, Nilforoshan, Nyarko, Ogut, Orr,
  Papadimitriou, Park, Piech, Portelance, Potts, Raghunathan, Reich, Ren, Rong,
  Roohani, Ruiz, Ryan, Ré, Sadigh, Sagawa, Santhanam, Shih, Srinivasan,
  Tamkin, Taori, Thomas, Tramèr, Wang, Wang, Wu, Wu, Wu, Xie, Yasunaga, You,
  Zaharia, Zhang, Zhang, Zhang, Zhang, Zheng, Zhou, and
  Liang]{bommasani_opportunities_2022}
Bommasani, R., Hudson, D.~A., Adeli, E., Altman, R., Arora, S., von Arx, S.,
  Bernstein, M.~S., Bohg, J., Bosselut, A., Brunskill, E., Brynjolfsson, E.,
  Buch, S., Card, D., Castellon, R., Chatterji, N., Chen, A., Creel, K., Davis,
  J.~Q., Demszky, D., Donahue, C., Doumbouya, M., Durmus, E., Ermon, S.,
  Etchemendy, J., Ethayarajh, K., Fei-Fei, L., Finn, C., Gale, T., Gillespie,
  L., Goel, K., Goodman, N., Grossman, S., Guha, N., Hashimoto, T., Henderson,
  P., Hewitt, J., Ho, D.~E., Hong, J., Hsu, K., Huang, J., Icard, T., Jain, S.,
  Jurafsky, D., Kalluri, P., Karamcheti, S., Keeling, G., Khani, F., Khattab,
  O., Koh, P.~W., Krass, M., Krishna, R., Kuditipudi, R., Kumar, A., Ladhak,
  F., Lee, M., Lee, T., Leskovec, J., Levent, I., Li, X.~L., Li, X., Ma, T.,
  Malik, A., Manning, C.~D., Mirchandani, S., Mitchell, E., Munyikwa, Z., Nair,
  S., Narayan, A., Narayanan, D., Newman, B., Nie, A., Niebles, J.~C.,
  Nilforoshan, H., Nyarko, J., Ogut, G., Orr, L., Papadimitriou, I., Park,
  J.~S., Piech, C., Portelance, E., Potts, C., Raghunathan, A., Reich, R., Ren,
  H., Rong, F., Roohani, Y., Ruiz, C., Ryan, J., Ré, C., Sadigh, D., Sagawa,
  S., Santhanam, K., Shih, A., Srinivasan, K., Tamkin, A., Taori, R., Thomas,
  A.~W., Tramèr, F., Wang, R.~E., Wang, W., Wu, B., Wu, J., Wu, Y., Xie,
  S.~M., Yasunaga, M., You, J., Zaharia, M., Zhang, M., Zhang, T., Zhang, X.,
  Zhang, Y., Zheng, L., Zhou, K., and Liang, P.
\newblock On the opportunities and risks of foundation models.
\newblock \emph{arXiv preprint arXiv:2108.07258}, 2022.

\bibitem[Brown et~al.(2020)Brown, Mann, Ryder, Subbiah, Kaplan, Dhariwal,
  Neelakantan, Shyam, Sastry, Askell, Agarwal, Herbert-Voss, Krueger, Henighan,
  Child, Ramesh, Ziegler, Wu, Winter, Hesse, Chen, Sigler, Litwin, Gray, Chess,
  Clark, Berner, McCandlish, Radford, Sutskever, and
  Amodei]{brown_language_2020}
Brown, T.~B., Mann, B., Ryder, N., Subbiah, M., Kaplan, J., Dhariwal, P.,
  Neelakantan, A., Shyam, P., Sastry, G., Askell, A., Agarwal, S.,
  Herbert-Voss, A., Krueger, G., Henighan, T., Child, R., Ramesh, A., Ziegler,
  D.~M., Wu, J., Winter, C., Hesse, C., Chen, M., Sigler, E., Litwin, M., Gray,
  S., Chess, B., Clark, J., Berner, C., McCandlish, S., Radford, A., Sutskever,
  I., and Amodei, D.
\newblock Language models are few-shot learners.
\newblock In \emph{Advances in {Neural} {Information} {Processing} {Systems}},
  2020.

\bibitem[Chattopadhay et~al.(2018)Chattopadhay, Sarkar, Howlader, and
  Balasubramanian]{chattopadhay2018grad}
Chattopadhay, A., Sarkar, A., Howlader, P., and Balasubramanian, V.~N.
\newblock Grad-cam++: Generalized gradient-based visual explanations for deep
  convolutional networks.
\newblock In \emph{2018 IEEE winter conference on applications of computer
  vision (WACV)}, pp.\  839--847. IEEE, 2018.

\bibitem[Chen et~al.(2022)Chen, Rosenfeld, Sellke, Ma, and
  Risteski]{chen_iterative_2022}
Chen, Y., Rosenfeld, E., Sellke, M., Ma, T., and Risteski, A.
\newblock Iterative feature matching: {Toward} provable domain generalization
  with logarithmic environments.
\newblock In \emph{Advances in {Neural} {Information} {Processing} {Systems}},
  2022.

\bibitem[Chen et~al.(2023)Chen, Huang, Zhou, Bian, Han, and
  Cheng]{chen_understanding_2023}
Chen, Y., Huang, W., Zhou, K., Bian, Y., Han, B., and Cheng, J.
\newblock Understanding and improving feature learning for out-of-distribution
  generalization.
\newblock In \emph{Advances in {Neural} {Information} {Processing} {Systems}},
  2023.

\bibitem[DeGrave et~al.(2021)DeGrave, Janizek, and Lee]{degrave_ai_2021}
DeGrave, A.~J., Janizek, J.~D., and Lee, S.-I.
\newblock {AI} for radiographic {COVID}-19 detection selects shortcuts over
  signal.
\newblock \emph{Nature Machine Intelligence}, 3\penalty0 (7):\penalty0
  610--619, 2021.
\newblock ISSN 2522-5839.

\bibitem[Deng et~al.(2009)Deng, Dong, Socher, Li, Li, and
  Fei-Fei]{deng_imagenet:_2009}
Deng, J., Dong, W., Socher, R., Li, L.-J., Li, K., and Fei-Fei, L.
\newblock Imagenet: {A} large-scale hierarchical image database.
\newblock In \emph{{CVPR}}, pp.\  248--255, 2009.

\bibitem[Elhage et~al.(2022)Elhage, Hume, Olsson, Schiefer, Henighan, Kravec,
  Hatfield-Dodds, Lasenby, Drain, and Chen]{elhage_toy_2022}
Elhage, N., Hume, T., Olsson, C., Schiefer, N., Henighan, T., Kravec, S.,
  Hatfield-Dodds, Z., Lasenby, R., Drain, D., and Chen, C.
\newblock Toy models of superposition.
\newblock \emph{arXiv preprint arXiv:2209.10652}, 2022.

\bibitem[Frankle \& Carbin(2019)Frankle and Carbin]{frankle_lottery_2019}
Frankle, J. and Carbin, M.
\newblock The lottery ticket hypothesis: {Finding} sparse, trainable neural
  networks.
\newblock In \emph{{ICLR}}, 2019.

\bibitem[Gandelsman et~al.(2024)Gandelsman, Efros, and
  Steinhardt]{gandelsman_interpreting_2024}
Gandelsman, Y., Efros, A.~A., and Steinhardt, J.
\newblock Interpreting {CLIP}'s image representation via text-based
  decomposition.
\newblock In \emph{International {Conference} on {Learning} {Representations}},
  2024.

\bibitem[Ganin et~al.(2016)Ganin, Ustinova, Ajakan, Germain, Larochelle,
  Laviolette, Marchand, and Lempitsky]{ganin_domain-adversarial_2016}
Ganin, Y., Ustinova, E., Ajakan, H., Germain, P., Larochelle, H., Laviolette,
  F., Marchand, M., and Lempitsky, V.
\newblock Domain-adversarial training of neural networks.
\newblock \emph{Journal of Machine Learning Research}, 17\penalty0
  (59):\penalty0 1--35, 2016.

\bibitem[Geirhos et~al.(2018)Geirhos, Temme, Rauber, Schütt, Bethge, and
  Wichmann]{geirhos_generalisation_2018}
Geirhos, R., Temme, C. R.~M., Rauber, J., Schütt, H.~H., Bethge, M., and
  Wichmann, F.~A.
\newblock Generalisation in humans and deep neural networks.
\newblock In \emph{Advances in {Neural} {Information} {Processing} {Systems}},
  pp.\  7549--7561, 2018.

\bibitem[Geirhos et~al.(2020)Geirhos, Jacobsen, Michaelis, Zemel, Brendel,
  Bethge, and Wichmann]{geirhos_shortcut_2020}
Geirhos, R., Jacobsen, J.-H., Michaelis, C., Zemel, R., Brendel, W., Bethge,
  M., and Wichmann, F.~A.
\newblock Shortcut learning in deep neural networks.
\newblock \emph{Nature Machine Intelligence}, 2\penalty0 (11):\penalty0
  665--673, 2020.
\newblock ISSN 2522-5839.

\bibitem[Gould et~al.(2023)Gould, Ong, Ogden, and Conmy]{gould_successor_2023}
Gould, R., Ong, E., Ogden, G., and Conmy, A.
\newblock Successor heads: {Recurring}, interpretable attention heads in the
  wild.
\newblock \emph{arXiv preprint arXiv:2312.09230}, 2023.

\bibitem[Gulrajani \& Lopez-Paz(2021)Gulrajani and
  Lopez-Paz]{gulrajani_search_2021}
Gulrajani, I. and Lopez-Paz, D.
\newblock In search of lost domain generalization.
\newblock In \emph{{ICLR}}, 2021.

\bibitem[Gurnee \& Tegmark(2024)Gurnee and Tegmark]{gurnee_language_2024}
Gurnee, W. and Tegmark, M.
\newblock Language models represent space and time.
\newblock In \emph{International {Conference} on {Learning} {Representations}},
  2024.

\bibitem[HaoChen et~al.(2022)HaoChen, Wei, Kumar, and Ma]{haochen_beyond_2022}
HaoChen, J.~Z., Wei, C., Kumar, A., and Ma, T.
\newblock Beyond separability: {Analyzing} the linear transferability of
  contrastive representations to related subpopulations.
\newblock \emph{arXiv preprint arXiv:2204.02683}, 2022.

\bibitem[Heinzerling \& Inui(2024)Heinzerling and
  Inui]{heinzerling_monotonic_2024}
Heinzerling, B. and Inui, K.
\newblock Monotonic representation of numeric properties in language models.
\newblock \emph{arXiv preprint arXiv:2403.10381}, 2024.

\bibitem[Hendrycks \& Gimpel(2016)Hendrycks and
  Gimpel]{hendrycks_gaussian_2016}
Hendrycks, D. and Gimpel, K.
\newblock Gaussian error linear units ({GELUs}).
\newblock \emph{arXiv preprint arXiv:1606.08415}, 2016.

\bibitem[Hendrycks et~al.(2021{\natexlab{a}})Hendrycks, Basart, Mu, Kadavath,
  Wang, Dorundo, Desai, Zhu, Parajuli, Guo, Song, Steinhardt, and
  Gilmer]{hendrycks_many_2021}
Hendrycks, D., Basart, S., Mu, N., Kadavath, S., Wang, F., Dorundo, E., Desai,
  R., Zhu, T., Parajuli, S., Guo, M., Song, D., Steinhardt, J., and Gilmer, J.
\newblock The many faces of robustness: {A} critical analysis of
  out-of-distribution generalization.
\newblock In \emph{{ICCV}}, 2021{\natexlab{a}}.

\bibitem[Hendrycks et~al.(2021{\natexlab{b}})Hendrycks, Zhao, Basart,
  Steinhardt, and Song]{hendrycks_natural_2021}
Hendrycks, D., Zhao, K., Basart, S., Steinhardt, J., and Song, D.
\newblock Natural adversarial examples.
\newblock In \emph{{CVPR}}, 2021{\natexlab{b}}.

\bibitem[Hinton et~al.(2014)Hinton, Vinyals, and Dean]{hinton_distilling_2014}
Hinton, G., Vinyals, O., and Dean, J.
\newblock Distilling the knowledge in a neural network.
\newblock In \emph{Advances in {Neural} {Information} {Processing} {Systems}
  {Deep} {Learning} {Workshop}}, 2014.

\bibitem[Huang et~al.(2020)Huang, Wang, Xing, and
  Huang]{huang_self-challenging_2020}
Huang, Z., Wang, H., Xing, E.~P., and Huang, D.
\newblock Self-challenging improves cross-domain generalization.
\newblock In \emph{European {Conference} on {Computer} {Vision}}, pp.\
  124--140, 2020.

\bibitem[Idnani et~al.(2023)Idnani, Madan, Goyal, Schwab, and
  Vedantam]{idnani_dont_2023}
Idnani, D., Madan, V., Goyal, N., Schwab, D.~J., and Vedantam, S.~R.
\newblock Don’t forget the nullspace! {Nullspace} occupancy as a mechanism
  for out of distribution failure.
\newblock 2023.

\bibitem[Kamath et~al.(2021)Kamath, Tangella, Sutherland, and
  Srebro]{kamath_does_2021}
Kamath, P., Tangella, A., Sutherland, D.~J., and Srebro, N.
\newblock Does invariant risk minimization capture invariance?
\newblock In \emph{{AISTATS}}, 2021.

\bibitem[Karp et~al.(2021)Karp, Winston, Li, and Singh]{karp_local_2021}
Karp, S., Winston, E., Li, Y., and Singh, A.
\newblock Local signal adaptivity: {Provable} feature learning in neural
  networks beyond kernels.
\newblock In \emph{Advances in {Neural} {Information} {Processing} {Systems}},
  pp.\  24883--24897, 2021.

\bibitem[Kim et~al.(2021)Kim, Yoo, Park, Kim, and Lee]{kim_selfreg_2021}
Kim, D., Yoo, Y., Park, S., Kim, J., and Lee, J.
\newblock {SelfReg}: {Self}-supervised contrastive regularization for domain
  generalization.
\newblock In \emph{2021 {IEEE}/{CVF} {International} {Conference} on {Computer}
  {Vision} ({ICCV})}, pp.\  9599--9608, 2021.

\bibitem[Kingma \& Ba(2015)Kingma and Ba]{kingma_adam:_2015}
Kingma, D.~P. and Ba, J.
\newblock Adam: {A} method for stochastic optimization.
\newblock In \emph{International {Conference} on {Learning} {Representations}},
  2015.

\bibitem[Koh et~al.(2021)Koh, Sagawa, Marklund, Xie, Zhang, Balsubramani, Hu,
  Yasunaga, Phillips, Gao, Lee, David, Stavness, Guo, Earnshaw, Haque, Beery,
  Leskovec, Kundaje, Pierson, Levine, Finn, and Liang]{koh_wilds_2021}
Koh, P.~W., Sagawa, S., Marklund, H., Xie, S.~M., Zhang, M., Balsubramani, A.,
  Hu, W., Yasunaga, M., Phillips, R.~L., Gao, I., Lee, T., David, E., Stavness,
  I., Guo, W., Earnshaw, B.~A., Haque, I.~S., Beery, S., Leskovec, J., Kundaje,
  A., Pierson, E., Levine, S., Finn, C., and Liang, P.
\newblock Wilds: {A} benchmark of in-the-wild distribution shifts.
\newblock In \emph{{ICML}}, 2021.

\bibitem[Krueger et~al.(2021)Krueger, Caballero, Jacobsen, Zhang, Binas, Zhang,
  Priol, and Courville]{krueger_out--distribution_2021}
Krueger, D., Caballero, E., Jacobsen, J.-H., Zhang, A., Binas, J., Zhang, D.,
  Priol, R.~L., and Courville, A.
\newblock Out-of-distribution generalization via risk extrapolation (rex).
\newblock In \emph{{ICML}}, 2021.

\bibitem[Kumar et~al.(2022)Kumar, Raghunathan, Jones, Ma, and
  Liang]{kumar_fine-tuning_2022}
Kumar, A., Raghunathan, A., Jones, R., Ma, T., and Liang, P.
\newblock Fine-tuning can distort pretrained features and underperform
  out-of-distribution.
\newblock In \emph{International {Conference} on {Learning} {Representations}},
  2022.

\bibitem[Lee et~al.(2023)Lee, Chen, Tajwar, Kumar, Yao, Liang, and
  Finn]{lee_surgical_2023}
Lee, Y., Chen, A.~S., Tajwar, F., Kumar, A., Yao, H., Liang, P., and Finn, C.
\newblock Surgical fine-tuning improves adaptation to distribution shifts.
\newblock In \emph{International {Conference} on {Learning} {Representations}},
  2023.

\bibitem[Li \& Flanigan(2023)Li and Flanigan]{li_task_2023}
Li, C. and Flanigan, J.
\newblock Task contamination: {Language} models may not be few-shot anymore.
\newblock \emph{arXiv preprint arXiv:2312.16337}, 2023.

\bibitem[Li et~al.(2017)Li, Yang, Song, and Hospedales]{li_deeper_2017}
Li, D., Yang, Y., Song, Y.-Z., and Hospedales, T.~M.
\newblock Deeper, broader and artier domain generalization.
\newblock In \emph{2017 {IEEE} {International} {Conference} on {Computer}
  {Vision} ({ICCV})}, pp.\  5543--5551, 2017.

\bibitem[Li et~al.(2018)Li, Yang, Song, and Hospedales]{li_learning_2018}
Li, D., Yang, Y., Song, Y.-Z., and Hospedales, T.~M.
\newblock Learning to generalize: {Meta}-learning for domain generalization.
\newblock In \emph{{AAAI}}, 2018.

\bibitem[Liu et~al.(2023)Liu, Ning, Teng, Liu, Zhou, and
  Zhang]{liu_evaluating_2023}
Liu, H., Ning, R., Teng, Z., Liu, J., Zhou, Q., and Zhang, Y.
\newblock Evaluating the logical reasoning ability of chatgpt and gpt-4.
\newblock \emph{arXiv preprint arXiv:2304.03439}, 2023.

\bibitem[Loshchilov \& Hutter(2017)Loshchilov and Hutter]{loshchilov_sgdr_2017}
Loshchilov, I. and Hutter, F.
\newblock {SGDR}: {Stochastic} gradient descent with warm restarts.
\newblock In \emph{International {Conference} on {Learning} {Representations}},
  2017.

\bibitem[Loshchilov \& Hutter(2019)Loshchilov and
  Hutter]{loshchilov_decoupled_2019}
Loshchilov, I. and Hutter, F.
\newblock Decoupled weight decay regularization.
\newblock In \emph{International {Conference} on {Learning} {Representations}},
  2019.

\bibitem[Marks \& Tegmark(2023)Marks and Tegmark]{marks_geometry_2023}
Marks, S. and Tegmark, M.
\newblock The geometry of truth: {Emergent} linear structure in large language
  model representations of true/false datasets.
\newblock \emph{arXiv preprint arXiv:2310.06824}, 2023.

\bibitem[Mayilvahanan et~al.(2024)Mayilvahanan, Wiedemer, Rusak, Bethge, and
  Brendel]{mayilvahanan_does_2024}
Mayilvahanan, P., Wiedemer, T., Rusak, E., Bethge, M., and Brendel, W.
\newblock Does {CLIP}'s generalization performance mainly stem from high
  train-test similarity?
\newblock In \emph{International {Conference} on {Learning} {Representations}},
  2024.

\bibitem[Miller et~al.(2021)Miller, Taori, Raghunathan, Sagawa, Koh, Shankar,
  Liang, Carmon, and Schmidt]{miller_accuracy_2021}
Miller, J., Taori, R., Raghunathan, A., Sagawa, S., Koh, P.~W., Shankar, V.,
  Liang, P., Carmon, Y., and Schmidt, L.
\newblock Accuracy on the line: {On} the strong correlation between
  out-of-distribution and in-distribution generalization.
\newblock In \emph{{ICML}}, 2021.

\bibitem[Mitrovic et~al.(2021)Mitrovic, McWilliams, Walker, Buesing, and
  Blundell]{mitrovic_representation_2021}
Mitrovic, J., McWilliams, B., Walker, J.~C., Buesing, L.~H., and Blundell, C.
\newblock Representation learning via invariant causal mechanisms.
\newblock In \emph{{ICLR}}, 2021.

\bibitem[Morcos et~al.(2018)Morcos, Barrett, Rabinowitz, and
  Botvinick]{morcos_importance_2018}
Morcos, A.~S., Barrett, D. G.~T., Rabinowitz, N.~C., and Botvinick, M.
\newblock On the importance of single directions for generalization.
\newblock In \emph{International {Conference} on {Learning} {Representations}},
  2018.
\newblock URL \url{http://arxiv.org/abs/1803.06959}.

\bibitem[Nagarajan et~al.(2021)Nagarajan, Andreassen, and
  Neyshabur]{nagarajan_understanding_2021}
Nagarajan, V., Andreassen, A., and Neyshabur, B.
\newblock Understanding the failure modes of out-of-distribution
  generalization.
\newblock In \emph{{ICLR}}, 2021.

\bibitem[Nam et~al.(2021)Nam, Lee, Park, Yoon, and Yoo]{nam_reducing_2021}
Nam, H., Lee, H., Park, J., Yoon, W., and Yoo, D.
\newblock Reducing domain gap by reducing style bias.
\newblock In \emph{Proceedings of the {IEEE}/{CVF} {Conference} on {Computer}
  {Vision} and {Pattern} {Recognition}}, pp.\  8690--8699, 2021.

\bibitem[Oliveira(2010)]{oliveira2009concentration}
Oliveira, R.~I.
\newblock Concentration of the adjacency matrix and of the laplacian in random
  graphs with independent edges.
\newblock \emph{arXiv preprint arXiv:0911.0600}, 2010.

\bibitem[Park et~al.(2023)Park, Choe, and Veitch]{park_linear_2023}
Park, K., Choe, Y.~J., and Veitch, V.
\newblock The linear representation hypothesis and the geometry of large
  language models.
\newblock \emph{arXiv preprint arXiv:2311.03658}, 2023.

\bibitem[Peng et~al.(2019)Peng, Bai, Xia, Huang, Saenko, and
  Wang]{peng_moment_2019}
Peng, X., Bai, Q., Xia, X., Huang, Z., Saenko, K., and Wang, B.
\newblock Moment matching for multi-source domain adaptation.
\newblock In \emph{{ICCV}}, pp.\  1406--1415, 2019.

\bibitem[Pezeshki et~al.(2021)Pezeshki, Kaba, Bengio, Courville, Precup, and
  Lajoie]{pezeshki_gradient_2021}
Pezeshki, M., Kaba, S.-O., Bengio, Y., Courville, A., Precup, D., and Lajoie,
  G.
\newblock Gradient starvation: {A} learning proclivity in neural networks.
\newblock In \emph{Advances in {Neural} {Information} {Processing} {Systems}},
  pp.\  1256--1272, 2021.

\bibitem[Radford et~al.(2021)Radford, Kim, Hallacy, Ramesh, Goh, Agarwal,
  Sastry, Askell, Mishkin, Clark, Krueger, and
  Sutskever]{radford_learning_2021}
Radford, A., Kim, J.~W., Hallacy, C., Ramesh, A., Goh, G., Agarwal, S., Sastry,
  G., Askell, A., Mishkin, P., Clark, J., Krueger, G., and Sutskever, I.
\newblock Learning transferable visual models from natural language
  supervision.
\newblock In \emph{International {Conference} on {Machine} {Learning}}, pp.\
  8748--8763, 2021.

\bibitem[Recht et~al.(2019)Recht, Roelofs, Schmidt, and
  Shankar]{recht_imagenet_2019}
Recht, B., Roelofs, R., Schmidt, L., and Shankar, V.
\newblock Do imagenet classifiers generalize to imagenet?
\newblock In \emph{{ICML}}, 2019.

\bibitem[Rosenfeld et~al.(2021)Rosenfeld, Ravikumar, and
  Risteski]{rosenfeld_risks_2021}
Rosenfeld, E., Ravikumar, P., and Risteski, A.
\newblock The risks of invariant risk minimization.
\newblock In \emph{{ICLR}}, 2021.

\bibitem[Rosenfeld et~al.(2022)Rosenfeld, Ravikumar, and
  Risteski]{rosenfeld_domain-adjusted_2022}
Rosenfeld, E., Ravikumar, P., and Risteski, A.
\newblock Domain-adjusted regression or: {ERM} may already learn features
  sufficient for out-of-distribution generalization.
\newblock \emph{arXiv preprint arXiv:2202.06856}, 2022.

\bibitem[Sagawa et~al.(2020{\natexlab{a}})Sagawa, Koh, Hashimoto, and
  Liang]{sagawa_distributionally_2020}
Sagawa, S., Koh, P.~W., Hashimoto, T.~B., and Liang, P.
\newblock Distributionally robust neural networks for group shifts: {On} the
  importance of regularization for worst-case generalization.
\newblock In \emph{{ICLR}}, 2020{\natexlab{a}}.

\bibitem[Sagawa et~al.(2020{\natexlab{b}})Sagawa, Raghunathan, Koh, and
  Liang]{sagawa_investigation_2020}
Sagawa, S., Raghunathan, A., Koh, P.~W., and Liang, P.
\newblock An investigation of why overparameterization exacerbates spurious
  correlations.
\newblock In \emph{{ICML}}, 2020{\natexlab{b}}.

\bibitem[Salman et~al.(2019)Salman, Li, Razenshteyn, Zhang, Zhang, Bubeck, and
  Yang]{salman2019provably}
Salman, H., Li, J., Razenshteyn, I., Zhang, P., Zhang, H., Bubeck, S., and
  Yang, G.
\newblock Provably robust deep learning via adversarially trained smoothed
  classifiers.
\newblock In \emph{Advances in Neural Information Processing Systems},
  volume~32, 2019.

\bibitem[Schölkopf et~al.(2021)Schölkopf, Locatello, Bauer, Ke, Kalchbrenner,
  Goyal, and Bengio]{scholkopf_toward_2021}
Schölkopf, B., Locatello, F., Bauer, S., Ke, N.~R., Kalchbrenner, N., Goyal,
  A., and Bengio, Y.
\newblock Toward causal representation learning.
\newblock \emph{Proceedings of the IEEE}, 109\penalty0 (5):\penalty0 612--634,
  2021.
\newblock ISSN 1558-2256.

\bibitem[Shah et~al.(2020)Shah, Tamuly, and Raghunathan]{shah_pitfalls_2020}
Shah, H., Tamuly, K., and Raghunathan, A.
\newblock The pitfalls of simplicity bias in neural networks.
\newblock In \emph{Advances in {Neural} {Information} {Processing} {Systems}},
  2020.

\bibitem[Shen et~al.(2022)Shen, Jones, Kumar, Xie, HaoChen, Ma, and
  Liang]{shen_connect_2022}
Shen, K., Jones, R., Kumar, A., Xie, S.~M., HaoChen, J.~Z., Ma, T., and Liang,
  P.
\newblock Connect, not collapse: {Explaining} contrastive learning for
  unsupervised domain adaptation.
\newblock In \emph{International {Conference} on {Machine} {Learning}}, volume
  19847-19878, 2022.

\bibitem[Shi et~al.(2022)Shi, Seely, Torr, Siddharth, Hannun, Usunier, and
  Synnaeve]{shi_gradient_2022}
Shi, Y., Seely, J., Torr, P. H.~S., Siddharth, N., Hannun, A., Usunier, N., and
  Synnaeve, G.
\newblock Gradient matching for domain generalization.
\newblock In \emph{International {Conference} on {Learning} {Representations}},
  2022.

\bibitem[Singh et~al.(2022)Singh, Gustafson, Adcock, De~Freitas~Reis, Gedik,
  Kosaraju, Mahajan, Girshick, Dollar, and Van
  Der~Maaten]{singh_revisiting_2022}
Singh, M., Gustafson, L., Adcock, A., De~Freitas~Reis, V., Gedik, B., Kosaraju,
  R.~P., Mahajan, D., Girshick, R., Dollar, P., and Van Der~Maaten, L.
\newblock Revisiting weakly supervised pre-training of visual perception
  models.
\newblock In \emph{{IEEE}/{CVF} {Conference} on {Computer} {Vision} and
  {Pattern} {Recognition} ({CVPR})}, pp.\  794--804, 2022.

\bibitem[Sun \& Saenko(2016)Sun and Saenko]{sun_deep_2016}
Sun, B. and Saenko, K.
\newblock Deep coral: {Correlation} alignment for deep domain adaptation.
\newblock In \emph{European conference on computer vision}, pp.\  443--450,
  2016.

\bibitem[Tan et~al.(2020)Tan, Peng, and Saenko]{tan_class-imbalanced_2020}
Tan, S., Peng, X., and Saenko, K.
\newblock Class-imbalanced domain adaptation: {An} empirical odyssey.
\newblock \emph{arXiv preprint arXiv:1910.10320}, 2020.

\bibitem[Taori et~al.(2020)Taori, Dave, Shankar, Carlini, Recht, and
  Schmidt]{taori_measuring_2020}
Taori, R., Dave, A., Shankar, V., Carlini, N., Recht, B., and Schmidt, L.
\newblock Measuring robustness to natural distribution shifts in image
  classification.
\newblock In \emph{Advances in {Neural} {Information} {Processing} {Systems}},
  pp.\  18583--18599, 2020.

\bibitem[Tian et~al.(2020)Tian, Krishnan, and Isola]{tian_contrastive_2020}
Tian, Y., Krishnan, D., and Isola, P.
\newblock Contrastive representation distillation.
\newblock In \emph{International {Conference} on {Learning} {Representations}},
  2020.

\bibitem[Torralba \& Efros(2011)Torralba and Efros]{torralba_unbiased_2011}
Torralba, A. and Efros, A.~A.
\newblock Unbiased look at dataset bias.
\newblock In \emph{{CVPR}}, pp.\  1521--1528, 2011.

\bibitem[Tripuraneni et~al.(2020)Tripuraneni, Jordan, and
  Jin]{tripuraneni_theory_2020}
Tripuraneni, N., Jordan, M.~I., and Jin, C.
\newblock On the theory of transfer learning: {The} importance of task
  diversity.
\newblock In \emph{Advances in {Neural} {Information} {Processing} {Systems}},
  pp.\  7852--7862, 2020.

\bibitem[Tropp(2012)]{tropp2012user}
Tropp, J.~A.
\newblock User-friendly tail bounds for sums of random matrices.
\newblock \emph{Foundations of computational mathematics}, 12:\penalty0
  389--434, 2012.

\bibitem[Vapnik(1999)]{vapnik_nature_1999}
Vapnik, V.
\newblock \emph{The nature of statistical learning theory}.
\newblock 1999.

\bibitem[Wang et~al.(2019)Wang, Ge, Xing, and Lipton]{wang_learning_2019}
Wang, H., Ge, S., Xing, E.~P., and Lipton, Z.~C.
\newblock Learning robust global representations by penalizing local predictive
  power.
\newblock In \emph{Advances in {Neural} {Information} {Processing} {Systems}},
  pp.\  10506--10518, 2019.

\bibitem[Wang et~al.(2022)Wang, Si, Li, and Zhao]{wang_provable_2022}
Wang, H., Si, H., Li, B., and Zhao, H.
\newblock Provable domain generalization via invariant-feature subspace
  recovery.
\newblock In \emph{International {Conference} on {Machine} {Learning}}, pp.\
  23018--23033, 2022.

\bibitem[Wiles et~al.(2022)Wiles, Gowal, Stimberg, Rebuffi, Ktena, Dvijotham,
  and Cemgil]{wiles_fine-grained_2022}
Wiles, O., Gowal, S., Stimberg, F., Rebuffi, S.-A., Ktena, I., Dvijotham, K.,
  and Cemgil, T.
\newblock A fine-grained analysis on distribution shift.
\newblock In \emph{{ICLR}}, 2022.

\bibitem[Wortsman et~al.(2022)Wortsman, Ilharco, Kim, Li, Kornblith, Roelofs,
  Lopes, Hajishirzi, Farhadi, Namkoong, and Schmidt]{wortsman_robust_2022}
Wortsman, M., Ilharco, G., Kim, J.~W., Li, M., Kornblith, S., Roelofs, R.,
  Lopes, R.~G., Hajishirzi, H., Farhadi, A., Namkoong, H., and Schmidt, L.
\newblock Robust fine-tuning of zero-shot models.
\newblock In \emph{2022 {IEEE}/{CVF} {Conference} on {Computer} {Vision} and
  {Pattern} {Recognition} ({CVPR})}, pp.\  7949--7961, 2022.

\bibitem[Xu et~al.(2021)Xu, Li, Zhang, Du, Kawarabayashi, and
  Jegelka]{xu_how_2021}
Xu, K., Li, J., Zhang, M., Du, S.~S., Kawarabayashi, K.-i., and Jegelka, S.
\newblock How neural networks extrapolate: {From} feedforward to graph neural
  networks.
\newblock In \emph{{ICLR}}, 2021.

\bibitem[Ye et~al.(2021)Ye, Xie, Cai, Li, Li, and Wang]{ye_towards_2021}
Ye, H., Xie, C., Cai, T., Li, R., Li, Z., and Wang, L.
\newblock Towards a theoretical framework of out-of-distribution
  generalization.
\newblock In \emph{Advances in {Neural} {Information} {Processing} {Systems}},
  2021.

\bibitem[Zhang et~al.(2018)Zhang, Cisse, Dauphin, and
  Lopez-Paz]{zhang_mixup_2018}
Zhang, H., Cisse, M., Dauphin, Y.~N., and Lopez-Paz, D.
\newblock Mixup: {Beyond} empirical risk minimization.
\newblock In \emph{{ICLR}}, 2018.

\bibitem[Zhang \& Bottou(2023)Zhang and Bottou]{zhang_learning_2023}
Zhang, J. and Bottou, L.
\newblock Learning useful representations for shifting tasks and distributions.
\newblock In \emph{International {Conference} on {Machine} {Learning}}, 2023.

\bibitem[Zhang et~al.(2021)Zhang, Marklund, Dhawan, Gupta, Levine, and
  Finn]{zhang_adaptive_2021}
Zhang, M., Marklund, H., Dhawan, N., Gupta, A., Levine, S., and Finn, C.
\newblock Adaptive risk minimization: {Learning} to adapt to domain shift.
\newblock In \emph{Advances in {Neural} {Information} {Processing} {Systems}},
  pp.\  23664--23678, 2021.

\end{thebibliography}
\bibliographystyle{icml2024}

\newpage
\appendix
\onecolumn



\begin{center}
\LARGE
\textsc{Appendix}
\vspace{0.5em}
\end{center}

The appendix is divided into two parts for readability. In~\hyperlink{app:theory}{Appendix \uppercase\expandafter{\romannumeral 1}}, we provide complete proofs of our theoretical results. In~\hyperlink{app:exp}{Appendix \uppercase\expandafter{\romannumeral 2}}, we present experimental details and additional empirical results.

\begin{center}
\vspace{0.5em}
\hypertarget{app:theory}{}
\Large
\textsc{Appendix {\uppercase\expandafter{\romannumeral 1}}: Proofs of Theoretical Results}
\vspace{0.5em}
\end{center}

In this part of the appendix, we provide complete proofs of our theorems in the main text. A quick overview of the structure of this part is as follows:
\begin{itemize}[leftmargin=2em]
\setlength\itemsep{0.5em}
	\item In~\secref{appsec:preliminary}, we introduce the preliminaries and some lemmas that characterize the neuron properties at random initialization.
	\item In~\secref{appsec:proof_main}, we provide the proofs of our main theorems on activation asymmetry (\theoref{theo:activation}), feature contamination (\theoref{theo:feature}), and ID/OOD risks (\theoref{theo:risk}).
	\item In~\secref{appsec:proof_nonlinear}, we provide the proof of~\theoref{prop:linear} on linear neural networks.
	\item In~\secref{appsec:prob_lemma}, we provide the basic probability theory lemmas used in our proofs for completeness.
\end{itemize}

\section{Preliminaries}
\label{appsec:preliminary}

\paragraph{Notation.} Throughout the appendix, we overload $\dtrain$ and $\dtest$ to allow them to denote (joint) training and test distributions on both $\xspace\times\yspace$ and $\zspace\times\yspace$. We also use $\dtrain$ and $\dtest$ to denote the corresponding marginal distributions on $\xspace$, $\yspace$ and $\zspace$.
For presentation brevity, unless otherwise stated, we use the shorthand $\expect{(\cdot)}{}$ and $\prob{(\cdot)}{}$ to denote $\expect{(\cdot)\sim\dtrain}{}$ and $\prob{(\cdot)\sim\dtrain}{}$, respectively, and use the shorthand $h$ to denote $h^{(t)}$ when it is clear from the context. As in~\defref{def:dgp}, we denote the moments of each $\rvz_j$ on the training distribution by $\mu_{jp}\defeq\expect{\rvz\sim\dtrain}{\rvz_j^p}$ for every $j\in[\di]$ and $p\in[3]$, and use the shorthand $\muj$ to denote $\mu_{j1}$ when it is clear from the context.

\subsection{Weight Decomposition and Gradient Calculations}

We begin by recalling that each weight vector $\rvw_k\in\mathbb{R}^d,k\in[m]$ (i.e., the learned feature of the $k$-th neuron) in the network can be decomposed into the sum of its projections to different feature vectors:
\begin{equation}
\its{\rvw_k}{t} = \sum_{j\in\score}\inprod{\its{\rvw_k}{t}}{\bm{m}_j}\bm{m}_j + \sum_{j\in\sbg}\inprod{\its{\rvw_k}{t}}{\bm{m}_j}\bm{m}_j +  \sum_{j\in[d]\setminus[\di]}\inprod{\its{\rvw_k}{t}}{\bm{m}_j}\bm{m}_j,
\label{eq:app_decompose}
\end{equation}
where $(\bm{m}_{\di+1},\ldots,\bm{m}_{d})$ are an orthogonal complement of $\bm{M}$. Since all possible inputs are generated to be in $\mathrm{span}\{\bm{m}_1,\ldots,\bm{m}_{\di}\}$ as in~\defref{def:dgp}, the last term in the RHS of Eq.~\eqref{eq:app_decompose} (i.e., the residual term in Eq.~\eqref{eq:weight_decompose} in the main text) can be neglected due to the orthogonality of different feature vector $\bm{m}_j$s. Therefore, throughout the following analysis, we will overload the notation $\wkt$ and let
\begin{equation}
\its{\rvw_k}{t} = \sum_{j\in\score}\inprod{\its{\rvw_k}{t}}{\bm{m}_j}\bm{m}_j + \sum_{j\in\sbg}\inprod{\its{\rvw_k}{t}}{\bm{m}_j}\bm{m}_j.
\end{equation}
A direct consequence of Eq.~\eqref{eq:app_decompose} is that we can analyze the feature learning process of the network by tracking the correlations between each weight vector $\its{\rvw_k}{t}$ and different feature vectors $\bm{m}_j,j\in[d]$ as the training proceeds. To this end, we need to first analyze the gradient of each neuron at every iteration.

\textbf{Gradient of each neuron.} Recall that at each iteration $t = 0,\ldots,T-1$, the SGD update for each weight vector $\rvw_k,k\in[m]$ is given by
\[
\its{\rvw_k}{t+1} \leftarrow (1 - \eta\lambda) \its{\rvw_k}{t} - \eta \nabla_{\its{\rvw_k}{t}}\frac{1}{N}\sum_{i\in[N]} \ell\left(\its{h}{t}(\its{\rvx_i}{t}), \its{\rvy_i}{t}\right),
\]
where
\[
	\its{h}{t}(\its{\rvx_i}{t}) = \sum_{k\in[m]}a_k\cdot\relu{\big(\inprod{\its{\rvw_k}{t}}{\its{\rvx_i}{t}}\big)}
\]
and $\ell(y,y') = \max\{1-yy',0\}$. We can then calculate the gradient of each neuron $\its{\rvw_k}{t}$ with regard to a certain example $(\rvx,\rvy)$:
\begin{lemma}[Gradient]
\label{lemma:gradient}
For every example $(x,y)\in\xspace\times\yspace$, every $k\in[m]$, and every iteration $t$, the following holds:
\begin{equation}
\nabla_{\its{\rvw_k}{t}} \ell\left(h(x), y\right) = -
a_k y \indicator{h(x) \le 1} \indicator{\inprod{\its{\rvw_k}{t}}{x}\ge 0} x.
\end{equation}
\end{lemma}
\begin{proof}
	The proof follows from simple calculation.
\end{proof}
We then introduce a lemma that bounds the empirical growth of the correlation between each neuron $\its{\rvw_k}{t}$ and each feature vector $\bm{m}_j$ after an SGD update using population gradients.
\begin{lemma}[Gap between empirical and population gradients]
	\label{lemma:gradient_gap}
For every $k\in[m]$, every $j\in[d]$, and every iteration $t$, if the batch size $N = \poly{d}$ for some sufficiently large polynomial, then the following holds with probability at least $1 - e^{-\Omega(d)}$:
\begin{equation}
\bigg|\Biginprod{\nabla_{\its{\rvw_k}{t}}\frac{1}{N}\sum_{i\in[N]} \ell\left(\its{h}{t}(\its{\rvx_i}{t}), \its{\rvy_i}{t}\right)}{\bm{m}_j} - \Biginprod{\nabla_{\its{\rvw_k}{t}}\expect{(\rvx,\rvy)\sim\dtrain}{\ell\left(h(\rvx),\rvy\right)}}{\bm{m}_j} \bigg| \le \frac{1}{\poly{d}}.
\end{equation}
\end{lemma}
\begin{proof}
Recall that $\lVert\bm{m}_j\rVert_2 = 1$. Applying Cauchy-Schwarz inequality gives
\[
	\begin{aligned}
&\bigg|\Biginprod{\nabla_{\its{\rvw_k}{t}}\frac{1}{N}\sum_{i\in[N]} \ell\left(\its{h}{t}(\its{\rvx_i}{t}), \its{\rvy_i}{t}\right)}{\bm{m}_j} - \Biginprod{\nabla_{\its{\rvw_k}{t}}\expect{(\rvx,\rvy)\sim\dtrain}{\ell\left(h(\rvx), \rvy\right)}}{\bm{m}_j} \bigg|\\
 &\le \Big\lVert \underbrace{\nabla_{\its{\rvw_k}{t}}\frac{1}{N}\sum_{i\in[N]} \ell\left(\its{h}{t}(\its{\rvx_i}{t}), \its{\rvy_i}{t}\right) - \nabla_{\its{\rvw_k}{t}}\expect{(\rvx,\rvy)\sim\dtrain}{\ell\left(h(\rvx), \rvy\right)}}_{\its{\rmS}{t}} \Big\rVert_2.
	\end{aligned}
\]
We define
\[
	\its{\rmZ_i}{t} \defeq \frac{1}{N}\nabla_{\its{\rvw_k}{t}} \ell\left(\its{h}{t}(\its{\rvx_i}{t}),\its{\rvy_i}{t}\right) - \frac{1}{N}\nabla_{\its{\rvw_k}{t}}\expect{(\rvx,\rvy)\sim\dtrain}{\ell\left(h(\rvx), \rvy\right)},\ \forall i\in [N].
	\]
It is easy to see that $\its{\rmS}{t} = \sum_{i\in[N]} \its{\rmZ_i}{t}$, $\expect{}{\its{\rmZ_i}{t}} = 0$ for every $i\in[N]$, and $\forall i\ne j\in [N]$, $\its{\rmZ_i}{t}$ and $\its{\rmZ_j}{t}$ are independent.
By~\lemmaref{lemma:gradient}, we have
\[
	\its{\rmZ_i}{t} = \frac{1}{N}\expect{(\rvx,\rvy)\sim\dtrain}{a_k\rvy\indicator{h(\rvx)\le 1}\indicator{\inprod{\its{\rvw_k}{t}}{\rvx}\ge 0}\cdot \rvx} -\frac{1}{N}a_k\its{\rvy_i}{t} \indicator{\its{h}{t}(\its{\rvx_i}{t}) \le 1} \indicator{\inprod{\its{\rvw_k}{t}}{\its{\rvx_i}{t}}\ge 0} \cdot\its{\rvx_i}{t}.
	\]
Recall that $a_k\in\{-\frac{1}{m},\frac{1}{m}\}$ and $\rvx$ is generated by $\rvx = \sum_{j\in\score}\rvy\rvz_j\bm{m}_j + \sum_{j\in\sbg}\rvz_j\bm{m}_j$ according to~\defref{def:dgp}. We then have $\lVert \its{\rmZ_i}{t} \rVert_2 \le \frac{2\sqrt{\di}}{mN}$, which also indicates that $\expect{}{\inprod{\its{\rmZ_i}{t}}{\its{\rmZ_i}{t}}} \le \frac{4\di}{m^2 N^2}$. This gives
\[
	\expect{}{\inprod{\its{\rmS}{t}}{\its{\rmS}{t}}} = \sum_{i\in[N]} \expect{}{\inprod{\its{\rmZ_i}{t}}{\its{\rmZ_i}{t}}} \le \frac{4\di}{m^2 N}.
	\]
Applying matrix Bernstein's inequality (\lemmaref{lemma:matrix_bernstein}), we have
\[
\prob{}{\left[\lVert\its{\rmS}{t}\rVert_2 \ge \delta \right]} \le (d+1)\exp\left(-\frac{3m^2N^2\delta^2}{24\di + 4\sqrt{\di}\delta mN}\right)
	\]
hold with every $\delta = \frac{1}{\poly{d}}$.
Therefore, we have that for $N = \poly{d}$ with some sufficiently large polynomial, the following holds with probability at least $1 - e^{-\Omega(d)}$:
\[
	\lVert\its{\rmS}{t}\rVert_2 \le \frac{1}{\poly{d}}.
	\]
This gives the desired result.
\end{proof}
\lemmaref{lemma:gradient_gap} directly leads to the following corollary:
\begin{lemma}
\label{lemma:corr_gradient}
For every $k\in[m]$, every $j\in[\di]$, and every iteration $t$, the following holds:
\begin{equation}
	\begin{aligned}
	&\Biginprod{\nabla_{\its{\rvw_k}{t}}\frac{1}{N}\sum_{i\in[N]} \ell\left(\its{h}{t}(\its{\rvx_i}{t}), \its{\rvy_i}{t}\right)}{\bm{m}_j}\\
	 &\qquad= \biginprod{\nabla_{\its{\rvw_k}{t}}\expect{(\rvx,\rvy)\sim\dtrain}{\ell\left(h(\rvx), \rvy\right)}}{\bm{m}_j}\pm \frac{1}{\poly{d}}  \\ &\qquad= -a_k\expect{(\rvx,\rvy)\sim\dtrain}{\rvy\indicator{h(\rvx)\le 1}\indicator{\inprod{\its{\rvw_k}{t}}{\rvx}\ge 0}\cdot \rvz_j}\pm\frac{1}{\poly{d}},\quad j\in[\di].
	\end{aligned}
\end{equation}
\end{lemma}
\begin{proof}
	The proof directly follows from combining~\lemmaref{lemma:gradient} and the generation process of $\rvx$ in~\defref{def:dgp}.
\end{proof}

\lemmaref{lemma:corr_gradient} allows us to directly work with population gradients instead of empirical gradients when analyzing the trajectory of SGD iterations in the subsequent sections.

\subsection{Neuron Characterization}

In this section, we define two subsets of neurons that will be used throughout our proofs.
\begin{definition}[Neuron characterization]
	\label{def:neuron}
	For each label $y\in\yspace=\{-1,1\}$ and every iteration $t$, we define the set $\nyt\subseteq[m]$ as:
	\begin{equation}
		\begin{aligned}
		\nyt\defeq\Bigg\{k\in[m]: &\sum_{j\in\score}y\mu_j\inprod{\wkt}{\bm{m}_j} + \sum_{j\in\sbg}\muj\inprod{\wkt}{\mj} \ge \Theta\left(\sqrt\frac{\di}{d}\right), \\
		& \sign{(a_k)} = y \Bigg\}.
		\end{aligned}
	\end{equation}
\end{definition}

\paragraph{Intuition.} For each label $y\in\yspace$ and iteration $t$,~\defref{def:neuron} characterizes a subset of neurons $\nyt$ in which
\begin{itemize}[leftmargin=2em]
	\item each neuron has (in expectation) enough positive correlations with the examples from class $y$ (recall that $\rvx = \sum_{j\in\score}\rvy\rvz_j\bm{m}_j + \sum_{j\in\sbg}\rvz_j\bm{m}_j$);
	\item each neuron positively contributes to the classification of examples from class $y$ (i.e., $\sign{(a_k)} = y$).
\end{itemize}

In our main proof, we will show in an iterative fashion that each neuron in $\nyt$ will accumulate either positive (if random initialization gives $a_k = \frac{1}{m}$) or negative (if random initialization gives $a_k = -\frac{1}{m}$) correlations with features in $\score$ (\emph{core feature learning}), while also accumulating positive correlations with features in $\sbg$ (\emph{feature contamination}).

For each neuron, we formally define the notion of \emph{positive examples} and \emph{negative examples} which are informally mentioned in~\secref{sec:main}:

\begin{definition}[Positive examples and negative examples]
	\label{def:example}
	Let $(x,y)\in\xspace\times\yspace$ be an example. For every $k\in[m]$, we say that $(x,y)$ is a \textbf{positive example} of neuron $k$ if $\sign(a_k) = y$, and say that $(x,y)$ is a \textbf{negative example} of neuron $k$ if $\sign(a_k) = -y$.
\end{definition}

\subsection{Properties at Initialization}

In this section, we introduce some useful properties of the neurons at initialization $t=0$, which serve as a basis for our inductive proofs in the subsequent sections.

\begin{lemma}
\label{lemma:neuron_init_base}
For every $j\in[\di]$, every $\mathcal{S}\subseteq[\di]$, and every $\{y_j\}_{j\in\mathcal{S}}\in\{-1, 1\}^{|\mathcal{S}|}$, the following holds for every $\delta>0$ over random initialization:
\begin{equation}
\begin{aligned}
\prob{\rmW^{(0)}}{\Big[\sum_{j\in\mathcal{S}}y_j\muj\inprod{\its{\rvw_k}{0}}{\bm{m}_j}\ge \frac{\delta}{\sqrt{d}}\Big]} &\ge \frac{1}{\sqrt{2\pi}}\frac{\delta\sqrt{\sum_{j\in\mathcal{S}}\muj^2}}{\delta^2 + \sum_{j\in\mathcal{S}}\muj^2}\exp{\left(-\frac{\delta^2}{2\sum_{j\in\mathcal{S}}\muj^2}\right)},\\
\prob{\rmW^{(0)}}{\Big[\sum_{j\in\mathcal{S}}y_j\muj\inprod{\its{\rvw_k}{0}}{\bm{m}_j}\ge \frac{\delta}{\sqrt{d}}\Big]} &\le \frac{1}{\sqrt{2\pi}}\frac{\sqrt{\sum_{j\in\mathcal{S}}\muj^2}}{\delta}\exp{\left(-\frac{\delta^2}{2\sum_{j\in\mathcal{S}}\muj^2}\right)}.
\end{aligned}
\end{equation}
\end{lemma}
\begin{proof}
Recall that different neurons are independently initialized by $\rvw_k^{(0)} \sim \mathcal{N}(\mathbf{0},\sigma_0^2\bm{I}_d),\forall k\in[m]$ with $\sigma_0^2 = \frac{1}{d}$. Using the fact that $\lVert\mj\rVert_2 = 1,\forall j\in[\di]$ and $y_j^2 = 1, \forall j\in\mathcal{S}$, we have
\[
\sum_{j\in\mathcal{S}}y_j\mu_j\inprod{\its{\rvw_k}{0}}{\bm{m}_j}\sim\mathcal{N}\Big(0,\ \frac{1}{d}\sum_{j\in\mathcal{S}}\muj^2\Big)
\]
Applying standard bounds for the Gaussian distribution function (\lemmaref{lemma:gaussian}) gives that for every $\delta > 0$,
\[
\frac{1}{\sqrt{2\pi}}\frac{\delta}{\delta^2+1}\exp{\left(-\frac{\delta^2}{2}\right)} \le \prob{\rmW^{(0)}}{\left[\frac{\sqrt{d}\sum_{j\in\mathcal{S}}y_j\mu_j\inprod{\its{\rvw_k}{0}}{\bm{m}_j}}{\sqrt{\sum_{j\in\mathcal{S}}\muj^2}} \ge \delta\right]} \le \frac{1}{\sqrt{2\pi}}\frac{1}{\delta}\exp{\left(-\frac{\delta^2}{2}\right)}.
\]
A simple transformation completes the proof.
\end{proof}

\begin{lemma}[Neuron properties at initialization]
\label{lemma:neuron_init}
For each label $y\in\yspace$, the following holds with probability at least $1 - e^{-\Omega(m)}$ over random initialization:
\begin{equation}
\big|\nyinit\big| = \Theta(m).
\end{equation}
\end{lemma}
\begin{proof}
For each neuron $k\in[m]$, define events $E_{k1}$ and $E_{k2}$ to be
\[
\begin{gathered}
E_{k1} \defeq \Bigg\{\sum_{j\in\score}y\mu_j\inprod{\wkinit}{\bm{m}_j} + \sum_{j\in\sbg}\muj\inprod{\wkinit}{\mj} \ge {\Theta\left(\sqrt{\frac{\di}{d}}\right)} \Bigg\}, \\
E_{k2} \defeq \Big\{\sign{(a_k)} = y\Big\}.
\end{gathered}
\]
By $a_k\sim\mathsf{Uniform}\{-\frac{1}{m},\frac{1}{m}\}$, we immediately have $\prob{}{[E_{k2}]} = \frac{1}{2}$ for every $k\in[m]$. For $E_{k1}$, by applying~\lemmaref{lemma:neuron_init_base} with $\delta = \Theta(\sqrt{\di})$ we obtain
\[
\begin{aligned}
\prob{\its{\rmW}{0}}{[E_{k1}]} &\ge \frac{1}{\sqrt{2\pi}}\frac{\Theta\left(\sqrt{\di\sum_{j\in[\di]}\muj^2}\right)}{\Theta(\di) + \sum_{j\in[\di]}\muj^2}\exp{\left(-\Theta\left(\frac{\di}{\sum_{j\in[\di]}\muj^2}\right)\right)}, \\
\prob{\its{\rmW}{0}}{[E_{k1}]} &\le  \frac{1}{\sqrt{2\pi}}\Theta\left(\sqrt{\frac{\sum_{j\in[\di]}\muj^2}{\di}}\right)\exp{\left(-\Theta\left(\frac{\di}{\sum_{j\in[\di]}\muj^2}\right)\right)}.
\end{aligned}
\]
Together with $\muj^2 = \Theta(1)$ for every $j\in[\di]$, we have that $\prob{\its{\rmW}{0}}{[E_{k1}]} = \Theta(1)$ for every $k\in [m]$. Since events $E_{k1}$ and $E_{k2}$ are independent, we have that for each neuron $k\in[m]$, the probability of it belonging to $\nyinit$ is given by $\prob{}{(k\in\nyinit)} = \prob{}{(E_{k1}\cap E_{k2})} = \Theta(1)$.

Since different neurons are independently initialized, $|\nyinit|$ follows a binomial distribution with trial number $m$ and some success probability $\Theta(1)$. Applying the standard tail bound for binomial variables (\lemmaref{lemma:binomial}) then gives $|\nyinit| \ge \Theta(m)$ with probability at least $1 - e^{-\Omega(m)}$. Together with the trivial upper bound that $|\nyinit|\le m$, we have that $|\nyinit| = \Theta(m)$ with probability at least $1 - e^{-\Omega(m)}$.
\end{proof}

\begin{lemma}[Neuron properties at initialization, continued]
\label{lemma:neuron_init_cont}
With probability at least $1 - O(\frac{1}{m})$ over random initialization, for every $y\in\yspace$, the following holds:
\[
\max_{k\in[m]}\bigabs{\expect{\rvx|\rvy=y\sim\dtrain}{\inprod{\wkinit}{\rvx}}} \le {O}\left(\sqrt{\frac{\di\log m}{d}}\right).
\]
\end{lemma}
\begin{proof}
Recall that different neurons are independently initialized by $\rvw_k^{(0)} \sim \mathcal{N}(\mathbf{0},\sigma_0^2\bm{I}_d),\forall k\in[m]$ with $\sigma_0^2 = \frac{1}{d}$. By $\lVert\mj\rVert_2 = 1$, we have
\[
\begin{aligned}
\expect{\rvx|\rvy=y\sim\dtrain}{\inprod{\its{\rvw_k}{0}}{\rvx}} &= \sum_{j\in\score}y\mu_j\inprod{\wkinit}{\bm{m}_j} + \sum_{j\in\sbg}\muj\inprod{\wkinit}{\mj} \\
&\sim\mathcal{N}\Big(0,\ \frac{1}{d}\sum_{j\in[\di]}\muj^2\Big).
\end{aligned}
\]
Applying~\lemmaref{lemma:max_gaussian} over the i.i.d. random variables $\inprod{\its{\rvw_1}{0}}{\rvx},\ldots,\inprod{\its{\rvw_m}{0}}{\rvx}$ gives
\[
\prob{\rmW^{(0)}}{\left[\expect{\rvx|\rvy=y\sim\dtrain}{\inprod{\wkinit}{\rvx}} \ge 2\sqrt{\frac{\sum_{j\in[\di]}\muj^2}{d}\log m}\,\,\right]} \le \frac{1}{m}.
\]
Finally, using $\sum_{j\in[\di]}\muj^2 = \Theta(\di)$ and $m \in [\Theta(\di), \Theta(d)]$ completes the proof.
\end{proof}

\begin{lemma}[Output magnitude at initialization]
\label{lemma:output}
For every $x\in\xspace$, the following holds with probability at least $1 - e^{-\Omega(\di)}$ over random initialization:
\begin{equation}
\big|h^{(0)}(x)\big| = O\left(\frac{1}{\sqrt{\di}}\right).
\end{equation}
\end{lemma}
\begin{proof}
By $\wkinit\sim\mathcal{N}(\mathbf{0},\sigma_0^2\bm{I}_d)$ with $\sigma_0^2 = \frac{1}{d}$ and $\norm{\mj} = 1$, we have
\[
\sum_{k\in[m]}\frac{1}{m}\sum_{j\in[\di]}\inprod{\wkinit}{\mj}\sim\mathcal{N}\Big(0, \frac{\di}{md}\Big).
\]
Applying standard bounds for the Gaussian distribution function (\lemmaref{lemma:gaussian}) gives
\[
\frac{1}{\sqrt{2\pi}}\frac{\delta}{\delta^2+1}\exp{\left(-\frac{\delta^2}{2}\right)} \le \prob{\rmW^{(0)}}{\left[\sum_{k\in[m]}\frac{1}{m}\sum_{j\in[\di]}\inprod{\wkinit}{\mj} \ge \delta \sqrt{\frac{\di}{md}} \right]} \le \frac{1}{\sqrt{2\pi}}\frac{1}{\delta}\exp{\left(-\frac{\delta^2}{2}\right)}
\]
for every $\delta > 0$. Substituting $\delta$ by $\Theta(\sqrt{\di})$ and using the symmetry of Gaussian then yield
\[
\prob{\rmW^{(0)}}{\bigg[\,\bigg|\sum_{k\in[m]}\frac{1}{m}\sum_{j\in[\di]}\inprod{\wkinit}{\mj}\bigg| \ge \frac{\Theta(\di)}{\sqrt{md}}\bigg]} \le \exp(-\Omega(\di)).
\]
We then have
\[
\begin{aligned}
\big|h^{(0)}(x)\big| &= \Big|\sum_{k\in[m]}a_k\cdot\relu{\big(\inprod{\its{\rvw_k}{0}}{\rvx}\big)}\Big| \\
&\le \bigg|\sum_{k\in[m]}\frac{1}{m}\inprod{\its{\rvw_k}{0}}{\rvx}\bigg| \\
&\le \bigg|\sum_{k\in[m]}\frac{1}{m}\sum_{j\in[\di]}\inprod{\wkinit}{\mj}\bigg| \\
&\le \frac{\Theta(\di)}{\sqrt{md}} = O\left(\frac{1}{\sqrt{\di}}\right).
\end{aligned}
\]
holds with probability at least $1 - e^{-\Omega(\di)}$, where in the last equality we use the fact that $m = \Omega(\di)$ and $d = \Omega(\di^{2.5})$.
\end{proof}

In what follows, we will always assume that the high-probability events at initialization in~\lemmaref{lemma:neuron_init},~\lemmaref{lemma:neuron_init_cont}, and~\lemmaref{lemma:output} hold---by a union bound argument and the fact that $m = \Omega(\di)$, the probability that they all hold is at least $1 - O(\frac{1}{m}) - e^{-\Omega(\di)}$.

\section{Activation Asymmetry, Feature Contamination, and OOD Failure: Proofs of~\theoref{theo:activation},~\theoref{theo:feature}, and~\theoref{theo:risk}}
\label{appsec:proof_main}

Before we delve into the main proofs, we first introduce some technical lemmas that characterize the gradient updates starting from random initialization.
We begin by introducing an important lemma that characterizes the activation probability of the ReLU function using the Berry-Esseen theorem:

\begin{lemma}[Activation probability]
\label{lemma:activation}
Assume that the training (ID) data is generated according to~\defref{def:dgp} and $\left|\frac{\inp{\wkt}{\bm{m}_i}}{\inp{\wkt}{\bm{m}_j}}\right| = \Theta(1)$ for every $k\in[m]$ and for every $i,j\in[\di]$. Then, for every label $y\in\yspace$, every $k\in[m]$, and every iteration $t$, the following holds:
\begin{equation}
\prob{\rvx|\rvy=y\sim\dtrain}{[\inp{\wkt}{\rvx} \ge 0]} = \Phi\left(\frac{\expect{\rvx|\rvy=y}{\inp{\wkt}{\rvx}}}{\Theta(1)\sqrt{\sum_{j\in[\di]}\inp{\wkt}{\mj}^2}}\right) \pm O\left(\frac{1}{\sqrt{\di}}\right),
\end{equation}
where $\Phi$ denotes the cumulative distribution function of $\mathcal{N}(0,1)$.
\end{lemma}
\begin{proof}
Recall~\defref{def:dgp} that given label $y\in\yspace$, $\rvx$ is generated by
\[
\rvx = \sum_{j\in\score}y\rvz_j\bm{m}_j + \sum_{j\in\sbg}\rvz_j\bm{m}_j.
\]
Therefore,
\[
\inp{\wkt}{\rvx} = \sum_{j\in\score}y\zj\inp{\wkt}{\mj} + \sum_{j\in\sbg}\zj\inp{\wkt}{\mj}.
\]
For every $j\in[\di]$, define the random variable
\[
\rvr_j \defeq y_j(\zj - \muj)\inp{\wkt}{\mj},
\]
where $y_j \defeq \left\{\begin{aligned} &y,\ j\in\score\\ &1,\ j\in\sbg\end{aligned}\right.$. Recall that $\muj\defeq\expect{\rvz\sim\dtrain}{\zj}$. It is then easy to derive that $\expect{}{\rvr_j} = 0$ and $\expect{}{\rvr_j^2} = \Theta(1)\inp{\wkt}{\mj}^2$ (recall that $\expect{}{(\zj - \muj)^2} = \Theta(1)$). We now upper bound $\expect{}{|\rvr_j^3|}$: first recall that by~\defref{def:dgp} we have $\expect{}{\rvz_j^3} = \Theta(1)$. For every $p\ge 1$, denote the $\ell_p$ norm of the random variable $\zj$ by $\pnorm{\zj}\defeq(\expect{}{\abs{\zj}^p})^{\frac{1}{p}}$. Applying Minkowsky's inequality gives
\[
\begin{aligned}
\pnorm{\zj-\muj} &\le \pnorm{\zj} + \pnorm{\muj} \\
&\stackrel{(a)}{=} \pnorm{\zj} + \snorm{\zj} \\
&\stackrel{(b)}{\le} 2\pnorm{\zj},
\end{aligned}
\]
where $(a)$ is due to the fact that $\pnorm{\muj} = |\muj| = \snorm{\zj}$ and $(b)$ is due to the power norm inequality indicating that $\pnorm{\cdot}$ is non-decreasing with regard to $p$. Letting $p=3$ and cubing the above inequality gives
\[
\expect{}{\abs{\zj-\muj}^3} \le 8\expect{}{\abs{\zj^3}} = 8\expect{}{{\zj^3}} = \Theta(1),
\]
from which we obtain $\expect{}{\abs{\rvr_j^3}} = \Theta(1)\cdot |\inp{\wkt}{\mj}|^3$.

We then define the normalized sum of $\rvr_j$:
\[
\rvs_{\di}\defeq\frac{\sum_{j\in[\di]}\rvr_j}{\sqrt{\sum_{j\in[\di]}\expect{}{\rvr_j^2}}}.
\]
Since $\rvr_i$ and $\rvr_j$ are independent and zero-mean for every $i\ne j\in[\di]$, we can apply the Berry-Esseen theorem (\lemmaref{lemma:berry}) to the normalized sum $\rvs_{\di}$ and obtain
\[
\begin{aligned}
\sup_{\delta\in\mathbb{R}}\left|\prob{\rvx|\rvy=y}{[\rvs_{\di} < \delta]} - \Phi(\delta)\right| &\le C_0 \Bigg(\sum_{j\in[\di]} \expect{}{\rvr_j^2}\Bigg)^{-\frac{3}{2}}\sum_{j\in[\di]}\expect{}{\abs{\rvr_j^3}}\\
&= C_0 \Bigg(\sum_{j\in[\di]}\Theta(1)\inp{\wkt}{\mj}^2\Bigg)^{-\frac{3}{2}} \sum_{j\in[\di]} \Theta(1)\left|\inp{\wkt}{\mj}\right|^3 \\
&\stackrel{(c)}{=} O\left(\frac{1}{\sqrt{\di}}\right),
\end{aligned}
\]
where $(c)$ is due to the assumption that $\left|\frac{\inp{\wkt}{\bm{m}_i}}{\inp{\wkt}{\bm{m}_j}}\right| = \Theta(1)$.
Note that $\sum_{j\in[\di]}\rvr_j = \inp{\wkt}{\rvx} - \expect{\rvx|\rvy=y}{\inp{\wkt}{\rvx}}$. We then have for every $\delta\in\mathbb{R}$,
\[
\prob{\rvx|\rvy=y}{\left[\inp{\wkt}{\rvx} \ge \expect{\rvx|\rvy=y}{\inp{\wkt}{\rvx}} + \delta\sqrt{\sum_{j\in[\di]}\expect{}{\rvr_j^2}}\right]} = 1 - \Phi(\delta) \pm O\left(\frac{1}{\sqrt{\di}}\right).
\]
Finally, letting $\delta = -\dfrac{\expect{\rvx|\rvy=y}{\inp{\wkt}{\rvx}}}{\sqrt{\sum_{j\in[\di]}\expect{}{\rvr_j^2}}}$ and using the symmetry of unit Gaussian $1 - \Phi(\delta) = \Phi(-\delta)$ give the desired result.
\end{proof}



We then define two terms that will be frequently used when analyzing gradients.

\begin{definition}
\label{def:grad}
For each label $y\in\yspace$, every $k\in[m]$, every feature vector $\mj,j\in[\di]$, every iteration $t$, and every subset $\mathcal{S}\subseteq[\di]$, define
\begin{equation}
\begin{aligned}
\gkytj &\defeq \frac{1}{m}\expect{(\rvx,\rvy)\sim\dtrain}{\indicator{\htx\le 1}\indicator{\rvy=y}\indicator{\inprod{\its{\rvw_k}{t}}{\rvx}\ge 0} \muj\rvz_j},\\
\gkyt{\mathcal{S}} &\defeq \sum_{j\in\mathcal{S}} \gkytj.
\end{aligned}
\end{equation}
\end{definition}

Given the above notation, we now introduce two lemmas that separately bound the gradient projection onto the core features and the gradient projection onto the background features for neurons in $\nyt$, which will be helpful for us to track the trajectory of SGD starting from network initialization.

\begin{lemma}[Gradient projection onto core features, neurons in $\nyt$] For every iteration $t\le O(\frac{m}{\eta\di})$, the following holds for every $y\in\mathcal{Y}$ and every neuron $k\in\nyt$:
\label{lemma:grad_core}
\begin{equation}
-\Biginprod{\nabla_{\its{\rvw_k}{t}}\frac{1}{N}\sum_{i\in[N]} \ell\left(\its{h}{t}(\its{\rvx_i}{t}), \its{\rvy_i}{t}\right)}{\sum_{j\in \score}\muj\mj} = {y}\left(\gkyt{\score} + \gknyt{\score}\right),
\label{eq:grad_core}
\end{equation}
where
\begin{equation}
\gkyt{\score} = \Theta\left(\frac{\di}{m}\right).
\end{equation}
\end{lemma}
\begin{proof}
Recall~\defref{def:dgp} that given a label $\rvy\in\yspace$, $\rvx$ is generated by
\[
\rvx = \sum_{j\in\score}\rvy\rvz_j\bm{m}_j + \sum_{j\in\sbg}\rvz_j\bm{m}_j.
\]
Then, applying~\lemmaref{lemma:corr_gradient} to the LHS of Eq.~\eqref{eq:grad_core} and using $\sign(a_k) = {y}$ for every $k\in\nyt$, we obtain
\begin{equation}
\begin{aligned}
	&-\Biginprod{\nabla_{\its{\rvw_k}{t}}\frac{1}{N}\sum_{i\in[N]} \ell\left(\its{h}{t}(\its{\rvx_i}{t}), \its{\rvy_i}{t}\right)}{\sum_{j\in\score}\muj\mj}\\
	 &\qquad= -\biginprod{\nabla_{\its{\rvw_k}{t}}\expect{(\rvx,\rvy)\sim\dtrain}{\ell\left(\htx, \rvy\right)}}{\sum_{j\in\score}\muj\mj}\pm \frac{O(\di)}{\poly{d}}  \\
	 &\qquad= a_k\expect{(\rvx,\rvy)\sim\dtrain}{\rvy\indicator{\htx\le 1}\indicator{\inprod{\its{\rvw_k}{t}}{\rvx}\ge 0} \sum_{j\in\score}\rvy\muj\zj}\pm\frac{O(\di)}{\poly{d}}\\
	 &\qquad= a_k \expect{(\rvx,\rvy)\sim\dtrain}{\indicator{\htx \le 1}\indicator{\rvy=y}\indicator{\inprod{\its{\rvw_k}{t}}{\rvx}\ge 0} \sum_{j\in\score}\muj\zj}\\
	 &\qquad\qquad +  a_k\expect{(\rvx,\rvy)\sim\dtrain}{\indicator{\htx \le 1}\indicator{\rvy=-y}\indicator{\inprod{\its{\rvw_k}{t}}{\rvx}\ge 0} \sum_{j\in\score}\muj\zj} \pm\frac{O(\di)}{\poly{d}}\\
	 &\qquad = {y}\left(\gkyt{\score} + \gknyt{\score}\right) \pm\frac{O(\di)}{\poly{d}}.
	\end{aligned}
\label{eq:app_grad_pos_base}	
\end{equation}
For $\gkyt{\score}$, by the law of total expectation we have
\[
\begin{aligned}
\gkyt{\score} &= \frac{1}{m}\expect{(\rvx,\rvy)\sim\dtrain}{\indicator{\htx\le 1}\indicator{\rvy=y}\indicator{\inprod{\its{\rvw_k}{t}}{\rvx}\ge 0} \sum_{j\in\score}\muj\rvz_j} \\
&= \frac{1}{2m}\expect{\rvx|\rvy=y}{\bigg[\indicator{\htx\le 1} \sum_{j\in\score}\muj\rvz_j\,\Big|\,\inprod{\wkt}{\rvx}\ge 0\bigg]}\prob{\rvx|\rvy=y}{[\inprod{\wkt}{\rvx}\ge 0]}\\
&= \frac{1}{2m}\expect{\rvx|\rvy=y}{\bigg[\indicator{\htx\le 1} \sum_{j\in\score}\muj\rvz_j\bigg]}\\
&\qquad -\frac{1}{2m}\expect{\rvx|\rvy=y}{\bigg[\indicator{\htx\le 1} \sum_{j\in\score}\muj\rvz_j\,\Big|\,\inprod{\wkt}{\rvx}< 0\bigg]}\prob{\rvx|\rvy=y}{[\inprod{\wkt}{\rvx} < 0]}.
\end{aligned}
\]
Applying~\lemmaref{lemma:activation} gives
\[
\prob{\rvx|\rvy=y}{[\inp{\wkt}{\rvx} < 0]} = \Phi\left(-\frac{\expect{\rvx|\rvy=y}{\inp{\wkt}{\rvx}}}{\Theta(1)\sqrt{\sum_{j\in[\di]}\inp{\wkt}{\mj}^2}}\right) \pm O\left(\frac{1}{\sqrt{\di}}\right).
\]
Recall that for $\rvx\sim\dtrain|\rvy=y$,
\[
\inp{\wkt}{\rvx} = \sum_{j\in\score}y\zj\inp{\wkt}{\mj} + \sum_{j\in\sbg}\zj\inp{\wkt}{\mj}.
\]
By~\defref{def:neuron}, we have for every $k\in\nyt$, $\expect{\rvx|\rvy=y}{\inp{\wkt}{\rvx}} \ge \Theta\Big(\sqrt{\dfrac{\di}{d}}\Big) > 0$, which indicates that $\Phi\left(-\dfrac{\expect{\rvx|\rvy=y}{\inp{\wkt}{\rvx}}}{\Theta(1)\sqrt{\sum_{j\in[\di]}\inp{\wkt}{\mj}^2}}\right) < \dfrac{1}{2}$.
Together with $\htx \le 1$, this gives
\begin{equation}
\label{eq:app_core_pos_lower}
\begin{aligned}
\gkyt{\score} &\ge \frac{1}{2m}\expect{\rvx|\rvy=y}{\bigg[\indicator{\htx\le 1}\sum_{j\in\score}\muj\rvz_j\bigg]}\\
&\qquad -\frac{1}{2m}\expect{\rvx|\rvy=y}{\bigg[\indicator{\htx\le 1}\sum_{j\in\score}\muj\rvz_j\,\Big|\,\inprod{\wkt}{\rvx}< 0\bigg]}\cdot\left(\frac{1}{2} \pm O\left(\frac{1}{\sqrt{\di}}\right)\right)\\
&\ge \sum_{j\in\score}\frac{\muj^2}{2m} - \sum_{j\in\score}\frac{\muj^2}{4m} \pm \sum_{j\in\score}\frac{\muj^2}{\Theta(m\sqrt{\di})}\\
&= \Theta\left(\frac{\di}{m}\right).
\end{aligned}
\end{equation}
Meanwhile, we also have the upper bound
\begin{equation}
\label{eq:app_core_pos_upper}
\begin{aligned}
\gkyt{\score} &= \frac{1}{m}\expect{(\rvx,\rvy)}{\indicator{\htx\le 1}\indicator{\rvy=y}\indicator{\inprod{\its{\rvw_k}{t}}{\rvx}\ge 0} \sum_{j\in\score}\muj\zj}\\
&\le \frac{1}{2m}\mathbb{E}_{\rvx|\rvy=y}\sum_{j\in\score}\muj\zj\\
&= \Theta\left(\frac{\di}{m}\right).
\end{aligned}
\end{equation}
Combining Eqs.~\eqref{eq:app_core_pos_lower} and~\eqref{eq:app_core_pos_upper} gives
\[
\begin{aligned}
\gkyt{\score} = \Theta\left(\frac{\di}{m}\right).
\end{aligned}
\]
Finally, plugging the above equation and $m = O(d)$ into Eq.~\eqref{eq:app_grad_pos_base} completes the proof.
\end{proof}



\begin{lemma}[Gradient projection onto background features, neurons in $\nyt$] For every iteration $t\le O(\frac{m}{\eta\di})$, the following holds for every $y\in\mathcal{Y}$ and every neuron $k\in\nyt$:
\label{lemma:grad_bg}
\begin{equation}
-\Biginprod{\nabla_{\its{\rvw_k}{t}}\frac{1}{N}\sum_{i\in[N]} \ell\left(\its{h}{t}(\its{\rvx_i}{t}), \its{\rvy_i}{t}\right)}{\sum_{j\in \sbg}\muj\mj} = \gkyt{\sbg} - \gknyt{\sbg},
\label{eq:grad_bg}
\end{equation}
where
\begin{equation}
\gkyt{\sbg} = \wt{\Theta}\left(\frac{\di}{m}\right).
\end{equation}
\end{lemma}
\begin{proof}
Similar to the proof of~\lemmaref{lemma:grad_core}, we apply~\lemmaref{lemma:corr_gradient} to the LHS of Eq.~\eqref{eq:grad_bg} and using $\sign(a_k) = {y}$ for every $k\in\nyt$, which gives
\begin{equation}
\begin{aligned}
	&-\Biginprod{\nabla_{\its{\rvw_k}{t}}\frac{1}{N}\sum_{i\in[N]} \ell\left(\its{h}{t}(\its{\rvx_i}{t}), \its{\rvy_i}{t}\right)}{\sum_{j\in\sbg}\muj\mj}\\
	 &\qquad= -\biginprod{\nabla_{\its{\rvw_k}{t}}\expect{(\rvx,\rvy)\sim\dtrain}{\ell\left(\htx, \rvy\right)}}{\sum_{j\in\sbg}\muj\mj}\pm \frac{O(\di)}{\poly{d}}  \\
	 &\qquad= a_k\expect{(\rvx,\rvy)\sim\dtrain}{\rvy\indicator{\htx\le 1}\indicator{\inprod{\its{\rvw_k}{t}}{\rvx}\ge 0} \sum_{j\in\sbg}\muj\zj}\pm\frac{O(\di)}{\poly{d}}\\
	 &\qquad= \frac{1}{m} \expect{(\rvx,\rvy)\sim\dtrain}{\indicator{\htx \le 1}\indicator{\rvy=y}\indicator{\inprod{\its{\rvw_k}{t}}{\rvx}\ge 0} \sum_{j\in\sbg}\muj\zj}\\
	 &\qquad\qquad - \frac{1}{m}\expect{(\rvx,\rvy)\sim\dtrain}{\indicator{\htx \le 1}\indicator{\rvy=-y}\indicator{\inprod{\its{\rvw_k}{t}}{\rvx}\ge 0} \sum_{j\in\sbg}\muj\zj} \pm\frac{O(\di)}{\poly{d}}\\
	 &\qquad = \gkyt{\sbg} - \gknyt{\sbg} \pm\frac{O(\di)}{\poly{d}}.
	\end{aligned}
\label{eq:app_grad_bg_base}	
\end{equation}
Also, by a nearly identical argument to~\lemmaref{lemma:grad_core} and $\dbg = \Theta\left(\dfrac{\di}{\log\di}\right)$, we have
\begin{equation}
\gkyt{\sbg} = \wt{\Theta}\left(\frac{\di}{m}\right).
\label{eq:app_grad_bg_pos}
\end{equation}
This completes the proof.
\end{proof}

Next, we also introduce a lemma that bound the gradient projection onto core features for all neurons:

\begin{lemma}[Gradient projection onto core features, all neurons] For every iteration $t\le O(\frac{m}{\eta\di})$, the following holds for every $y\in\mathcal{Y}$ and every neuron $k\in[m]$ with $\sign{(a_k)} = y$:
\label{lemma:grad_core_all}
\begin{equation}
-\Biginprod{\nabla_{\its{\rvw_k}{t}}\frac{1}{N}\sum_{i\in[N]} \ell\left(\its{h}{t}(\its{\rvx_i}{t}), \its{\rvy_i}{t}\right)}{\sum_{j\in \score}\muj\mj} = y\cdot O\left(\frac{\di}{m}\right).
\label{eq:grad_core_all}
\end{equation}
\end{lemma}
\begin{proof}
By an identical proof to~\lemmaref{lemma:grad_core}, we have
\[
-\Biginprod{\nabla_{\its{\rvw_k}{t}}\frac{1}{N}\sum_{i\in[N]} \ell\left(\its{h}{t}(\its{\rvx_i}{t}), \its{\rvy_i}{t}\right)}{\sum_{j\in \score}\muj\mj} = {y}\left(\gkyt{\score} + \gknyt{\score}\right)\pm\frac{\Theta(\di)}{\poly{d}}.
\]
By Eq.~\eqref{eq:app_core_pos_upper}, we have the upper bound $\gkyt{\score} \le \Theta\left(\frac{\di}{m}\right)$. By a similar argument, we also have $\gknyt{\score} \le \Theta\left(\frac{\di}{m}\right)$. Plugging those upper bounds and $m = O(d)$ into the above equation completes the proof.
\end{proof}

We then introduce a lemma that bounds the expected correlation between each neuron in $\nyt$ and its positive examples.

\begin{lemma}[Correlation with positive examples, neurons in $\nyt$]
\label{lemma:corr_pos}
For every iteration $t\le O(\frac{m}{\eta\di})$, every $y\in\yspace$, and every $k\in\nyt$, the following holds:
\begin{equation}
\begin{aligned}
\expect{\rvx|\rvy=y\sim\dtrain}[\inp{\wk{t+1}}{\rvx}] \ge (1 - \lambda\eta) \expect{\rvx|\rvy=y\sim\dtrain}[\inp{\wkt}{\rvx}] + \Theta\left(\frac{\eta \di}{m}\right).
\end{aligned}
\end{equation}
\end{lemma}
\begin{proof}
Recall~\defref{def:dgp} that given the label $y\in\yspace$, $\rvx$ is generated by
\[
\rvx = \sum_{j\in\score}y\rvz_j\bm{m}_j + \sum_{j\in\sbg}\rvz_j\bm{m}_j.
\]
We can thus obtain
\[
\begin{aligned}
&\expect{\rvx|\rvy=y\sim\dtrain}[\inp{\wk{t+1}}{\rvx}] - \expect{\rvx|\rvy=y\sim\dtrain}[\inp{\wkt}{\rvx}]\\
&\qquad= \underbrace{y \Big(\biginprod{\wk{t+1}}{\sum_{j\in\score}\muj\mj} - \biginprod{\wk{t}}{\sum_{j\in\score}\muj\mj}\Big)}_{\dtcore} \\
&\qquad\qquad+ \Big(\underbrace{\biginprod{\wk{t+1}}{\sum_{j\in\sbg}\muj\mj} - \biginprod{\wk{t}}{\sum_{j\in\sbg}\muj\mj}}_{\dtbg}\Big).
\end{aligned}
\]
For $\dtcore$, by the SGD iteration~\eqref{eq:sgd} we have
\[
\dtcore = -\eta y\Biginprod{\nabla_{\its{\rvw_k}{t}}\frac{1}{N}\sum_{i\in[N]} \ell\left(\its{h}{t}(\its{\rvx_i}{t}), \its{\rvy_i}{t}\right)}{\sum_{j\in \score}\muj\mj} -\lambda\eta y\biginprod{\wk{t}}{\sum_{j\in\score}\muj\mj}.
\]
Applying~\lemmaref{lemma:grad_core} gives
\[
\dtcore = \eta \left(\gkyt{\score} + \gknyt{\score}\right) -\lambda\eta y \biginprod{\wk{t}}{\sum_{j\in\score}\muj\mj},
\]
which results in the iterative expression
\begin{equation}
y\biginprod{\wk{t+1}}{\sum_{j\in\score}\muj\mj} = y(1-\lambda\eta)\biginprod{\wk{t}}{\sum_{j\in\score}\muj\mj} + \eta \left(\gkyt{\score} + \gknyt{\score}\right).
\label{eq:iter_core}
\end{equation}
For $\dtbg$, by the SGD iteration~\eqref{eq:sgd} we have
\[
\dtbg = -\eta\Biginprod{\nabla_{\its{\rvw_k}{t}}\frac{1}{N}\sum_{i\in[N]} \ell\left(\its{h}{t}(\its{\rvx_i}{t}), \its{\rvy_i}{t}\right)}{\sum_{j\in \sbg}\muj\mj} -\lambda\eta \biginprod{\wk{t}}{\sum_{j\in\sbg}\muj\mj}.
\]
Applying~\lemmaref{lemma:grad_bg} gives
\[
\dtbg = \eta\left(\gkyt{\sbg} - \gknyt{\sbg}\right) -\lambda\eta \biginprod{\wk{t}}{\sum_{j\in\sbg}\muj\mj},
\]
which results in the iterative expression
\begin{equation}
\biginprod{\wk{t+1}}{\sum_{j\in\sbg}\muj\mj} = (1-\lambda\eta)\biginprod{\wk{t}}{\sum_{j\in\sbg}\muj\mj} + \eta\left(\gkyt{\sbg} - \gknyt{\sbg}\right).
\label{eq:iter_bg}
\end{equation}
Combining Eqs.~\eqref{eq:iter_core} and~\eqref{eq:iter_bg} gives
\[
\begin{aligned}
\expect{\rvx|\rvy=y\sim\dtrain}[\inp{\wk{t+1}}{\rvx}] &= y\biginprod{\wk{t+1}}{\sum_{j\in\score}\muj\mj} + \biginprod{\wk{t+1}}{\sum_{j\in\sbg}\muj\mj} \\
&= y(1-\lambda\eta)\biginprod{\wk{t}}{\sum_{j\in\score}\muj\mj} + \eta \left(\gkyt{\score} + \gknyt{\score}\right) \\
&\qquad + (1-\lambda\eta)\biginprod{\wk{t}}{\sum_{j\in\sbg}\muj\mj} + \eta\left(\gkyt{\sbg} - \gknyt{\sbg}\right) \\
&= y(1-\lambda\eta)\expect{\rvx|\rvy=y\sim\dtrain}{[\inp{\wkt}{\rvx}]} \\
&\qquad + \eta \left(\gkyt{\score} + \gknyt{\score} + \gkyt{\sbg} - \gknyt{\sbg}\right)\\
&\stackrel{(a)}{\ge} (1-\lambda\eta)\expect{\rvx|\rvy=y\sim\dtrain}{[\inp{\wkt}{\rvx}]} + \Theta\left(\frac{\eta \di}{m}\right),
\end{aligned}
\]
where $(a)$ is due to $\gkyt{\score} = \Theta\left(\dfrac{\di}{m}\right)$ (\lemmaref{lemma:grad_core}), $\gkyt{\sbg} = O\left(\dfrac{\di}{m}\right)$ (\lemmaref{lemma:grad_bg}), and
\[
\begin{aligned}
\gknyt{\score} - \gknyt{\sbg} &= \frac{1}{m}\expect{(\rvx,\rvy)}{\indicator{\htx\le 1}\indicator{\rvy=-y}\indicator{\inprod{\its{\rvw_k}{t}}{\rvx}\ge 0} \bigg(\sum_{j\in\score}\muj\rvz_j - \sum_{j\in\sbg}\muj\rvz_j\bigg)} \\
&\stackrel{(b)}{\ge} \frac{1}{m}\expect{(\rvx,\rvy)}{\indicator{\htx\le 1}\indicator{\rvy=-y}\indicator{\inprod{\its{\rvw_k}{t}}{\rvx}\ge 0} \left(\Theta(\di) - \Theta\left(\frac{\di}{\log\di}\right)\right)} \\
&\ge 0,
\end{aligned}
\label{eq:tmp_1}
\]
where $(b)$ is due to the fact that $\dcore = \Theta(\di)$ and $\dbg = \Theta\left(\dfrac{\di}{\log\di}\right)$.
\end{proof}

We also introduce a general upper bound on the expected correlation between every neuron in the network and its positive examples.

\begin{lemma}[Correlation with positive examples, all neurons]
\label{lemma:corr_pos_all}
For every iteration $t\le O(\frac{m}{\eta\di})$, the following holds for every $y\in\yspace$ and every $k\in\nyt$ with $\sign(a_k) = y$:
\begin{equation}
\begin{aligned}
\expect{\rvx|\rvy=y\sim\dtrain}[\inp{\wk{t+1}}{\rvx}] \le (1 - \lambda\eta) \expect{\rvx|\rvy=y\sim\dtrain}[\inp{\wkt}{\rvx}] + O\left(\frac{\eta \di}{m}\right).
\end{aligned}
\end{equation}
\end{lemma}
\begin{proof}
By an identical proof to~\lemmaref{lemma:corr_pos}, we have
\[
\begin{aligned}
\expect{\rvx|\rvy=y\sim\dtrain}[\inp{\wk{t+1}}{\rvx}] &= y(1-\lambda\eta)\expect{\rvx|\rvy=y\sim\dtrain}{[\inp{\wkt}{\rvx}]} \\
&\qquad + \eta \left(\gkyt{\score} + \gknyt{\score} + \gkyt{\sbg} - \gknyt{\sbg}\right) \\
\end{aligned}
\]
By Eq.~\eqref{eq:app_core_pos_upper}, we have the upper bound $\gkyt{\score} \le \Theta\left(\frac{\di}{m}\right)$. By a similar argument, we also have the upper bounds $\gknyt{\score} \le \Theta\left(\frac{\di}{m}\right)$ and $\gkyt{\sbg} \le O\left(\frac{\di}{m}\right)$. Plugging those upper bounds and the trivial lower bound $\gknyt{\sbg} \ge 0$ into the above equation completes the proof.
\end{proof}

The above two lemmas directly lead to the following result saying that if a neuron is initialized to have large enough correlation to its positive examples (i.e., belonging to $\nyinit$), then this large enough correlation will be retained during training.
\begin{lemma}[Neuron properties during training]
\label{lemma:neuron_train}
For every label $y\in\yspace$, every iteration $t\le \wt{O}(\frac{m}{\eta\di})$, and every step size $\eta \le \frac{1}{\poly{\di}}$, we have $\ny{t+1} \supseteq \nyt$.
\end{lemma}
\begin{proof}
By~\lemmaref{lemma:neuron_init_cont}, we have at initialization
\[
\max_{k\in[m]}\bigabs{\expect{\rvx|\rvy=y\sim\dtrain}{\inprod{\wkinit}{\rvx}}} \le \wt{O}\left(\sqrt{\frac{\di}{d}}\right),\ \forall y\in\yspace.
\]
By~\lemmaref{lemma:corr_pos_all} and our choice of $T = \Theta(\frac{m}{\eta\di})$, we have
\[
\expect{\rvx|\rvy=y\sim\dtrain}[\inp{\wk{t}}{\rvx}] \le O\left(\frac{\eta \di T}{m}\right) + \max_{k\in[m]}\bigabs{\expect{\rvx|\rvy=y\sim\dtrain}{\inprod{\wkinit}{\rvx}}} = O(1).
\]
By~\lemmaref{lemma:corr_pos}, we have
\[
\expect{\rvx|\rvy=y\sim\dtrain}[\inp{\wk{t+1}}{\rvx}] \ge (1 - \lambda\eta) \expect{\rvx|\rvy=y\sim\dtrain}[\inp{\wkt}{\rvx}] + \Theta\left(\frac{\eta \di}{m}\right).
\]

Recall that $\lambda = O(\frac{\di}{m^{1.01}})$. Therefore, as long as $\expect{\rvx|\rvy=y\sim\dtrain}[\inp{\wk{t}}{\rvx}] = \wt{O}(1)$,
\[
\begin{aligned}
\expect{\rvx|\rvy=y\sim\dtrain}[\inp{\wk{t+1}}{\rvx}] - \expect{\rvx|\rvy=y\sim\dtrain}[\inp{\wkt}{\rvx}] &\ge \Theta\left(\frac{\eta \di}{m}\right) - \lambda\eta\expect{\rvx|\rvy=y\sim\dtrain}[\inp{\wkt}{\rvx}] \\
&= \Theta\left(\frac{\eta \di}{m}\right) > 0.
\end{aligned}
\]
Finally, recall that $\expect{\rvx|\rvy=y\sim\dtrain}[\inp{\wk{t}}{\rvx}] = \sum_{j\in\score}y\mu_j\inprod{\wkinit}{\bm{m}_j} + \sum_{j\in\sbg}\muj\inprod{\wkinit}{\mj}$. By~\defref{def:neuron}, we immediately have $\ny{t+1} \supseteq \nyt$.
\end{proof}

Finally, we introduce a lemma that bounds the expected correlation between every neuron in the network and its negative examples.

\begin{lemma}[Correlation with negative examples, all neurons]
\label{lemma:corr_neg}
For every iteration $t$, every $y\in\yspace$, and every $k\in[m]$ such that $\sign{(a_k)} = y$, the following holds:
\begin{equation}
\begin{aligned}
\expect{\rvx|\rvy=-y\sim\dtrain}[\inp{\wk{t+1}}{\rvx}] &= (1 - \lambda\eta) \expect{\rvx|\rvy=-y\sim\dtrain}[\inp{\wkt}{\rvx}] \\
&\qquad - \Theta\left(\frac{\eta\di}{m}\right)\prob{\rvx|\rvy=-y}{[\inprod{\wkt}{\rvx} \ge 0]}.
\end{aligned}
\end{equation}
\end{lemma}
\begin{proof}
Similar to the proof of~\lemmaref{lemma:corr_pos}, we have
\[
\begin{aligned}
&\expect{\rvx|\rvy=-y\sim\dtrain}[\inp{\wk{t+1}}{\rvx}] - \expect{\rvx|\rvy=-y\sim\dtrain}[\inp{\wkt}{\rvx}]\\
&\qquad= \underbrace{-y \Big(\biginprod{\wk{t+1}}{\sum_{j\in\score}\muj\mj} - \biginprod{\wk{t}}{\sum_{j\in\score}\muj\mj}\Big)}_{\dtcore} \\
&\qquad\qquad+ \Big(\underbrace{\biginprod{\wk{t+1}}{\sum_{j\in\sbg}\muj\mj} - \biginprod{\wk{t}}{\sum_{j\in\sbg}\muj\mj}}_{\dtbg}\Big).
\end{aligned}
\]
For $\dtcore$, by the SGD iteration~\eqref{eq:sgd} we have
\[
\dtcore = \eta y\Biginprod{\nabla_{\its{\rvw_k}{t}}\frac{1}{N}\sum_{i\in[N]} \ell\left(\its{h}{t}(\its{\rvx_i}{t}), \its{\rvy_i}{t}\right)}{\sum_{j\in \score}\muj\mj} + \lambda\eta y\biginprod{\wk{t}}{\sum_{j\in\score}\muj\mj}.
\]
Applying~\lemmaref{lemma:grad_core} gives
\[
\dtcore = -\eta \left(\gkyt{\score} + \gknyt{\score}\right) + \lambda\eta y \biginprod{\wk{t}}{\sum_{j\in\score}\muj\mj},
\]
which results in the iterative expression
\begin{equation}
-y\biginprod{\wk{t+1}}{\sum_{j\in\score}\muj\mj} = -y(1-\lambda\eta)\biginprod{\wk{t}}{\sum_{j\in\score}\muj\mj} - \eta \left(\gkyt{\score} + \gknyt{\score}\right).
\label{eq:iter_core}
\end{equation}
For $\dtbg$, by the SGD iteration~\eqref{eq:sgd} we have
\[
\dtbg = -\eta\Biginprod{\nabla_{\its{\rvw_k}{t}}\frac{1}{N}\sum_{i\in[N]} \ell\left(\its{h}{t}(\its{\rvx_i}{t}), \its{\rvy_i}{t}\right)}{\sum_{j\in \sbg}\muj\mj} -\lambda\eta \biginprod{\wk{t}}{\sum_{j\in\sbg}\muj\mj}.
\]
Applying~\lemmaref{lemma:grad_bg} gives
\[
\dtbg = \eta\left(\gkyt{\sbg} - \gknyt{\sbg}\right) -\lambda\eta \biginprod{\wk{t}}{\sum_{j\in\sbg}\muj\mj},
\]
which results in the iterative expression
\begin{equation}
\biginprod{\wk{t+1}}{\sum_{j\in\sbg}\muj\mj} = (1-\lambda\eta)\biginprod{\wk{t}}{\sum_{j\in\sbg}\muj\mj} + \eta\left(\gkyt{\sbg} - \gknyt{\sbg}\right).
\label{eq:iter_bg}
\end{equation}
Combining Eqs.~\eqref{eq:iter_core} and~\eqref{eq:iter_bg} gives
\begin{equation}
\begin{aligned}
\expect{\rvx|\rvy=-y\sim\dtrain}[\inp{\wk{t+1}}{\rvx}] &= -y\biginprod{\wk{t+1}}{\sum_{j\in\score}\muj\mj} + \biginprod{\wk{t+1}}{\sum_{j\in\sbg}\muj\mj} \\
&= -y(1-\lambda\eta)\biginprod{\wk{t}}{\sum_{j\in\score}\muj\mj} - \eta \left(\gkyt{\score} + \gknyt{\score}\right) \\
&\qquad + (1-\lambda\eta)\biginprod{\wk{t}}{\sum_{j\in\sbg}\muj\mj} + \eta\left(\gkyt{\sbg} - \gknyt{\sbg}\right) \\
&= (1-\lambda\eta)\expect{\rvx|\rvy=-y\sim\dtrain}[\inp{\wkt}{\rvx}] \\
&\qquad+ \eta\left(\gkyt{\sbg} - \gkyt{\score}\right) - \eta\left(\gknyt{\score} + \gknyt{\sbg}\right).
\end{aligned}
\label{eq:iter_all}
\end{equation}
For $\gkyt{\sbg} - \gkyt{\score}$, we have
\begin{equation}
\begin{aligned}
\gkyt{\sbg} - \gkyt{\score} &= \frac{1}{m}\expect{(\rvx,\rvy)}{\indicator{\htx\le 1}\indicator{\rvy=y}\indicator{\inprod{\its{\rvw_k}{t}}{\rvx}\ge 0} \bigg(\sum_{j\in\sbg}\muj\rvz_j - \sum_{j\in\score}\muj\rvz_j\bigg)} \\
&= \frac{1}{m}\expect{(\rvx,\rvy)}{\indicator{\htx\le 1}\indicator{\rvy=y}\indicator{\inprod{\its{\rvw_k}{t}}{\rvx}\ge 0} \left(\Theta\left(\frac{\di}{\log\di}\right) - \Theta(\di)\right)} \\
&\le 0.
\end{aligned}
\label{eq:tmp_1}
\end{equation}

For $\gknyt{\score} + \gknyt{\sbg}$, we have
\begin{equation}
\begin{aligned}
\gknyt{\score} &+ \gknyt{\sbg} = \frac{1}{m}\expect{(\rvx,\rvy)\sim\dtrain}{\indicator{\htx\le 1}\indicator{\rvy=-y}\indicator{\inprod{\its{\rvw_k}{t}}{\rvx}\ge 0} \bigg(\sum_{j\in\sbg}\muj\rvz_j + \sum_{j\in\score}\muj\rvz_j\bigg)} \\
&= \frac{1}{m}\expect{(\rvx,\rvy)\sim\dtrain}{\indicator{\htx\le 1}\indicator{\rvy=-y}\indicator{\inprod{\its{\rvw_k}{t}}{\rvx}\ge 0} \sum_{j\in[\di]}\muj\rvz_j} \\
&= \frac{1}{2m}\expect{\rvx|\rvy=-y}{\bigg[\indicator{\htx\le 1}\sum_{j\in[\di]}\muj\rvz_j\Big| \inprod{\its{\rvw_k}{t}}{\rvx}\ge 0\bigg]}\prob{\rvx|\rvy=-y}{[\inprod{\its{\rvw_k}{t}}{\rvx}\ge 0]} \\
&= \Theta\left(\frac{\di}{m}\right)\prob{\rvx|\rvy=-y}{[\inprod{\its{\rvw_k}{t}}{\rvx}\ge 0]}
\end{aligned}
\label{eq:tmp_2}
\end{equation}
Finally, plugging Eqs.~\eqref{eq:tmp_1} and~\eqref{eq:tmp_2} into Eq.~\eqref{eq:iter_all} gives the desired result.
\end{proof}

We are now ready to introduce the proofs of our main theoretical results.

\subsection{Proof of~\theoref{theo:activation}}
\label{appsec:proof_activation}

For ease of presentation, we first restate the theorem and then introduce its proof.
\setcounter{theorem}{0}

\begin{theorem}[Activation asymmetry]
For every $\eta \le \frac{1}{\poly{\di}}$ and every $y\in\yspace$, there exists $T_0 = \wt{\Theta}(\frac{m}{\eta\sqrt{d}})$ such that with high probability, for every $t\ge T_0$, there exist $\Theta(m)$ neurons in which the weight $\rvw_k^{(t)}$ for each neuron satisfies:
\begin{equation}
\begin{aligned}
&\prob{\rvx|\rvy=y\sim\dtrain}{[\inp{\wk{t}}{\rvx} \ge 0]} = 1 - O\left(\di^{-\frac{1}{2}}\right),\\
&\prob{\rvx|\rvy=-y\sim\dtrain}{[\inp{\wk{t}}{\rvx} \ge 0]} = o(1).
\end{aligned}
\end{equation}
\end{theorem}

\begin{proof}
For every $y\in\yspace$, consider the neuron set $\nyt$ defined in~\defref{def:neuron}. By~\lemmaref{lemma:neuron_init} and~\lemmaref{lemma:neuron_train}, we have $\abs{\nyt} = \Theta(m)$ for every iteration $t\le \Theta(\frac{m}{\eta\di})$. We then prove that after at most $T_0$ iterations, for every neuron $k\in\ny{T_0}$ we have $\prob{\rvx|\rvy=y\sim\dtrain}{[\inp{\wk{T_0}}{\rvx} \ge 0]} = 1 - O\left(\dfrac{1}{\sqrt{\di}}\right)$ and $\prob{\rvx|\rvy=-y\sim\dtrain}{[\inp{\wk{T_0}}{\rvx} \ge 0]} = o(1)$.

\textbf{Part 1: proving $\prob{\rvx|\rvy=y\sim\dtrain}{[\inp{\wk{T_0}}{\rvx} \ge 0]} = 1 - O\left(\dfrac{1}{\sqrt{\di}}\right)$.}

Let $T_0 = \Theta(\frac{m\sqrt{\log m\di}}{\eta\sqrt{d}}) = \wt{\Theta}(\frac{m}{\eta\sqrt{d}})$. By~\lemmaref{lemma:corr_pos} and~\lemmaref{lemma:corr_pos_all} we have
\[
\begin{aligned}
\expect{\rvx|\rvy=y\sim\dtrain}[\inp{\wk{t+1}}{\rvx}] &= (1 - \lambda\eta) \expect{\rvx|\rvy=y\sim\dtrain}[\inp{\wkt}{\rvx}] + \Theta\left(\frac{\eta \di}{m}\right) \\
&\le \expect{\rvx|\rvy=y\sim\dtrain}[\inp{\wkt}{\rvx}] + \Theta\left(\frac{\eta \di}{m}\right),
\end{aligned}
\]
which gives $\expect{\rvx|\rvy=y\sim\dtrain}[\inp{\wk{t}}{\rvx}] \le \wt{\Theta}(\frac{\di}{\sqrt{d}}) = O(1)$. Recall that $\lambda = o(\frac{\di}{m})$, we then have
\[
\begin{aligned}
\expect{\rvx|\rvy=y\sim\dtrain}[\inp{\wk{t+1}}{\rvx}] &= (1 - \lambda\eta) \expect{\rvx|\rvy=y\sim\dtrain}[\inp{\wkt}{\rvx}] + \Theta\left(\frac{\eta \di}{m}\right) \\
&\ge \expect{\rvx|\rvy=y\sim\dtrain}[\inp{\wkt}{\rvx}] + \Theta\left(\frac{\eta \di}{m}\right) - o\left(\frac{\eta \di}{m}\right)\\
&= \expect{\rvx|\rvy=y\sim\dtrain}[\inp{\wkt}{\rvx}] + \Theta\left(\frac{\eta \di}{m}\right),
\end{aligned}
\]
which gives $\expect{\rvx|\rvy=-y\sim\dtrain}[\inp{\wk{T_{0}}}{\rvx}] \ge {\Theta}\left(\frac{\di\sqrt{\log m\di}}{\sqrt{d}}\right)$.

By~\lemmaref{lemma:activation}, we have
\[
\begin{aligned}
\prob{\rvx|\rvy=y}{[\inp{\wk{T_0}}{\rvx} \ge 0]} &= \Phi\left(\frac{\expect{\rvx|\rvy=y}{\inp{\wkt}{\rvx}}}{\Theta(1)\sqrt{\sum_{j\in[\di]}\inp{\wkt}{\mj}^2}}\right) \pm O\left(\frac{1}{\sqrt{\di}}\right) \\
&\ge \Phi\left(\frac{{\Theta}\left(\frac{\di\sqrt{\log m\di}}{\sqrt{d}}\right)}{\Theta(\sqrt{\di})\max_{j\in[\di]} |\inprod{\wkt}{\mj}|}\right)\pm O\left(\frac{1}{\sqrt{\di}}\right)\\
&\ge \Phi\left(\frac{{\Theta}\left(\frac{\di\sqrt{\log m\di}}{\sqrt{d}}\right)}{\Theta(\sqrt{\di})\left(O\left(\sqrt{\frac{\di\log m}{d}}\right) + {\Theta}\left(\frac{\sqrt{\log m\di}}{\sqrt{d}}\right)\right)}\right)\pm O\left(\frac{1}{\sqrt{\di}}\right)\\
&= \Phi\left(\Theta\left(\sqrt{\log{\di}}\right)\right) \pm O\left(\frac{1}{\sqrt{\di}}\right).
\end{aligned}
\]
Applying~\lemmaref{lemma:gaussian} gives $\Phi\left(\Theta\left(\sqrt{\log\di}\right)\right) = 1 - \Theta(\frac{1}{\sqrt{\di}})$. We then have
\[
\prob{\rvx|\rvy=y}{[\inp{\wk{T_0}}{\rvx} \ge 0]} = 1 - O\left(\frac{1}{\sqrt{\di}}\right).
\]
\textbf{Part 2: proving $\prob{\rvx|\rvy=-y\sim\dtrain}{[\inp{\wk{T_0}}{\rvx} \ge 0]} = o(1)$.}

By~\lemmaref{lemma:corr_neg}, we have for every $t$ and $k\in\nyt$:
\begin{equation}
\begin{aligned}
\expect{\rvx|\rvy=-y\sim\dtrain}[\inp{\wk{t+1}}{\rvx}] &= (1 - \lambda\eta) \expect{\rvx|\rvy=-y\sim\dtrain}[\inp{\wkt}{\rvx}] \\
&\qquad - \Theta\left(\frac{\eta\di}{m}\right)\prob{\rvx|\rvy=-y}{[\inprod{\wkt}{\rvx} \ge 0]}.
\end{aligned}
\end{equation}
By~\lemmaref{lemma:activation}, we have
\begin{equation}
\prob{\rvx|\rvy=-y}{[\inp{\wkt}{\rvx} \ge 0]} = \Phi\left(\frac{\expect{\rvx|\rvy=-y}{\inp{\wkt}{\rvx}}}{\Theta(1)\sqrt{\sum_{j\in[\di]}\inp{\wkt}{\mj}^2}}\right) \pm O\left(\frac{1}{\sqrt{\di}}\right).
\label{eq:tmp3}
\end{equation}
Assume that a neuron $k\in\nyinit$ satisfies $\expect{\rvx|\rvy=-y}{\inp{\wkt}{\rvx}} \ge 0$. Then by Eq.~\eqref{eq:tmp3}, we have $\prob{\rvx|\rvy=-y}{[\inp{\wkt}{\rvx} \ge 0]} \ge \frac{1}{2} \pm O(\frac{1}{\sqrt{\di}})$, which gives
\[
\begin{aligned}
\expect{\rvx|\rvy=-y\sim\dtrain}[\inp{\wk{t+1}}{\rvx}] &= (1 - \lambda\eta) \expect{\rvx|\rvy=-y\sim\dtrain}{[\inp{\wkt}{\rvx}]} - \Theta\left(\frac{\eta\di}{m}\right)\\
&\le \expect{\rvx|\rvy=-y\sim\dtrain}{[\inp{\wkt}{\rvx}]} - \Theta\left(\frac{\eta\di}{m}\right).
\end{aligned}
\]
By~\lemmaref{lemma:neuron_init_cont}, we have at initialization $t=0$:
\begin{equation}
\max_{k\in[m]}\bigabs{\expect{\rvx|\rvy=-y\sim\dtrain}{\inprod{\wkinit}{\rvx}}} \le \wt{O}\left(\sqrt{\frac{\di}{d}}\right).
\end{equation}
Therefore, for any step size $\eta = \frac{1}{\poly{\di}}$, after at most $T_{01} \defeq \wt{\Theta}(\frac{m}{\eta\sqrt{\di d}})$ iterations, we must have $\expect{\rvx|\rvy=-y}{\inp{\wkt}{\rvx}} \le 0$ for every $k\in\nyt$.

Now, let $T_{02}\defeq \Theta(\frac{m\sqrt{\log m\di}}{\eta\sqrt{d}})$. Suppose that $\prob{\rvx|\rvy=-y}{[\inp{\wkt}{\rvx} \ge 0]} \ge \Theta(1)$ after $t = T_{01} + T_{02} = \wt{\Theta}(\frac{m}{\eta\sqrt{d}})$ steps. We then have for $t = T_{01},\ldots,T_{01}+T_{02} - 1$ that
\[
\begin{aligned}
\expect{\rvx|\rvy=-y\sim\dtrain}[\inp{\wk{t+1}}{\rvx}] &= (1 - \lambda\eta) \expect{\rvx|\rvy=-y\sim\dtrain}[\inp{\wkt}{\rvx}] - \Theta\left(\frac{\eta\di}{m}\right) \\
&\ge\expect{\rvx|\rvy=-y\sim\dtrain}{[\inp{\wkt}{\rvx}]} - \Theta\left(\frac{\eta\di}{m}\right),
\end{aligned}
\]
which gives $\expect{\rvx|\rvy=-y\sim\dtrain}[\inp{\wk{t}}{\rvx}] \ge -\wt{O}\left(\sqrt{\frac{\di}{d}}\right) - \Theta\left(\frac{\di\sqrt{\log m\di}}{\sqrt{d}}\right) \ge -O(1)$. Since $\lambda = o(\frac{\di}{m})$, we then have
\[
\begin{aligned}
\expect{\rvx|\rvy=-y\sim\dtrain}[\inp{\wk{t+1}}{\rvx}] &= (1 - \lambda\eta) \expect{\rvx|\rvy=-y\sim\dtrain}[\inp{\wkt}{\rvx}] - \Theta\left(\frac{\eta\di}{m}\right) \\
&\le \expect{\rvx|\rvy=-y\sim\dtrain}[\inp{\wkt}{\rvx}] - \Theta\left(\frac{\eta\di}{m}\right) + o\left(\frac{\eta\di}{m}\right)\\
&= \expect{\rvx|\rvy=-y\sim\dtrain}[\inp{\wkt}{\rvx}] - \Theta\left(\frac{\eta\di}{m}\right),
\end{aligned}
\]
which gives $\expect{\rvx|\rvy=-y\sim\dtrain}[\inp{\wk{T_{01} + T_{02}}}{\rvx}] \le -\Theta\left(\frac{\di\sqrt{\log m\di}}{\sqrt{d}}\right)$. Plugging this into Eq.~\eqref{eq:tmp3}, we obtain
\[
\begin{aligned}
\prob{\rvx|\rvy=-y}{[\inp{\wk{T_{01} + T_{02}}}{\rvx} \ge 0]} &= \Phi\left(\frac{\expect{\rvx|\rvy=-y}{\inp{\wkt}{\rvx}}}{\Theta(1)\sqrt{\sum_{j\in[\di]}\inp{\wkt}{\mj}^2}}\right) \pm O\left(\frac{1}{\sqrt{\di}}\right) \\
&\le \Phi\left(-\frac{\Theta\left(\frac{\di\sqrt{\log m\di}}{\sqrt{d}}\right)}{\Theta(\sqrt{\di})\max_{j\in[\di]} |\inprod{\wkt}{\mj}|}\right) \pm O\left(\frac{1}{\sqrt{\di}}\right) \\
&\le \Phi\left(-\frac{\Theta\left(\frac{\di\sqrt{\log m\di}}{\sqrt{d}}\right)}{\Theta(\sqrt{\di})\left({O}\left(\sqrt{\frac{\di\log m}{d}}\right) + \Theta\left(\frac{\sqrt{\log\di}}{\sqrt{d}}\right)\right)}\right)\pm O\left(\frac{1}{\sqrt{\di}}\right) \\
&= \Phi\left(-\Theta\left(\sqrt{\log\di}\right)\right) \pm O\left(\frac{1}{\sqrt{\di}}\right).
\end{aligned}
\]
Applying~\lemmaref{lemma:gaussian} gives $\Phi\left(-\Theta\left(\sqrt{\log\di}\right)\right) = \Theta(\frac{1}{\sqrt{\di}})$, which leads to
\[
\prob{\rvx|\rvy=-y}{[\inp{\wk{T_{01} + T_{02}}}{\rvx} \ge 0]} = O\left(\frac{1}{\sqrt{\di}}\right).
\]
This contradicts with our assumption that $\prob{\rvx|\rvy=-y}{[\inp{\wk{T_{01} + T_{02}}}{\rvx} \ge 0]} \ge \Theta(1)$. Hence, we must have $\prob{\rvx|\rvy=-y}{[\inp{\wk{T_{01} + T_{02}}}{\rvx} \ge 0]} = o(1)$.

Finally, combining \textbf{Part 1} and \textbf{Part 2} finishes the proof.
\end{proof}

\subsection{Proof of~\theoref{theo:feature}}
\label{appsec:proof_feature}

For ease of presentation, we first restate the theorem and then introduce its proof.

\begin{theorem}[Learned features]
For every $\eta \le \frac{1}{\poly{\di}}$ and every $y\in\yspace$, there exists $T_1 = \Theta(\frac{m}{\eta\di})$ such that with high probability, after $T_1$ iterations, there exist $\Theta(m)$ neurons in which the weight $\wk{T_1}$ for each neuron satisfies the following:
\begin{equation}
\begin{aligned}
&\sum\nolimits_{j\in\score}\mu_{j1}\inp{\wk{T_1}}{\mj} = y\cdot\Theta(1),\\
&\sum\nolimits_{j\in\sbg}\mu_{j1}\inp{\wk{T_1}}{\mj} = \wt{\Theta}(1).
\label{eq:feature}
\end{aligned}
\end{equation}
\end{theorem}
\begin{proof}
For every $y\in\yspace$, consider the neuron set $\nyt$ defined in~\defref{def:neuron}. By~\lemmaref{lemma:neuron_init} and~\lemmaref{lemma:neuron_train}, we have $\abs{\nyt} = \Theta(m)$ for every iteration $t\le T_1$.
We break the subsequent proof into two parts: in the first part we prove the desired result for core features $\score$ for all neurons $k\in\ny{T_1}$; in the second part we prove the desired result for background features $\sbg$ for all neurons $k\in\ny{T_1}$. Recall that we use the shorthand $\mu_j$ to denote $\mu_{j1} = \expect{\rvz\sim\dtrain}{\zj}$.

\textbf{Part 1: proving $\sum_{j\in\score}\muj\inp{\wk{T_1}}{\mj} = \inp{\wk{T_1}}{\sum_{j\in\score}\muj\mj} = \Theta(1)$.}

The SGD update~\eqref{eq:sgd} gives
\[
\begin{aligned}
&\biginprod{\wk{t+1}}{\sum_{j\in\score}\muj\mj} - \biginprod{\wk{t}}{\sum_{j\in\score}\muj\mj}\\
&\qquad= -\eta \Biginprod{\nabla_{\its{\rvw_k}{t}}\frac{1}{N}\sum_{i\in[N]} \ell\left(\its{h}{t}(\its{\rvx_i}{t}), \its{\rvy_i}{t}\right)}{\sum_{j\in \score}\muj\mj} - \lambda\eta \biginprod{\wk{t}}{\sum_{j\in\score}\muj\mj}
\end{aligned}
\]
for every $t = 0,\ldots,T_1 - 1$.

Applying~\lemmaref{lemma:grad_core}, we obtain
\[
\begin{aligned}
&\biginprod{\wk{t+1}}{\sum_{j\in\score}\muj\mj} - \biginprod{\wk{t}}{\sum_{j\in\score}\muj\mj}\\
&\qquad= y\cdot\Theta\left(\frac{\eta\di}{m}\right) + y\gknyt{\score} - \lambda\eta \biginprod{\wk{t}}{\sum_{j\in\score}\muj\mj} \\
&\qquad = y\cdot\Theta\left(\frac{\eta\di}{m}\right) - \lambda\eta \biginprod{\wk{t}}{\sum_{j\in\score}\muj\mj}.
\end{aligned}
\]
Without loss of generality, assume that $y = 1$ (the case of $y=-1$ is similar).
By the choice of $T_1 = \Theta(\frac{m}{\eta\di})$, we have
\[
\begin{aligned}
\biginprod{\wk{T_1}}{\sum_{j\in\score}\muj\mj} &\le \Theta\left(\frac{\eta T_1\di}{m}\right) + \biginprod{\wk{0}}{\sum_{j\in\score}\muj\mj}\\
&\le \Theta(1) + \wt{O}\left(\frac{\di}{d}\right) = \Theta(1),
\end{aligned}
\]
where in the second inequality we apply the concentration inequality of the maximum absolute Gaussian (\lemmaref{lemma:max_gaussian}). By our choice of $\lambda = o(\frac{\di}{m})$, we have
\[
\begin{aligned}
&\biginprod{\wk{t+1}}{\sum_{j\in\score}\muj\mj} - \biginprod{\wk{t}}{\sum_{j\in\score}\muj\mj}\\
&\qquad= \Theta\left(\frac{\eta\di}{m}\right) - \lambda\eta \biginprod{\wk{t}}{\sum_{j\in\score}\muj\mj} \\
&\qquad \ge \Theta\left(\frac{\eta\di}{m}\right) - o\left(\frac{\eta\di}{m}\right) = \Theta\left(\frac{\eta\di}{m}\right).
\end{aligned}
\]
Summing the above inequality from $t=0$ to $t=T_1 - 1$ yields
\[
\biginprod{\wk{T_1}}{\sum_{j\in\score}\muj\mj} = \Theta(1).
\]
Similarly, for $y = -1$ we have $\biginprod{\wk{T_1}}{\sum_{j\in\score}\muj\mj} = -\Theta(1)$.

\textbf{Part 2: proving $\sum_{j\in\sbg}\muj\inp{\wk{T_1}}{\mj} = \inp{\wk{T_1}}{\sum_{j\in\sbg}\muj\mj} = \wt{\Theta}(1)$.}

Similar to the first part of the proof, we have the SGD update
\[
\begin{aligned}
&\biginprod{\wk{t+1}}{\sum_{j\in\sbg}\muj\mj} - \biginprod{\wk{t}}{\sum_{j\in\sbg}\muj\mj}\\
&\qquad= -\eta \Biginprod{\nabla_{\its{\rvw_k}{t}}\frac{1}{N}\sum_{i\in[N]} \ell\left(\its{h}{t}(\its{\rvx_i}{t}), \its{\rvy_i}{t}\right)}{\sum_{j\in \sbg}\muj\mj} - \lambda\eta \biginprod{\wk{t}}{\sum_{j\in\sbg}\muj\mj}.
\end{aligned}
\]
Applying~\lemmaref{lemma:grad_bg}, we obtain
\[
\begin{aligned}
&\biginprod{\wk{t+1}}{\sum_{j\in\sbg}\muj\mj} - \biginprod{\wk{t}}{\sum_{j\in\sbg}\muj\mj}\\
&\qquad= \wt{\Theta}\left(\frac{\eta\di}{m}\right) - \eta\gknyt{\sbg} - \lambda\eta \biginprod{\wk{t}}{\sum_{j\in\sbg}\muj\mj},
\end{aligned}
\]
where
\[
\begin{aligned}
\gknyt{\sbg} &= \frac{1}{m}\sum_{j\in\sbg}\expect{(\rvx,\rvy)\sim\dtrain}{\indicator{\htx\le 1}\indicator{\rvy=-y}\indicator{\inprod{\its{\rvw_k}{t}}{\rvx}\ge 0} \muj\rvz_j}\\
&= \frac{1}{2m}\expect{\rvx|\rvy=-y}{\bigg[\indicator{\htx\le 1} \sum_{j\in\sbg}\muj\rvz_j\bigg| \inp{\wkt}{\rvx}\ge 0\bigg]} \prob{\rvx|\rvy=-y}{\left[\inp{\wkt}{\rvx}\ge 0\right]}\\
&\le \Theta\left(\frac{\di}{m}\right)\prob{\rvx|\rvy=-y}{\left[\inp{\wkt}{\rvx}\ge 0\right]}.
\end{aligned}
\]
Using~\theoref{theo:activation}, we have after at most $T_0 = \wt{\Theta}(\frac{m}{\eta\sqrt{d}})$ iterations, $\prob{\rvx|\rvy=-y}{\left[\inp{\wkt}{\rvx}\ge 0\right]} = o(1)$. We thus have
\[
\biginprod{\wk{t+1}}{\sum_{j\in\sbg}\muj\mj} = (1-\lambda\eta) \biginprod{\wk{t}}{\sum_{j\in\sbg}\muj\mj} + \wt{\Theta}\left(\frac{\eta\di}{m}\right)
\]
for every $t \ge T_0$. By a similar argument as in the first part of the proof, we have $\biginprod{\wk{T}}{\sum_{j\in\sbg}\muj\mj} \le \Theta(1)$ and
\[
\begin{aligned}
\biginprod{\wk{T_1}}{\sum_{j\in\sbg}\muj\mj} &= (T_1 - T_0) \wt{\Theta}\left(\frac{\eta\di}{m}\right) + \biginprod{\wk{T_0}}{\sum_{j\in\sbg}\muj\mj}\\
&\ge \wt{\Theta}(1) - T_0\cdot\Theta\left(\frac{\eta\di}{m}\right) - \wt{O}\left(\sqrt{\frac{\di}{d}}\right)\\
&\stackrel{(a)}{=}\wt{\Theta}(1) - \wt{\Theta}\left(\frac{m}{\eta\sqrt{d}}\right) \Theta\left(\frac{\eta\di}{m}\right) \\
&= \wt{\Theta}(1),
\end{aligned}
\]
where $(a)$ is due to $d \in [\Omega(\di^{2.01}),\poly{\di}]$.

Finally, combining \textbf{Part 1} and \textbf{Part 2} completes the proof.
\end{proof}

\subsection{Proof of~\theoref{theo:risk}}
\label{appsec:proof_risk}

For ease of presentation, we first restate the theorem and then introduce its proof.

\begin{theorem}[ID and OOD risks]
For every $\eta \le \frac{1}{\poly{\di}}$, there exists $T_2 = \wt{\Theta}(\frac{m}{\eta\di})$ such that with high probability, after $T_2$ iterations, the trained model $\its{h}{T_2}$ satisfies the following:
\begin{equation}
\begin{aligned}
&\exrisk{\dtrain}{\its{h}{T_2}} \le o(1),\\
&\oodrisk{\its{h}{T_2}} = \wt{\Theta}(1).
\end{aligned}
\end{equation}
\end{theorem}
\begin{proof}
We break the subsequent proof into two parts: in the first part we prove the desired result for the ID risk; in the second part we prove the desired result for the OOD risk.

\textbf{Part 1: proving $\exrisk{\dtrain}{\its{h}{T_2}} \le o(1)$.}

By definition, we have
\begin{equation}
\begin{aligned}
\exrisk{\dtrain}{\its{h}{T_2}} &= \expect{(\rvx,\rvy)\sim\dtrain}{\max\left\{1 - \rvy\hx{T_2},0\right\}}\\
&= \frac{1}{2}\underbrace{\expect{\rvx|\rvy=1}{\left[1 - \hx{T_2}\Big|\hx{T_2}\le 1\right]}\prob{\rvx|\rvy=1}{\left[\hx{T_2}\le 1\right]}}_{\mathcal{R}_1}\\
&\qquad + \frac{1}{2}\underbrace{\expect{\rvx|\rvy=-1}{\left[1 + \hx{T_2}\Big|\hx{T_2}\ge -1\right]}\prob{\rvx|\rvy=-1}{\left[\hx{T_2}\ge -1\right]}}_{\mathcal{R}_{-1}}.
\end{aligned}
\end{equation}
We first consider $\mathcal{R}_1$.
By~\theoref{theo:feature}, we have that after $T_1 = \Theta(\frac{m}{\eta\di})$ iterations, for every neuron $k\in\npos{t}$ with $t\ge T_1$, we have
\[
\sum_{j\in\score}\muj\rvz_j\inp{\wk{t}}{\mj} = \Theta(1),\quad \sum_{j\in\sbg}\muj\rvz_j\inp{\wk{t}}{\mj} = \wt{\Theta}(1).
\label{eq:feature_app}
\]
We can then obtain
\[
\expect{\rvx|\rvy=1}{\inp{\wk{t}}{\rvx}} = \sum_{j\in[\di]}\muj\zj\inp{\wk{t}}{\rvx} = \Theta(1).
\]
On the other hand, by~\lemmaref{lemma:corr_pos_all}, we know that for every neuron $k$ satisfying $\sign(a_k) = y$, its correlation grow rate is asymptotically not larger than the correlation grow rate of neurons in $\ny{t}$. Denoting the set of those neurons as $\mathcal{M}_y \defeq \{k\in[m]: \sign(a_k) = y\},\,\forall y\in\yspace$, we then have
\[
\expect{\rvx|\rvy=1}{\inp{\wk{t}}{\rvx}} = O(1),\,\forall k\in\mathcal{M}_1,t\ge T_1.
\]
Meanwhile, for all neurons $k\in\mathcal{M}_{-1}$, by~\lemmaref{lemma:corr_neg} and~\theoref{theo:activation} we have for all $t\ge T_0 = \wt{\Theta}(\frac{m}{\eta\sqrt{d}})$,
\[
\prob{\rvx|\rvy=1}{[\inp{\wk{t}}{\rvx} \ge 0]} = o(1).
\]
Therefore, we have
\[
\begin{aligned}
\expect{\rvx|\rvy=1}{\hx{T_1}} &= \frac{1}{m}\sum_{k\in\mathcal{M}_1}\expect{\rvx|\rvy=1}{\left[\relu\left(\biginprod{\wk{T_1}}{\rvx}\right)\right]} - \frac{1}{m} \sum_{k\in\mathcal{M}_{-1}} \expect{\rvx|\rvy=1}{\left[\relu\left(\biginprod{\wk{T_1}}{\rvx}\right)\right]}\\
&= \frac{1}{m}\sum_{k\in\mathcal{M}_1} \Theta(1) - \frac{1}{m}\sum_{k\in\mathcal{M}_{-1}} o(1)\\
&= \Theta(1).
\end{aligned}
\]
Now, suppose that $\mathcal{R}_1 \ge \Theta(1)$. Choose $T_2 = \Theta(\frac{m\sqrt{\log m}}{\eta\di}) = \wt{\Theta}(\frac{m}{\eta\di})$. Then, for every $t = T_1,\ldots,T_2-1$ we have
\[
\prob{\rvx|\rvy=1}{\left[\hx{T_2}\le 1\right]} = \Theta(1).
\]
This further leads to
\begin{equation}
\begin{aligned}
&\frac{1}{m}\sum_{k\in\mathcal{M}_1}\expect{\rvx|\rvy=1}{\left[\relu\left(\biginprod{\wk{t+1}}{\rvx}\right)\right]} - \frac{1}{m}\sum_{k\in\mathcal{M}_1}\expect{\rvx|\rvy=1}{\left[\relu\left(\biginprod{\wk{t}}{\rvx}\right)\right]}\\
&\qquad = \frac{1}{m}\sum_{k\in\mathcal{M}_1}\expect{\rvx|\rvy=1}{\left[\biginprod{\wk{t+1}}{\rvx} - \biginprod{\wkt}{\rvx}\right]} \\
&\qquad = \frac{1}{m}\sum_{k\in\npos{t}}\expect{\rvx|\rvy=1}{\left[\biginprod{\wk{t+1}}{\rvx} - \biginprod{\wkt}{\rvx}\right]} + \frac{1}{m}\sum_{k\in\mathcal{M}_1\setminus\npos{t}}\expect{\rvx|\rvy=1}{\left[\biginprod{\wk{t+1}}{\rvx} - \biginprod{\wkt}{\rvx}\right]}
\end{aligned}
\label{eq:tmp5}
\end{equation}
For the first term in the RHS of the last equality in~\eqref{eq:tmp5}, by~\lemmaref{lemma:corr_pos} we have
\[
\begin{aligned}
&\frac{1}{m}\sum_{k\in\npos{t}}\expect{\rvx|\rvy=1}{\left[\biginprod{\wk{t+1}}{\rvx} - \biginprod{\wkt}{\rvx}\right]} \\
&\qquad = \frac{1}{m}\sum_{k\in\npos{t}} \left(\Theta\left(\frac{\eta\di}{m}\right) - \lambda\eta \expect{\rvx|\rvy=1}{\biginprod{\wkt}{\rvx}}\right)\\
&\qquad = \Theta\left(\frac{\eta\di}{m}\right),
\end{aligned}
\]
where in the last equality we use $\abs{\npos{t}} = \Theta(m)$, $\lambda = O(\frac{\di}{m^{0.01}})$ and $\expect{\rvx|\rvy=1}{\biginprod{\wkt}{\rvx}} = \wt{O}(1)$ for $t\le T_2$.

For the second term in the RHS of the last equality in~\eqref{eq:tmp5}, by~\lemmaref{lemma:corr_pos_all} we have
\[
\frac{1}{m}\sum_{k\in\mathcal{M}_1\setminus\npos{t}}\expect{\rvx|\rvy=1}{\left[\biginprod{\wk{t+1}}{\rvx} - \biginprod{\wkt}{\rvx}\right]} \le O\left(\frac{\eta\di}{m}\right).
\]
Therefore,
\[
\begin{aligned}
\expect{\rvx|\rvy=1}{\hx{T_2}} &= \expect{\rvx|\rvy=1}{\hx{T_1} + \Theta\left(\frac{\eta\di (T_2-T_1)}{m}\right)} \pm o(1) \\
&= \Theta(1) + \Theta(\sqrt{\log m}) \pm o(1) = \Theta(\sqrt{\log m}).
\end{aligned}
\]
Applying one-sided Bernstein's inequality (\lemmaref{lemma:onesided_bernstein}) then gives
\[
\prob{\rvx|\rvy=1}{\left[\hx{T_2}\le 1\right]} = O\left(\frac{1}{\sqrt{m}}\right),
\]
which contradicts with $\prob{\rvx|\rvy=1}{\left[\hx{T_2}\le 1\right]} = \Theta(1)$. Hence, we must have $\mathcal{R}_1 = o(1)$. By a similar argument, we also have $\mathcal{R}_{-1} = o(1)$. We then have that $\exrisk{\dtrain}{\its{h}{T_2}} = o(1)$ holds.

\textbf{Part 2: proving $\oodrisk{\its{h}{T}} = \wt{\Theta}(1)$.}

This part of the proof directly follows from~\theoref{theo:feature}. Since after $t = T_1$ iterations we have $\sum_{j\in\sbg}\muj\inp{\wkt}{\mj} = \wt{\Theta}(1)$ for every neuron $k\in\nyt$, it can be shown that perturbing each $\mu_j$ from $\Theta(1)$ to $-\Theta(1)$ for $j\in\sbg$ (recall the generation process of the OOD data in~\defref{def:dgp}) changes the output of the network by at least $-\frac{1}{m}\sum_{k\in\nyt}\wt{\Theta}(1) = -\wt{\Theta}(1)$ using the fact that $\abs{\nyt} = \Theta(m)$ for every $t$ (using~\lemmaref{lemma:neuron_init} and~\lemmaref{lemma:neuron_train}). By the definition of the OOD risk we then arrive at the desired result.

Finally, combining \textbf{Part 1} and \textbf{Part 2} completes the proof.
\end{proof}

\section{Separation between Linear Networks and Non-Linear Networks: Proof of~\theoref{prop:linear}}
\label{appsec:proof_nonlinear}


Before providing the main proof, we first introduce some lemmas that characterize the gradients of the two-layer linear network. In general, the gradients of two-layer linear networks take a similar form to those of two-layer ReLU networks except for not having the ReLU derivative. We can thus reuse some of our lemmas in~\secref{appsec:preliminary} and~\secref{appsec:proof_main} in the analysis of the gradients.

\textbf{Notation.} In this section, we overload the notation from the previous sections such as $\htx$ and $\wkt$ and let them also denote the linear network model/weights.

\begin{lemma}[Gradient of linear networks]
\label{lemma:gradient_linear}
For every example $(x,y)\in\xspace\times\yspace$, every $k\in[m]$, and every iteration $t$, the following holds:
\begin{equation}
\nabla_{\its{\rvw_k}{t}} \ell\left(h(x), y\right) = -
a_k y \indicator{h(x) \le 1} x.
\end{equation}
\begin{proof}
	The proof follows from simple calculation.
\end{proof}
\end{lemma}

\begin{lemma}[Gap between empirical and population gradients]
	\label{lemma:gradient_gap_linear}
For every $k\in[m]$, every $j\in[d]$, and every iteration $t$, if the batch size $N = \poly{d}$ for some sufficiently large polynomial, then the following holds with probability at least $1 - e^{-\Omega(d)}$:
\begin{equation}
\bigg|\Biginprod{\nabla_{\its{\rvw_k}{t}}\frac{1}{N}\sum_{i\in[N]} \ell\left(\its{h}{t}(\its{\rvx_i}{t}), \its{\rvy_i}{t}\right)}{\bm{m}_j} - \Biginprod{\nabla_{\its{\rvw_k}{t}}\expect{(\rvx,\rvy)\sim\dtrain}{\ell\left(h(\rvx),\rvy\right)}}{\bm{m}_j} \bigg| \le \frac{1}{\poly{d}}.
\end{equation}
\end{lemma}
\begin{proof}
The proof is nearly identical to~\lemmaref{lemma:gradient_gap}, hence we omit here.
\end{proof}

Since in linear models we do not need to consider the activation probability (equivalently, this can be viewed as each neuron being fully activated for every example), we can analyze the gradient projections for all neurons without resorting to characterizing a subset of neurons as in~\defref{def:neuron}.

\begin{lemma}[Gradient projection onto background features, linear networks] For every iteration $t \le O(\frac{m}{\eta\di})$, every $k\in[m]$, and every $j\in\sbg$, the following holds:
\label{lemma:grad_bg_linear}
\begin{equation}
\left|\Biginprod{\nabla_{\its{\rvw_k}{t}}\frac{1}{N}\sum_{i\in[N]} \ell\left(\its{h}{t}(\its{\rvx_i}{t}), \its{\rvy_i}{t}\right)}{\mj}\right| = \frac{1}{\poly{d}},
\label{eq:grad_bg_linear}
\end{equation}
\end{lemma}

\begin{proof}
Applying~\lemmaref{lemma:gradient_linear} and~\lemmaref{lemma:gradient_gap_linear}, we obtain
\begin{equation}
\begin{aligned}
	&-\Biginprod{\nabla_{\its{\rvw_k}{t}}\frac{1}{N}\sum_{i\in[N]} \ell\left(\its{h}{t}(\its{\rvx_i}{t}), \its{\rvy_i}{t}\right)}{\mj}\\
	 &\qquad= -\biginprod{\nabla_{\its{\rvw_k}{t}}\expect{(\rvx,\rvy)\sim\dtrain}{\ell\left(\htx, \rvy\right)}}{\mj}\pm \frac{1}{\poly{d}}  \\
	 &\qquad= a_k\expect{(\rvx,\rvy)\sim\dtrain}{\rvy\indicator{\htx\le 1}\zj}\pm\frac{1}{\poly{d}}\\
	 &\qquad= a_k \expect{(\rvx,\rvy)\sim\dtrain}{\indicator{\htx \le 1}\indicator{\rvy=1} \zj}\\
	 &\qquad\qquad - a_k\expect{(\rvx,\rvy)\sim\dtrain}{\indicator{\htx \le 1}\indicator{\rvy=-1} \zj} \pm\frac{1}{\poly{d}}\\
	 &\qquad = \pm\frac{1}{\poly{d}}.
	\end{aligned}
\label{eq:app_grad_bg_base_linear}	
\end{equation}
This gives the desired result.
\end{proof}

We are now ready to prove~\theoref{prop:linear}.
For ease of presentation, we first restate the theorem and then introduce its proof.

\begin{theorem}[Linear networks]
If we replace the ReLU functions in the network with identity functions and keep other conditions the same as in~\theoref{theo:feature}, then with high probability, we have $|\inprod{\its{\rvw_k}{T_1}}{\bm{m}_j}| \le \wt{O}(\frac{1}{\sqrt{d}})$ for every $k\in[m]$ and every $j\in\sbg$.
\end{theorem}
\begin{proof}
For every $k\in[m]$ and every $j\in\sbg$, by the SGD update~\eqref{eq:sgd} we have
\[
\inp{\wk{t+1}}{\mj} = -\eta \Biginprod{\nabla_{\its{\rvw_k}{t}}\frac{1}{N}\sum_{i\in[N]} \ell\left(\its{h}{t}(\its{\rvx_i}{t}), \its{\rvy_i}{t}\right)}{\mj} + (1-\lambda\eta) \biginprod{\wk{t}}{\sum_{j\in\score}\muj\mj}.
\]
By~\lemmaref{lemma:grad_bg_linear}, we obtain
\[
\inp{\wk{t+1}}{\mj} = (1-\lambda\eta)\inp{\wk{t}}{\mj} \pm \frac{\eta}{\poly{d}}.
\]
By~\lemmaref{lemma:max_gaussian}, with probability at least $1 - O(\frac{1}{m})$, we have at initialization
\[
\max_{k\in[m]}|\inprod{\wkinit}{\mj}| \le 2\sqrt{\frac{\log m}{d}}.
\]
Recall that $\lambda = O(\frac{\di}{m^{1.01}})$. Combining the above equations gives the desired result.
\end{proof}

\textbf{Remark.} Similar to our analysis of two-layer ReLU networks, for two-layer linear networks we can also analyze the correlation growth between every neuron and the core features and show that SGD can converge to a solution with small ID risk. Since~\theoref{prop:linear} indicates that linear networks do not have feature contamination (i.e., background features do not accumulate in the weights), we can show that the network would also have small OOD risk at convergence. Since this analysis has a similar procedure to (and is also much simpler than) our analysis on two-layer ReLU networks we do not include it here.







\section{Probability Theory Lemmas}
\label{appsec:prob_lemma}

In this section, we provide the probability theory lemmas used in our proofs for completeness. Since those lemmas are standard results in the probability theory we omit the proofs of them.

We first state an one-sided form of Bernstein's inequality.

\begin{lemma}[One-sided Bernstein's inequality]
\label{lemma:onesided_bernstein}
Given $n$ independent random variables $\{\rvx_i\}_{i\in[n]}$ with $\rvx_i\le b$ almost surely for every $i\in[n]$, the following holds for every $\delta \ge 0$:
\begin{equation}
\prob{}{\Big[\sum_{i\in[n]}(\rvx_i - \expect{}{\rvx_i}) \ge n\delta}\Big] \le \exp{\left(-\frac{n\delta^2}{\frac{1}{n}\sum_{i\in[n]}\expect{}{\rvx_i^2} + \frac{b\delta}{3}}\right)}.
\end{equation}
\end{lemma}

Note that the above result can also be used to derive bounds on the lower tail by applying it to the random variables $\{-\rvx_i\}_{i\in[n]}$ if each $\rvx_i$ is bounded from below.

We then state a matrix extension of Bernstein's inequality; such type of inequalities is useful for bounding the gradients of the network in our proofs.

\begin{lemma}[Matrix Bernstein's inequality~\citep{oliveira2009concentration,tropp2012user}]
\label{lemma:matrix_bernstein}
Given $n$ independent random matrices $\{\rmX_i\}_{i\in[n]}$ with dimension $d_1\times d_2$ and $\expect{}{\rmX_i} = \mathbf{0}$, $\lVert \rmX_i \rVert_2\le b$ almost surely for every $i\in[n]$, define the sum $\rmS\defeq\sum_{i\in[n]}\rmX_i$ and let $v(\rmS)$ denote the matrix variance statistic of the sum:
\begin{equation}
v(\rmS) \defeq \max{\{ \lVert\expect{}{[\rmS\rmS^*]}\rVert_2, \lVert\expect{}{[\rmS^*\rmS]}\rVert_2 \}},
\end{equation}
where $\lVert \cdot \rVert_2$ denotes the spectral norm a matrix or the $\ell_2$ norm of a vector (when $d_1=1$ or $d_2=1$). Then, the following holds for every $\delta \ge 0$:
\begin{equation}
\prob{}{\big[\lVert\rmS\rVert_2 \ge \delta \big]} \le (d_1 + d_2)\cdot\exp{\left(-\frac{\delta^2}{2v(\rmS) + \frac{2b\delta}{3}}\right)}.
\end{equation}
\end{lemma}

We then state a basic result for bounding the cumulative distribution function of the unit Gaussian distribution that is repeatedly used in deriving neuron properties in initialization.

\begin{lemma}[Bounds for unit Gaussian variables]
\label{lemma:gaussian}
Let $\rvx\sim\mathcal{N}(0,1)$ be a unit Gaussian random variable. Then, the following holds for every $\delta > 0$:
\begin{equation}
\frac{1}{\sqrt{2\pi}}\frac{\delta}{\delta^2+1}\exp{\left(-\frac{\delta^2}{2}\right)}\le \prob{}{[\rvx\ge \delta]} \le \frac{1}{\sqrt{2\pi}}\frac{1}{\delta}\exp{\left(-\frac{\delta^2}{2}\right)}.
\end{equation}
\end{lemma}

Finally, we state a result for lower bounding the upper tail of the cumulative distribution function for binomial variables using Hoeffding's inequality:
\begin{lemma}[Tail bound for binomial variables]
\label{lemma:binomial}
Let $\rvx\sim\mathcal{B}(n,p)$ be a binomial random variable with trial number $n$ and success probability $p\in [0,1]$. Then, the following holds for every $n, p$ and integer $k\le np$:
\begin{equation}
\prob{}{[\rvx\ge k]} \ge 1 - \exp{\left(-2n\left(p - \frac{k-1}{n}\right)^2\right)}.
\end{equation}
\end{lemma}

\begin{lemma}[Concentration of the maximum of absolute Gaussian]
\label{lemma:max_gaussian}
Let $\rvx_1,\ldots,\rvx_n$ be i.i.d. random variables that follow the zero-mean Gaussian distribution $\mathcal{N}(0,\sigma^2)$. Then, the following holds for every positive integer $n$:
\begin{equation}
\prob{}{\left[\max_{i\in[n]} \abs{\rvx_i} \ge 2\sigma\sqrt{\log n}\right]} \le \frac{2}{n}.
\end{equation}
\end{lemma}

\begin{lemma}[Berry–Esseen theorem]
\label{lemma:berry}
Let $\rvx_1,\ldots,\rvx_n$ be independent random variables with $\expect{}{\rvx_i} = 0$, $\expect{}{\rvx_i^2} = \sigma_i^2 > 0$, and $\rho_i\defeq\expect{}{|\rvx_i^3|} < \infty$. Also, define the normalized sum
\begin{equation}
\rvs_n \defeq \frac{\sum_{i\in[n]}\rvx_i}{\sqrt{\sum_{i\in[n]}\sigma_i^2}}.
\end{equation}
Denote $\Phi$ the cumulative distribution function of $\mathcal{N}(0,1)$. Then, there exists a constant $C_0 \in [0.40, 0.56]$ such that
\begin{equation}
\sup_{\delta\in\mathbb{R}}\left|\prob{}{[\rvs_n < \delta]} - \Phi(\delta)\right| \le C_0 \left(\sum_{i=1}^n \sigma_i^2\right)^{-\frac{3}{2}}\sum_{i=1}^n\rho_i.
\end{equation}
\end{lemma}

\newpage

\begin{center}
\vspace{0.5em}
\hypertarget{app:exp}{}
\Large
\textsc{Appendix {\uppercase\expandafter{\romannumeral 2}}: Experimental Details and Additional Results}
\vspace{0.5em}
\end{center}

In this part of the appendix, we present the details of the experiments in the main text and include additional empirical results in both real-world datasets and synthetic distribution shift settings. A quick overview of the structure of this part is as follows:
\begin{itemize}[leftmargin=2em]
\setlength\itemsep{0.5em}
	\item In~\secref{appsec:distill}, we provide the implementation details and more results of the representation distillation experiments in~\secref{sec:main_exp}.
	\item In~\secref{appsec:numerical}, we present more numerical results, implementation details, more results of class activation histograms, and additional feature visualization for deep neural networks.
	\item In~\secref{appsec:conjecture}, we provide empirical evidence that supports the conjecture in~\secref{sec:discussion} and more discussion on related work.
\end{itemize}

\section{Representation Distillation Details}
\label{appsec:distill}

\subsection{Natural Distribution Shifts of ImageNet}

\paragraph{Datasets.} Following~\citep{taori_measuring_2020,radford_learning_2021,wortsman_robust_2022}, we consider 5 natural distribution shift test sets of ImageNet that are representative of real-world distribution shifts without artificial perturbations to images, including ImageNetV2~\citep{recht_imagenet_2019}, ImageNet-R~\citep{hendrycks_many_2021}, ObjectNet~\citep{barbu_objectnet_2019}, ImageNet Sketch~\citep{wang_learning_2019}, and ImageNet-A~\citep{hendrycks_natural_2021}. Compared to the original training and validation (ID test) sets of ImageNet, those test sets are reflective of changes in data distribution due to natural variations in the data collection process such as lighting, geographic location, image background, and styles.

\paragraph{Pre-trained models.} We used pre-trained checkpoints provided by CLIP~\citep{radford_learning_2021}, which is reported to exhibit remarkable robustness to distribution shifts of ImageNet. The official CLIP repository provide CLIP models pre-trained on the same dataset with varying sizes and architectures (ResNets and ViTs). In our experiments, we used five different CLIP models, including four ResNets and one ViT: CLIP-ResNet-50 (CLIP-RN50), CLIP-ResNet-101 (CLIP-RN101), CLIP-ResNet-50x4 (CLIP-RN50x4), CLIP-ResNet-50x16 (CLIP-RN50x16), and CLIP-ViT-B/16. For linear probing, we freezed the weights of the pre-trained models and trained randomly-initialized linear classification heads on top of the extracted representations on the ImageNet training set for 10 epochs. Following the hyperparameters used by~\citet{wortsman_robust_2022}, we used the AdamW optimizer~\citep{loshchilov_decoupled_2019} with learning rate 0.001, $\ell_2$ weight decay 0.1, batch size 256, and a cosine learning rate scheduler~\citep{loshchilov_sgdr_2017}. The results are reported based on the model with the best ID validation accuracy.

\paragraph{Representation distillation.} For each pre-trained CLIP model (teacher model), we freezed its weights and randomly initialized another model with identical architecture to the teacher model. We used the Mean Squared Error (MSE) loss to train the student model on the ImageNet training set, minimizing the mean Euclidean distance between the representations extracted by the student model and the representations extracted by the teacher model. We did not perform extensive grid search on the distillation hyperparameters and sticked to the following hyperparameter choices based on our preliminary experiments:
\begin{itemize}[leftmargin=2em]
	\item CLIP-RN50: AdamW optimizer with learning rate 0.001, $\ell_2$ weight decay 0.05, batch size 256, and a cosine learning rate schedular with warmup for 10000 steps; 100 distillation epochs.
	\item CLIP-RN101: AdamW optimizer with learning rate 0.001, $\ell_2$ weight decay 0.1, batch size 256, and a cosine learning rate scheduler with warmup for 10000 steps; 100 distillation epochs.
	\item CLIP-RN50x4 and CLIP-RN50x16: AdamW optimizer with learning rate 0.0001, $\ell_2$ weight decay 0.5, batch size 256, and a cosine learning rate scheduler with warmup for 10000 steps; 100 distillation epochs.
	\item CLIP-ViT-B/16: AdamW optimizer with learning rate 0.0001, $\ell_2$ weight decay 0.1, batch size 256, and a cosine learning rate scheduler with warmup for 10000 steps; 200 distillation epochs. Besides minimizing the difference between final representations (i.e., the output of the last layer of the networks) of student and teacher networks, we also minimized the difference between student and teacher network's intermediate representations of each residual attention block with a weighting coefficient 0.1.
\end{itemize}

In the linear probing stage, we freezed the parameters of the student models and trained a randomly initialized linear classification head for each student model on the ImageNet training set for 10 epochs. We used the AdamW optimizer with learning rate 0.001, $\ell_2$ weight decay of 0.001, batch size 256, and a cosine learning rate scheduler. The results are reported based on the model with the best ID validation accuracy.

\paragraph{Baseline models.} We reported the results of baseline models provided by the testbed of~\citet{taori_measuring_2020}. In their testbed,~\citet{taori_measuring_2020} catogory the models into different types, where some type of models are trained with more data than the original ImageNet training set. Since our aim is to explore the limit of representation learning using only ID data, we omit the results of those models trained with more data. We also omit the results of models with significantly lower accuracy than common ImageNet models, such as linear classifier on pixels or random features, classifiers based on nearest neighbors, and low accuracy CNNs. Concretely, we reported the results of the following two types of models defined by~\citet{taori_measuring_2020}:

\begin{itemize}[leftmargin=2em]
	\item \texttt{STANDARD}: models obtained by standard training (i.e., ERM) on the ImageNet training set.
	\item \texttt{ROBUST\_INTV}: models trained with existing robust intervention techniques on the ImageNet training set.
\end{itemize}

\paragraph{Detailed results.} 
We list detailed OOD generalization performance of linear probes on pre-trained and distilled representations on all 5 distribution shift test sets as well as the ID generalization results on the original ImageNet validation set in~\tableref{table:results_imagenet}.

\begin{table}[h]
\centering
\caption{Detailed ID and OOD top-1 accuracy (\%) of linear probes on pre-trained and distilled representations on ImageNet-based test sets. ``Im'' refers to ``ImageNet''.}
\label{table:results_imagenet}
\begin{tabular}{cccccccc}
\toprule
 & Im (ID) & OOD Avg. & ImV2 & Im-R & ObjectNet & Im Sketch & Im-A \\
\midrule
CLIP-RN50 & 70.37 & 39.42 & 59.03 & 51.18 & 37.72 & 31.87 & 17.31 \\
Distilled RN50 & 69.85 & 31.63 & 57.97 & 38.22 & 32.72 & 20.97 & 8.25  \\
\midrule
CLIP-RN101 & 72.33 & 45.27 & 61.70 & 59.92 & 43.07 & 37.93 & 23.73 \\
Distilled RN101 & 72.28 & 35.18 & 60.46 & 44.09 & 35.89 & 23.88 & 11.56 \\
\midrule
CLIP-ViT-B/16 & 79.40 & 57.59 & 69.72 & 72.42 & 51.85 & 47.33 & 46.64\\
Distilled ViT-B/16 & 73.58 & 37.14 & 62.45 & 44.43 & 35.52 & 23.83 & 19.47 \\
\midrule
CLIP-RN50x4 & 76.18 & 51.45 & 65.83 & 64.80 & 48.74 & 42.19 & 35.67 \\
Distilled RN50x4 & 76.25 & 41.40 & 65.20 & 49.22 & 42.71 & 29.23 & 20.64 \\
\midrule
CLIP-RN50x16 & 80.24 & 60.61 & 70.13 & 73.67 & 56.92 & 48.52 & 53.79 \\
Distilled RN50x16 & 79.65 & 48.26 & 68.49 & 55.03 & 48.90 & 32.93 & 35.97 \\
\bottomrule
\end{tabular}
\end{table}

\subsection{iWildCam-WILDS}

\paragraph{Dataset.} We used the official version of the dataset provided by WILDS~\citep{koh_wilds_2021}.

\paragraph{Pre-trained models.} In order to obtain a feature extractor that exhibits sufficient generalization ability on the dataset, we explored different pre-trained models including ViTs in CLIP~\citep{radford_learning_2021}, RegNets in SWAG~\citep{singh_revisiting_2022} as well as ResNets pre-trained on ImageNet~\citep{deng_imagenet:_2009}. In the end, we chose a fine-tuned ResNet-50 (RN50) that is pre-trained on ImageNet as the teacher model since we observed that ImageNet-scale pre-training already leads to considerable robustness improvements compared to models trained from scratch on this dataset (also reported by~\citet{miller_accuracy_2021}), while being consistent to the network architecture used in the official WILDS repository. For linear probing, we freezed the parameters of the pre-trained model and trained a randomly initialized linear classification head using the hyperparameters provided by the official WILDS repository. The results are reported based on the model with the best OOD validation accuracy, following the protocol used by the WILDS paper~\citep{koh_wilds_2021}.
 
\paragraph{Representation distillation.} We freezed the weights of the teacher model and randomly initialized a ResNet-50 as the student model. We trained the student model by minimizing the Euclidean distance between its extracted representations and the representations extracted by the teacher model using the MSE loss on the training domains of iWildCam-WILDS. The student model was trained for 150 epochs using AdamW with batch size 128, learning rate 0.0001, and $\ell_2$ weight decay 0.1. In the linear probing stage, we freezed the parameters of the student model and trained a randomly initialized linear classification head using the hyperparameters provided by the official WILDS repository. The results are reported based on the model with the best OOD validation accuracy, following the protocol used by the WILDS paper.

\paragraph{Baseline models.} We reported the results of baseline models provided by~\citep{miller_accuracy_2021}. In their result file,~\citet{miller_accuracy_2021} report both results for ImageNet-pre-trained neural networks (corresponding to models with \texttt{model\_type} as ``Neural Network'' in the result file) and results for neural networks trained from scratch (corresponding to models with \texttt{model\_type} as ``ImageNet Pretrained Network''). Since our aim is to explore the limit of representation learning using only ID data, we omit the results of the models with pre-training.

\paragraph{Detailed results.} We list detailed ID and OOD generalization performance of linear probes on pre-trained and distilled representations on iWildCam-WILDS in~\tableref{table:results_iwildcam}.

\begin{table}[h]
\centering
\caption{Detailed ID and OOD Macro F1 of linear probes on pre-trained and distilled representations on iWildCam-WILDS.}
\label{table:results_iwildcam}
\begin{tabular}{ccc}
\toprule
 & ID Macro F1 & OOD Macro F1 \\
\midrule
ImageNet RN50 & 49.30 & 32.46 \\
Distilled RN50 & 32.32 & 13.83 \\
\bottomrule
\end{tabular}
\end{table}


\subsection{Camelyon17-WILDS}

\paragraph{Dataset.} We used the official version of the dataset provided by WILDS~\citep{koh_wilds_2021}.

\paragraph{Pre-trained models.} 
After preliminary experiments, we chose a ViT-B/16 pre-trained by CLIP as our teacher model. For linear probing, we freezed the parameters of the pre-trained model and trained a randomly initialized linear classification head using the hyperparameters provided by the official WILDS repository. The results are reported based on the model with the best OOD validation accuracy, following the protocol used by the WILDS paper~\citep{koh_wilds_2021}.

\paragraph{Representation distillation.} We freezed the weights of the teacher model and randomly initialized a ViT-B/16 with identical architecture to the teacher model as the student model. We trained the student model by minimizing the Euclidean distance between its extracted representations and the representations extracted by the teacher model using the MSE loss on the training domains of Camelyon17-WILDS. The student model was trained for 120 epochs with batch size 128, learning rate 0.0001 and $\ell_2$ weight decay 0.1 using AdamW. For linear probing, we freezed the parameters of the student model and trained a randomly initialized linear classification head using the hyperparameters provided by the official WILDS repository. The results are reported based on the model with the best OOD validation accuracy, following the protocol used by the WILDS paper.

\paragraph{Baseline models.}  We reported the results of all algorithms from the offcial WILDS leaderboard (accessed at September 26th, 2023) that do not use custom data augmentation or pre-training (including ``SGD (Freeze-Embed)'' that uses CLIP pre-training and ``ContriMix'', ``MBDG'', ``ERM w/ targeted aug'' and ``ERM w/ H\&E jitter'' that use custom, task-specific data augmentations) as baseline results.

\paragraph{Detailed results.} We list detailed ID and OOD generalization performance of linear probes on pre-trained and distilled representations on Camelyon17-WILDS in~\tableref{table:results_camelyon}.

\begin{table}[h]
\centering
\caption{Detailed ID validation and OOD test accuracy (\%) of linear probes on pre-trained and distilled representations on Camelyon17-WILDS.}
\label{table:results_camelyon}
\begin{tabular}{ccc}
\toprule
 & ID Validation Accuracy & OOD Test Accuracy \\
\midrule
CLIP-ViT-B/16 & 98.22 & 92.88 \\
Distilled ViT-B/16 & 98.28 & 89.83 \\
\bottomrule
\end{tabular}
\end{table}

\subsection{DomainNet}
\label{appsec:domainnet}

\paragraph{Dataset.} Following the setup of~\citet{tan_class-imbalanced_2020,kumar_fine-tuning_2022}, we used a pruned version of the original DomainNet dataset~\citep{peng_moment_2019}. The pruned dataset consists of 4 domains \{Clipart, Painting, Real, Sketch\} and 40 commonly occurring classes, selected from the original DomainNet which consists of 6 domains and 345 classes.

\paragraph{Implementation details.} We adhered to the experimental settings as in DomainBed~\citep{gulrajani_search_2021}, which encompassed protocols for data augmentation, dataset partitioning, and hyperparameter search strategies. We opted for the widely adopted training domain validation for the model selection criterion. To reduce the computational cost, without loss of generality, we chose the Sketch domain with the largest distributional shifts as the test domain (OOD), while training on the other three domains (ID). For both our model and baseline models, we performed random searches on the hyperparameters with three different random seeds, each involving 5 trials.

\paragraph{Pre-trained models.} We used the official ResNet-50 (RN50), ResNet-101 (RN101), and ViT-B/32 pre-trained checkpoints provided by CLIP.

\paragraph{Representation distillation.}
Due to limitations imposed by the scale of the dataset, we exclusively employed the CLIP-RN50 as the teacher model---it turns out in our preliminary experiments that distilling the other two pre-trained models results in \emph{worse} performance both ID and OOD, which we believe is because the scale of the dataset is too small for distilling larger models.
In the distillation stage, we freezed the pre-trained CLIP-RN50 as the teacher model and used the MSE loss to train the student RN50 model with the exact same structure as CLIP-RN50 but with random initialization.
We used the AdamW optimizer with a cosine scheduler and learning rate 0.0003, $\ell_2$ weight decay 5e-5, batch size 32, and trained the student model for 95000 iterations.
In the linear probe stage, we freezed the parameters of the student model and add a randomly initialized single-layer linear classifier. We trained the linear probe on the training sets of the three training domains and performed zero-shot evaluation on the test domain. We ultimately select the checkpoints with the highest accuracy on the validation set from the training domain. During this stage, we used the Adam optimizer~\citep{kingma_adam:_2015} with a cosine scheduler and learning rate 0.003, $\ell_2$ weight decay 1e-6, batch size 32, and trained the linear probe for 5000 iterations.

\paragraph{Baseline models.}
We generally followed the settings of DomainBed, with the exception of using a modified RN50 model with the same structure as CLIP-RN50 but randomly initialized.
Additionally, we introduced a cosine scheduler with a warmup to enhance the convergence of models trained from scratch. We conducted extensive experiments with 15 representative domain generalization algorithms, including ERM~\citep{vapnik_nature_1999}, IRM~\citep{arjovsky_invariant_2019}, GroupDRO~\citep{sagawa_distributionally_2020}, Mixup~\citep{zhang_mixup_2018}, MLDG~\citep{li_learning_2018}, Deep CORAL~\citep{sun_deep_2016}, DANN~\citep{ganin_domain-adversarial_2016}, SagNet~\citep{nam_reducing_2021}, ARM~\citep{zhang_adaptive_2021}, VREx~\citep{krueger_out--distribution_2021}, RSC~\citep{huang_self-challenging_2020}, SelfReg~\citep{kim_selfreg_2021}, IB-ERM~\citep{ahuja_invariance_2021}, and IB-IRM~\citep{ahuja_invariance_2021}, and Fish~\citep{shi_gradient_2022}. We increased the training iterations from the default 5000 to 20000 to ensure the convergence of all methods.

\paragraph{Detailed results.} We list detailed ID and OOD generalization performance of linear probes on pre-trained and distilled representations on DomainNet in~\tableref{table:results_domainnet}.

\begin{table}[h]
\centering
\caption{Detailed ID test and OOD test accuracy (\%) of linear probes on pre-trained and distilled representations on DomainNet.}
\label{table:results_domainnet}
\begin{tabular}{ccc}
\toprule
 & ID Test Accuracy & OOD Test Accuracy \\
\midrule
CLIP-RN101 & 92.30 & 87.34 \\
CLIP-ViT-B/32 & 92.35 & 87.60 \\
\midrule
CLIP-RN50 & 87.02 & 82.58 \\
Distilled RN50 & 77.91 & 64.78 \\
\bottomrule
\end{tabular}
\end{table}


\section{Additional Experiments and Results}
\label{appsec:numerical}

\subsection{Numerical Experiments}
\label{appsubsec:numerical}

In this section, we present the results of our numerical experiments. The numerical experiments were conducted with parameters $\dcore=\dbg=32$, $d=256$, $m=256$, and $N=1000$. During training, each $\rvz_i,\, i\in[\di]$ was sampled from the uniform distribution on its support $[0,1]$; during testing, each $\rvz_i,\, i\in\score$ was sampled from the same distribution as in training, while each $\rvz_i,\, i\in\sbg$ was sampled from the uniform distribution on $[-1,0]$. We considered two experimental settings:
\begin{itemize}[leftmargin=2em]
\item \textbf{Classification:} We trained a two-layer ReLU network to predict the binary label for each input, which matches our theoretical setting in~\secref{sec:main}. As an ablation, we also trained a two-layer linear network for the same task, replacing the ReLU functions in the network by identity functions.
\item \textbf{Regression (representation distillation):} We trained a two-layer ReLU network to predict the vector $(\rvz_i)_{i\in\score}$ for each input---note that this is an optimal representation for OOD generalization, which matches the setting as our real-world representation distillation experiments in~\secref{sec:main_exp}. As an ablation, we also trained a two-layer linear network.
\end{itemize}

In both settings, we trained the network using SGD with a learning rate 0.001 and a weight decay $\lambda=0.001$. The results are in~\figref{fig:neurons},~\figref{fig:classification}, and~\figref{fig:regression}, which corroborate our theoretical results on
\begin{itemize}[leftmargin=2em]
\item \textbf{Activation asymmetry:} as shown by~\figref{fig:neurons}, each neuron evolves to have positive correlations with at most one class of examples during training.
\item \textbf{Feature contamination happens for non-linear networks:} as shown by~\figref{subfig:relu_classification} (classification) and~\figref{subfig:relu_regression} (regression), two-layer ReLU networks indeed accumulate weight projections onto the background features during training, leading to small ID risk yet large OOD risk.
\item \textbf{Feature contamination does \emph{not} happen for linear networks:} as shown by~\figref{subfig:linear_classification} (classification) and~\figref{subfig:linear_regression} (regression), two-layer linear networks does not accumulate weight projections onto the background features during training, leading to both small ID risk and small OOD risk when the concept class is linearly separable.
\end{itemize}

\paragraph{Extensions to more general settings:} in~\figref{fig:numerical}(d) in the main text, we provide numerical results of the average correlations between weights and background features for different activation functions. Here we detail our experimental settings and provide complete results in~\figref{fig:numerical_all}. Concretely, we consider a \emph{non-linear} relationship between core features and the label where core features for the two classes are distributed in a hyperball in $\mathbb{R}^{\dcore}$ with radii $1.0$ and $2.0$, respectively. We consider four different activation functions, namely ReLU, GELU~\citep{hendrycks_gaussian_2016}, Sigmoid, and Tanh. We consider a two-layer network where both layers have trainable weights and biases. We use the AdamW optimizer~\citep{loshchilov_decoupled_2019} instead of SGD.

\subsection{Class Activation Histograms}
\label{appsec:histograms}

In this section, we include average activation rate histograms for all blocks of ResNet-50 and ViT-B/16 as described in~\secref{sec:evidence} in the main text. For every block in ResNet, we compute the mean activation rate for every class averaged over all channels in the final ReLU layer; for every block in ViT, we compute the mean activation rate for every class averaged over all channels in the GELU layer in its MLP. For every channel with the activation function $\sigma$ and pre-activation input $x$, the activation rate is defined by $\mathsf{rate}(x)\defeq\left\{\begin{aligned} &1,\ \mathrm{if}\ \sigma(x)\ge 0\\ &0,\ \mathrm{otherwise} \end{aligned}\right.$, where $\sigma$ is ReLU for ResNet and GELU for ViT.

See~\figref{appfig:histograms_rn} for histograms of ResNet-50 blocks and~\figref{appfig:histograms_vit} for histograms of ViT-B/16 blocks. For earlier blocks in ResNet-50, activation asymmetry is not exhibited at random initalization but exhibited after training; for later blocks in ResNet-50 and all blocks in ViT, activation asymmetry is exhibited both at random initialization and after training.

\subsection{Neuron Selectivity}
\label{appsec:selectivity}

In this section, we detail the \emph{neuron selectivity} metric that we adopted to produce the results in~\figref{fig:selectivity} in the main text. Concretely, to examine whether the property of activation asymmetry also holds for deep neural networks, we adopt a similar selectivity metric as in~\citep{morcos_importance_2018} to quantify the difference of neuron activation magnitudes between different classes. For a $C$-way multi-class classification problem and for a neuron, we denote the class-conditional mean activation after the nonlinearity for each class by $\mu_1, \ldots, \mu_C$. In other words, each $\mu_i\in\mathbb{R}$ is calculated by averaging the activation of all inputs that belong to class $i$. Then, the \emph{selectivity} of this neuron is defined as
\begin{equation}
\mathsf{Selectivity} \defeq \frac{\mu_\mathrm{max} - \mu}{|\mu_\mathrm{max}| + |\mu| + \epsilon},
\label{eq:selectivity}
\end{equation}
where $\mu_\mathrm{max} = \mathrm{max}\{\mu_1,\ldots,\mu_C\}$ and $\mu = \frac{1}{C}{\sum_{i\in[C]} \mu_i}$ denote the largest class-conditional mean activation and the mean activation of all classes, respectively. $\epsilon > 0$ is a small constant for numerical stability and we set $\epsilon = 1e^{-6}$ in our experiments. In practice, we compute each $\mu_i,i\in[C]$ by averaging over examples in mini-batches with a batch size of 1024 on the ImageNet validation set, and then averaging over all mini-batches. For CLIP-RN50, the reported neuron selectivity is averaged over all dimensions of the output of the last attention pooling layer. For CLIP-ViT-B/16, the reported neuron selectivity is averaged over all dimensions of the output of every GELU layer in the last attention block.

\begin{figure}
\centering
\includegraphics[width=0.999\linewidth]{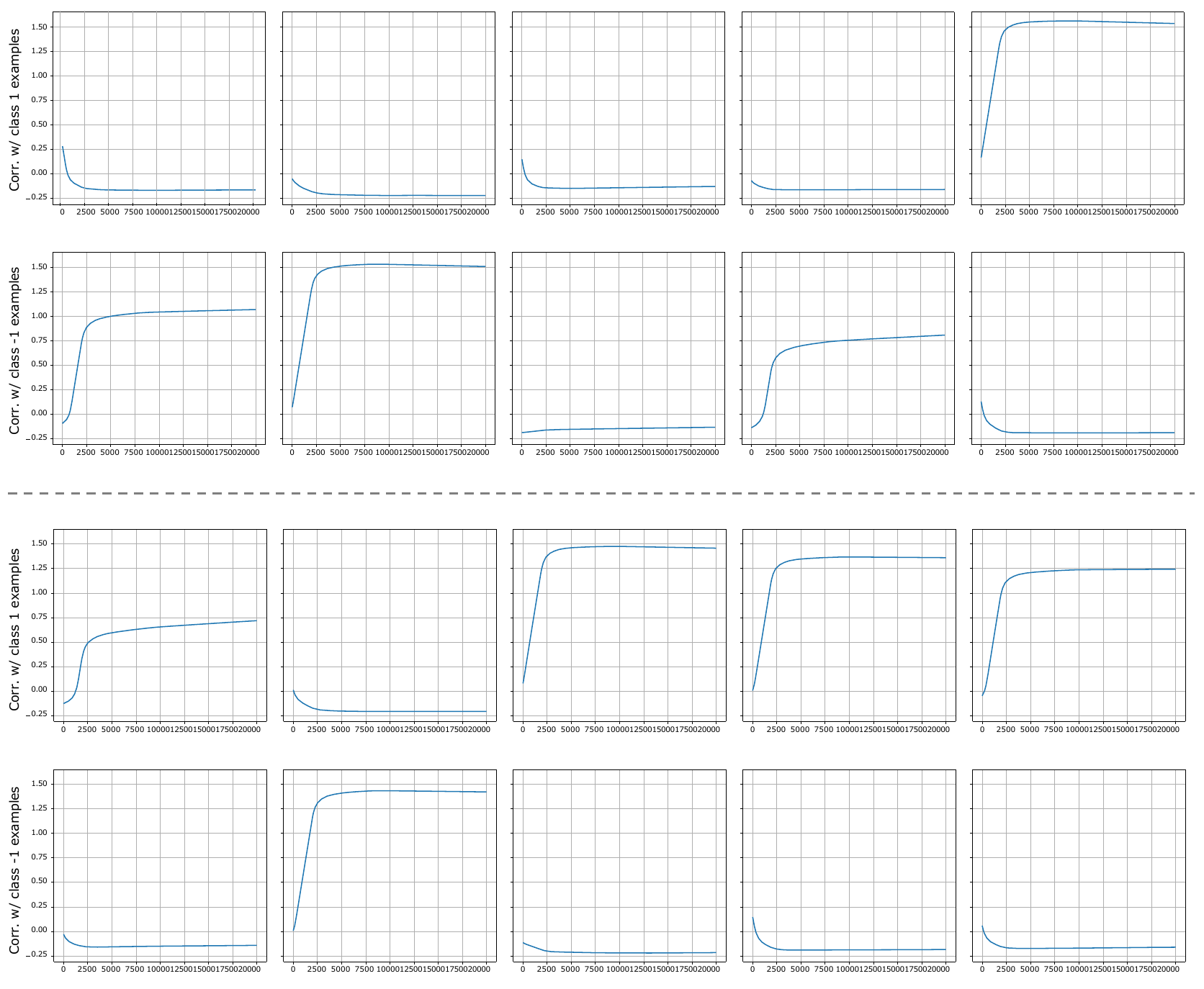}
\caption{\textbf{(Activation asymmetry)} The average correlation between 10 random neurons and examples from both classes as a function of training iterations in the classification setting. In each column, the top plot above shows the average correlation between the weight (learned feature) of the neuron and the examples from class $y=1$, while the bottom plot shows the average correlation between the weight (learned feature) of the neuron and the examples from class $y=-1$. As the training goes on, each neuron evolves to have positive correlation with at most one class of examples, resulting in activation asymmetry.}
\label{fig:neurons}
\end{figure}

\begin{figure}
\centering
\subcaptionbox{Two-layer \textbf{ReLU} network\label{subfig:relu_classification}}{\includegraphics[width=0.4\linewidth]{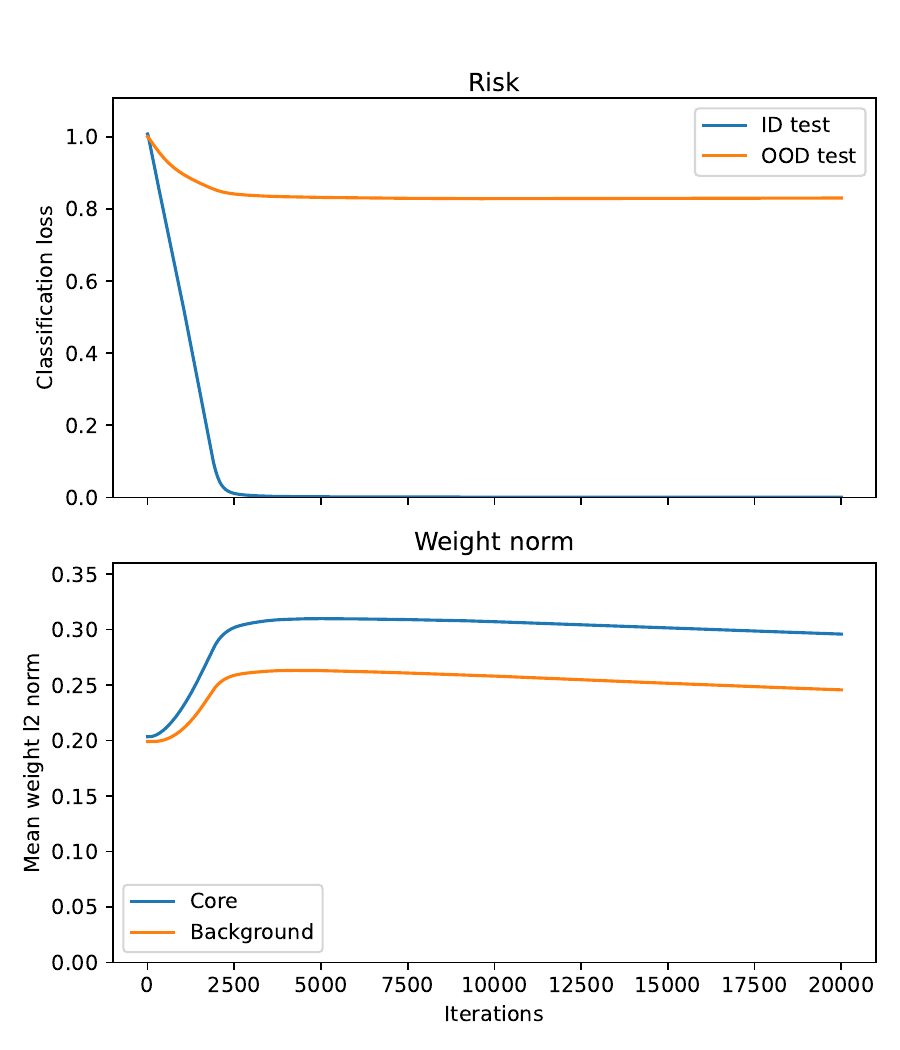}}
\hspace{2em}
\subcaptionbox{Two-layer \textbf{linear} network\label{subfig:linear_classification}}{\includegraphics[width=0.4\linewidth]{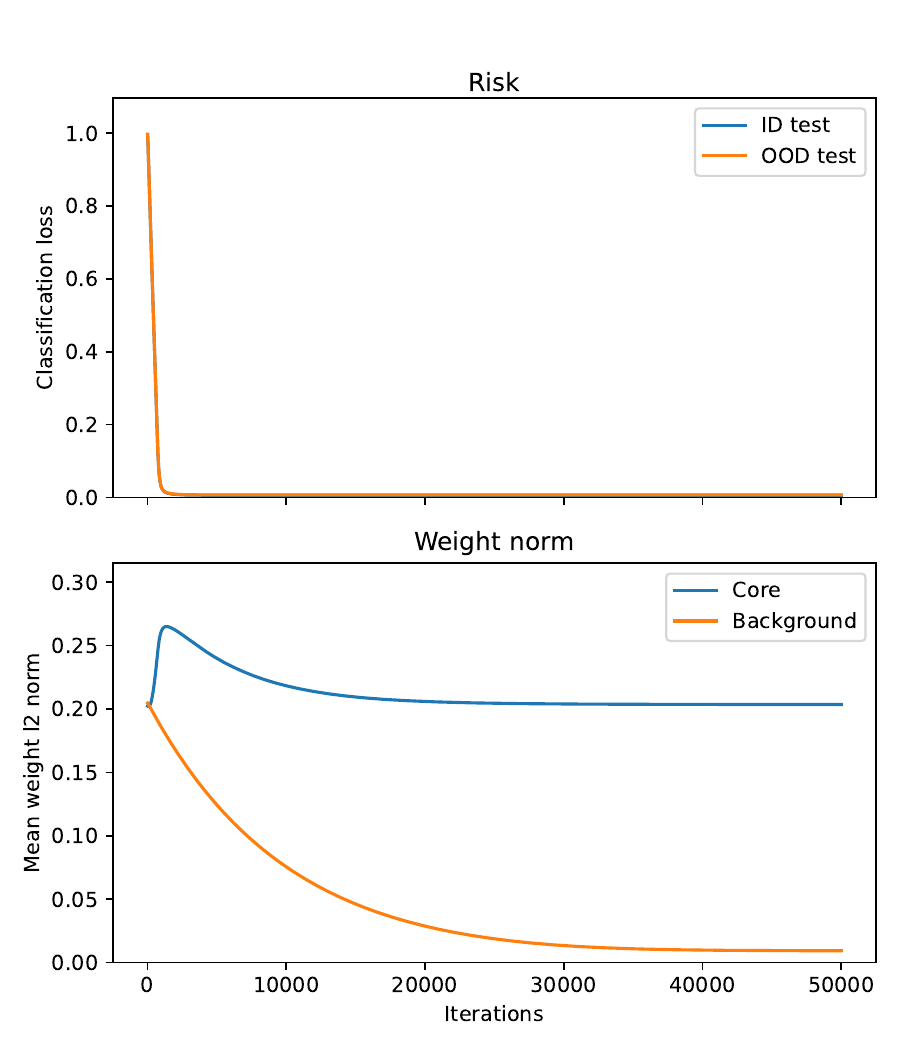}}
\caption{The ID and OOD risks (\textbf{top}) and the norm of weight projections onto core and background features (\textbf{bottom}) in the \textbf{classification} setting.}
\label{fig:classification}
\end{figure}

\begin{figure}
\centering
\subcaptionbox{Two-layer \textbf{ReLU} network\label{subfig:relu_regression}}{\includegraphics[width=0.4\linewidth]{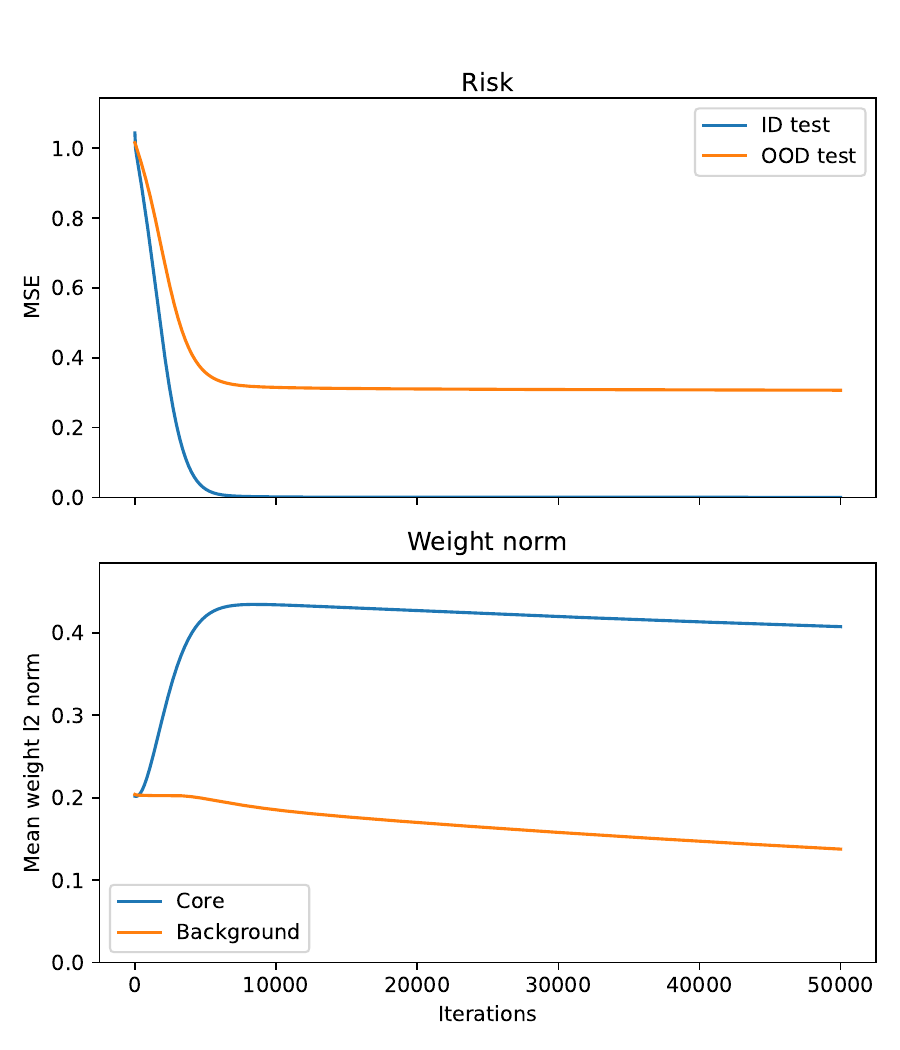}}
\hspace{2em}
\subcaptionbox{Two-layer \textbf{linear} network\label{subfig:linear_regression}}{\includegraphics[width=0.4\linewidth]{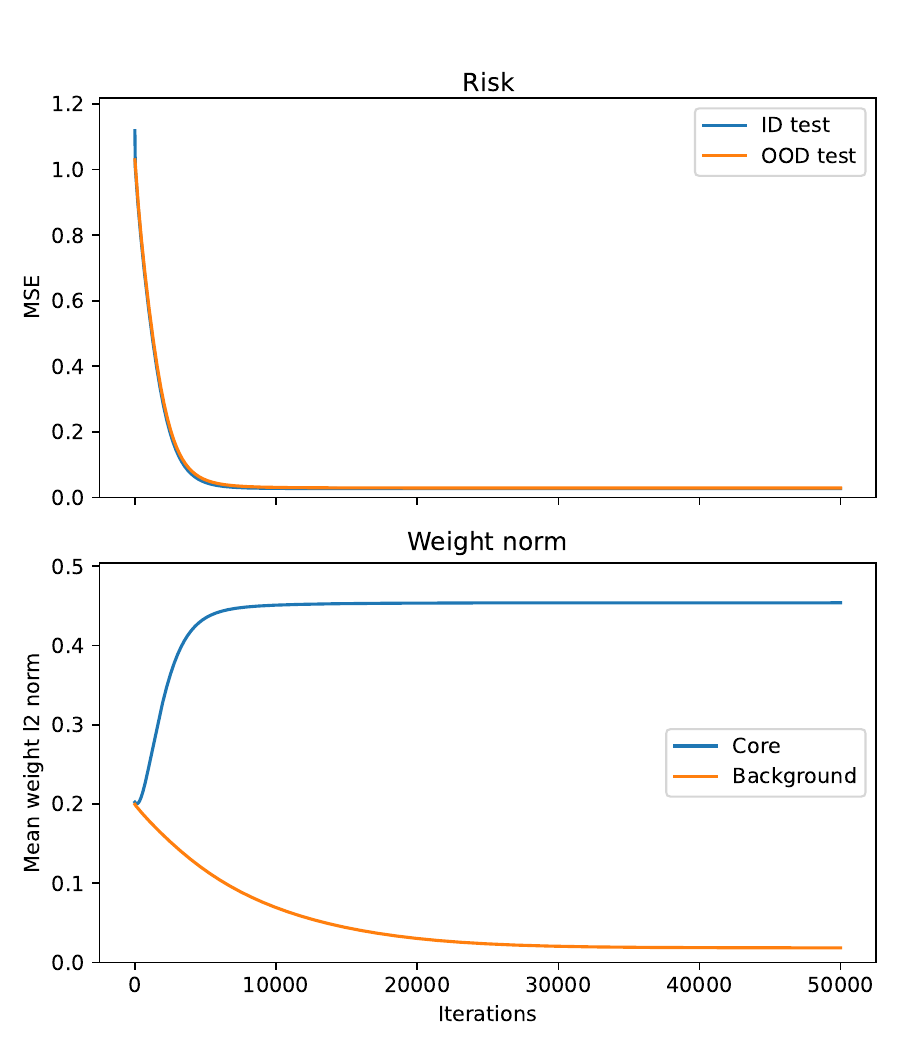}}
\caption{The ID and OOD risks (\textbf{top}) and the norm of weight projections onto core and background features (\textbf{bottom}) in the \textbf{regression} setting.}
\label{fig:regression}
\end{figure}

\begin{figure}
\centering
\subcaptionbox{ReLU\label{subfig:numerical_relu}}{\includegraphics[width=0.48\linewidth]{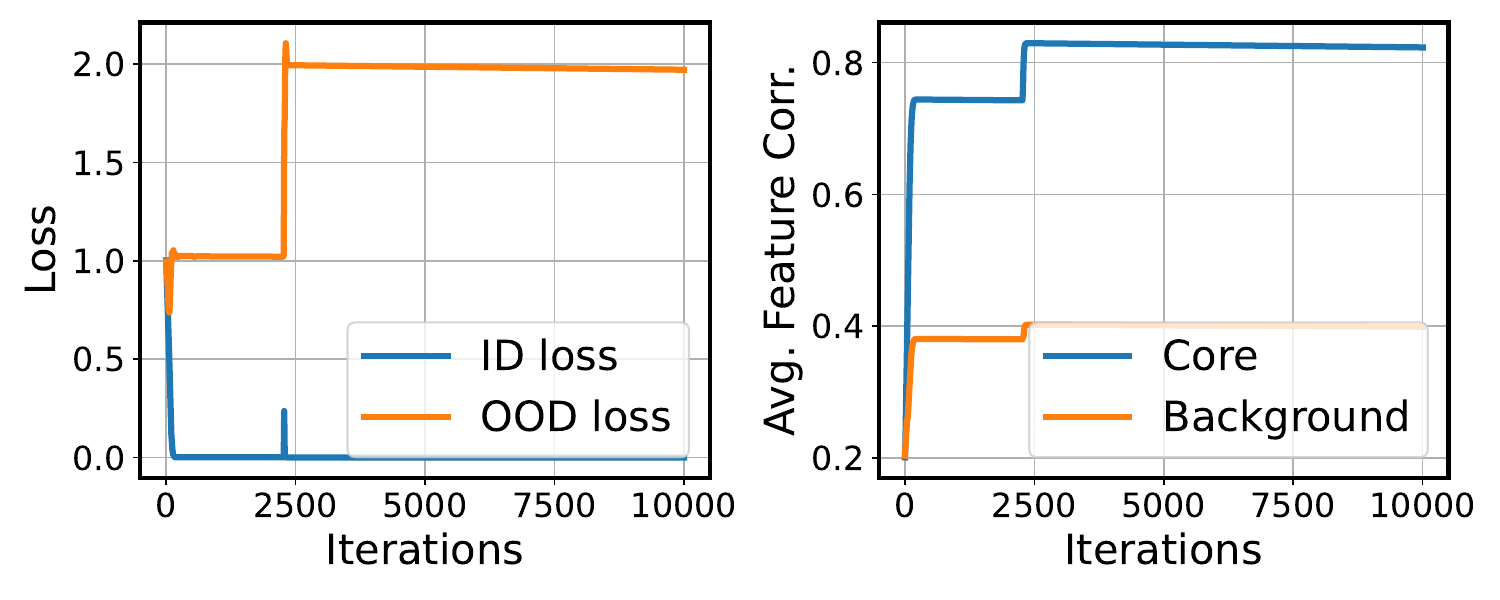}}
\hspace{1em}
\subcaptionbox{GELU\label{subfig:numerical_gelu}}{\includegraphics[width=0.48\linewidth]{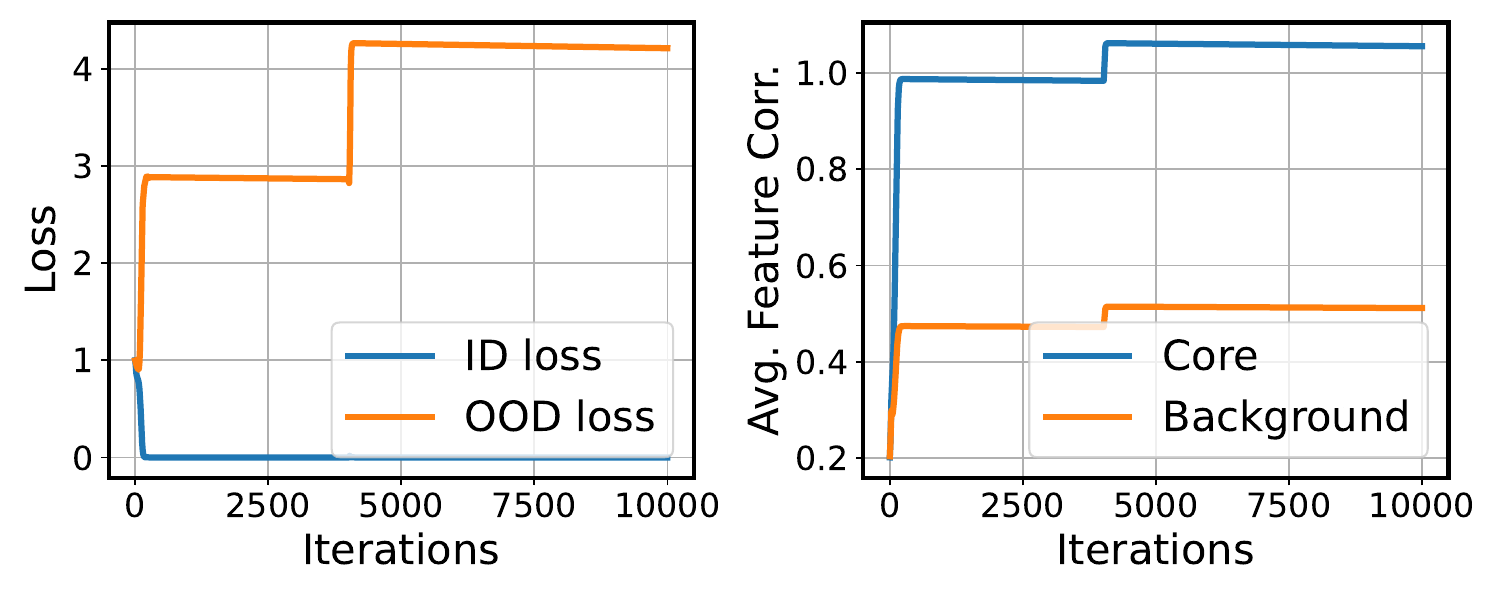}}\vspace{1em}\\
\subcaptionbox{Sigmoid\label{subfig:numerical_sigmoid}}{\includegraphics[width=0.48\linewidth]{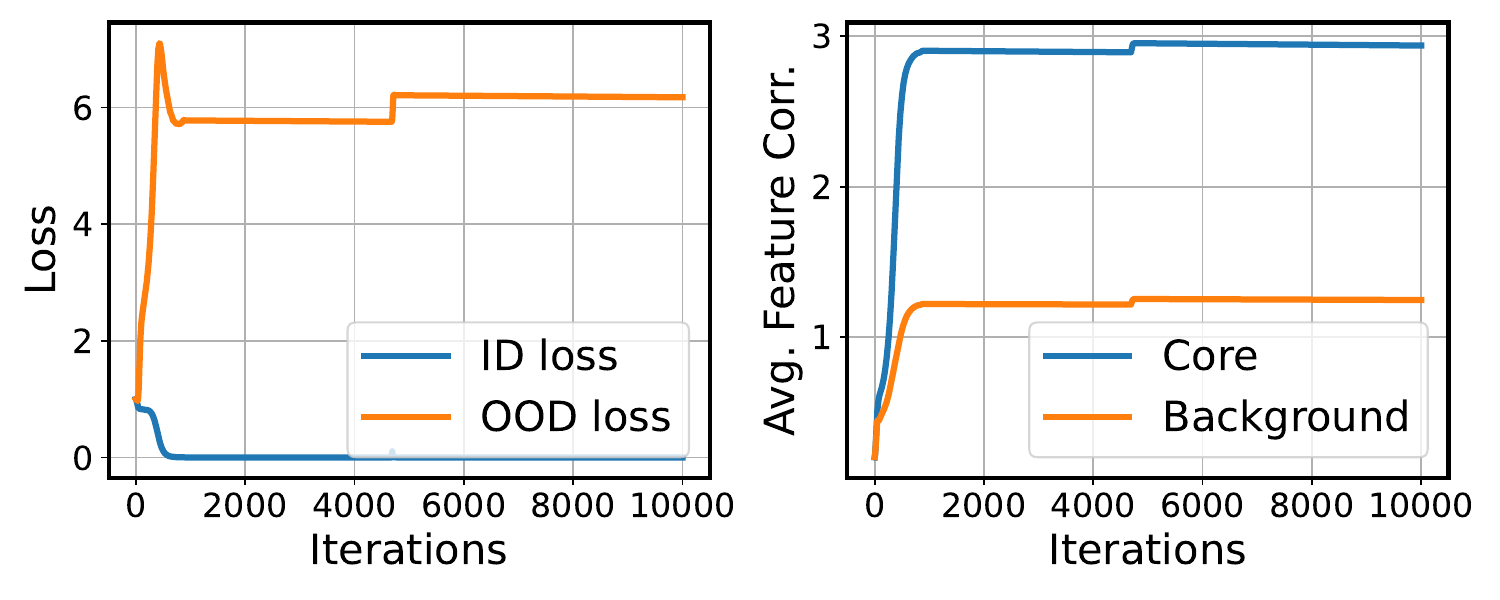}}
\hspace{1em}
\subcaptionbox{Tanh\label{subfig:numerical_tanh}}{\includegraphics[width=0.48\linewidth]{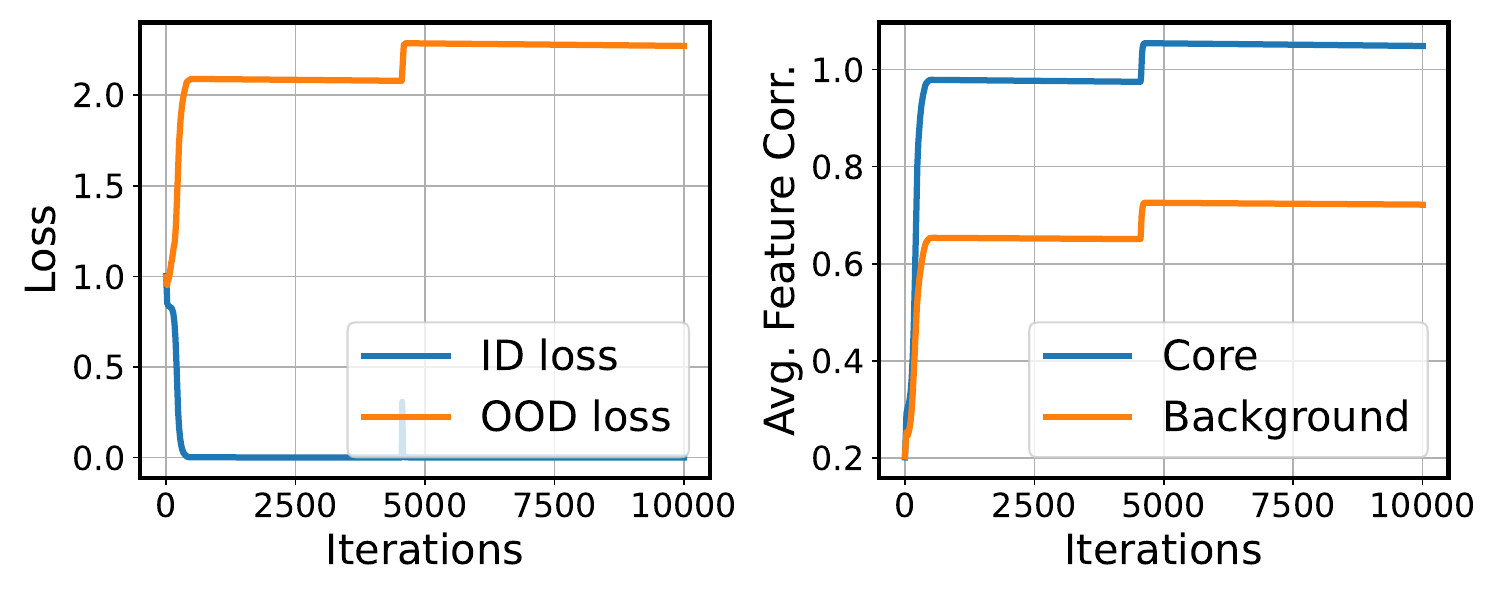}}
\caption{Complete numerical results of the ID/OOD loss and average weight-feature correlations for \textbf{different activation functions}. Feature contamination occurs for all activation functions considered, resulting in large OOD loss.}
\label{fig:numerical_all}
\end{figure}



\begin{figure}
\centering
\includegraphics[width=0.21\linewidth]{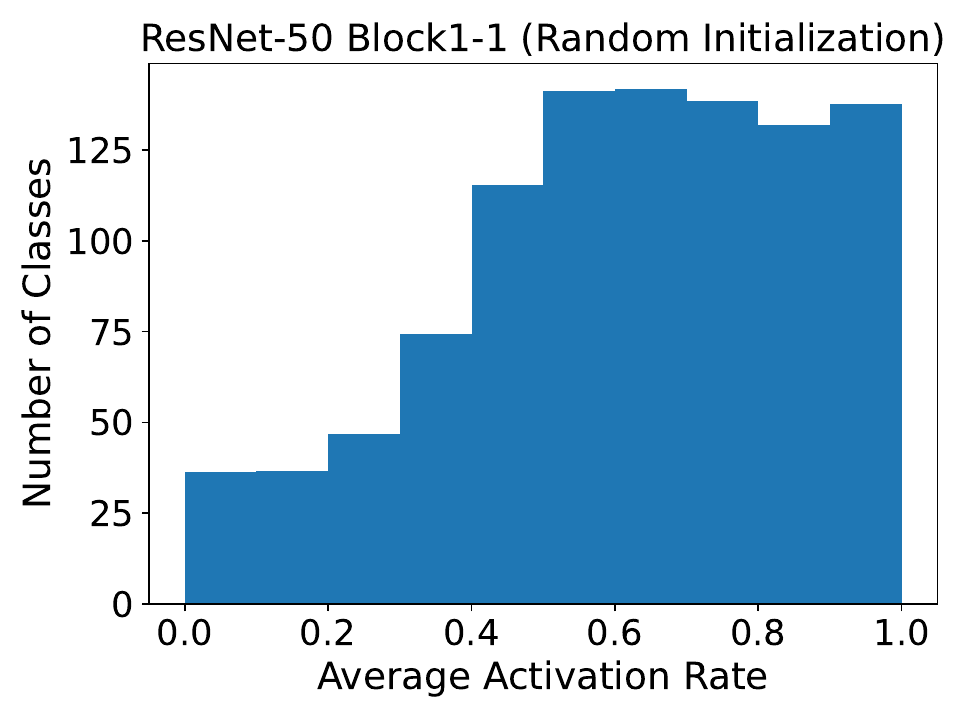}
\includegraphics[width=0.21\linewidth]{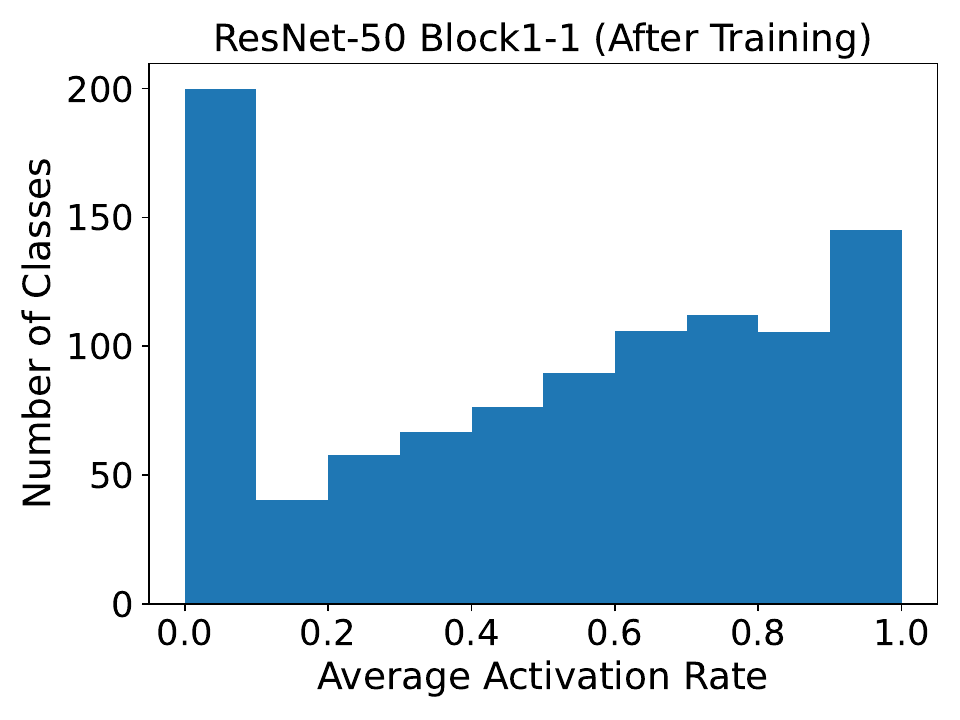}
\includegraphics[width=0.21\linewidth]{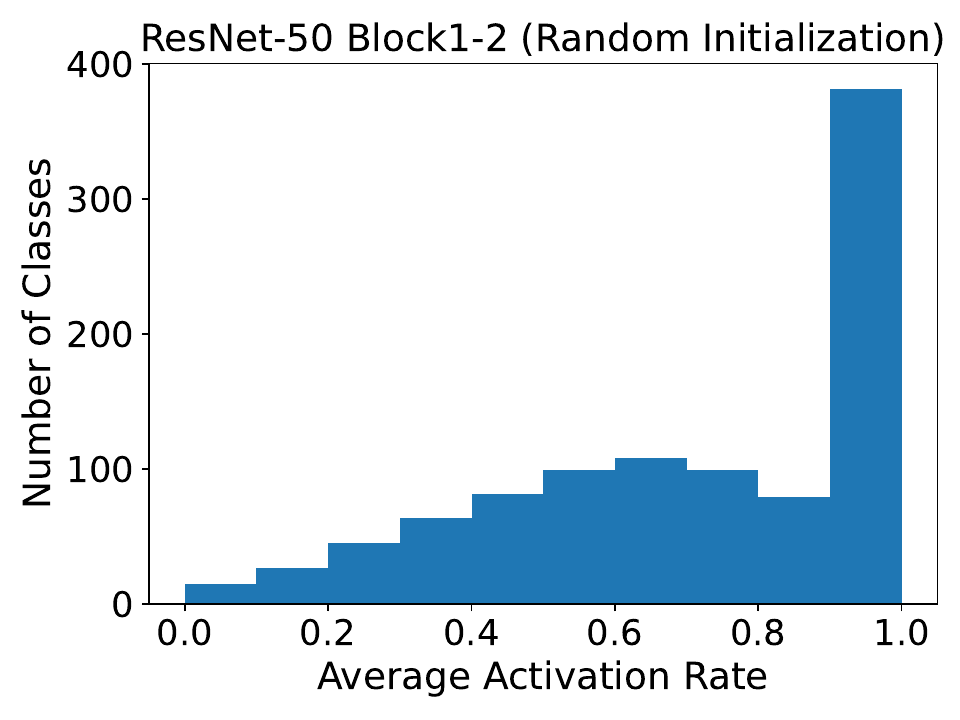}
\includegraphics[width=0.21\linewidth]{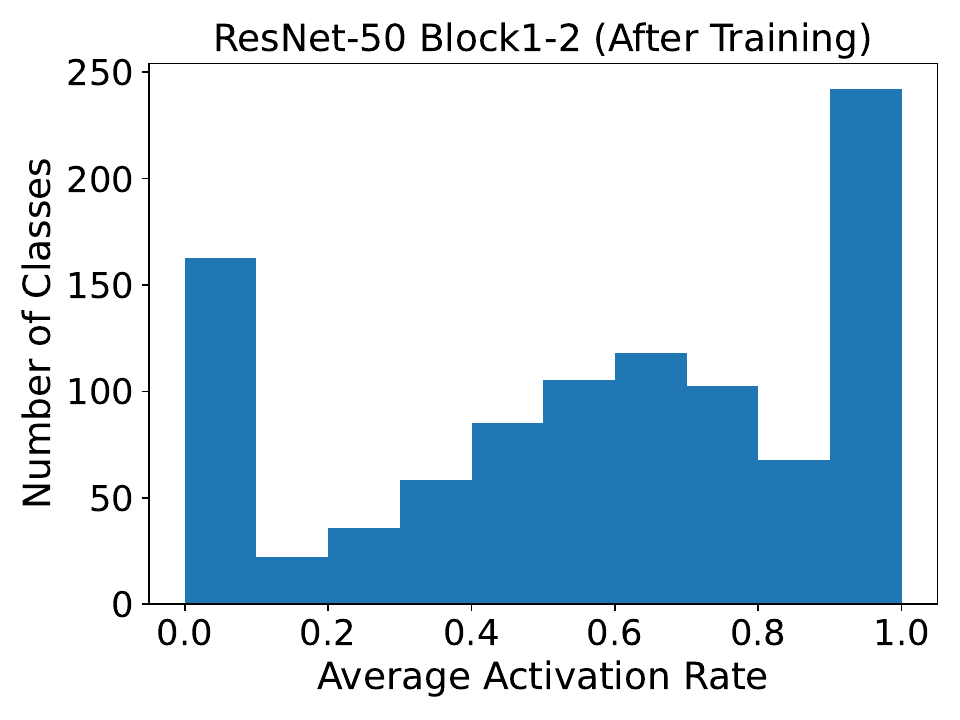} \\

\includegraphics[width=0.21\linewidth]{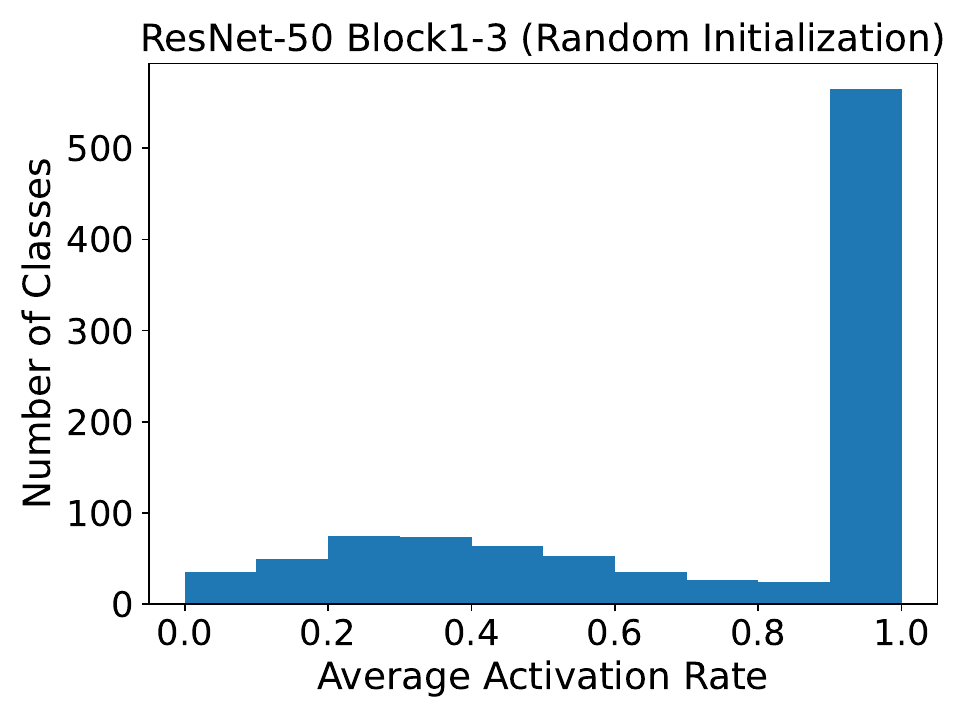}
\includegraphics[width=0.21\linewidth]{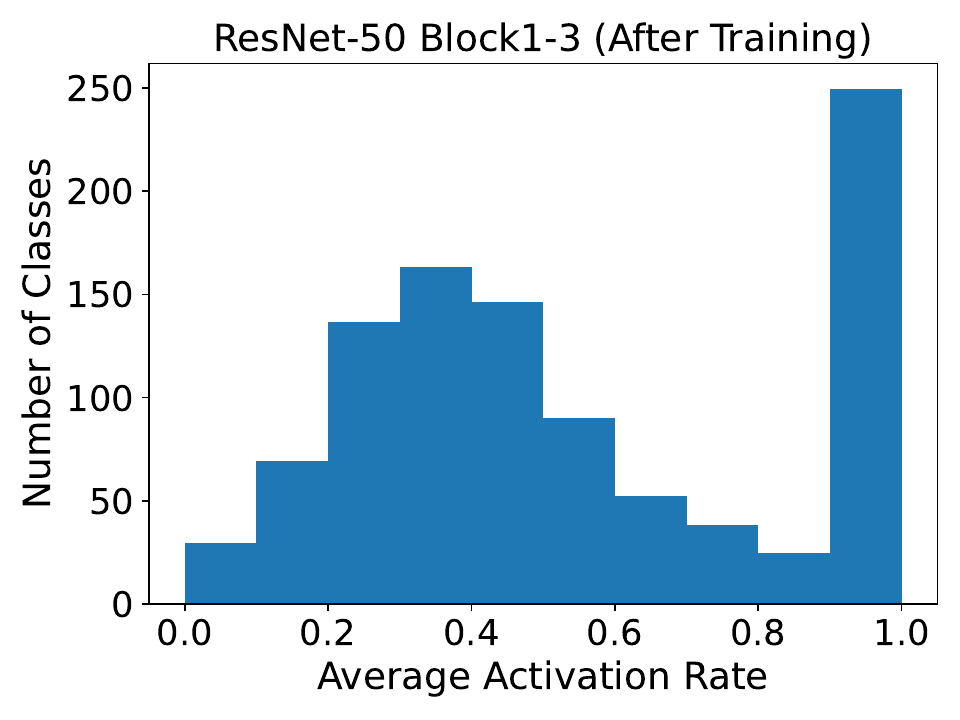}
\includegraphics[width=0.21\linewidth]{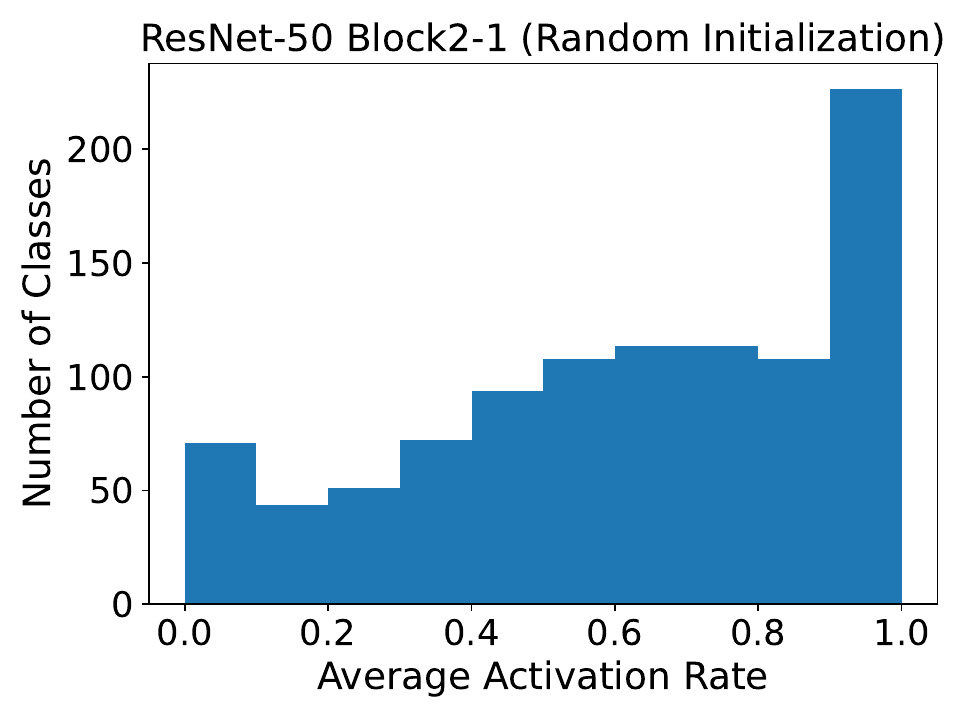}
\includegraphics[width=0.21\linewidth]{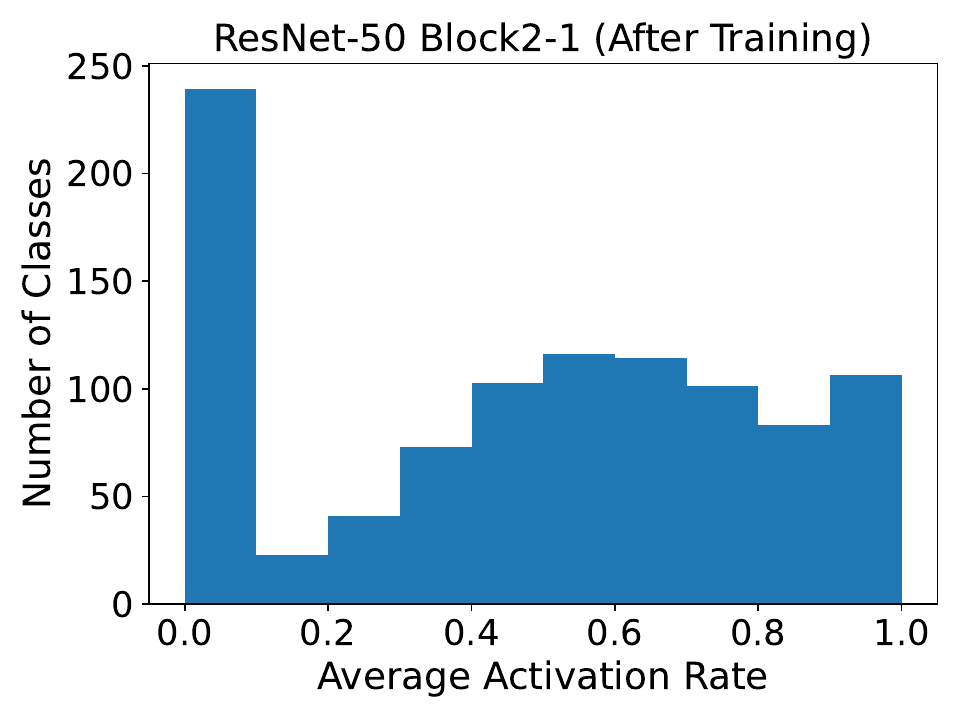} \\

\includegraphics[width=0.21\linewidth]{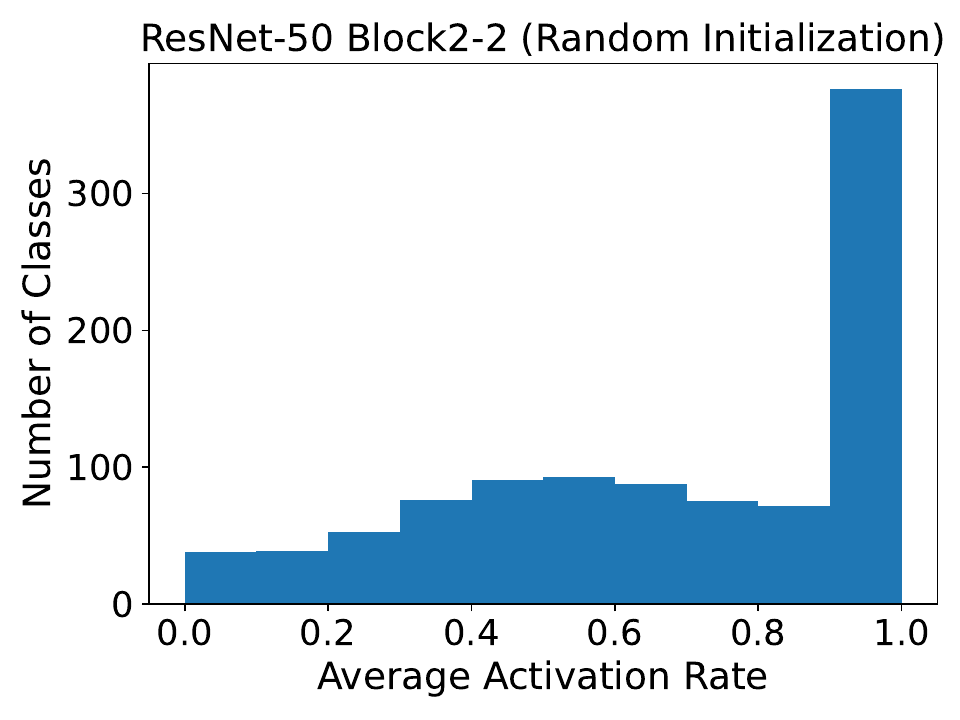}
\includegraphics[width=0.21\linewidth]{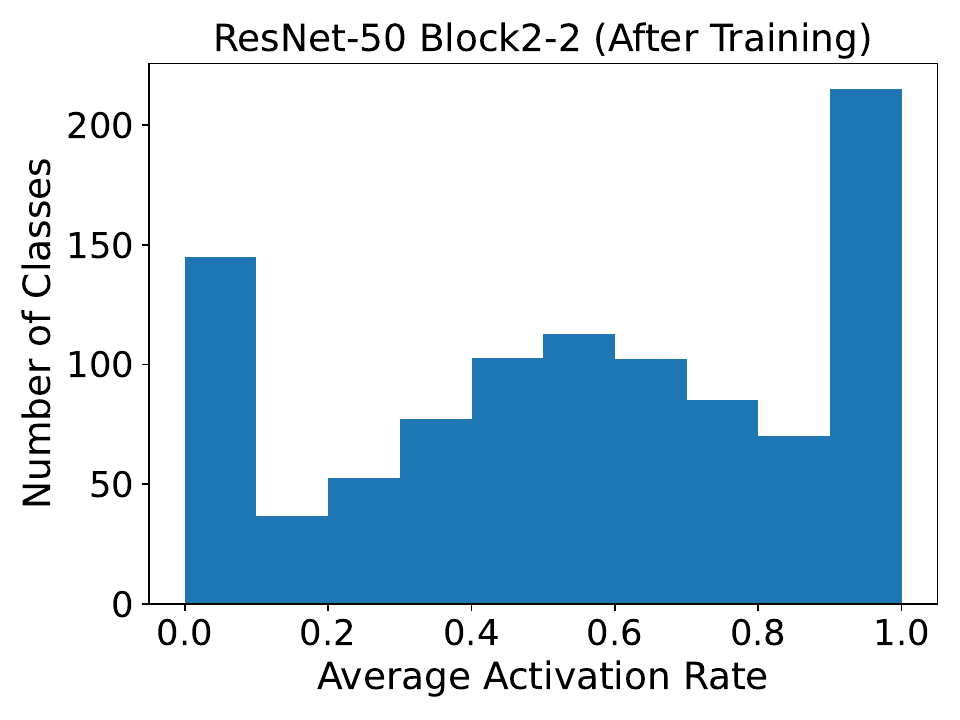}
\includegraphics[width=0.21\linewidth]{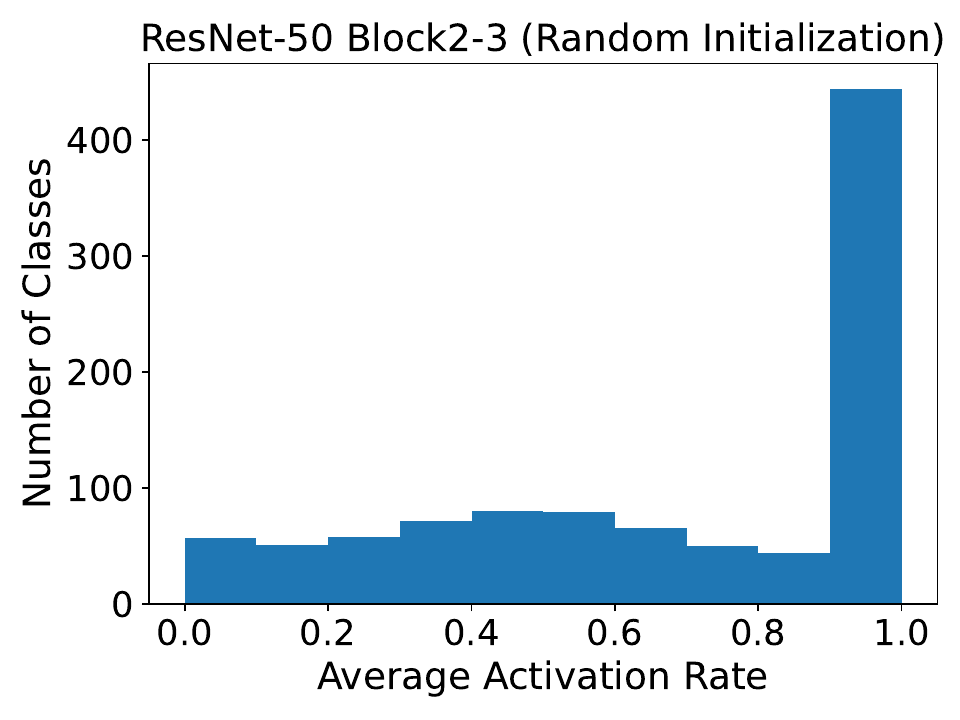}
\includegraphics[width=0.21\linewidth]{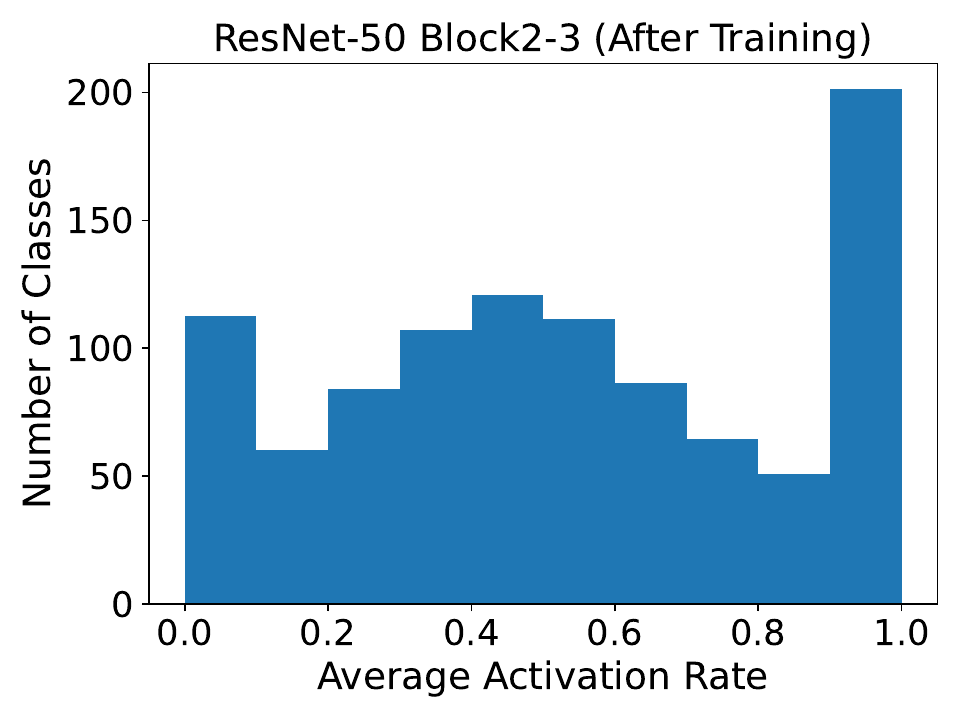} \\

\includegraphics[width=0.21\linewidth]{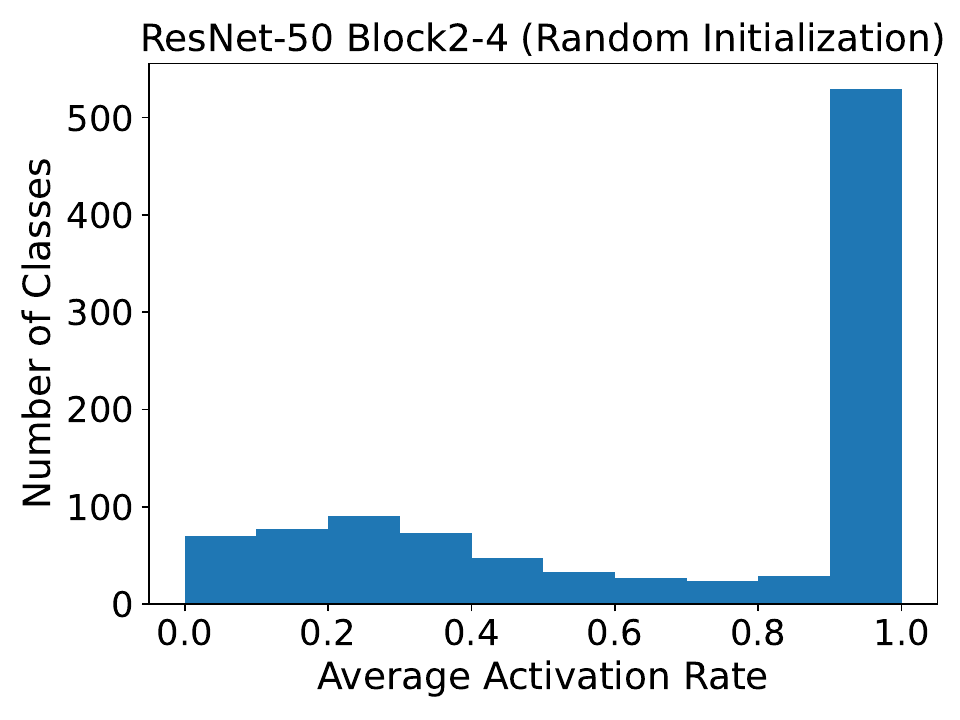}
\includegraphics[width=0.21\linewidth]{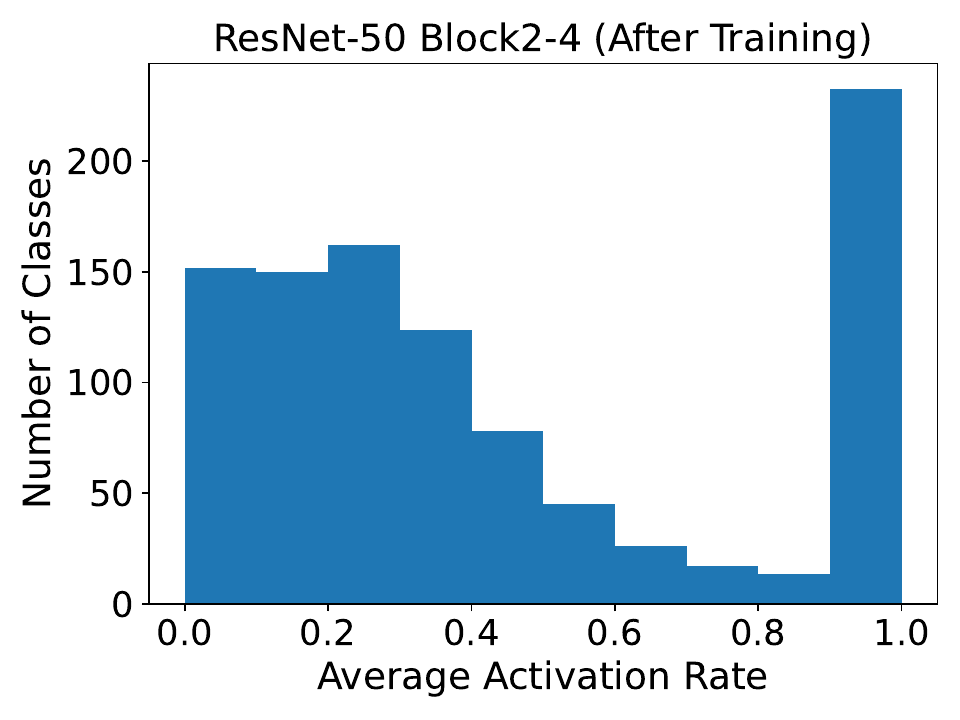}
\includegraphics[width=0.21\linewidth]{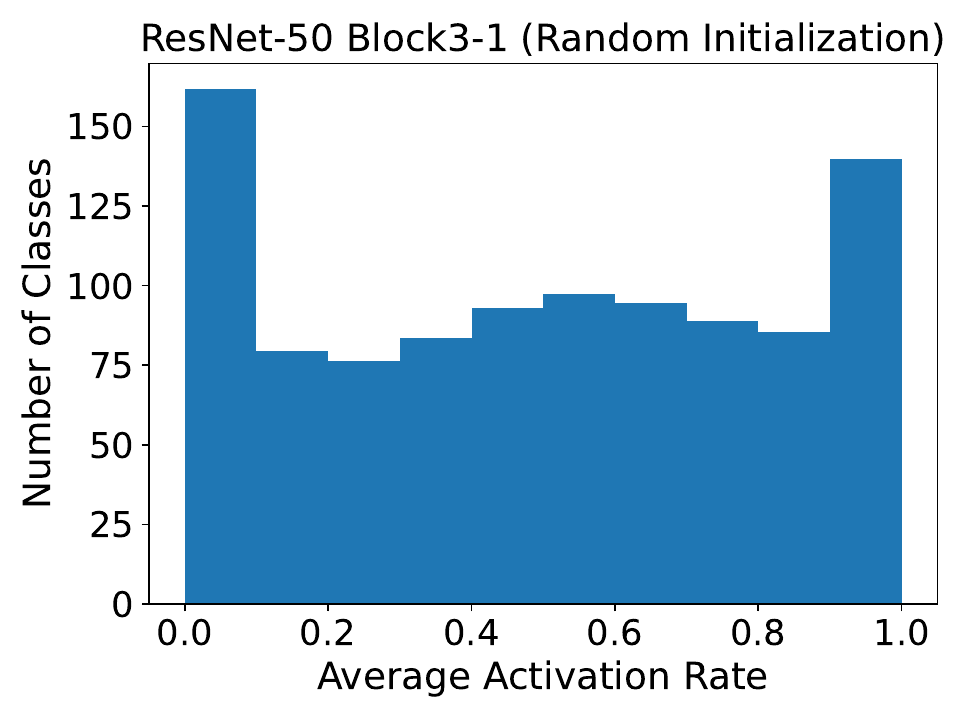}
\includegraphics[width=0.21\linewidth]{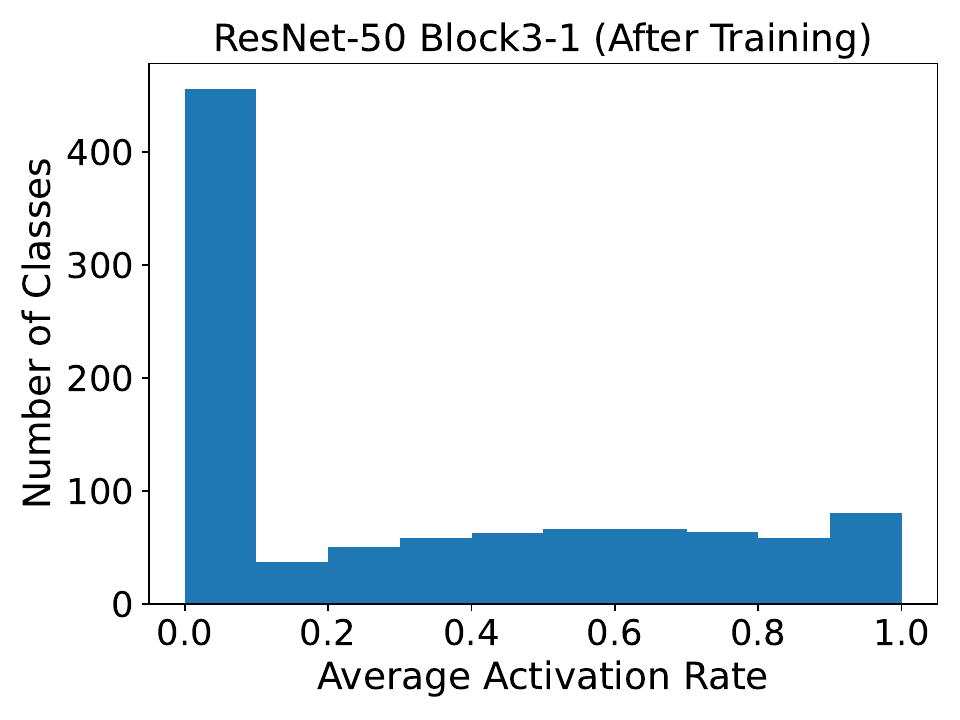} \\

\includegraphics[width=0.21\linewidth]{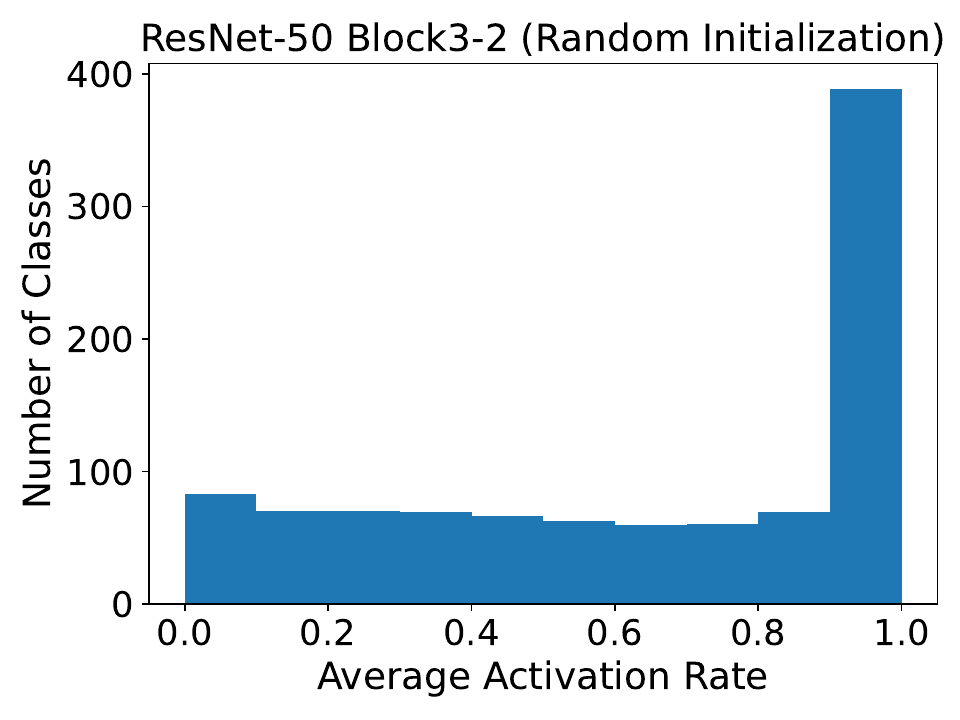}
\includegraphics[width=0.21\linewidth]{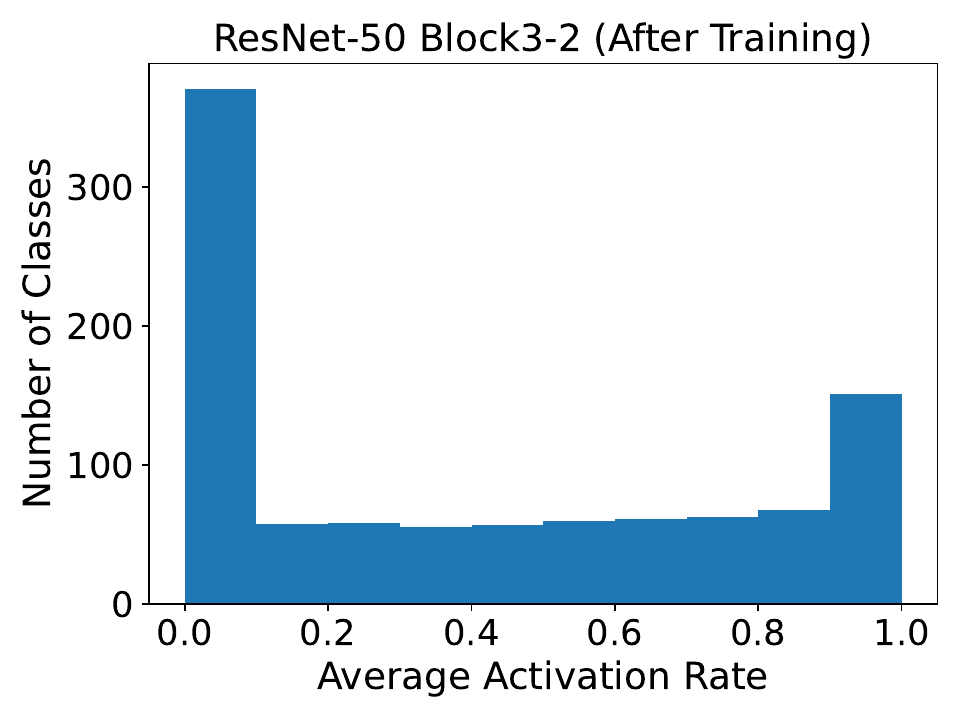}
\includegraphics[width=0.21\linewidth]{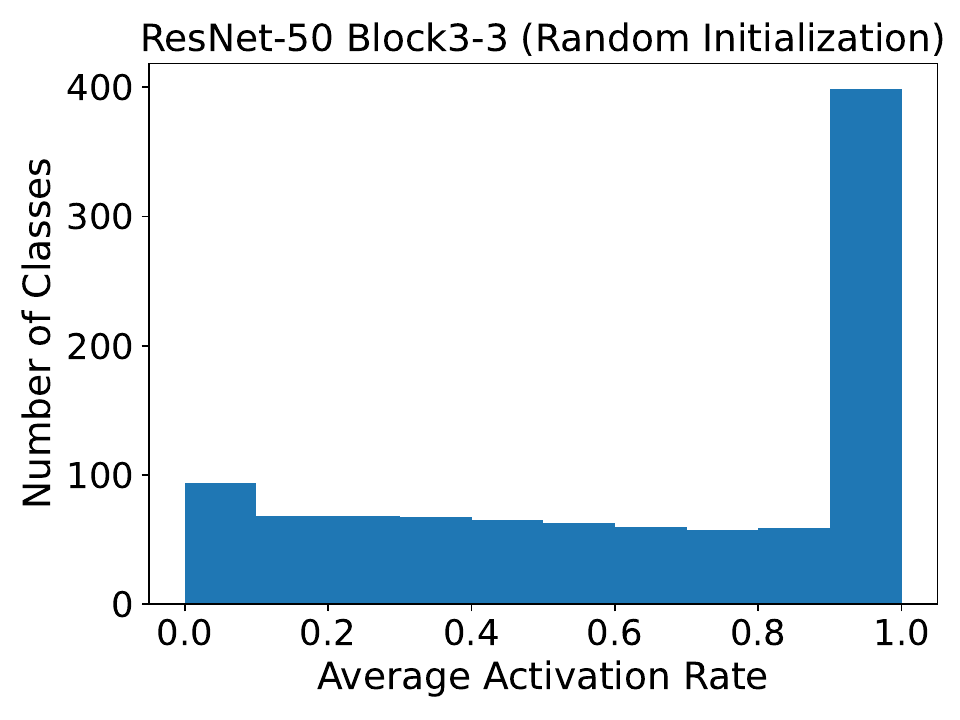}
\includegraphics[width=0.21\linewidth]{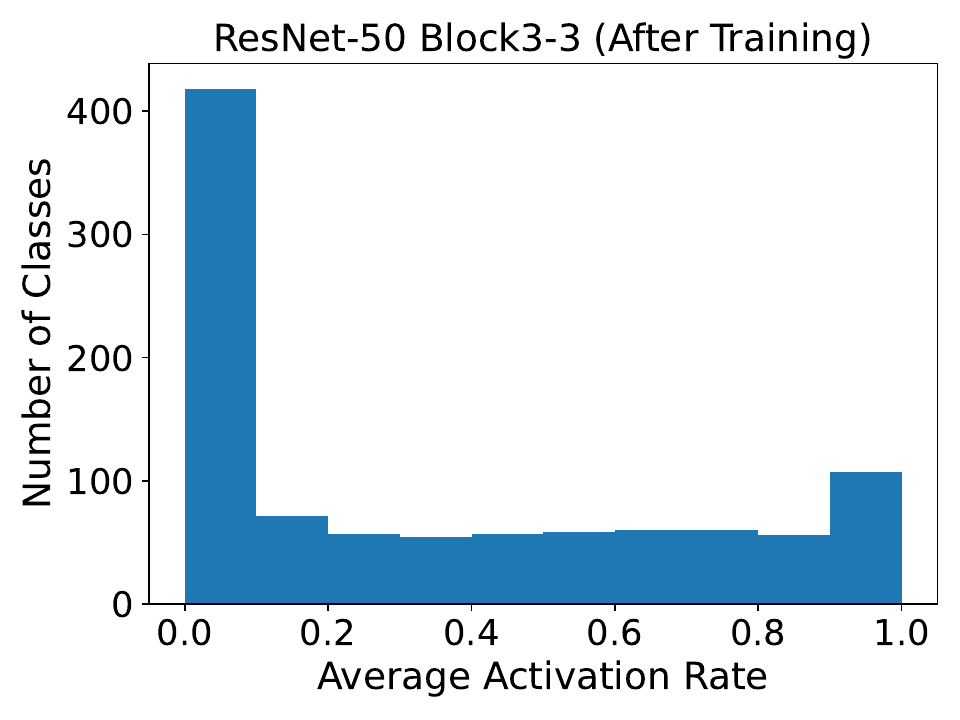} \\

\includegraphics[width=0.21\linewidth]{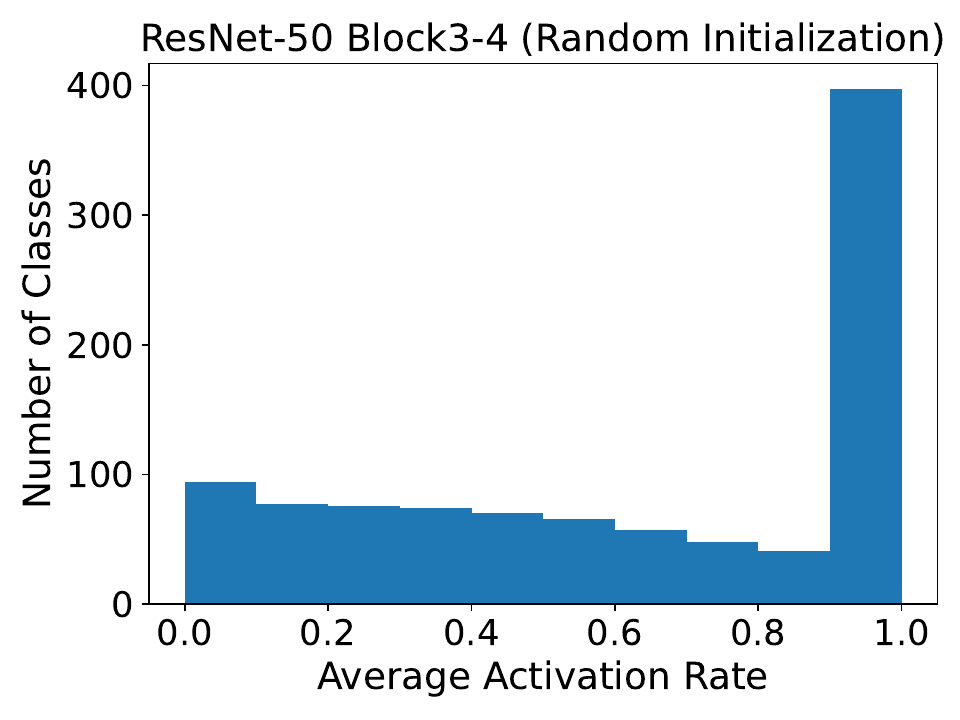}
\includegraphics[width=0.21\linewidth]{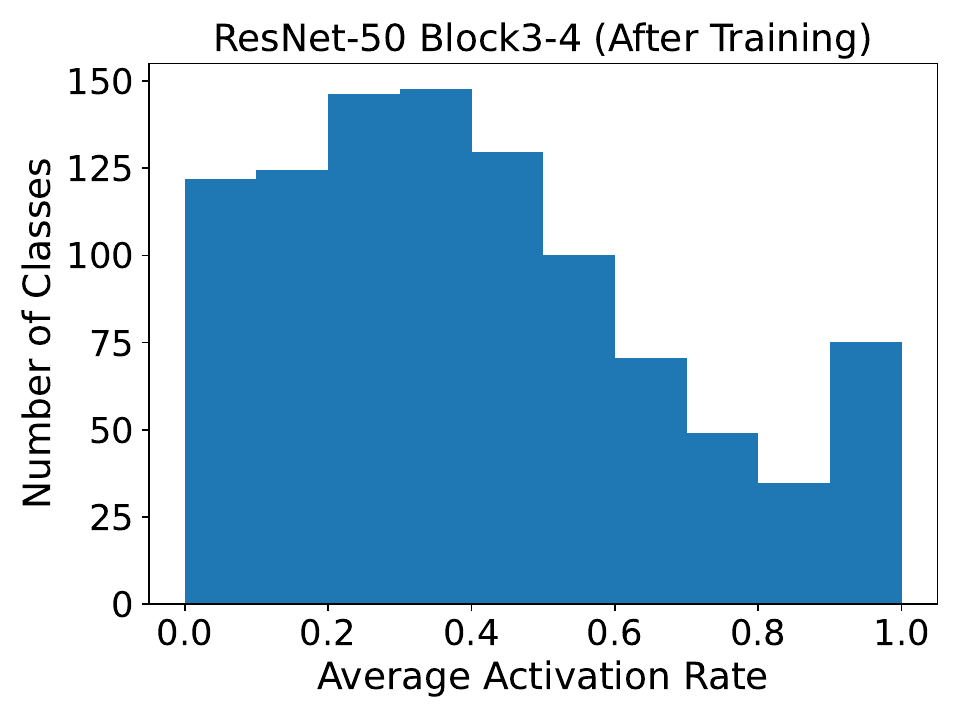}
\includegraphics[width=0.21\linewidth]{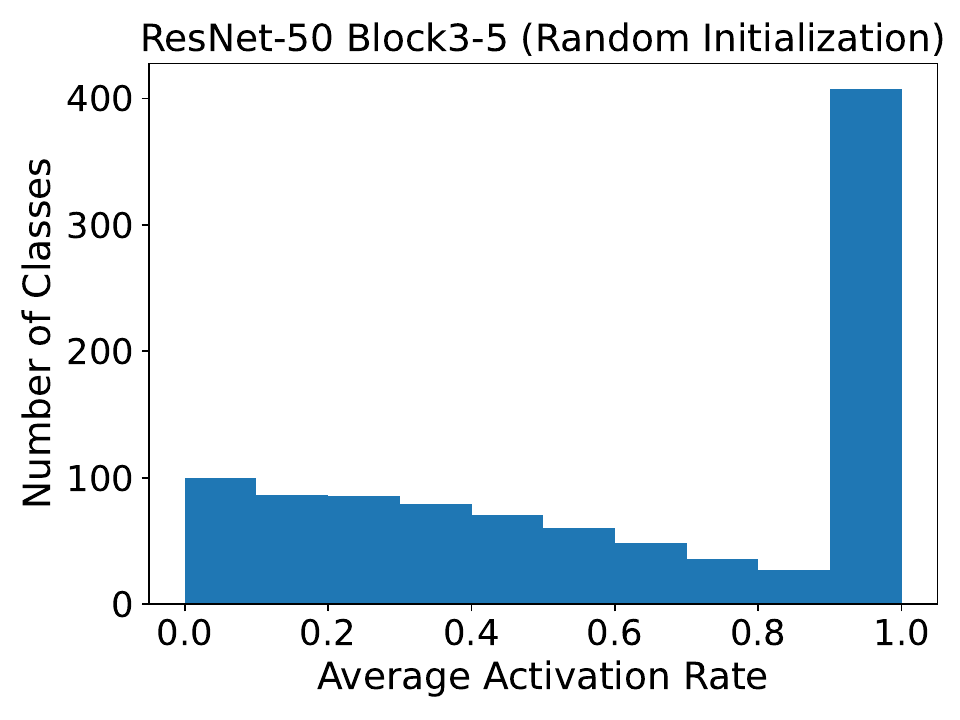}
\includegraphics[width=0.21\linewidth]{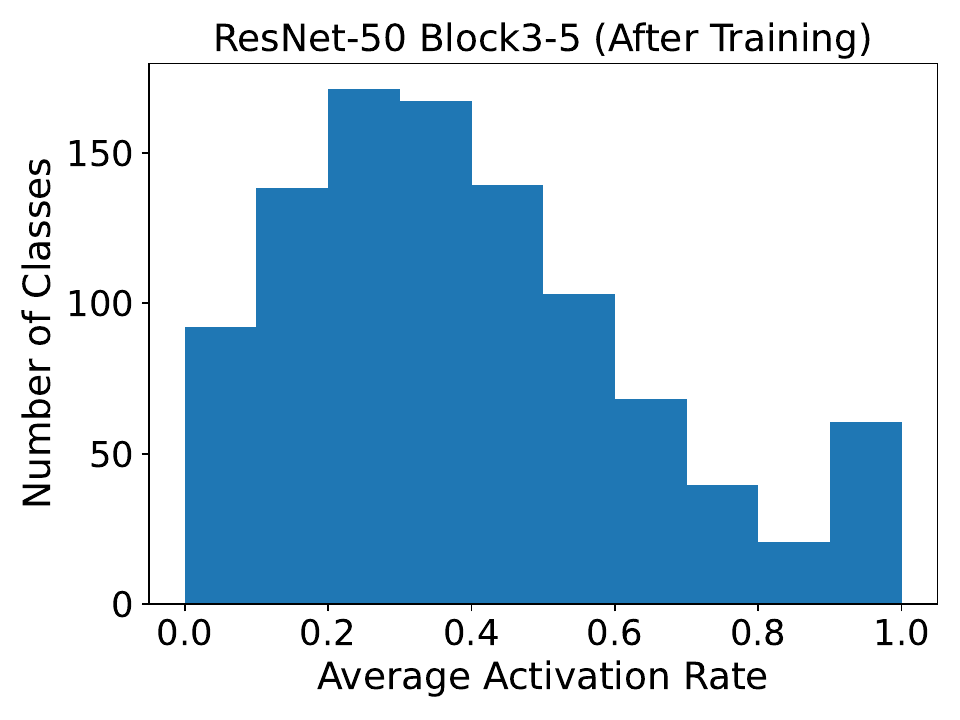} \\

\includegraphics[width=0.21\linewidth]{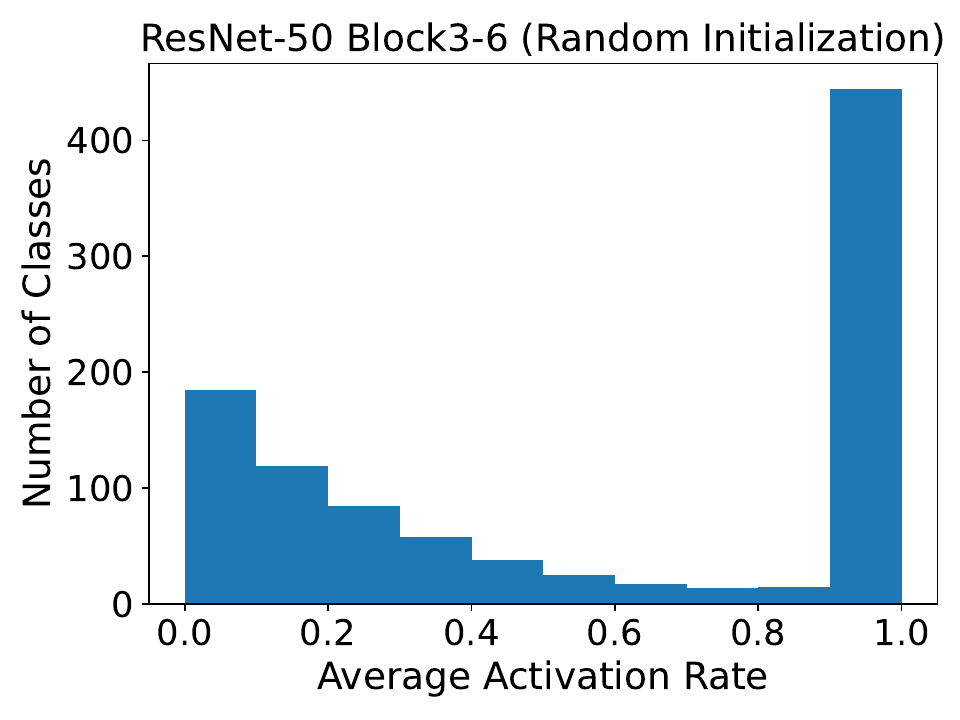}
\includegraphics[width=0.21\linewidth]{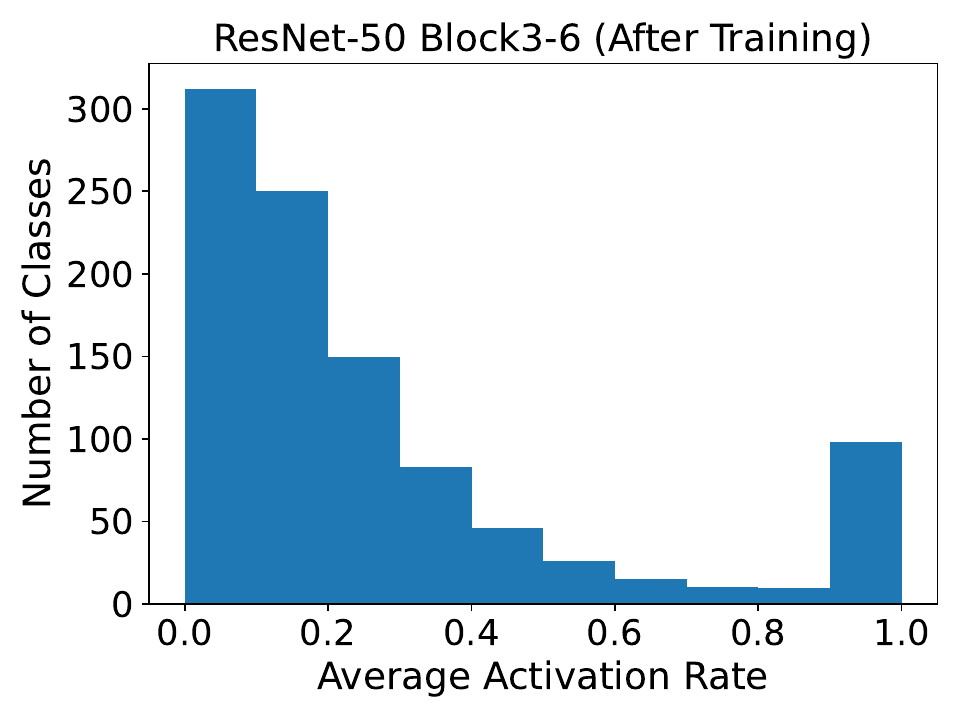}
\includegraphics[width=0.21\linewidth]{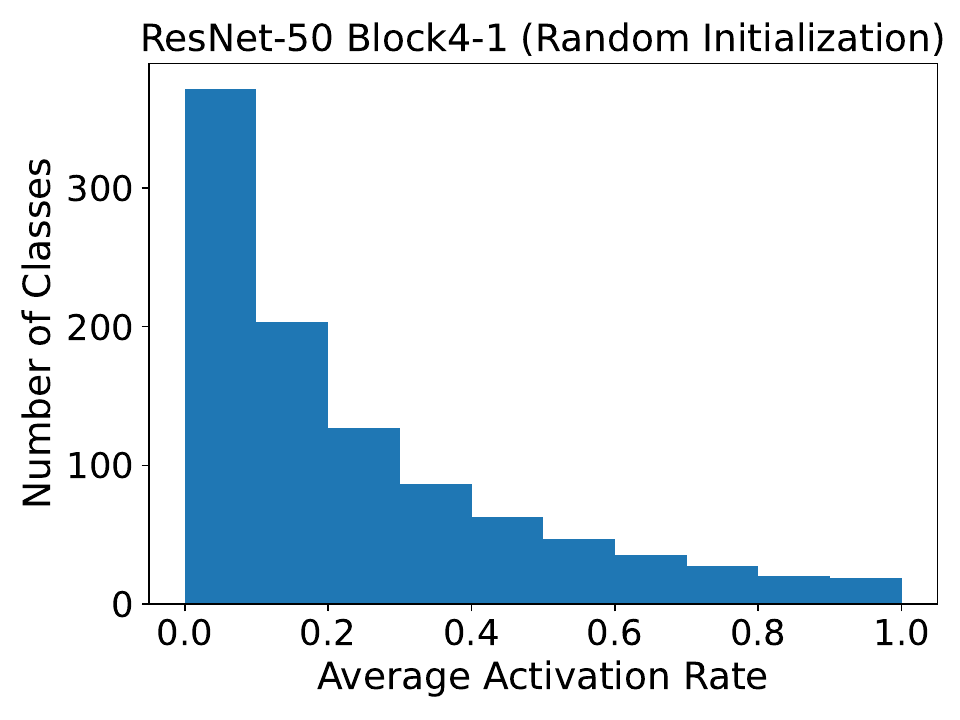}
\includegraphics[width=0.21\linewidth]{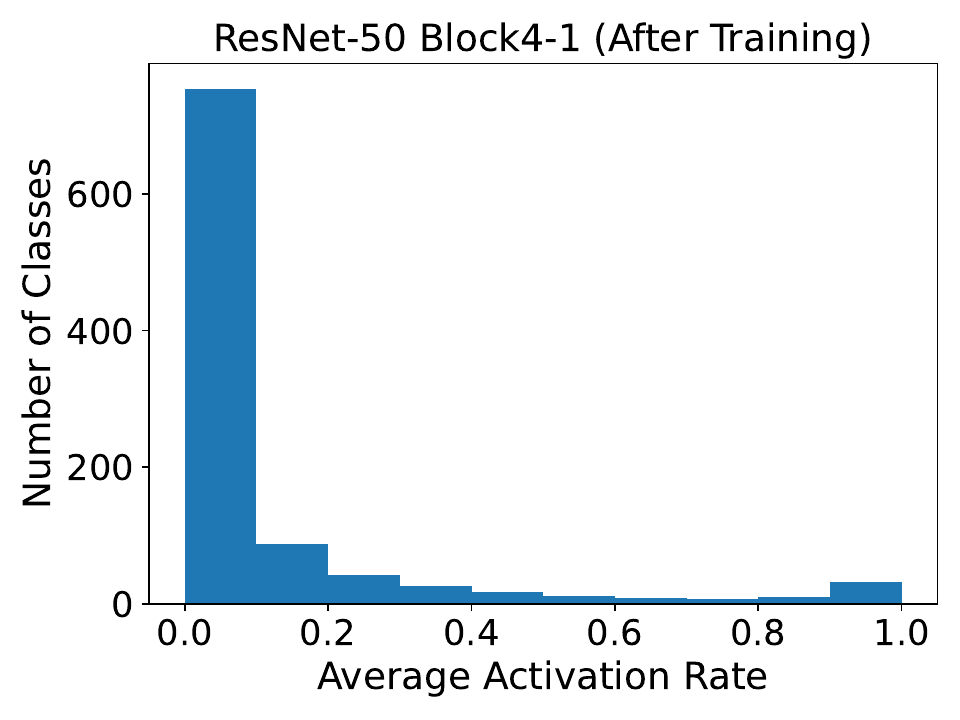} \\

\includegraphics[width=0.21\linewidth]{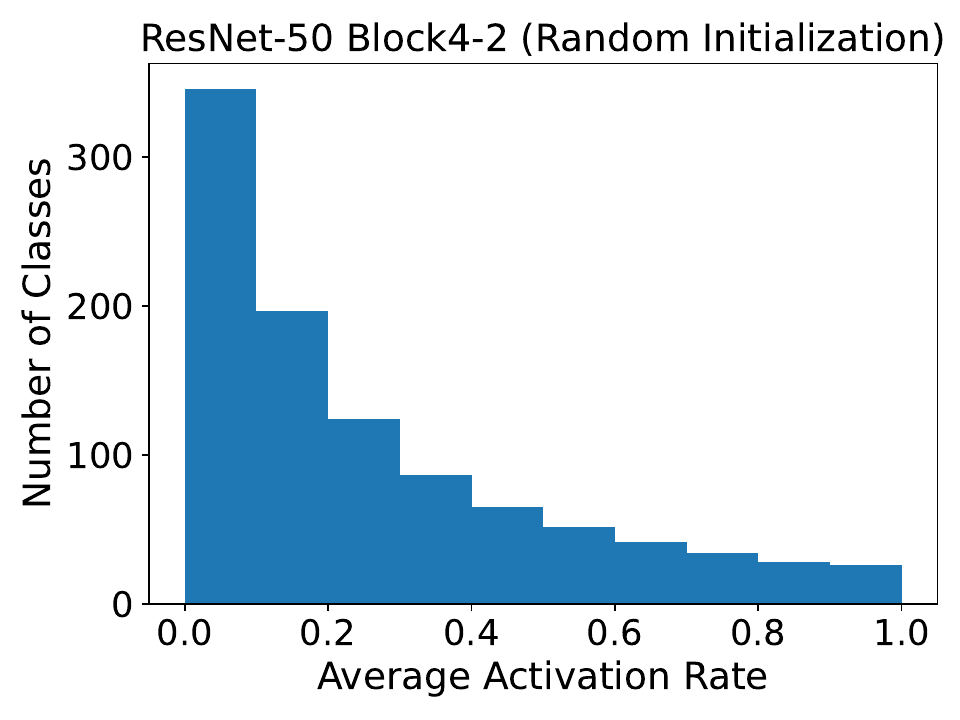}
\includegraphics[width=0.21\linewidth]{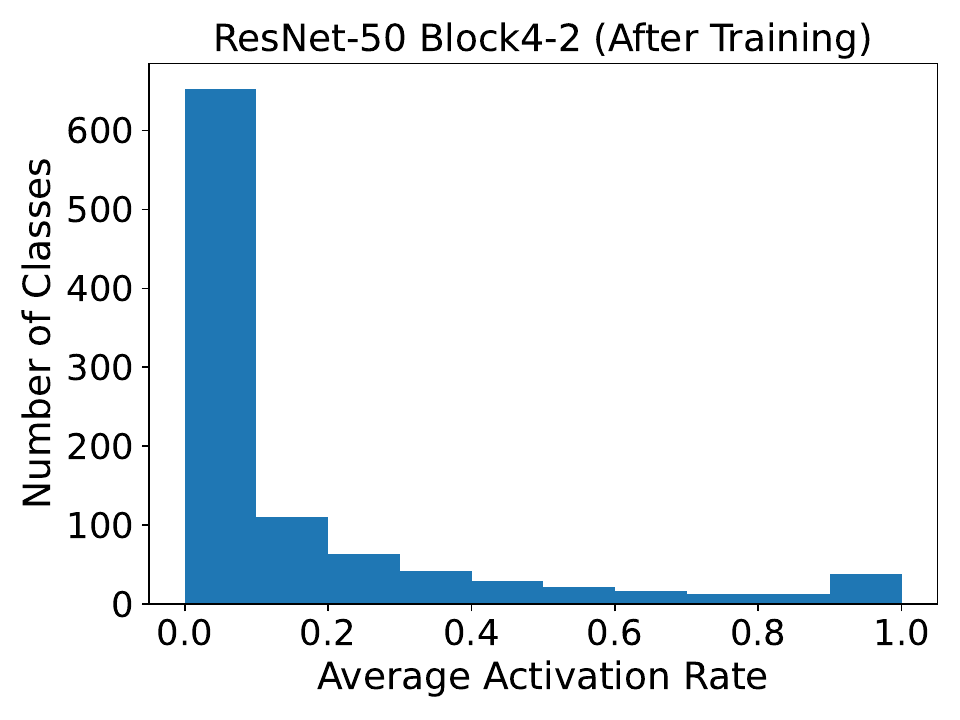}
\includegraphics[width=0.21\linewidth]{random_CLIP_RN50_stage4_block2_epoch0_rate_hist.pdf}
\includegraphics[width=0.21\linewidth]{distilled_CLIP_RN50-wd0_stage4_block2_epoch10_rate_hist.pdf}

\caption{Activation rate histograms of all blocks in randomly initialized and trained ResNet-50 networks, computed from the ImageNet validation set.
}
\label{appfig:histograms_rn}
\end{figure}

\begin{figure}
\centering
\includegraphics[width=0.245\linewidth]{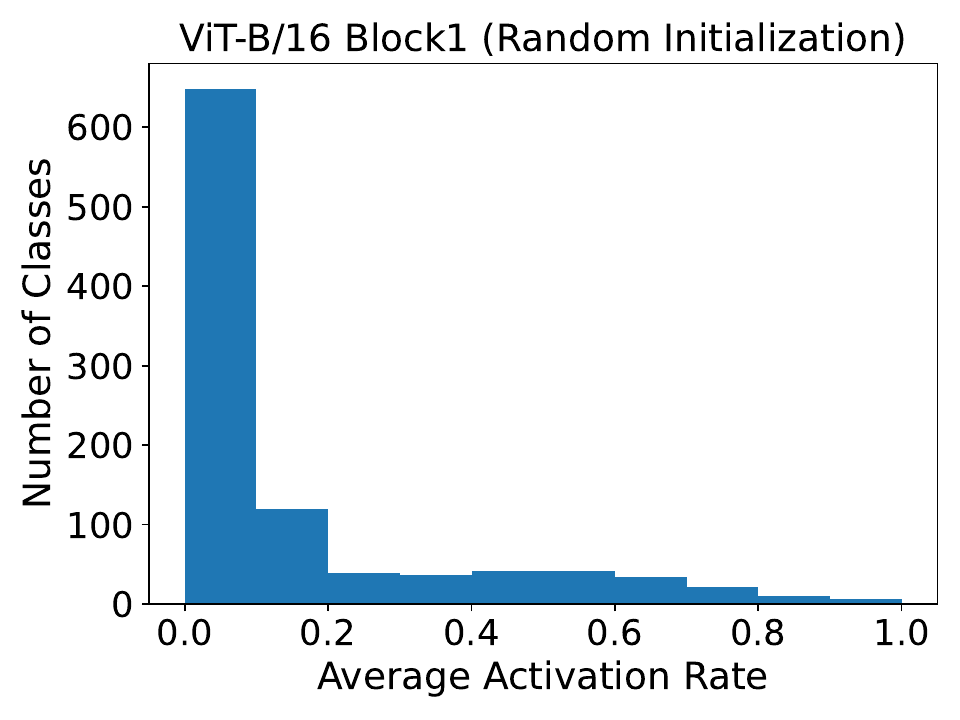}
\includegraphics[width=0.245\linewidth]{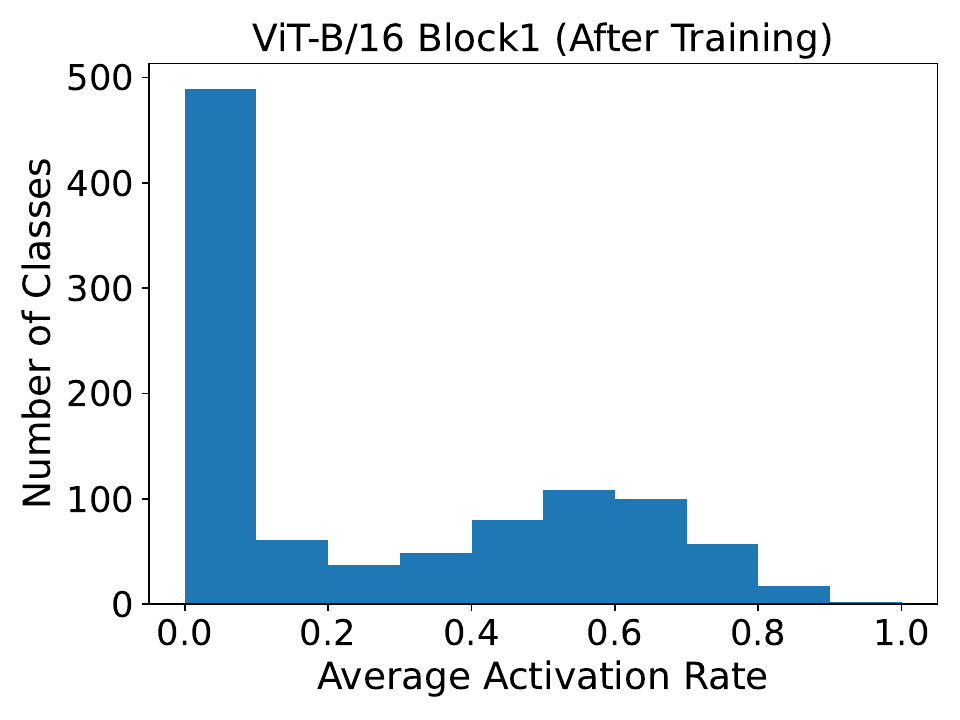}
\includegraphics[width=0.245\linewidth]{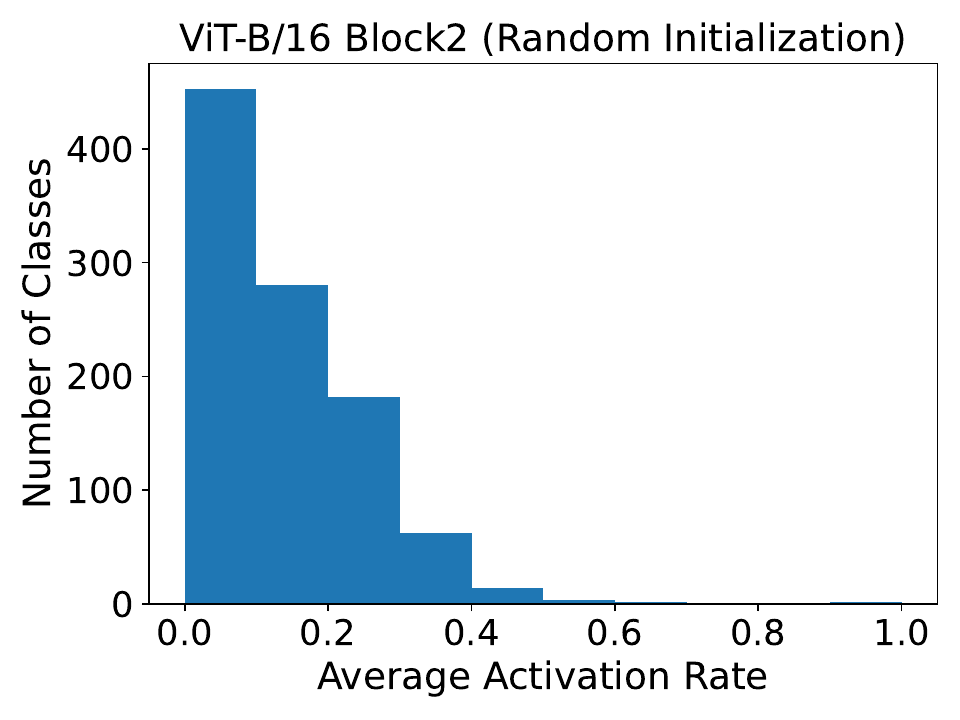}
\includegraphics[width=0.245\linewidth]{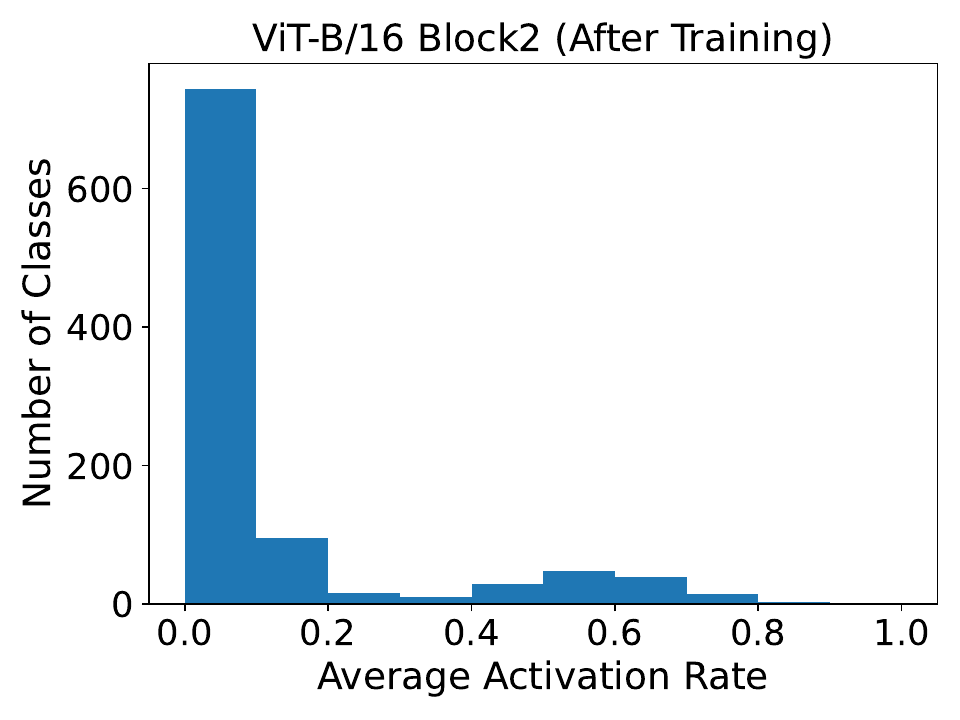} \\

\includegraphics[width=0.245\linewidth]{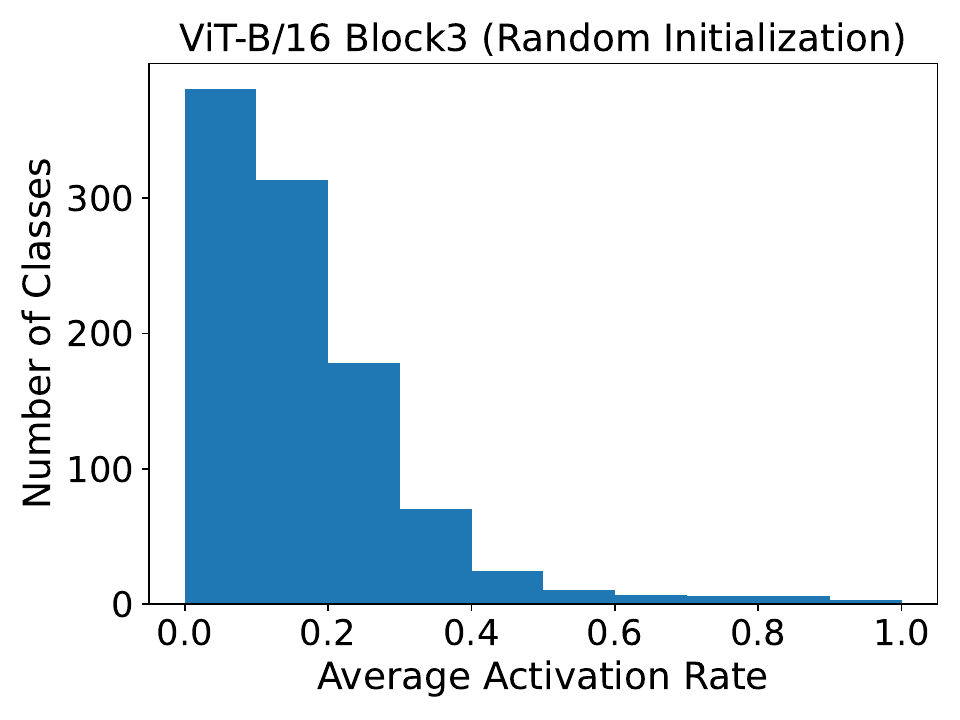}
\includegraphics[width=0.245\linewidth]{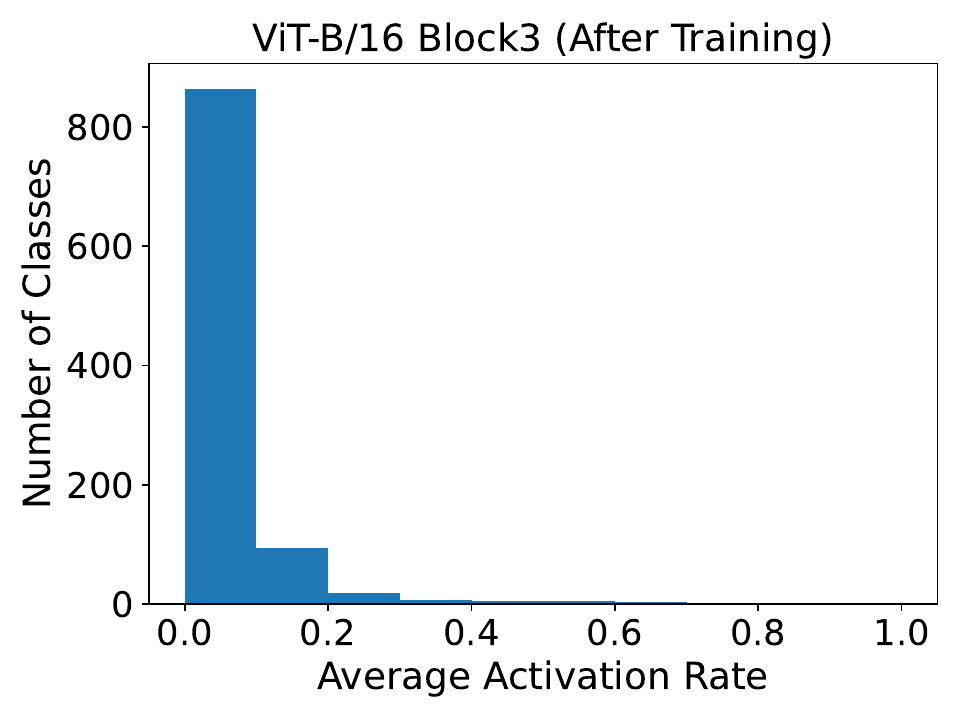}
\includegraphics[width=0.245\linewidth]{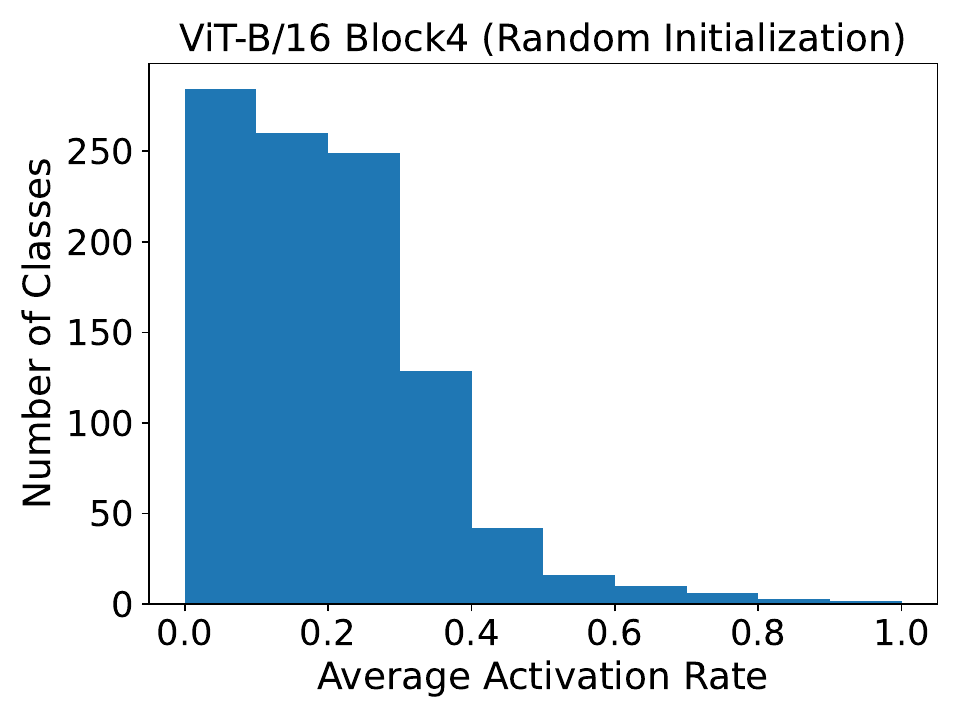}
\includegraphics[width=0.245\linewidth]{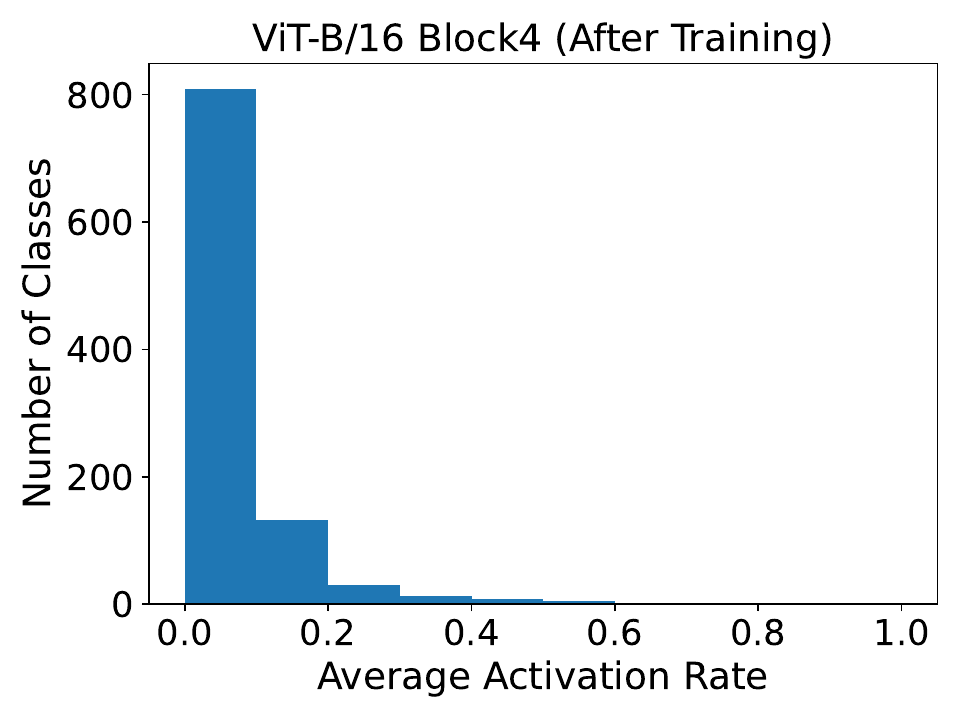} \\

\includegraphics[width=0.245\linewidth]{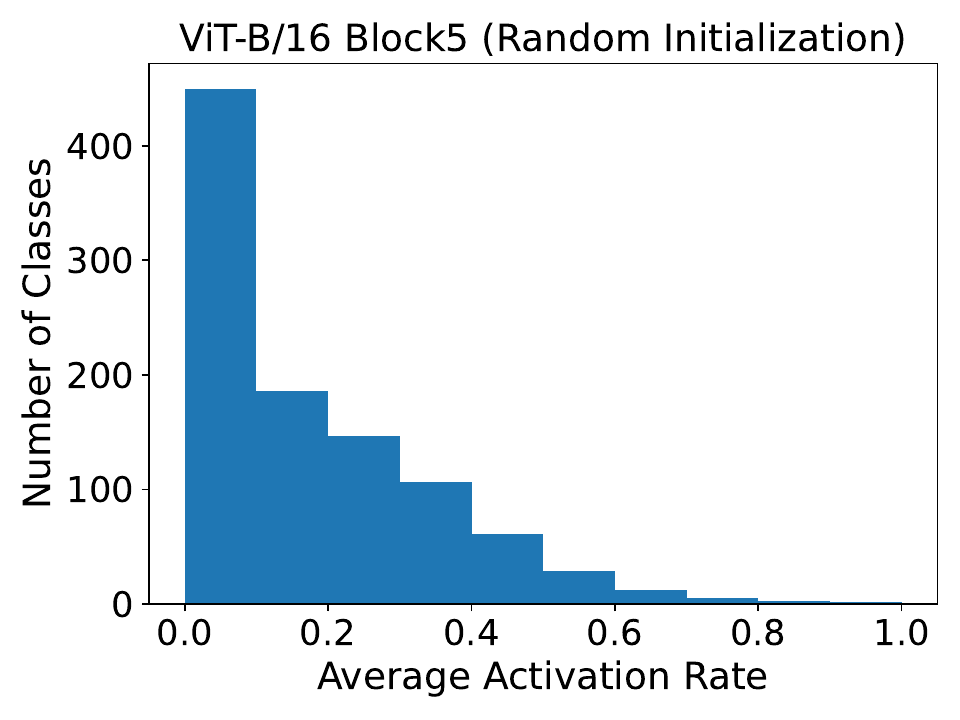}
\includegraphics[width=0.245\linewidth]{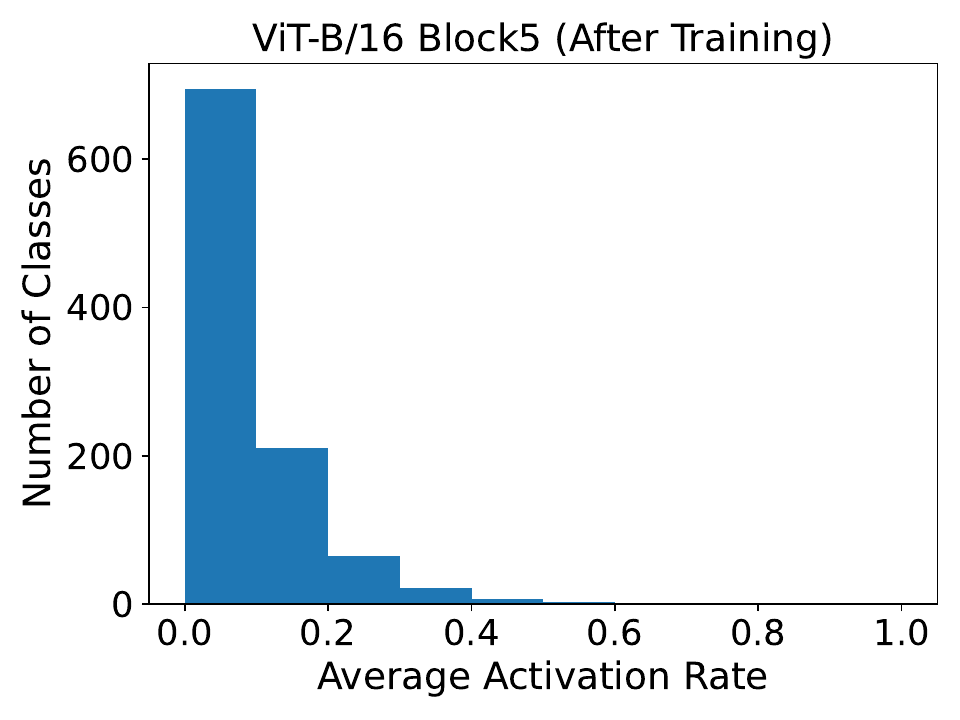}
\includegraphics[width=0.245\linewidth]{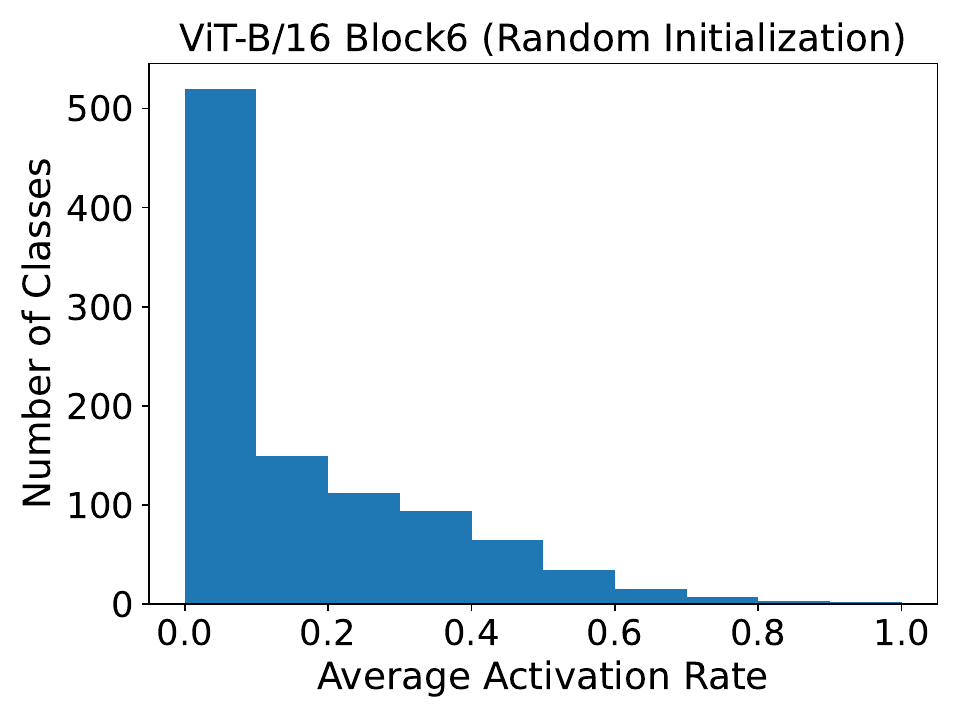}
\includegraphics[width=0.245\linewidth]{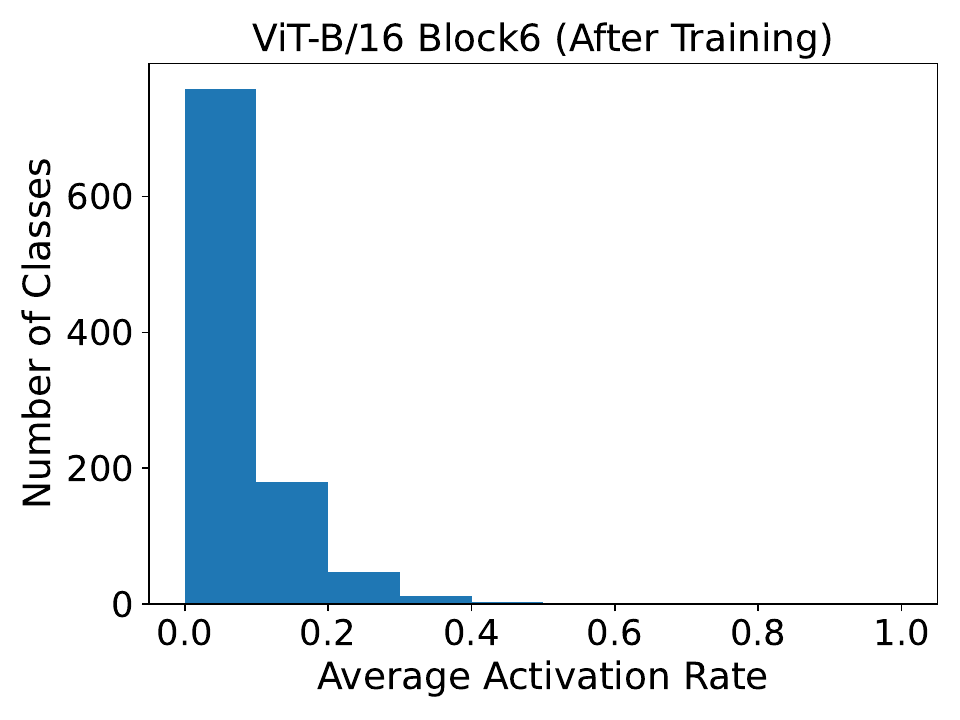} \\

\includegraphics[width=0.245\linewidth]{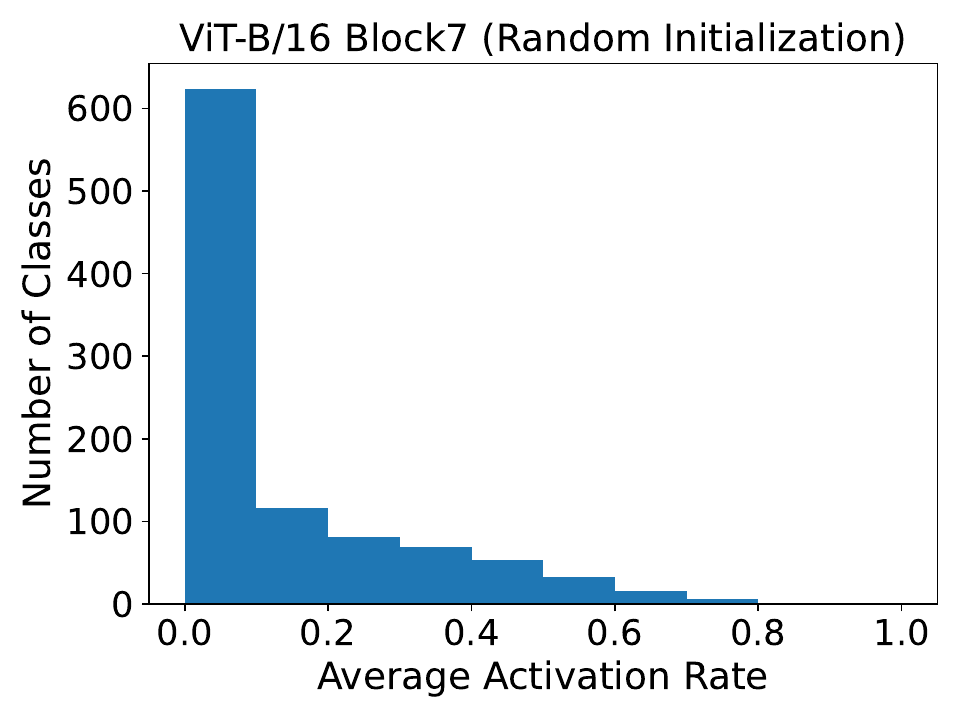}
\includegraphics[width=0.245\linewidth]{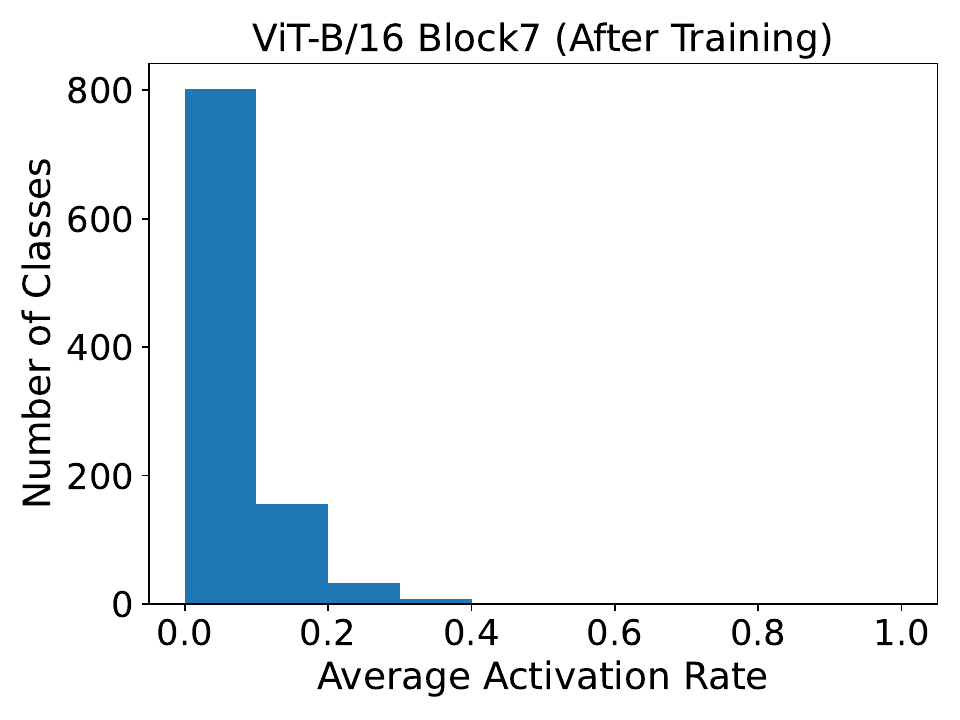}
\includegraphics[width=0.245\linewidth]{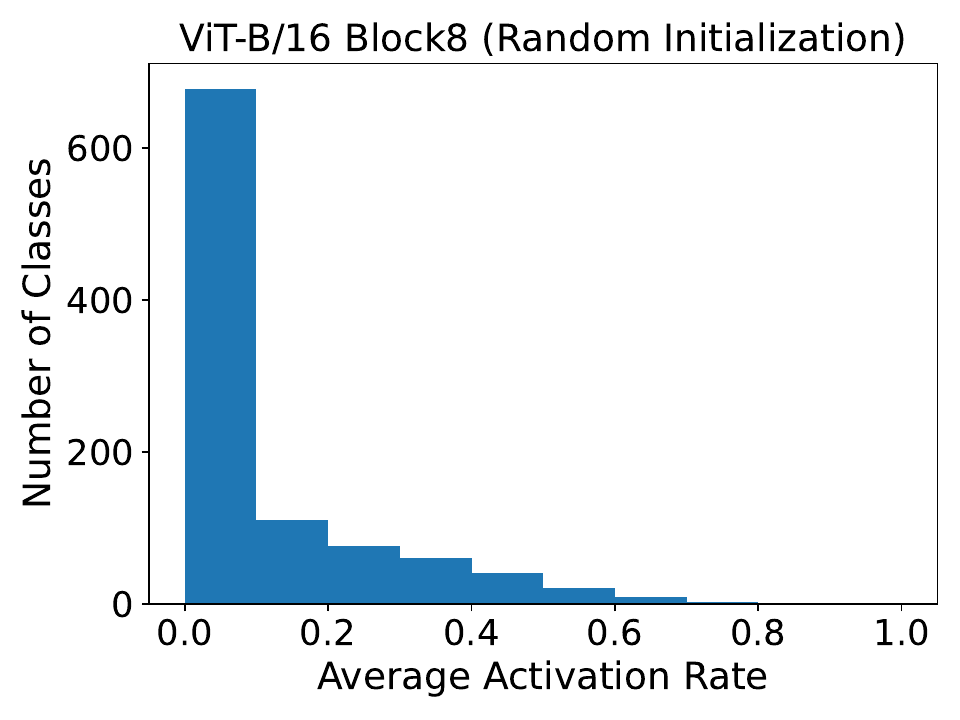}
\includegraphics[width=0.245\linewidth]{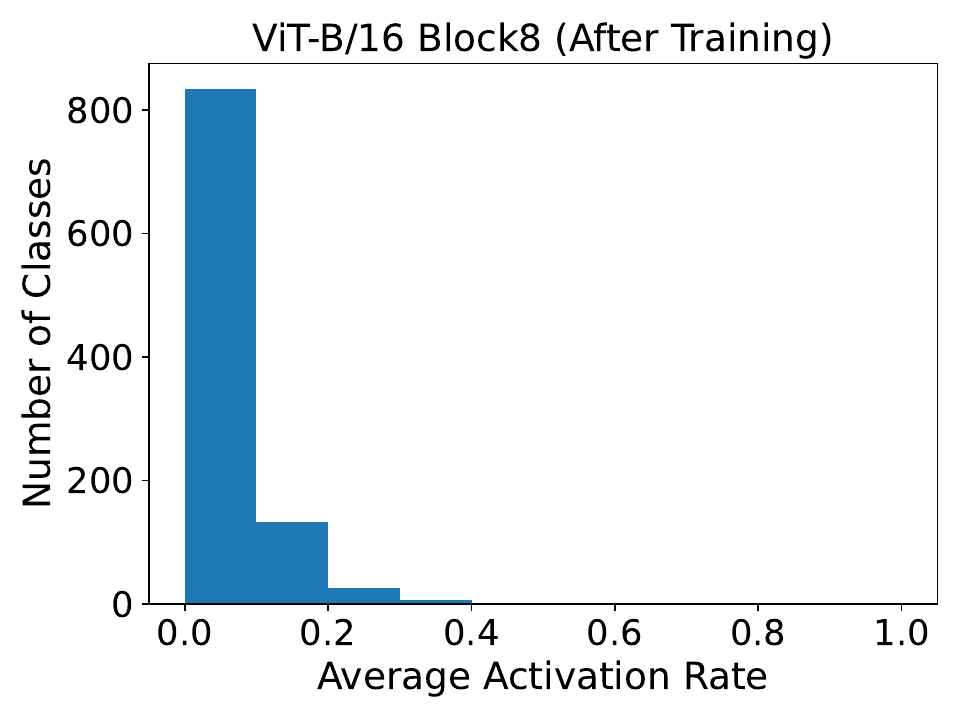} \\

\includegraphics[width=0.245\linewidth]{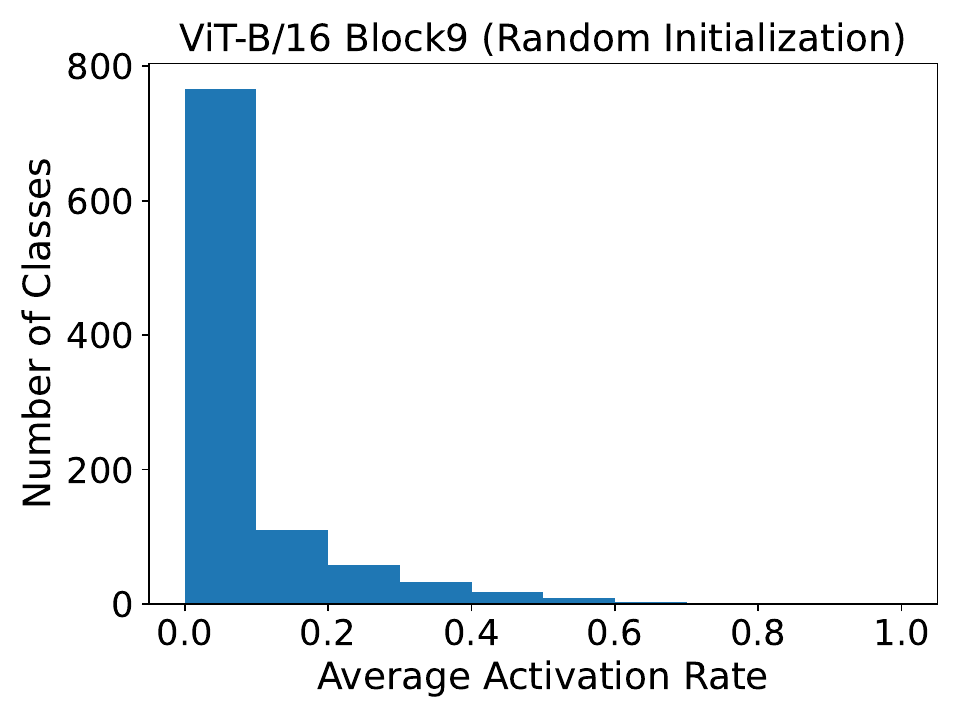}
\includegraphics[width=0.245\linewidth]{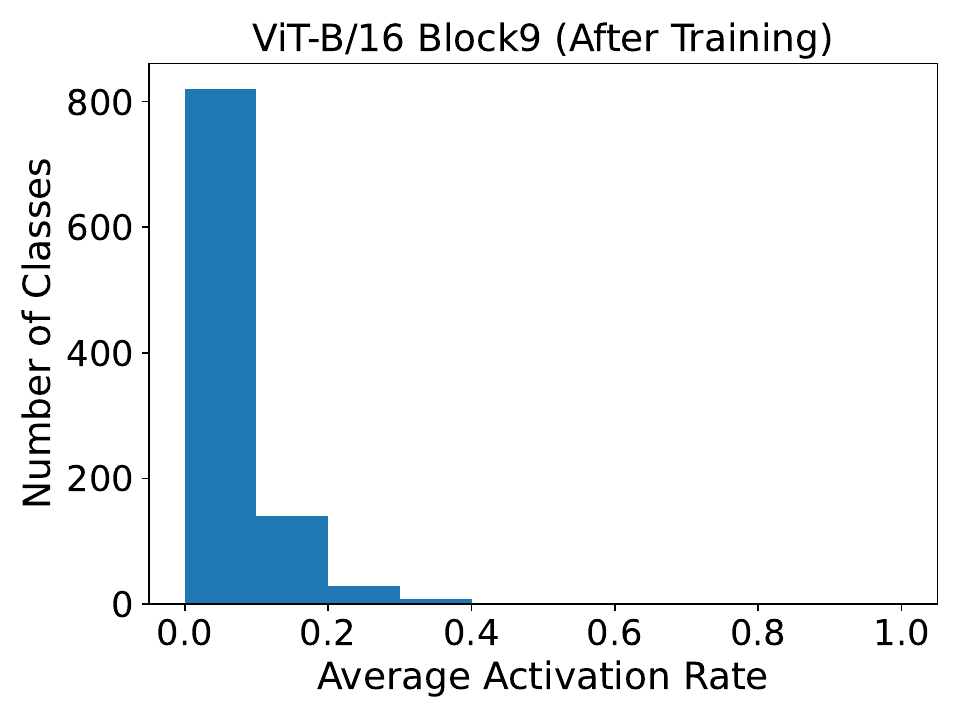}
\includegraphics[width=0.245\linewidth]{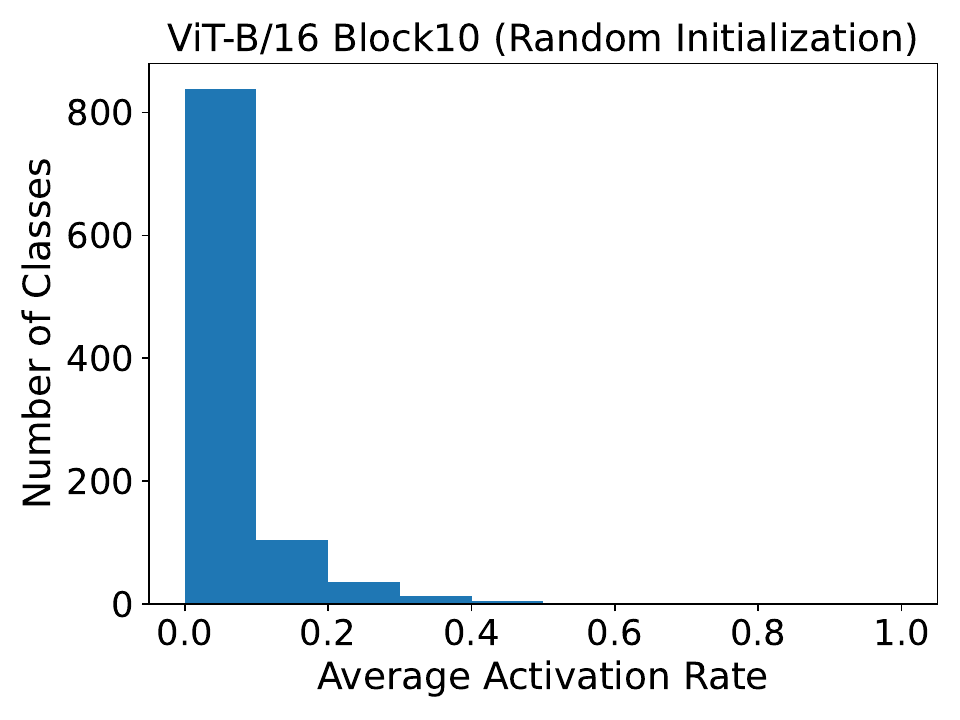}
\includegraphics[width=0.245\linewidth]{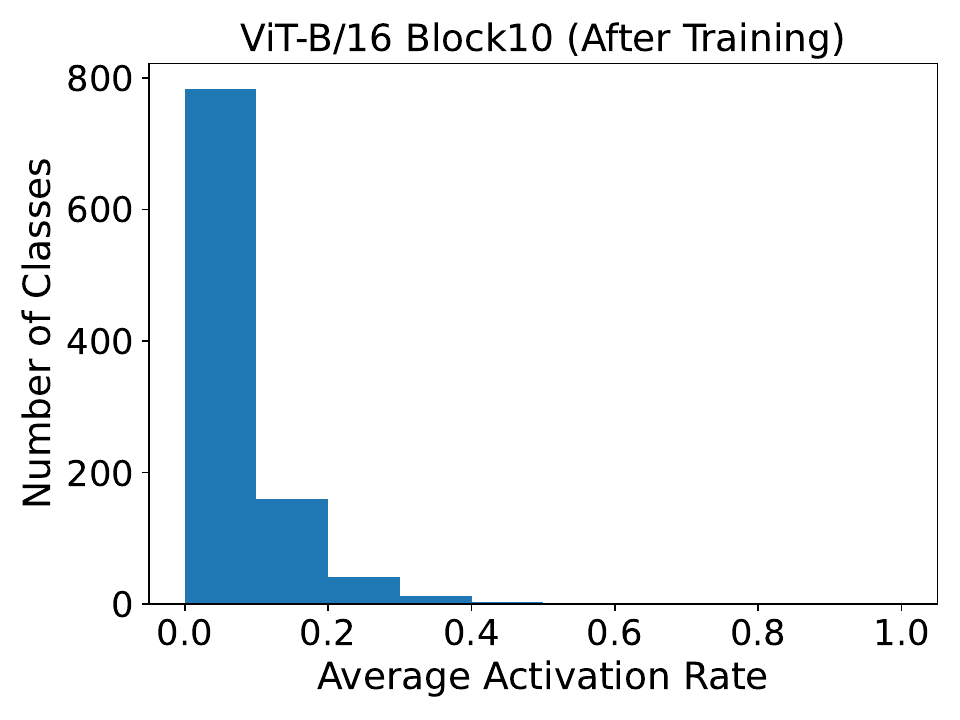} \\

\includegraphics[width=0.245\linewidth]{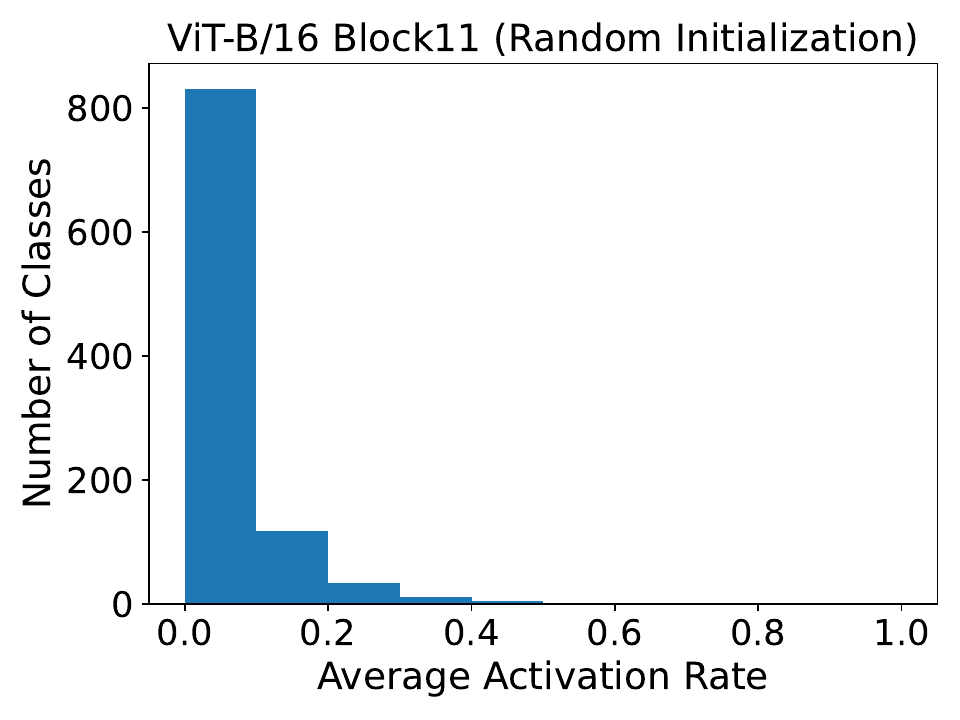}
\includegraphics[width=0.245\linewidth]{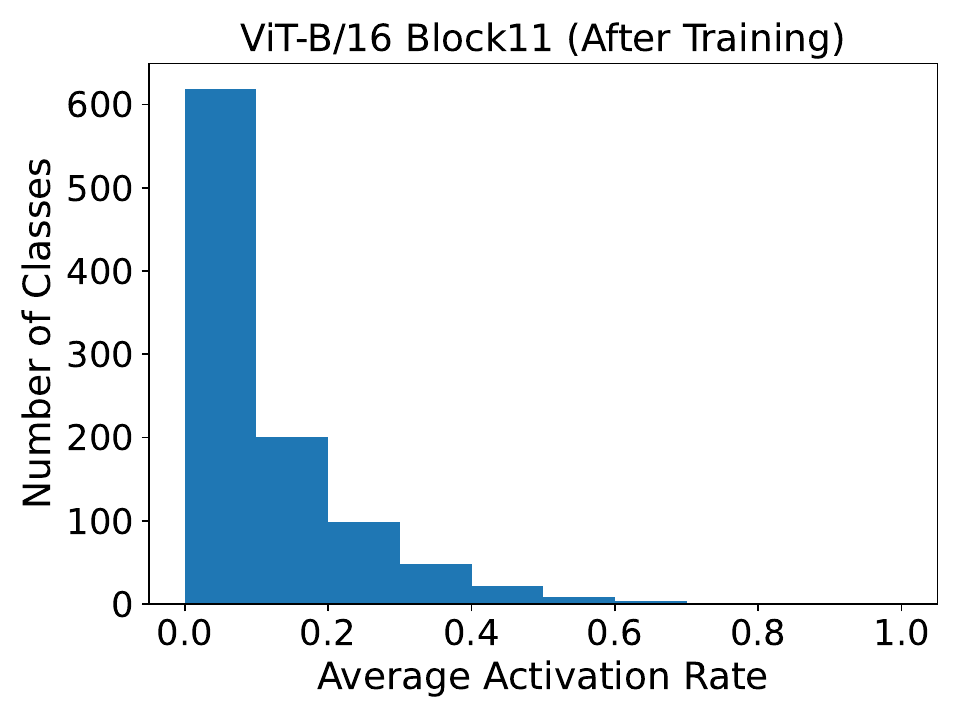}
\includegraphics[width=0.245\linewidth]{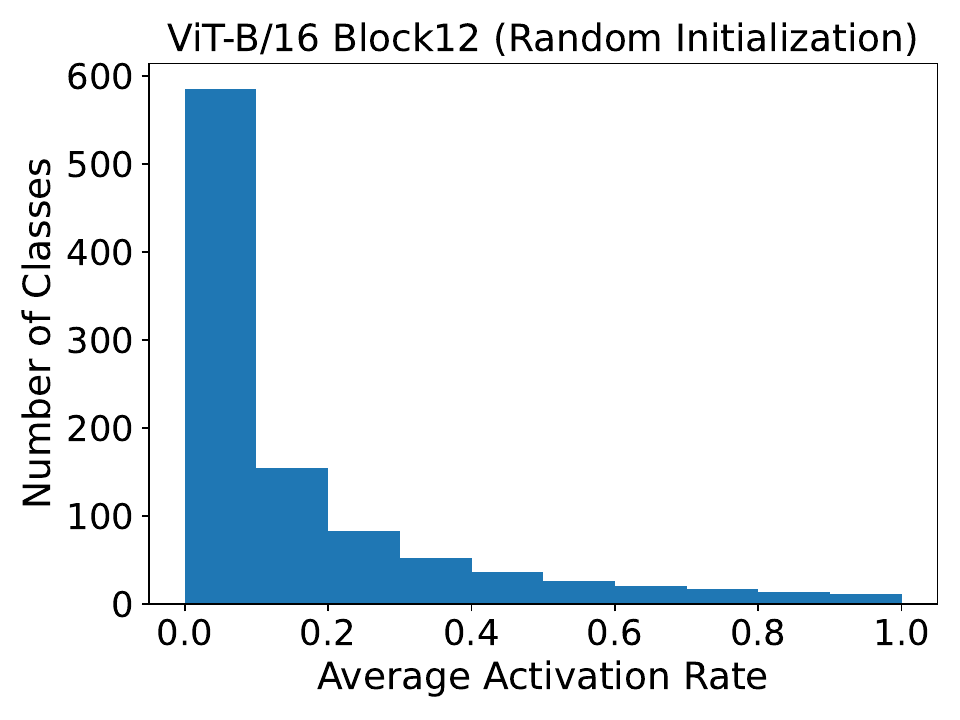}
\includegraphics[width=0.245\linewidth]{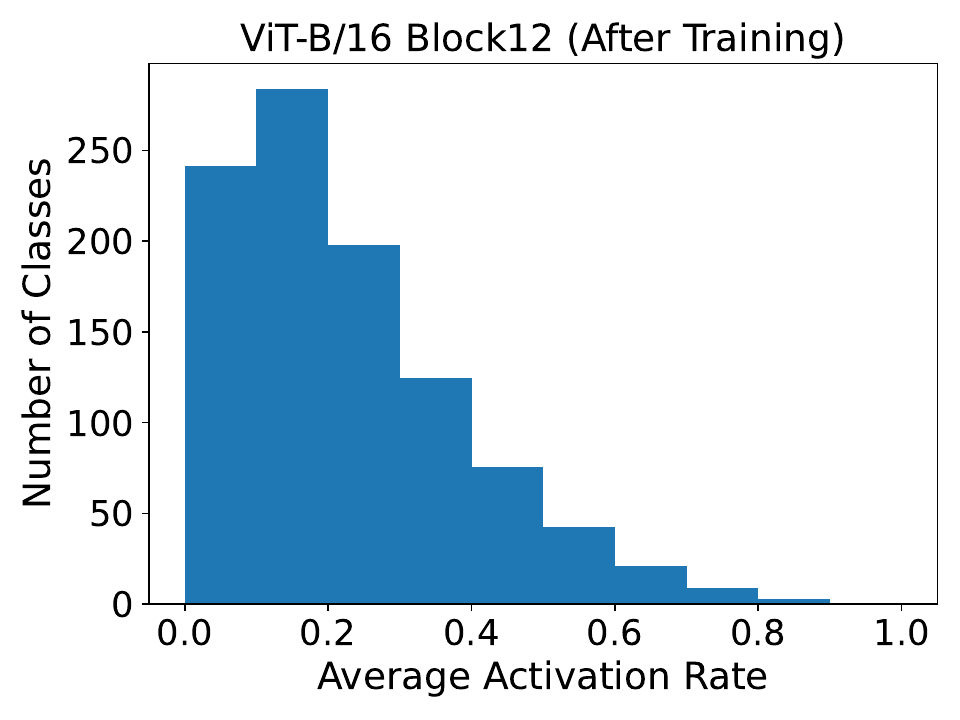}

\caption{Activation rate histograms of all blocks in randomly initialized and trained ViT-B/16 networks, computed from the ImageNet validation set.
}
\label{appfig:histograms_vit}
\end{figure}

\begin{figure}[t]
\centering
\subcaptionbox{Learned features of a ResNet-32 trained on the original CIFAR-10 dataset.}{\includegraphics[width=0.4\linewidth]{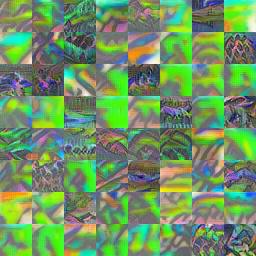}}
\hspace{1em}
\subcaptionbox{Learned features of a ResNet-32 trained on the modified CIFAR-10 dataset with background features.}{\includegraphics[width=0.4\linewidth]{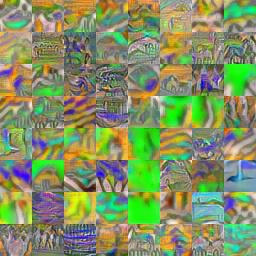}}
\caption{Additional visualization of the learned deep features in our CIFAR-10 experiment.}
\label{appfig:features}
\end{figure}

\subsection{Feature Visualization on CIFAR-10}
\label{appsec:visualization}

To investigate the presence of feature contamination in real-world datasets, we conducted a experiment based on a variant of the CIFAR-10 dataset that is explicitly modified to incorporate background features that have \emph{no correlation} with the label. Concretely, we augmented the CIFAR-10 training set by padding brick red pixels to the original images from CIFAR-10 and resized the padded images to the size of the original images, as shown in~\figref{fig:visualization}(c). Since our padding does not impact the original image contents, it follows the ``orthogonal'' setting in our theoretical model where the core features (the original image contents) and the background features (the padded pixels) are independent---there exists a \emph{linear} transformation in the \emph{input space} that can fully recover core features and remove background features.

We then visualize the learned features of a ResNet-32 network trained on the original CIAFR-10 dataset and another ResNet-32 trained on our modified dataset. Following the visualization technique in~\citet{allen-zhu_feature_2021}, we performed adversarial training using the method proposed by~\citet{salman2019provably} and visualized the features learned by the network's convolutional kernels in the 31st layer using the same hyperparameters as described in~\citet{allen-zhu_feature_2021}. As shown by~\figref{appfig:features}, we observe notable differences in the learned color information between models trained on the original CIFAR-10 dataset and its modified variant. Meanwhile, we note that there are no obvious geometric patterns in the red areas, which we conjecture is due to the augmentations used during training such as random cropping and flipping. In general, the visualization results suggest that background features are indeed learned by deep neural networks despite having no correlation with the label, which corroborates our theory and indicates that \emph{feature contamination also happens in \textbf{deep features} learned from real-world image data}.

\section{Empirical Evidence that Supports the Conjecture in~\secref{sec:discussion}}
\label{appsec:conjecture}

In this section, we provide preliminary empirical evidence that supports the conjecture stated in~\secref{sec:discussion} in the main text and discuss its relation with related observations in recent work. For ease of presentation, here we restate this conjecture:

\begin{conjecture*}
Pre-training on a sufficiently diverse dataset does not remove uncorrelated features, but \emph{linearizes} those features in the model's representations, hence mitigating feature contamination and improving OOD generalization.
\end{conjecture*}

\begin{figure}[t]
\centering
\subcaptionbox{ERM representations}{\includegraphics[width=0.32\linewidth]{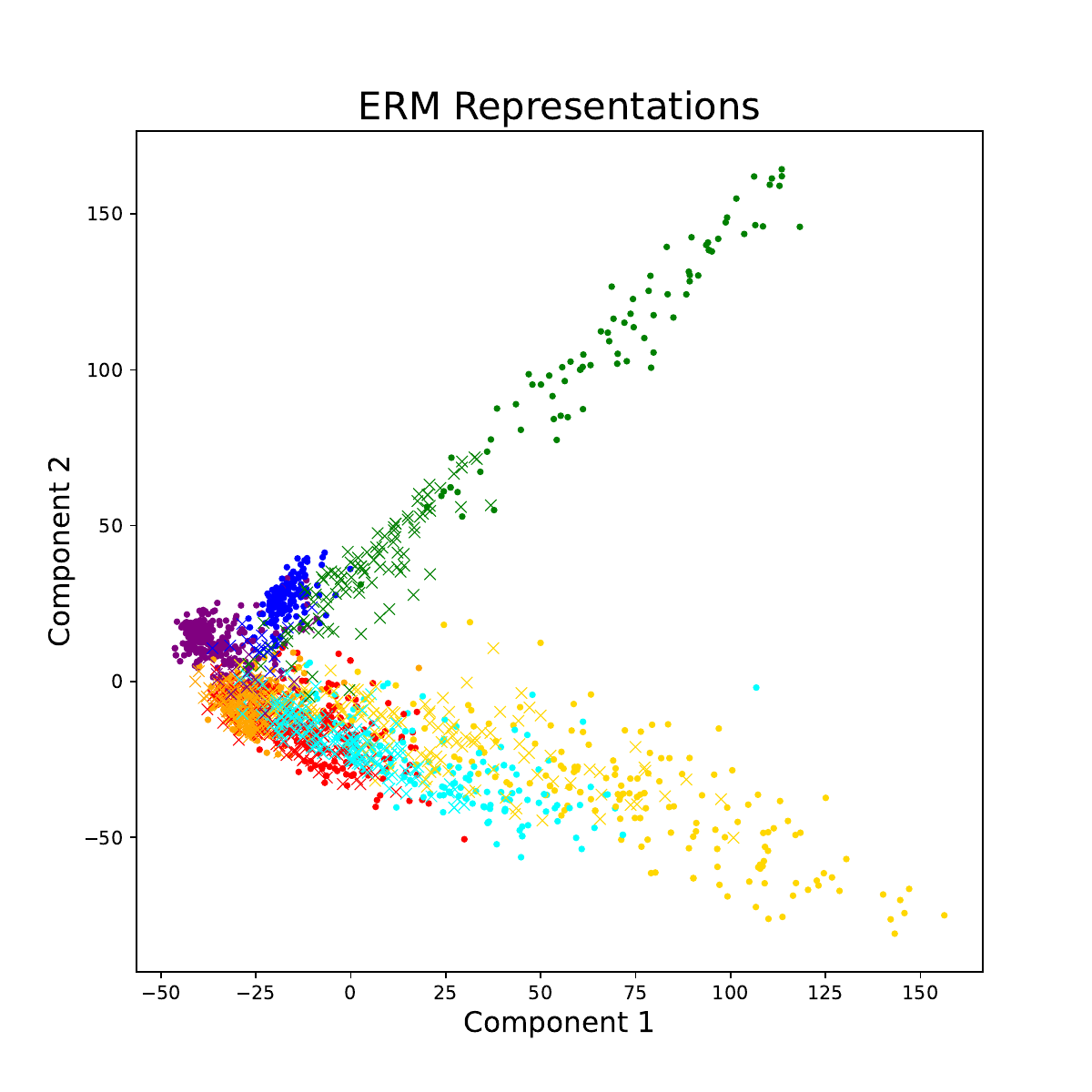}}
\subcaptionbox{CLIP-RN50 representations}{\includegraphics[width=0.32\linewidth]{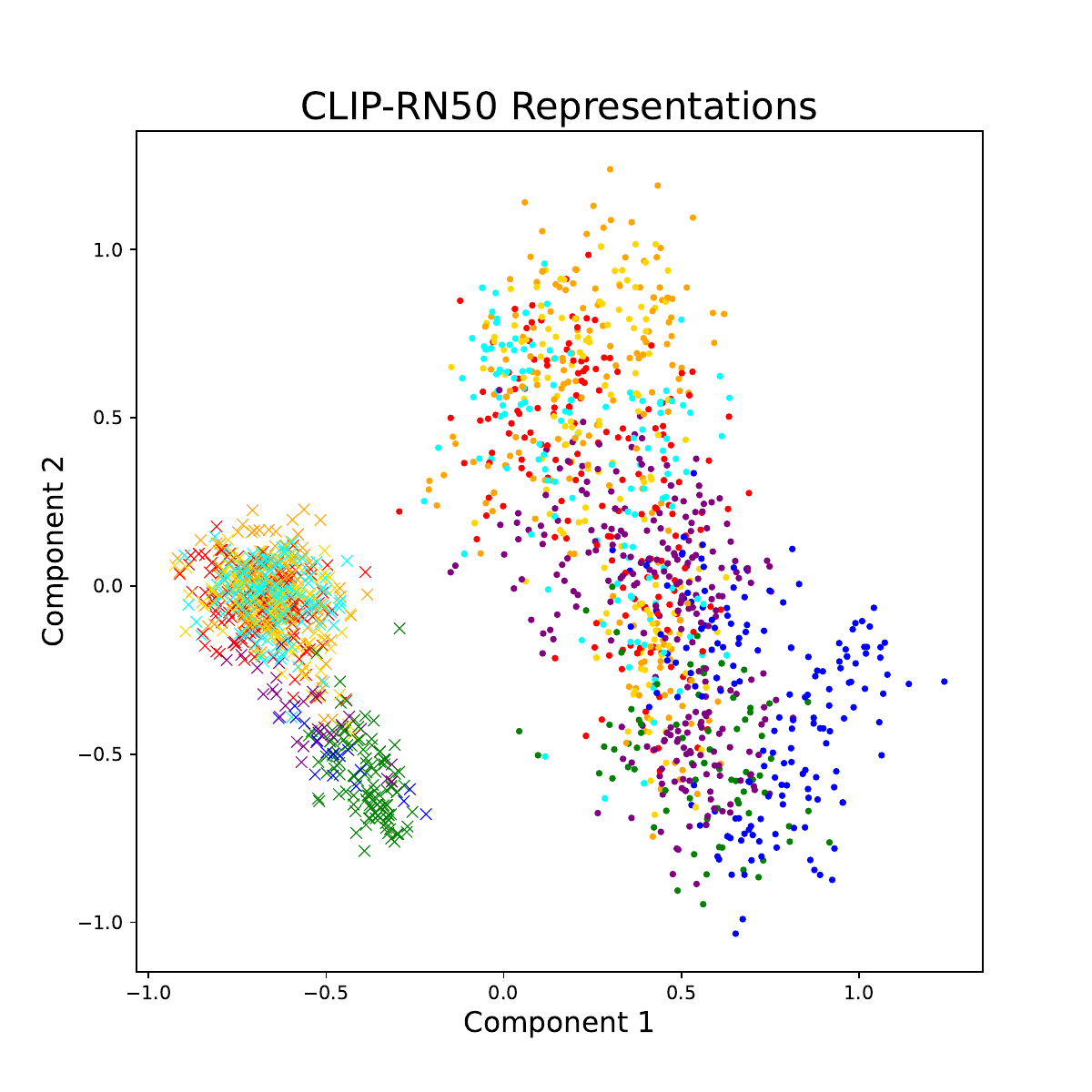}}
\subcaptionbox{CLIP-ViT-B/16 representations}{\includegraphics[width=0.32\linewidth]{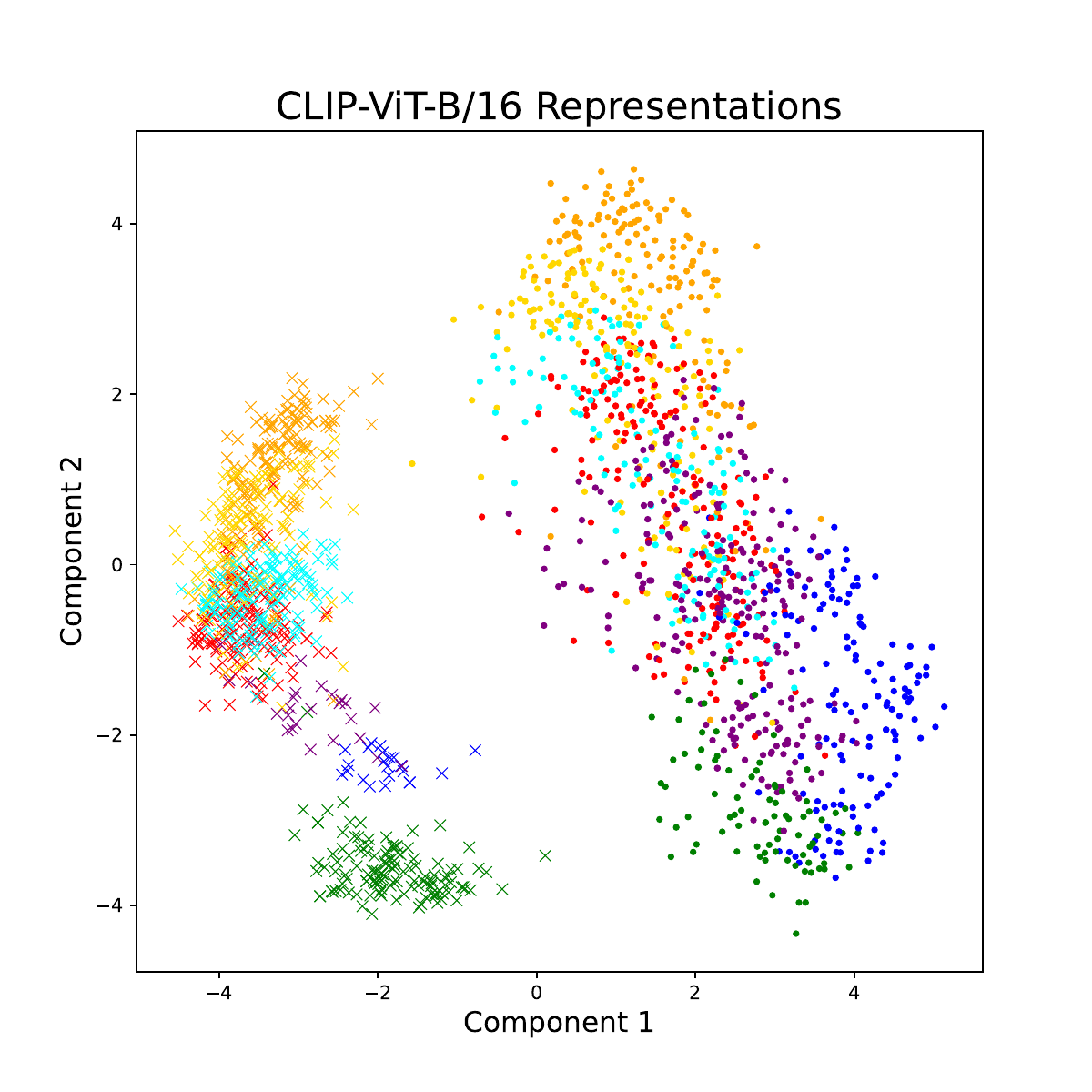}}
\caption{Visualizations of ERM and CLIP representations after PCA dimensionality reduction to two dimensions. Circles refer to image representations in the training domains, while crosses refer to image representations in the test domain. Different colors represent different classes. Compared to ERM representations where the examples from training and test domains are visually mixed, CLIP representations exhibit strong \emph{linear separability} of different domains.}
\label{fig:pca}
\end{figure}

\begin{table}[h]
\centering
\caption{Detailed ID test accuracy, OOD test accuracy, and domain classification error (\%) of linear probes on pre-trained and distilled representations on PACS.}
\label{table:results_pacs}
\begin{tabular}{cccc}
\toprule
 & ID Test Acc. & OOD Test Acc. & Domain Classification Error \\
\midrule
CLIP-ViT-B/16 & \bf 99.68 & \bf 91.59 & \bf 0.06 \\
CLIP-RN50 & 97.35 & 85.67 & 0.19 \\
\midrule
ERM-RN50 & 99.28 & 76.47 & 1.02 \\
\bottomrule
\end{tabular}
\end{table}

\subsection{Large-Scale Pre-training Leads to Linear Separability of Domains}

To empirically test this conjecture, we first examined the properties of the pre-trained representations from CLIP and the representations learned by ERM on a domain generalization dataset PACS~\citep{li_deeper_2017} for image classification. The images in PACS are divided into four domains, namely Photo, Art painting, Cartoon, and Sketch, with seven common categories. We trained a ResNet-50 ERM model using the examples from the first three domains (ID) and the Sketch domain was treated as the OOD domain. To evaluate the robustness of CLIP representations, we fitted a linear probe on top of freezed CLIP representations on ID domains and evaluated the learned linear probe on the OOD domain.

We begin by a 2-dimensional visualization of both the learned ERM representations and the CLIP representations using PCA dimensionality reduction. As shown in~\figref{fig:pca}, \emph{ERM representations and CLIP representations exhibit quite different properties in terms of domain separability}: while examples from training and test domains are visually mixed in ERM representations, examples from training and test domains are \emph{strongly linearly separable} in CLIP representations.

We then quantitatively examined this linear separability by fitting linear classifiers on top of ERM and CLIP representations for \emph{domain classification}. Concretely, we trained linear classifiers with the original ``class'' label of each example substituted by its domain index. We then evalute the accuracy of this classifier on a hold-out validation set. As shown in~\tableref{table:results_pacs}, domain classifiers on CLIP representations have considerably smaller error than domain classifiers on ERM representations, which is consistent with visualization. This phenomenon is related to recent work on unsupervised domain adaptation based on contrastive learning~\citep{shen_connect_2022,haochen_beyond_2022}, where it has been shown that contrastive learning can learn representations that disentangle domain and class information, enabling generalization that they refer to as ``linear transferability''~\citep{haochen_beyond_2022}. However, their analysis requires that unlabeled examples from the target domain are seen by the contrastive learning algorithm during training, while large-scale pre-training in our context seems to achieve a similar disentangling effect even without explicitly trained on the target distribution. Further theoretical explanations of this phenomenon is an important future work.

In summary, the results in this section suggest that the representations learned by large-scale pre-training is highly linearized, with features representing different factors of variation not as non-linearly coupled as in our analysis on feature contamination. We believe that such high linearity of representations plays a critical role in the OOD capability of pre-trained models.

\begin{figure}[t]
\centering
\begin{subfigure}{0.3\linewidth}
\centering
    \includegraphics[width=0.99\linewidth]{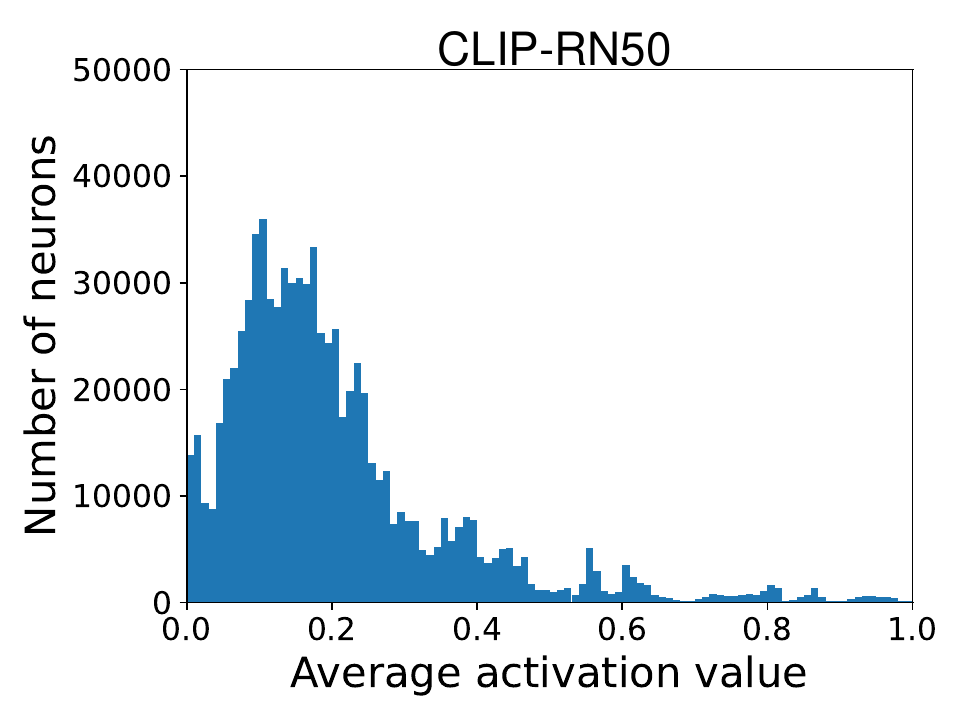}\\
    \includegraphics[width=0.99\linewidth]{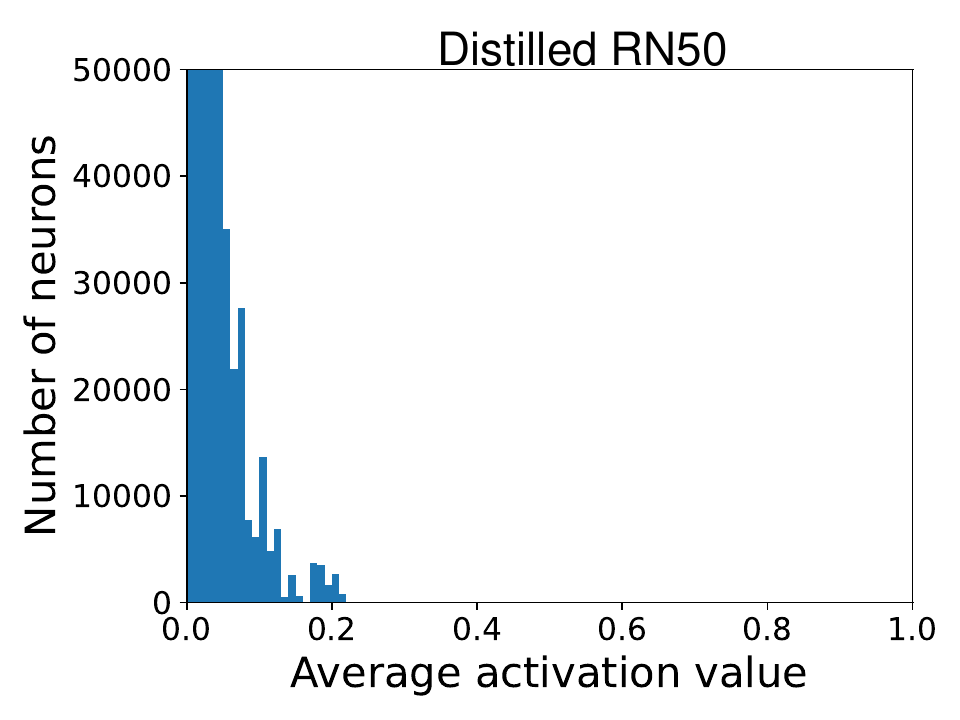}
    \caption{ImageNet}
\end{subfigure}
\begin{subfigure}{0.3\linewidth}
\centering
    \includegraphics[width=0.99\linewidth]{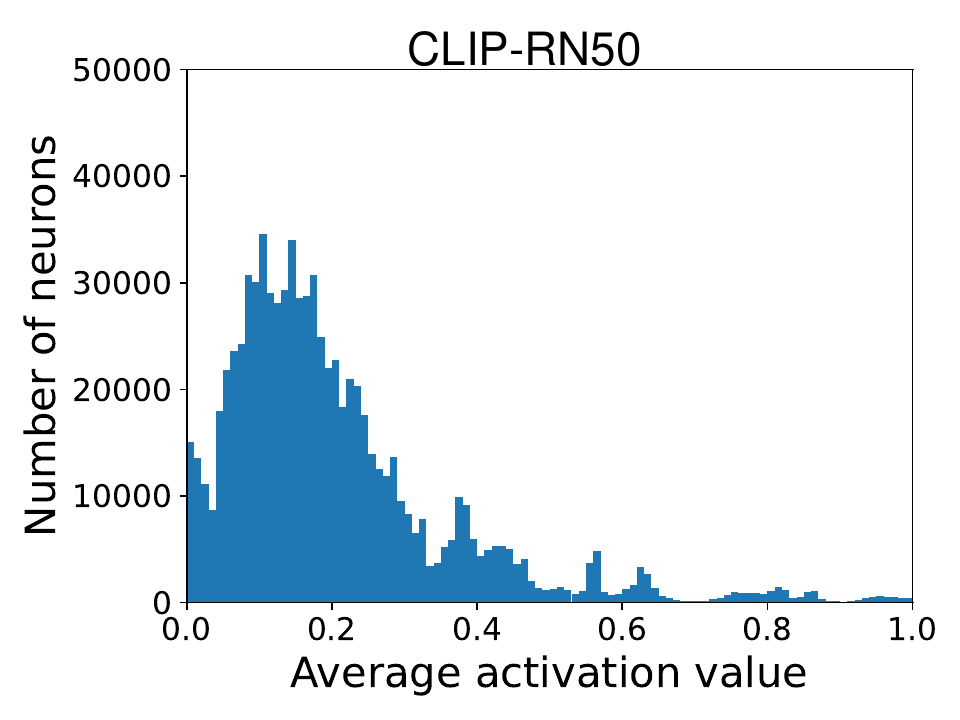}\\
    \includegraphics[width=0.99\linewidth]{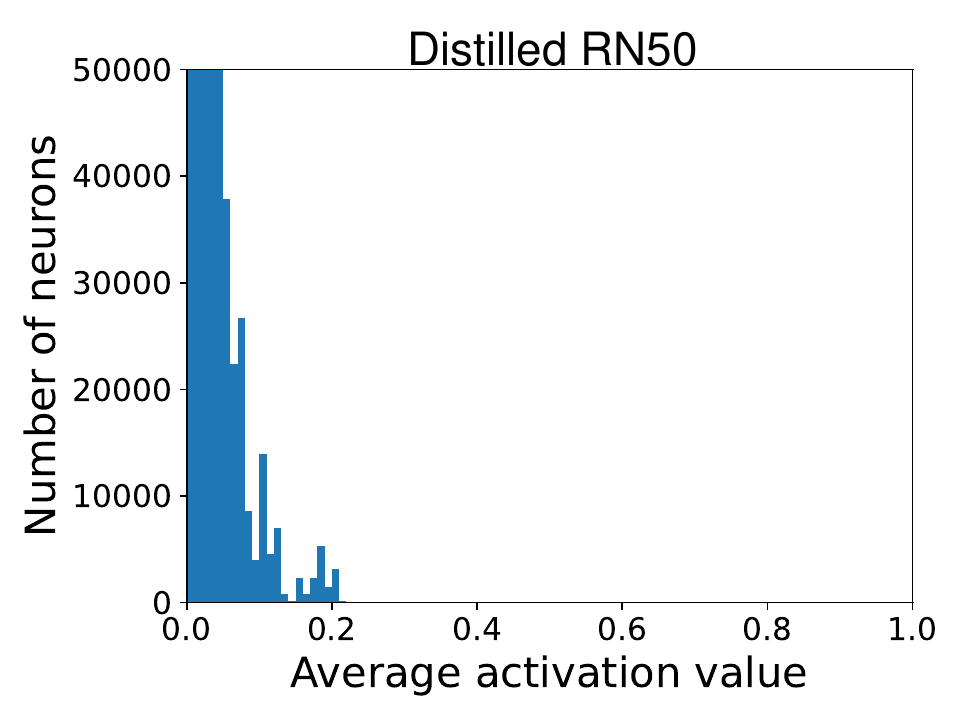}
    \caption{ImageNetV2}
\end{subfigure}
\begin{subfigure}{0.3\linewidth}
\centering
    \includegraphics[width=0.99\linewidth]{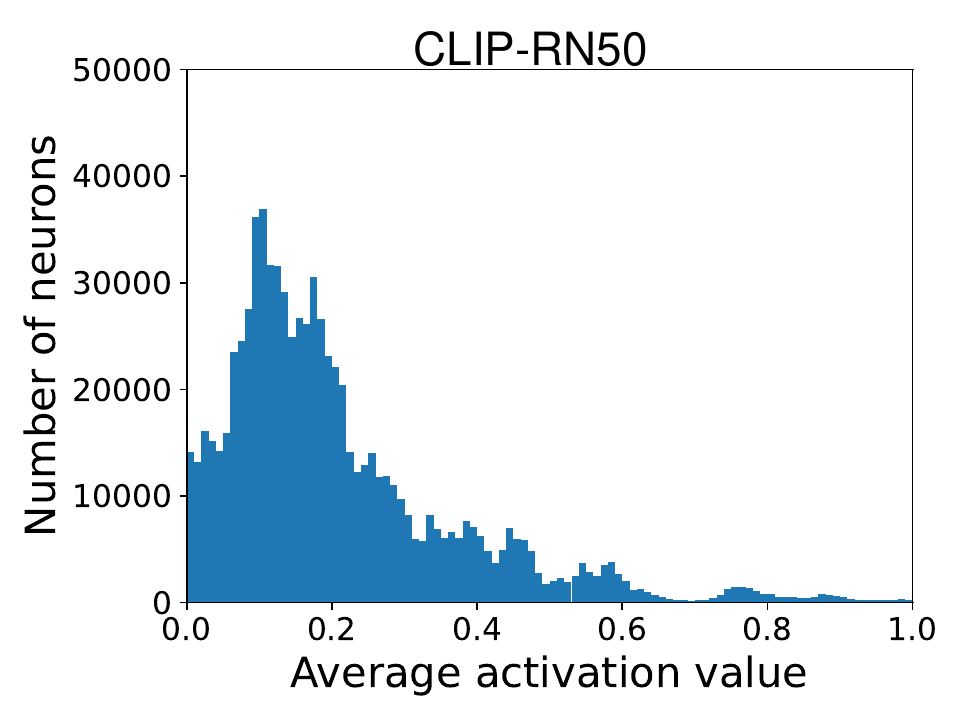}\\
    \includegraphics[width=0.99\linewidth]{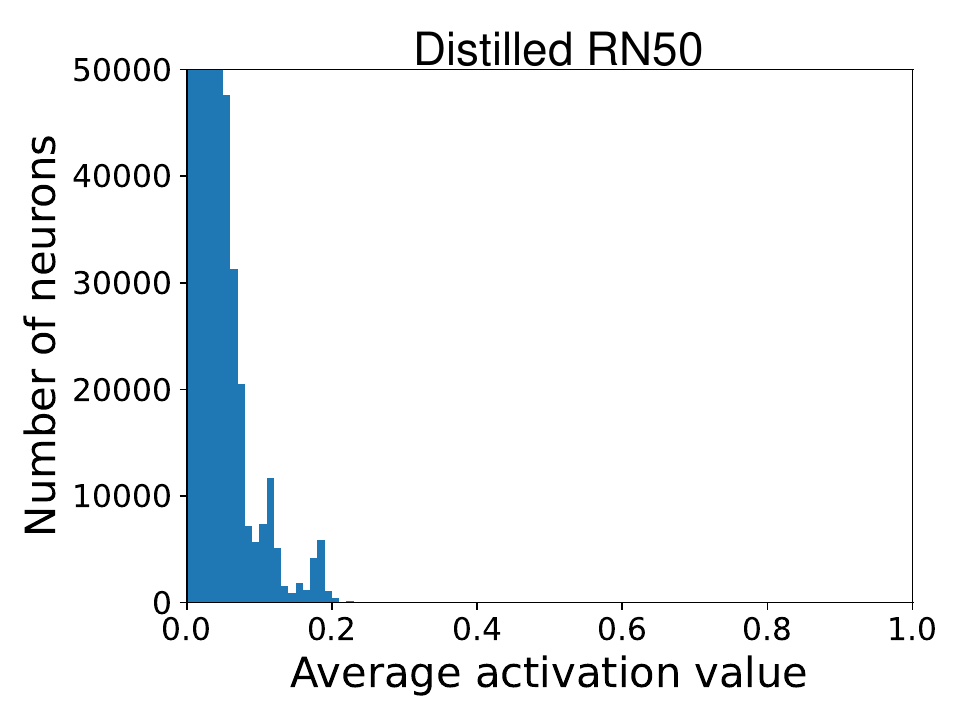}
    \caption{ImageNet-R}
\end{subfigure}\\

\begin{subfigure}{0.3\linewidth}
\centering
    \includegraphics[width=0.99\linewidth]{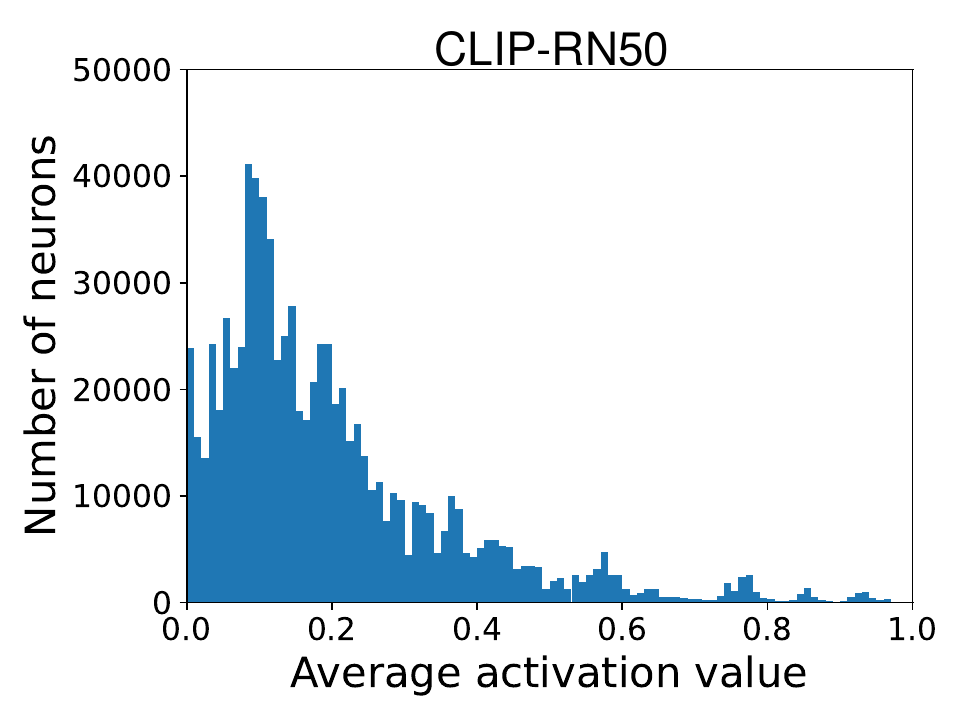}\\
    \includegraphics[width=0.99\linewidth]{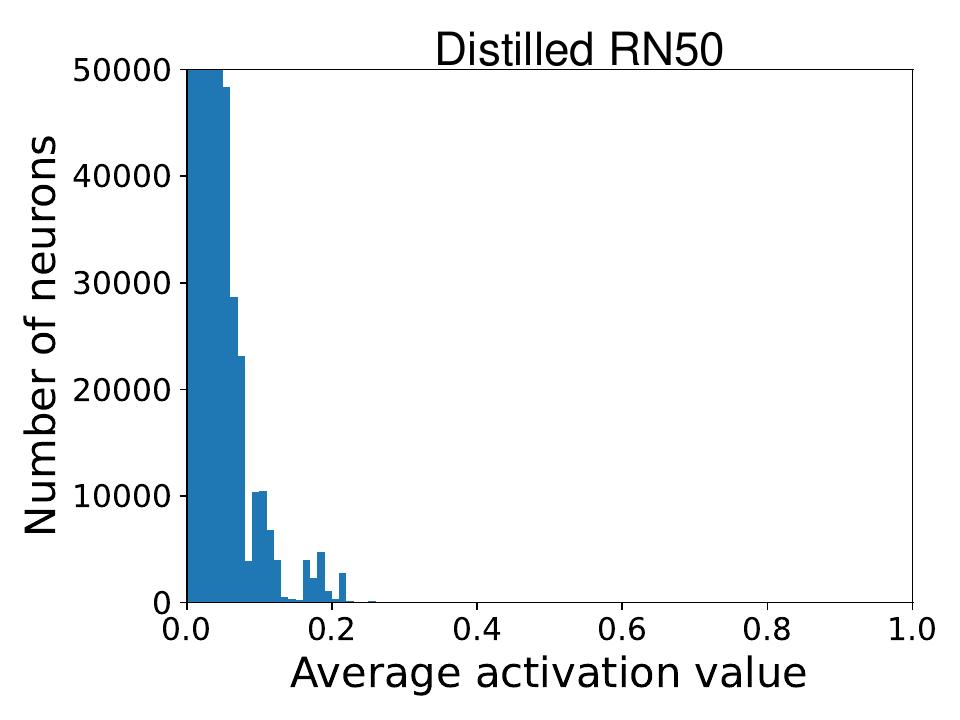}
    \caption{ObjectNet}
\end{subfigure}
\begin{subfigure}{0.3\linewidth}
\centering
    \includegraphics[width=0.99\linewidth]{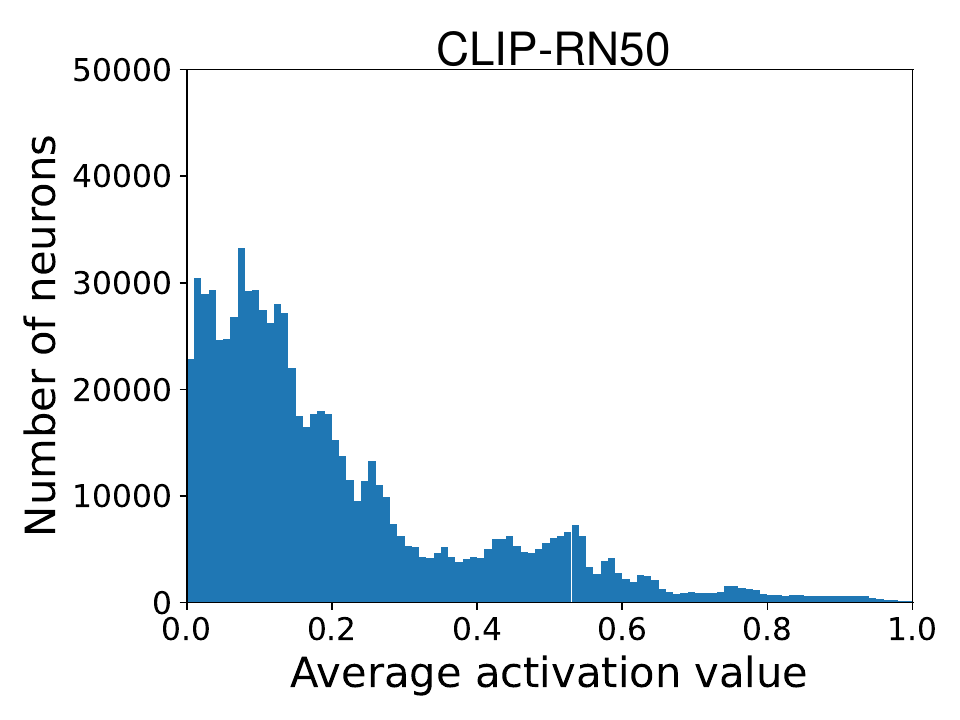}\\
    \includegraphics[width=0.99\linewidth]{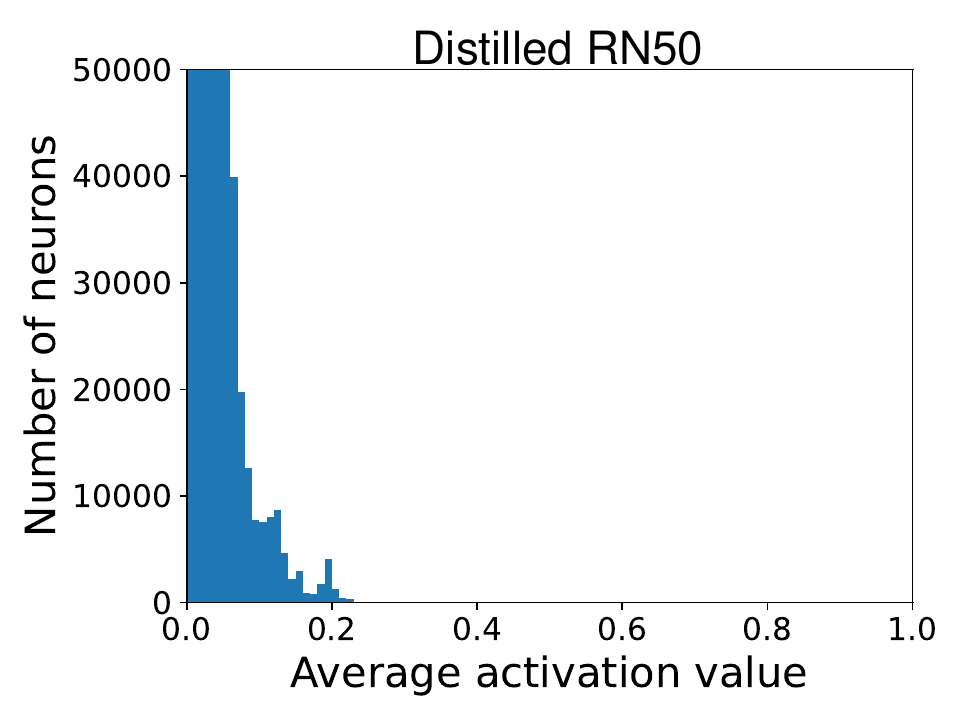}
    \caption{ImageNet Sketch}
\end{subfigure}
\begin{subfigure}{0.3\linewidth}
\centering
    \includegraphics[width=0.99\linewidth]{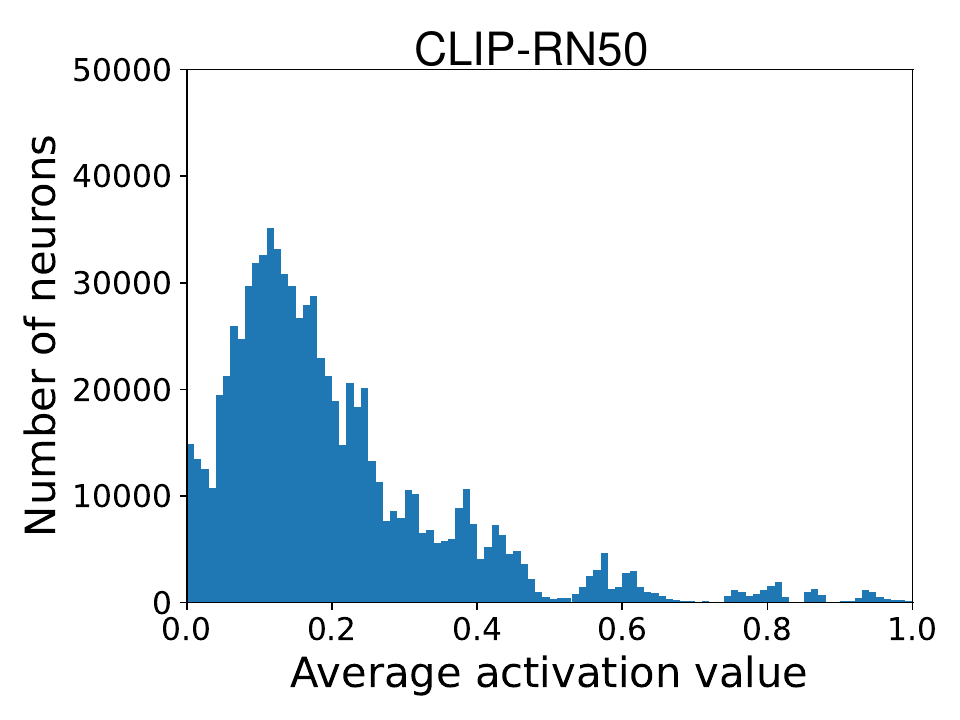}\\
    \includegraphics[width=0.99\linewidth]{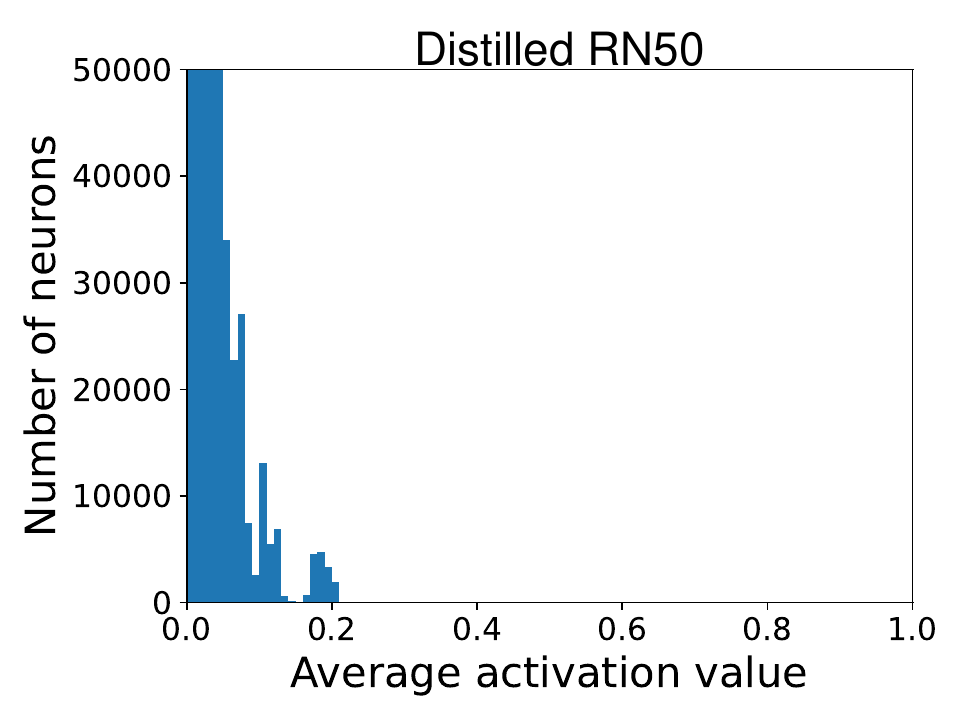}
    \caption{ImageNet-A}
\end{subfigure}

\caption{Histograms of average neuron activations of a pre-trained CLIP-RN50 and a distilled CLIP-RN50 on ImageNet distribution shift datasets. In each subfigure, the top plot shows the histogram of CLIP, and the bottom plot shows the histogram of the distilled model.}
\label{fig:activation}
\end{figure}

\subsection{Large-Scale Pre-training Leads to Denser Neuron Activation}

In this section, we study property of pre-trained representations from another angle of neuron activation. As we have formally proved in~\secref{sec:main}, feature contamination causes the neurons to learn non-linearly coupled features. The activation of each neuron is thus likely to involve multiple feature vectors due to this coupling. By the above deduction, if pre-training alleviates feature contamination and learns more linearized features, then the activation of different feature vectors would be more likely to involve different neurons, resulting in an increase in the total number of activated neurons for each input.

Empirically, we confirmed the above hypothesis by calculating the histogram of the neuron's expected activation value in pre-trained and distilled models from the ImageNet experiments in~\secref{sec:main_exp}. We considered the CLIP-RN50 teacher model and its corresponding student model obtained from representation distillation, and maintained an estimate of the average activation value for each output ReLU activation in the first residual block during one evaluation run. We plot the histogram of the neuron's average activation value in~\figref{fig:activation}. As shown by the figure, the pre-trained CLIP model indeed have considerably denser neuron activation than the distilled model, even on the ID ImageNet validation set where their top-1 accuracy is nearly the same (70.37\% for the pre-trained CLIP model and 69.85\% for the distilled model). This suggests that pre-trained models learn more ``decoupled'' features than models trained solely on the ID data.

\subsection{More Discussion on Related Work}

\paragraph{Explaining the distributional robustness of CLIP.} Understanding the remarkable distributional robustness of large-scale pre-trained models such as CLIP is an open research problem of its own. Due to the amount and diversity of pre-training data, a major confounder in this problem is that pre-trained models may have ``seen'' similar examples in standard distribution shift test sets during pre-training rather than be essentially robust to unseen distribution shifts. Recently,~\citet{mayilvahanan_does_2024} conducted controlled experiments suggesting that even if we remove examples that are semantically similar to those in OOD test sets during pre-training, CLIP still remains a large portion of its distributional robustness. Therefore, CLIP must have achieved good OOD performance in a non-trivial way by learning OOD generalizable representations rather than simply memorize the test distribution.~\citet{gandelsman_interpreting_2024} shows that CLIP's image representations can be decomposed as sums across individual image patches, layers and attention heads that are often \emph{property-specific}. Such ``decomposability'' in the representations implies that CLIP may represent different semantics in the input images in a decoupled way, which may be free from feature contamination. However, a rigorous connection between this observation and the linearity of features remains to be explored.

\paragraph{The linear representation hypothesis.} An important observation made by recent work on interpreting the representations of large language models (LLMs) is that many high-level, abstract concepts are \emph{linearly} represented in the LLMs' intermediate activation spaces~\citep{marks_geometry_2023,allen-zhu_physics_2023,gould_successor_2023,park_linear_2023,heinzerling_monotonic_2024,gurnee_language_2024}. At a high level, those results are related to our conjecture in~\secref{sec:discussion} on the \emph{feature linearization} effect of pre-training. However, it remains an open problem how pre-training leads to such effects. Another closely related concept in the literature is \emph{superposition}~\citep{elhage_toy_2022}, which hypothesizes that neural networks may represent more independent features than the number of neurons in the network by assigning different features to the same neuron when those features are sparse and correlated with the task. By contrast, we show that neural network can also learn \emph{uncorrelated} features even when it has enough neurons to represent all features separately.


\end{document}